\documentclass{article}

\usepackage{inputenc}
\usepackage{amsmath, dsfont}
\usepackage{times}
\usepackage[margin = 1in]{geometry}
\usepackage{color}
\usepackage{comment}
\usepackage{enumitem}
\usepackage{verbatim}
\usepackage{algorithmic}

\usepackage[colorlinks,urlcolor=blue,citecolor=blue,linkcolor=blue]{hyperref}

\usepackage{natbib}
\usepackage{amssymb, amsmath, amsthm}
\usepackage{amsfonts,mathtools,mathrsfs,mathtools,color}
\usepackage{graphicx}
\usepackage{caption}
\usepackage{subcaption}
\usepackage{bbm} 
\usepackage{algorithm} 
\usepackage{algorithmic} 
\usepackage{comment}
\usepackage{xcolor}
\usepackage{url}
\usepackage{mdframed}
\usepackage{soul}
\usepackage{blindtext}
\usepackage{titlesec}
\usepackage{multirow}
\usepackage{wrapfig}
\usepackage{booktabs}

\newcommand{\oureqref}[1]{(\ref{#1})}

\newcounter{simpletable}
\newcommand{\simpletablelabel}[1]{%
    \refstepcounter{simpletable}%
    \label{#1}%
}
\newcommand{\tabref}[1]{Table~\ref{#1}}

\newtheorem{theorem}{Theorem}

\newtheorem*{theorem*}{Theorem}
\newtheorem{lemma}{Lemma}
\newtheorem{corollary}{Corollary}

\newtheorem{assumption}{Assumption}

\newcommand{\reals}{\mathbb{R}} 

\newcommand{\E}{\mathbb{E}}

\newcommand{\loss}{\mathfrak{L}} 

\newcommand{\g}{\mathcal{G}} 
\newcommand{\q}{\mathcal{Q}} 

\newcommand{\bx}{\mathbf{x}} 
 
\newcommand{\bu}{\mathbf{u}} 
\newcommand{\bv}{\mathbf{v}} 
\newcommand{\bg}{\mathbf{g}} 
 
\newcommand{\bw}{\mathbf{w}}

\newcommand{\F}{\mathcal{F}}

\newcommand{\tildegt}[2]{\tilde{g}_{t,#2}^{#1}}

\newcommand{\gt}[2]{g_{t,#2}^{#1}}

\newcommand{\Gtvec}[1]{ \g_{t,#1}^{\boldsymbol{\lambda}_t} }
\newcommand{\tildeqt}[2]{\tilde{q}_{t,#2}^{#1}}

\newcommand{\qt}[2]{q_{t,#2}^{#1}}

\newcommand{\Qt}[1]{ \q_{t,#1}^{\lambda_t} }
\newcommand{\Q}[1]{ \q_{t,#1}^{\lambda} }

\title{Training with ML-labeled Data via Stochastic First- and Zeroth-Order Optimization}
\author{Maria-Eleni Sfyraki\footnote{
            University of California San Diego, 
            \texttt{msfyraki@ucsd.edu}}
\and Jun-Kun Wang\footnote{
            University of California San Diego, 
            \texttt{jkw005@ucsd.edu}}
            }
\date{}

\begin{document}
\maketitle

% REQUIRED
\begin{abstract}
Training a machine learning (ML) model is an optimization process that typically requires a set of examples with ground-truth labels beforehand. Obtaining these ground truths usually requires laborious human effort and can be costly.
The scalability of human labeling becomes a concern when there are a large number of data points that need to be labeled.
To address this practical issue, we consider an optimization setting where a subset of samples has ground-truth labels, while the remaining data points were labeled by an ML predictor, which could potentially assign incorrect labels to those points.
We propose both first-order and zeroth-order stochastic optimization methods for training on datasets consisting of a mix of true-labeled and ML-labeled samples.
We further conduct comprehensive empirical studies from training a linear regression model to fine-tuning a large language model (LLM) to showcase the effectiveness of the proposed algorithms for training with ML-labeled data.
\end{abstract}

\section{Introduction}
\label{sec:intro}

Stochastic optimization is the de facto approach for training modern ML models \citep{robbins1951stochastic,zhang2004solving,bottou2008,
rakhlin2011making,roux2012stochastic,shamir2013stochastic,ghadimi2013stochastic,needell2014stochastic,jain2018accelerating,nguyen2018sgd,gower2019sgd,cutkosky2019momentum,li2021page}.
The common machinery of first-order stochastic methods involves using an easy-to-compute estimate of the gradient (a.k.a.~a stochastic gradient) to update the underlying ML model. This results in a much lower iteration cost than deterministic methods such as gradient descent, which require gradient computations for all data points---a requirement that can become impractical at scale. The idea of using a mini-batch of samples to construct a gradient estimate can be naturally extended to zeroth-order optimization \citep{ghadimi2013stochastic,duchi2015optimal,nesterov2017random,balasubramanian2022zeroth,zhang2026zeroth}, where only function-value evaluations are used to construct an estimate. Recent work \citep{malladi2023fine} has observed that the memory footprint of a zeroth-order method is a few times smaller than that of its first-order counterpart in practice, which offers opportunities for fine-tuning an LLM using a single modern GPU.

In this work, we investigate stochastic optimization methods for datasets that consist of human-labeled and machine-labeled data.
Specifically, we assume access to a set of samples with ground-truth labels, denoted as $\{(\bx_i, y_i)\}_{i \geq 1}$, where $i$ is the sample index, and $\bx_i \in \reals^p$ and $y_i \in \reals$ are the feature vector and the true label for point $i$, respectively. Moreover, ML labels are also available for these data points with true labels, denoted as $\{f(\bx_i)\}_{i\geq1}$. 
In addition, we have a set of data without ground-truth labels but with ML-generated labels, denoted as $\{(\tilde{\mathbf{x}}_j, f(\tilde{\mathbf{x}}_j))\}_{j\geq 1}$, where $j$ is the sample index for ML-labeled data and that 
$f(\tilde{\mathbf{x}}_j)$ represents the ML-label for feature vector $\tilde{\mathbf{x}}_j \in \reals^p$.
We will assume that the feature vectors $\{\mathbf{x}_i\}_{i\geq 1}$ and $\{\tilde{\mathbf{x}}_j\}_{j\geq1}$ are drawn i.i.d.~from the same unknown distribution
and that the ML-labeling rule $f(\cdot)$ is independent from these samples.

This scenario is relevant given that modern large-scale datasets typically demand substantial human labeling and collaborative crowdsourcing \citep{ deng2009imagenet, lin2014microsoft, rajpurkar2016squad,mayo2023howhard}. Hence, the use of an ML predictor in labeling can help significantly reduce the associated costs and labor. While recent trends in active learning focus on designing methods to draw human attention to points where ML predictors lack confidence \citep{beygelzimer2009importance}, this work considers a different and complementary  direction. Specifically, we would like to answer the following question: \emph{Can we leverage the ML-labels to accelerate the optimization process?} 

In ML, a common objective function takes the following form:
\begin{equation} \label{obj:expected}
\min_{\bw \in \reals^d} \loss(\bw), \quad \text{ where } \quad \loss(\bw):= \E_{(\bx, y)}[ \ell(\bw; \bx, y)].
\end{equation}
Here, $\ell(\cdot; \bx, y)$ is the loss function that evaluates the predictive performance of a model $\bw$ on a sample $(\bx, y)$.
In particular, the objective function encompasses the finite-sum problem, i.e., $\min_{\bw \in \reals^d} \frac{1}{n} \sum_{i=1}^n  \ell(\bw; \bx_i, y_i)$, which fits into \oureqref{obj:expected} by interpreting the expectation as being taken with respect to a uniform draw from the training data.

We aim to leverage ML-labeled data to more efficiently solving stochastic optimization problem $\oureqref{obj:expected}$, compared to the setting without access to ML-labeled data.
Our results can be viewed as a counterpart of prediction-powered inference (PPI) in statistics \citep{angelopoulos2023prediction,angelopoulos2023ppi++,sfyraki2026revisiting}.
In PPI, the primary objective is to construct confidence intervals for parameter estimation by leveraging ML-labeled data. More specifically, the statistical inference problem of PPI shares a similar setting in which a small gold-standard dataset consisting of feature and outcome pairs $\{(\bx_i, y_i)\}_{i= 1}^n$ is available, along with a large collection of unlabeled samples without true outcomes $\{\tilde{\bx}_j\}_{j=1}^N$, i.e., $N \gg n$. Additionally, an ML model provides its predicted-labels for all the data points, denoted as $\{f(\bx_i)\}_{i=1}^n$ and $\{f(\tilde{\bx}_j)\}_{j=1}^N$.
On the other hand, their goal is to achieve statistical efficiency in estimating a parameter of interest, such as the population mean of the true outcomes $\E[y_i]$, while incorporating information from ML-labeled data with potentially biased predictions. To this end, \citet{angelopoulos2023prediction,angelopoulos2023ppi++} propose unbiased (or bias-corrected) estimators that combine gold-standard and ML-labeled data points. They show that the resulting confidence intervals are asymptotically valid and, in certain regimes where the predictor is sufficiently accurate and the number of ML-labeled samples is large, can be strictly shorter than those constructed using only data with true labels.

Our work is inspired by both the emergence of PPI in statistics \citep{angelopoulos2023prediction,angelopoulos2023ppi++,sfyraki2026revisiting} and a recent work by \citet{shoham2025prediction}, 
who extended the idea of leveraging ML-labeled data for statistical inference to the semi-supervised model training by proposing a gradient estimator that incorporates ML-labeled data points. Specifically, given $n$ samples with true labels and $N$ ML-labeled samples, one can construct a \emph{prediction-powered} gradient estimate:
$$\textstyle g_{\bw}^{\mathrm{PP}} := \frac{1}{N} \sum_{j=1}^N
\nabla \ell\big(\bw; \tilde{\bx}_j, f(\tilde{\bx}_j) \big)
+
\frac{1}{n} \sum_{i=1}^n 
\Big( \nabla \ell(\bw; \bx_i, y_i) - \nabla \ell(\bw; \bx_i, f(\bx_i)) \Big).$$
Under the assumption that all the samples $\{\tilde{\bx}_j\}_{j=1}^N$ and $\{\bx_i\}_{i=1}^n$ are drawn i.i.d.~from the same distribution, together with the condition that the ML-labeling rule $f(\cdot)$ is independent of the observed samples $\{\tilde{\bx}_j\}_{j=1}^N$ and $\{\bx_i\}_{i=1}^n$, the gradient estimator $g_{\bw}^{\mathrm{PP}} $ is unbiased for $\nabla \loss(\bw)$.
Motivated by PPI++ \citep{angelopoulos2023ppi++}, \citet{shoham2025prediction} further consider a tuning parameter that controls the contribution of the ML-labeled components in the prediction-powered gradient estimate. 
A more recent work by \citet{frankel2026abcalignpredictionpoweredalignmentadaptive}
extends this idea by introducing a second tuning parameter that allows controlled bias in estimating the population gradient $\nabla\loss(\bw)$ in exchange for a potential reduction in the estimator's variance.

For their prediction-powered gradient estimator, \citet{shoham2025prediction} further combine AdaGrad 
\citep{duchi2011adaptive,levy2017online} with an online learning algorithm \citep{hazan2022,orabona2019modern} that dynamically tunes $\lambda$ during optimization. They evaluate their method against a couple of baselines, including one using the best fixed $\lambda$ chosen in hindsight, and demonstrate its effectiveness in terms of final mean absolute error and final mean squared error on test sets for some regression and classification tasks.
However, their empirical evaluation focuses mainly on the predictive performance of the trained model rather than potential optimization speedups. Nevertheless, a plot of training loss versus epoch for CIFAR-10 in their supplementary (Figure~S12 in \citet{shoham2025prediction}) suggests that incorporating ML-labeled data might accelerate training. Intrigued by their plot, we investigate further to answer the aforementioned question, with the aim of strengthening the foundation of optimization with ML-labeled data.

To give the reader a flavor of our contributions, we summarize them here.
We first study standard stochastic gradient descent (SGD), a classical, general-purpose first-order stochastic optimization algorithm, and propose two theoretically grounded schemes that allow SGD leverage ML-labels for variance reduction.
Let $T$ be the number of iterations. In terms of the common stationarity metric in general smooth nonconvex optimization, i.e., $\frac{1}{T}\sum_{t=1}^T \E\left[ \| \nabla \loss(\bw_t) \|^2 \right]$, we provide convergence rate results that characterize the potential improvement in SGD achieved by incorporating ML labels, compared to using only true labels.
The proposed schemes aim to dynamically update the contribution of ML labels when constructing a prediction-powered gradient estimate at each step and to advance the perspective of \citet{shoham2025prediction}. Specifically, unlike prior work, we do not assume that the objective $\loss (\bw)$ is upper bounded or that the stochastic gradient $\nabla_{\bw}\ell(\bw; \bx, y)$ is $L$-Lipschitz in $y$.
At the same time, each of our proposed schemes still admits a lightweight closed-form update obtained by minimizing a quadratic function.
Furthermore, whereas prior works \citep{angelopoulos2023ppi++,shoham2025prediction} restrict the weights on the ML-labeled components in the estimates to the interval $[0,1]$, we lift this constraint so that the weights determining the contribution of these components at each step can take any real value, including those outside this interval. This allows the underlying optimization methods to fully exploit the ML-labeled data to reduce variance and further accelerate the optimization process. Of the two schemes that we propose, one updates online a single tuning parameter shared across all dimensions, as in prior work, while the other uses different weights to combine the components of the gradient estimate in different dimensions.

We then extend the machinery that we develop for first-order optimization to zeroth-order stochastic optimization, where we build upon a common stochastic zeroth-order oracle based on Gaussian perturbations, which was also considered in prior seminal works, e.g., \citet{ghadimi2013stochastic,balasubramanian2022zeroth}.
Interestingly, we find a similar variance reduction mechanism in the iteration complexity analysis of the stochastic gradient-free method with the two proposed online tuning schemes adapted to zeroth-order optimization. To our knowledge, this is the first result on stochastic gradient-free optimization with ML-labeled data in the literature.
Finally, to corroborate our theoretical results, we conduct experiments to showcase the effectiveness of the proposed method in accelerating optimization with ML-labeled data. These experiments range from fitting a linear regression model to training a image classifier and fine-tuning an LLM.

\section{First-Order Stochastic Optimization with ML-labeled Data} \label{sec:first_order_stochastic_optimization}

\noindent
\textbf{Notations.} We begin by introducing some notations first.
Throughout this paper, we use boldface $\boldsymbol{\lambda} \in \reals^d$ to denote a vector of tuning parameters, while $\lambda \in \reals$ denotes a scalar tuning parameter.
We use $\|\cdot\|$ to denote the Euclidean norm throughout.
For any vector $\bv$, $\bv[i]$ denotes its $i_{\mathrm{th}}$-element
Furthermore, $\odot$ represents the Hadamard product; specifically, for any pair of vectors $\bw, \bv \in \reals^d$, $\bw \odot \bv \in \reals^d$ is their element-wise product.
We further note that, given a scalar tuning parameter $\lambda \in \reals$, a scalar-vector multiplication $\lambda \bv$ can be rewritten as $\lambda \bv = \lambda \mathbf{1} \odot \bv$, where $\mathbf{1} \in \reals^d$ is the all-ones vector.
For conciseness, we will use this fact to present the proposed online tuning schemes of this work in a single algorithmic description in the following, i.e., Algorithm~\ref{alg:sgd} and Algorithm~\ref{alg:sgfm}. The distinction between scalar and vector tuning parameters should be clear from the context and will be highlighted when necessary. 
We will let $\mathcal{F}_t$ denote the filtration representing the information available up to and including step $t$.
For any collection of random variables or random elements, $\sigma(\cdot)$ denotes their generated sigma-algebra.
Also, we use $\Delta_1:=\loss(\bw_1)-\inf_{\bw\in\mathbb{R}^d}\loss(\bw)$ to represent the initial optimality gap, where $\bw_1$ is an initial point.
%Finally, for a set $B$, $|B|$ denotes its cardinality.

\begin{algorithm}[t]
\caption{SGD with ML-labeled data}
\label{alg:sgd}
% \small
\begin{algorithmic}[1]
\REQUIRE Initial point $\bw_1 \in \reals^d$ and step sizes $\{\eta_t\}_{t=1}^T$.
\FOR{$t=1,\ldots,T$}
    \STATE Sample a minibatch of points with true labels $B_t$ and compute
    $\gt{y}{\bw_t}:=\frac{1}{ |B_t| } \sum_{(\bx, y) \in B_t} \nabla \ell \left( \bw_t ; \bx, y \right)$ and 
    $\gt{f}{\bw_t}:=\frac{1}{ |B_t| } \sum_{(\bx, y) \in B_t} \nabla \ell \left( \bw_t ; \bx, f(\bx) \right)$.
   \STATE  Sample an ML-labeled minibatch $\tilde{B}_t $ and compute 
   $\tildegt{f}{\bw_t} := \frac{1}{ |\tilde{B}_t| } \sum_{\tilde{\bx} \in  \tilde{B}_t } \nabla \ell \left( \bw_t ; \tilde{\bx}, f\left( \tilde{\bx} \right) \right)$.
   \STATE Update $\boldsymbol{\lambda}_t$ using either the scalar tuning scheme \oureqref{scalar-tuning-first-order} or the vector tuning scheme \oureqref{vector-tuning-first-order}.
   \STATE Update $\bw_{t+1}= \bw_t - \eta_t \, \Gtvec{\bw_t} $,
   where $ \Gtvec{\bw_t} := \gt{y}{\bw_t} + \boldsymbol{\lambda}_t \odot \left( \tildegt{f}{\bw_t} - \gt{f}{\bw_t} \right).$
\ENDFOR
\end{algorithmic}
\end{algorithm}

\subsection{Our Methods} \label{sub:our_methods}

Algorithm~\ref{alg:sgd} describes the proposed SGD with ML-labeled data.
For further notation brevity, in the following, we will also denote
\begin{equation*}
\textstyle
    \mathbf{g}_t := \gt{y}{\bw_t}, %\text{ and }
    \quad \, \quad
    \mathbf{d}_t := \tildegt{f}{\bw_t} - \gt{f}{\bw_t}.
\end{equation*}
Hence the prediction-powered gradient estimate at $t$ in Algorithm~\ref{alg:sgd} can be rewritten as $\Gtvec{\bw_t} = \mathbf{g}_t + \lambda  \mathbf{d}_t $.

The critical step for effectively leveraging the ML-labeled data lies in how to update the parameter $\boldsymbol{\lambda}_t$ at each time $t$, i.e., Step 4 in Algorithm~\ref{alg:sgd}, which determines how much contribution from the stochastic gradients based on ML-labels should be weighted. We propose two online update schemes: one uses a single scalar tuning parameter $\lambda$, and the other weights the contribution differently across dimensions.

For the scalar case, we let $\boldsymbol{\lambda}_t \gets \lambda_t \odot \mathbf{1}$ in Algorithm~\ref{alg:sgd}, where $\mathbf{1}$ is the all-ones vector; consequently one can write
$\mathbf{g}_t + \boldsymbol{\lambda}_t \odot \mathbf{d}_t = \mathbf{g}_t  + \lambda_t \mathbf{d}_t$ for this case. This suggests that only a single scalar parameter $\lambda_t$ is required to be maintained in each step. 
We propose the following update scheme: 
\begin{equation}
\textstyle
\textbf{(Scalar tuning)} \qquad  
        \lambda_t 
		=	- \frac{\sum_{s=1}^{t-1} \langle \mathbf{g}_s, \mathbf{d}_s \rangle }{\sum_{s=1}^{t-1} \left\| \mathbf{d}_s \right\|^2},  \quad \text{if } \, \sum_{s=1}^{t-1} \left\| \mathbf{d}_s \right\|^2 > 0,    \label{scalar-tuning-first-order}
\end{equation}
where we let $\lambda_t = 0$ when $\sum_{s=1}^{t-1} \left\| \mathbf{d}_s \right\|^2 = 0$, such as the initial step. We remark that as long as the \textbf{d}ifference of the stochastic gradient components $\mathbf{d}_t$ at any time point is non-zero, the sum is evidently nonzero as well, which is likely to hold as early as $t=2$ in practice.

The other scheme uses the full vector tuning parameter $\boldsymbol{\lambda}_t$. 
Specifically, for each $i \in [d]$, we update
\begin{equation}
\textstyle
\textbf{(Vector tuning)} \, \quad
    \boldsymbol{\lambda}_t [i] =
        - \frac{\sum_{s=1}^{t-1}\mathbf{g}_s[i] \mathbf{d}_s[i]}{\sum_{s=1}^{t-1} \mathbf{d}_s[i]^2} ,   \quad \text{if } \, \sum_{s=1}^{t-1} \mathbf{d}_s[i]^2 > 0,
     \label{vector-tuning-first-order}
\end{equation}
and we let $\boldsymbol{\lambda}_t [i]  = 0$, otherwise.

\begin{wrapfigure}[18]{r}{0.34\textwidth}
    \vspace{-\baselineskip}
    \begin{minipage}{\linewidth}
    \begin{mdframed}
    \small
        \centering
        \textbf{Scalar Tuning as a Game}\par
        \medskip 
            \vspace{-0.7\baselineskip}
        \small
        \label{alg:online-scalar-tuning}
        \begin{algorithmic}[1]
            \FOR{$t=1,\ldots,T$}
                \STATE Play $\lambda_t \in \mathbb{R}$.
                \STATE Observe the loss function 
                $
                    \phi_t(\lambda)
                    := \|\mathbf{g}_t
                    + \lambda\mathbf{d}_t\|^2
                $
                \STATE Incur loss $\phi_t(\lambda_t)$.
            \ENDFOR
        \end{algorithmic}
        \end{mdframed}
            \begin{mdframed}
    \small
        \centering
        \textbf{Vector Tuning as a Game}\par
        \medskip 
            \vspace{-0.7\baselineskip}
        \small
        \label{alg:online-scalar-tuning}
        \begin{algorithmic}[1]
            \FOR{$t=1,\ldots,T$}
                \STATE Play $\boldsymbol{\lambda}_t \in \mathbb{R}^d$.
                \STATE Observe the loss function  \qquad \qquad
                    $\phi_t(\boldsymbol{\lambda}_t)
                    := \|\mathbf{g}_t
                    + \boldsymbol{\lambda}_t \odot \mathbf{d}_t\|^2$
                \STATE Incur loss $\phi_t(\boldsymbol{\lambda}_t )$.
            \ENDFOR
        \end{algorithmic}
        \end{mdframed}
    \end{minipage}
    % \vspace{-0.5\baselineskip}
        \caption{\small Game-theoretic interpretation of the online tuning schemes
        } \label{fig:online_tuning_schemes}
\end{wrapfigure}
\vspace{3mm}
\noindent 
\textbf{Algorithm Design Principle.} 
We now explain the mechanism underlying the proposed online tuning schemes.

Recall the standard SGD update: $\bw_{t+1} = \bw_t - \eta_t \mathcal{G}_t$,
where $\mathcal{G}_t$ is a stochastic gradient.
The machinery of standard SGD is usually built on the premise that $\mathcal{G}_t$ is a conditionally unbiased estimator, i.e.,
$\E[\mathcal{G}_t \mid \mathcal F_{t-1}] = \nabla \loss(\mathbf{w}_t)$.
In our problem, $\mathcal{G}_t = \Gtvec{\bw_t}$.
If the weight $\boldsymbol{\lambda}_t$ in $\Gtvec{\bw_t}$ is not
$\mathcal{F}_{t-1}$-measurable, then in general,
$\E[\Gtvec{\bw_t} \mid \mathcal F_{t-1}] \neq \nabla \loss(\mathbf{w}_t)$.
In other words, the prediction-powered gradient estimate $\Gtvec{\bw_t}$ is not necessarily unbiased unless
$\boldsymbol{\lambda}_t$ is $\mathcal{F}_{t-1}$-measurable.
While convergence results for SGD with biased gradient estimates are established in the literature, such methods typically converge only to a neighborhood, with a nonvanishing bias term in the bound of the iteration complexity, see e.g., \citet{demidovich2023guide}. Hence, we do not pursue biased SGD in this work and instead require $\boldsymbol{\lambda}_t$ to be $\mathcal{F}_{t-1}$-measurable. To satisfy this requirement, the update of $\boldsymbol{\lambda}_t$ cannot use stochastic gradient information from step $t$; only information available up to the previous step can be used.
Therefore, the problem is inherently an \emph{online learning} problem, where the learner chooses an action using only past information before observing the current round's loss function.

Figure~\ref{fig:online_tuning_schemes} presents the game-theoretic
interpretation of the online tuning schemes.
The loss functions in these games arise naturally from an intermediate
step in the standard convergence analysis of SGD with unbiased gradient
estimates, which yields
\[
\textstyle
\sum_{t=1}^T
\left(\eta_t - \frac{\beta}{2}\eta_t^2\right)
\E\left[\|\nabla \loss(\bw_t)\|^2\right]
\leq
\Delta_1
+ \frac{\beta}{2}\sum_{t=1}^T \eta_t^2 \,
\E\left[
    \mathbb{E}_{t-1}\left[\|\Gtvec{\bw_t}\|^2\right]
    - \|\nabla \loss(\bw_t)\|^2
\right],
\]
where we replicate the detailed derivation in Appendix~\ref{app:B}.
Hence, it is clear that one should update the weight $\boldsymbol{\lambda}_t$ to make this bound small, where $\boldsymbol{\lambda}_t$ enters the bound through $\mathbb{E}_{t-1}\left[ \|\Gtvec{\bw_t}\|^2 \right]$.
This motivates using the squared norm of the gradient estimator as the loss function in each of our online learning problems, i.e., $\phi_t(\cdot)$ in Figure~\ref{fig:online_tuning_schemes}.

The goal of online learning is to
minimize regret \citep{hazan2022,orabona2019modern}, i.e., minimizing
$
\mathrm{Regret}_T^{\phi}(\lambda^*)
:= \sum_{t=1}^T \phi_t(\lambda_t)
- \sum_{t=1}^T \phi_t(\lambda^*),
$
where $\lambda^*$ is a comparator in the same decision space as the
online learner, namely $\reals$ in the scalar tuning.
Their counterparts in the vector tuning game are defined analogously.
In both games, the quadratic loss functions motivate the use of the Follow-the-Leader (FTL) strategy, which is known to enjoy a logarithmic regret; see e.g., Section 2.2 in \citet{ShalevShwartz2011Online} and Section 3.3 in \citet{wang2024no}.
FTL selects a minimizer of the cumulative loss observed so far.
In the scalar tuning game, this gives
$
\lambda_t \in \arg\min_{\lambda \in \reals}
\sum_{s=1}^{t-1} \phi_s(\lambda),
$
while in the vector tuning game, the update is
$
\boldsymbol{\lambda}_t \in
\arg\min_{\boldsymbol{\lambda} \in \reals^d}
\sum_{s=1}^{t-1} \phi_s(\boldsymbol{\lambda}).
$
Here, we have overloaded the notation $\phi_t(\cdot)$ to denote the function in scalar tuning or vector tuning, as determined by context; see Figure~\ref{fig:online_tuning_schemes}.
Routine calculations show that these lead to the closed-form updates \oureqref{scalar-tuning-first-order} and \oureqref{vector-tuning-first-order}, respectively, where we provide the details in Appendix~\ref{app:B}.

\subsection{Theoretical analysis}

We first recall a couple of notions in optimization. 
\vspace{2mm} \\
\noindent
\textbf{Gradient Lipschitzness and Smoothness.}
A differentiable function $\ell(\cdot)$ has $\beta$-Lipschitz gradients if 
$\| \nabla \ell(\bw) - \nabla \ell(\bv) \| \leq \beta \| \bw  - \bv \|$.
Furthermore, a differentiable function $\ell(\cdot)$ is $\beta$-smooth if for any pair of points $\bw$, $\bv \in \reals^d$, it holds that
$\left| \ell(\bw) - \ell(\bv) - \langle \nabla \ell(\bv), \bw - \bv \rangle \right| \leq \frac{\beta}{2} \| \bw - \bv \|^2.$

We remark that $\beta$-smoothness implies $\beta$-Lipschitz gradients (see e.g., \citet{giselsson2017envelopefunctionsunificationsproperties}). Having introduced these notions, we are ready to state
the assumptions that are used later to analyze the convergence behavior of the proposed methods in this work.
\begin{assumption}(Bounded Infimum) \label{A1}
The infimum is bounded below, i.e., $\inf_{\bw \in \reals^d} \, \loss(\bw) > -\infty$.
\end{assumption}
\begin{assumption}(Smoothness) \label{A2}
The loss functions $\ell(\bw;\bx,z)$ satisfy the $\beta$-smoothness,
where $z$ is either the ground truth $y$ or the ML-label $f(\bx)$. Consequently, $\loss(\bw)$ is $\beta$-smooth. 
\end{assumption}
\begin{assumption}(Finite variance) \label{A3}
The stochastic gradients $\nabla_{\bw}\ell(\bw;\bx,y)$ have bounded variance, i.e., for some finite $\sigma \geq 0$, $\E\left[\|\nabla_{\bw}\ell(\bw;\bx,y)-\nabla\loss(\bw)\|^2\right]\leq\sigma^2$.
\end{assumption}
\begin{assumption}(Bounded stochastic gradient) \label{A4}
There exists $G>0$, such that $\|\nabla_{\bw}\ell(\bw;\bx,z)\| \leq G$ a.s. for all $\bw \in \mathbb{R}^d$, where $z$ is either the ground truth $y$ or the ML-label $f(\bx)$.
\end{assumption}

We note that lower-boundedness, smoothness, and finite variance are common conditions in the analysis of nonconvex optimization, e.g., \citep{ghadimi2013stochastic,sfyraki2025}.
Here, we impose smoothness on the individual loss function for the simplicity of the analysis and presentation in zeroth-order optimization later on. 
The boundedness of the stochastic gradients additionally ensures that each individual loss $\ell(\cdot;\bx,y)$ is $G$-Lipschitz in $\bw$.
We further note that the same set of assumptions was also considered in an elegant work of \citet{cutkosky2019momentum} for stochastic optimization.

In the following, we denote by
$\Phi_{1:T}(\lambda):=\sum_{t=1}^T \bigl(\phi_t(\lambda)-\phi_t(0)\bigr)$
the cumulative loss, and we also let
$\lambda^*\in\arg\min_{\lambda\in\reals}\E[\Phi_{1:T}(\lambda)]$
be a minimizer of its expectation.

\begin{mdframed}[backgroundcolor=black!10,rightline=false,leftline=false,topline=false,bottomline=false]
\vspace{-3mm}
\begin{theorem} \label{th: first-order-scalar}
Denote $C_0 := \sup_{T \geq 1} \mathbb E[\max_{1\le t\le T+1}|\lambda_t|^2]$.
Then, under Assumptions~\ref{A1}--\ref{A4}, Algorithm~\ref{alg:sgd} 
with scalar tuning and $\eta_t = \frac{1}{\beta \sqrt{T}}$ satisfies
\begin{align} 
     \frac{1}{T}\sum_{t=1}^T \E\left[ \| \nabla  \loss(\bw_t) \|^2 \right] 
        \leq
         \frac{2 \beta \Delta_1}{\sqrt{T}} 
        + \frac{8 G^2(1+2C_0)(3+\log T)}{T^{3/2}}  
        + \frac{\E[ \Phi_{1:T}(\lambda^*) ]}{T^{3/2}} +  \frac{\sigma^2}{\sqrt{T}} ,
        \label{eq: theorem_bound}
\end{align}
where 
$    \E \left[ \Phi_{1:T}(\lambda^*) \right]
    =
    -
    \frac{\left( \sum_{t=1}^T \E \left[ \left\langle \mathbf{g}_t , \mathbf{d}_t   \right\rangle \right] \right)^2 }{\sum_{t=1}^T \E \left[ \left\| \mathbf{d}_t \right\|^2 \right]}, \mathrm{ if } \, \sum_{t=1}^T \E \left[ \left\| \mathbf{d}_t \right\|^2 \right]>0,
$ and $\E \left[ \Phi_{1:T}(\lambda^*) \right]=0$ otherwise. In particular, it holds that
$  - T \sigma^2 \leq \E \left[ \Phi_{1:T}(\lambda^*) \right] \leq 0.$
\end{theorem}
\end{mdframed}

We remark that, compared to the standard SGD without using ML-labeled data \citep{ghadimi2013stochastic}, there are two additional middle terms in the upper bound in \oureqref{eq: theorem_bound}. The first middle term arises from the regret incurred by the online FTL tuning procedure and is of order $O\left(\frac{\log T}{T^{3/2}}\right)$, and hence is negligible relative to the $O\left(\frac{1}{\sqrt{T}}\right)$ terms when $T$ is sufficiently large. 
The prefactor of this regret term depends on the tuning coefficients through the second moment of their running maximum. Since scalar FTL aggregates all past rounds, Assumption~\ref{A4} and Cauchy--Schwarz give $|\lambda_{t+1}|\leq G/\sqrt{t^{-1}\sum_{s=1}^t\|\mathbf d_s\|^2}$ whenever the denominator is positive. 
Thus, the coefficients remain bounded
along any run in which the average squared norm of the stochastic gradient differences stays bounded away from zero. 
The coefficient trajectories reported in our experiments in Appendix~\ref{app:D} are consistent with this behavior over the training horizons considered (see, e.g., Figures~\ref{fig:sst_fo_lam}, \ref{fig:cifar10_fo_lambda}, \ref{fig:cal_fo_lam}, \ref{fig:sat_fo_lam}). Interestingly, we found that the magnitude of $\lambda_t$ is relatively large early in training and gradually settles or fluctuates around a smaller value in later stages. In Appendix~\ref{app:FTRL}, we provide additional convergence guarantees for a Follow-the-Regularized-Leader (FTRL) tuning scheme which can retain the $O(T^{-1/2})$ convergence bound while removing dependence on the maximum magnitude of the tuning parameters. The second-to-last term in \oureqref{eq: theorem_bound} quantifies how much the ML-labeled data can reduce the variance contribution to the bound.
Under the reasonable condition that $\E\left[\|\mathbf{d}_t\|^2\right]$ is nonzero for at least one $t \in [T]$,
the last two terms in \oureqref{eq: theorem_bound} can be combined 
as 
$ \frac{\E[ \Phi_{1:T}(\lambda^*) ]}{T^{3/2}} +  \frac{\sigma^2}{\sqrt{T}}
= \frac{1}{\sqrt{T}} \left( \sigma^2 - \frac{\left( \frac{1}{T} \sum_{t=1}^T \E \left[\left\langle \mathbf{g}_t , \mathbf{d}_t   \right\rangle\right] \right)^2}{ \frac{1}{T} \sum_{t=1}^T \E \left[ \left\| \mathbf{d}_t \right\|^2 \right]} \right)$,
which is no greater than the standard $\frac{\sigma^2}{\sqrt{T}}$ term and demonstrates the benefit of exploiting ML-labeled data.

To further understand the potential improvement in the convergence rate from using ML-labeled data, we present the following corollary,
where we consider constant batch sizes such that $\left| B_t \right| = n$ and $\left| \tilde{B}_t \right| = N$ for all $t \in [T]$.
Below, we define $\loss^f(\bw) := \E \left[ \ell \left( \bw ; \bx, f(\bx) \right) \right]$.
\begin{mdframed}[backgroundcolor=black!10,rightline=false,leftline=false,topline=false,bottomline=false]
\begin{corollary} \label{corollary}
Suppose the variance of a gradient computed from a single ML-labeled sample is uniformly bounded by $\sigma_f^2$, i.e.,
$\E \left[ \left\| \nabla \ell \left( \bw ; \bx, f(\bx) \right) - \nabla \loss^f(\bw) \right\|^2 \right] \leq \sigma_f^2$, while $\sigma_y^2$ represents its counterpart using true labels.
Additionally, let $\bar{\Xi}_T:=  \frac{1}{T}\sum_{t=1}^T \E\left[ \| \nabla \loss(\bw_t) \|^2 \right]$ 
and $\Sigma_{t}^{y,f} := \operatorname{Cov} \left( \nabla \ell \left( \bw_t ; \bx, y \right), \nabla \ell \left( \bw_t ; \bx, f(\bx) \right) \big| \mathcal{F}_{t-1}\right)$ denote the conditional covariance between stochastic gradients computed using true and ML labels at step $t$.
Consider the case that $\exists t \in [T]: \E[ \| \mathbf{d}_t \|^2] > 0$.
Under the assumptions of Theorem~\ref{th: first-order-scalar},
Algorithm~\ref{alg:sgd} satisfies
\begin{align*}
        \bar{\Xi}_T   &\leq
         \frac{2 \beta \Delta_1}{\sqrt{T}} 
        + \frac{8 G^2(1+2C_0)(3+\log T)}{T^{3/2}} 
        + \frac{1}{\sqrt{T}} \frac{\sigma_y^2}{n}
        - \frac{1}{\sqrt{T}} \frac{N}{n (n+N)} \frac{ \left( \frac{1}{T}\sum_{t=1}^T \E \left[ \operatorname{tr} \left( \Sigma_{t}^{y,f}  \right) \right] \right)^2 }{\sigma_f^2}.
\end{align*}
\end{corollary}
\end{mdframed}

Corollary~\ref{corollary} suggests that the variance reduction is larger when the deviations of the true-label and ML-label gradients from their respective population gradients are more strongly correlated in expectation, relative to the standard deviation of a stochastic gradient computed from an ML-labeled sample.
Intuitively, the variance reduction and hence the speedup are more significant when stochastic gradients computed from ML-labeled data are more ``predictive" of those computed from data with true labels. We further note that the unlabeled batch size $N$ affects the variance-reduction term through the factor $\frac{N}{n(n+N)}$ from the corollary, which suggests that while the overall improvement continues to increase with the number of ML-labeled samples $N$, the additional improvement obtained by increasing $N$ diminishes as $N$ grows, with $n$ fixed.

We defer the analysis of SGD with vector tuning to Theorem~\ref{th: first-order-vector} in Appendix~\ref{app:B}. As in the scalar tuning case, the bound contains a term that reflects the variance reduction enabled by incorporating ML-labeled data. We refer the reader to Appendix~\ref{app:B} for details.

\section{Zeroth-Order Stochastic Optimization with ML-labeled Data} \label{sec:zeroth_order_stochastic_optimization}

\subsection{Our Methods}

The machinery we developed in the previous section can be naturally extended to zeroth-order stochastic optimization. We propose Algorithm~\ref{alg:sgfm} with the two online tuning schemes for leveraging ML-labeled data effectively. We emphasize that only one random Gaussian perturbation needs to be sampled at each step and that it is shared across the different components used to construct the zeroth-order prediction-powered gradient estimate.
For further notation brevity, we will denote
 $\mathbf{q}_t := \qt{y}{\bw_t}$ and  
 $\mathbf{d}^0_t := \tildeqt{f}{\bw_t} - \qt{f}{\bw_t}$
below.

Our corresponding tuning schemes in zeroth-order optimization are:
\begin{equation}
\textstyle
\textbf{(Scalar tuning)} \qquad  
        \lambda_t 
		=	- \frac{\sum_{s=1}^{t-1} \langle \mathbf{q}_s, \mathbf{d}^0_s \rangle }{\sum_{s=1}^{t-1} \left\| \mathbf{d}^0_s \right\|^2},  \quad \text{if } \, \sum_{s=1}^{t-1} \left\| \mathbf{d}^0_s \right\|^2 > 0,    \label{scalar-tuning-first-order-zoo}
\end{equation}
where we let $\lambda_t = 0$ when $\sum_{s=1}^{t-1} \left\| \mathbf{d}^0_s \right\|^2 = 0$ such as the initial step.
The other scheme uses the full vector tuning parameter $\boldsymbol{\lambda}_t$;
specifically, for each $i \in [d]$, we update
\begin{equation}
\textstyle
\textbf{(Vector tuning)} \, \quad
    \boldsymbol{\lambda}_t [i] =
        - \frac{\sum_{s=1}^{t-1}\mathbf{q}_s[i] \mathbf{d}^0_s[i]}{\sum_{s=1}^{t-1} (\mathbf{d}^0_s[i])^2} ,   \quad \text{if } \, \sum_{s=1}^{t-1} (\mathbf{d}^0_s[i])^2 > 0,
     \label{vector-tuning-first-order-zoo}
\end{equation}
and we let $\boldsymbol{\lambda}_t [i]  = 0$, otherwise.
These updates are derived using the same principle of online tuning as a game as illustrated in the previous section, except that one observes a sequence of loss functions based on zeroth-order gradient estimates here. We refer the reader to Appendix~\ref{app:C} for details.

\subsection{Theoretical analysis}

\begin{algorithm}[t]
\caption{Stochastic Gradient-Free Method with ML-labeled data}
\label{alg:sgfm}
% \small
\begin{algorithmic}[1]
\REQUIRE Initial point $\bw_1 \in \reals^d$, smoothing parameter $\mu$, and step sizes $\{\eta_t\}_{t=1}^T$.
\FOR{$t=1,\ldots,T$}
    \STATE Sample $\zeta_t \sim \mathcal{N}(0,I_d)$.
    \STATE Compute
     $\qt{y}{\bw_t} := \frac{1}{ |B_t| } \sum_{(\bx, y) \in B_t}
    \frac{\ell\left(\bw_t + \mu \zeta_t; \bx,y \right) - \ell(\bw_t; \bx, y )  }{ \mu } \zeta_t$ by sampling a mini-batch $B_t$ of points with true labels and also compute
     $\qt{f}{\bw_t} := \frac{1}{ |B_t| } \sum_{(\bx, y)  \in B_t } \frac{\ell\left(\bw_t + \mu \zeta_t; \bx, f(\bx) \right) - \ell(\bw_t; \bx, f(\bx) )  }{ \mu } \zeta_t$. 
    \STATE Compute $\tildeqt{f}{\bw_t} := \frac{1}{ |\tilde{B}_t| } \sum_{\tilde{\bx} \in  \tilde{B}_t } \frac{\ell\left(\bw_t + \mu \zeta_t; \tilde{\bx}, f(\tilde{\bx}) \right) - \ell(\bw_t; \tilde{\bx}, f(\tilde{\bx}) )  }{ \mu } \zeta_t$ from
    a batch $\tilde{B}_t$ of ML-labeled points.
   \STATE Obtain $\boldsymbol{\lambda}_t$ using either the scalar tuning scheme \oureqref{scalar-tuning-first-order-zoo} or the vector tuning scheme \oureqref{vector-tuning-first-order-zoo}.
   \STATE Update $\bw_{t+1}= \bw_t - \eta_t \, \q_{t,\bw_t}^{\boldsymbol{\lambda}_t} $,
   where $ \q_{t,\bw_t}^{\boldsymbol{\lambda}_t} := \qt{y}{\bw_t} + \boldsymbol{\lambda}_t \odot \left( \tildeqt{f}{\bw_t} - \qt{f}{\bw_t} \right).$
\ENDFOR
\end{algorithmic}
\end{algorithm}

In the following, we denote by
$\Psi_{1:T}(\lambda):=\sum_{t=1}^T \bigl(\psi_t(\lambda)-\psi_t(0)\bigr)$
the cumulative loss, where $\psi_t(\lambda): = \| \q_{t, \bw_t}^\lambda \|^2 = \| \mathbf{q}_t + \lambda \mathbf{d}_t^0 \|^2$, and let
$\lambda^*\in\arg\min_{\lambda\in\reals}\E[\Psi_{1:T}(\lambda)]$
be a minimizer of its expectation.

\begin{mdframed}[backgroundcolor=black!10,rightline=false,leftline=false,topline=false,bottomline=false]
\vspace{-3mm}
\begin{theorem} \label{th: zeroth-order-scalar}
Denote $D_0:=
\sup_{T\ge1}\mathbb E\!\left[\max_{1\le t\le T+1}|\lambda_t|^2\right]$,
and $\bar{\Xi}_T:=\frac{1}{T} \sum_{t=1}^T \E\left[ \left \| \nabla  \loss(\bw_{t}) \right \|^2 \right] $. 
Then, under Assumptions~\ref{A1}-\ref{A4}, 
Algorithm~\ref{alg:sgfm} with scalar tuning, $\eta = \min \left\{ \frac{1}{4 \beta (d+2)}, \frac{1}{4 \beta \sqrt{(d+2) T}}   \right\}$, and $\mu=\frac{1}{\beta (d+3) \sqrt{T}}$ satisfies
    \begin{align}
\textstyle  \bar{\Xi}_T
        &\textstyle  \leq 
        16 \beta \Delta_1 \left( \frac{d+2}{T} + \frac{\sqrt{d+2}}{\sqrt{T}} \right) + \frac{16 d}{(d+3)^2} \left( \frac{d+2}{T^{2}} + \frac{\sqrt{d+2}}{T^{3/2}} \right) + \frac{d+3}{T} \notag \\
        &\textstyle \quad + \frac{16G^2(1+2D_0)\bigl(d+4\log(eT)\bigr)^2(4+\log T)}
        {\sqrt{d+2}\,T^{3/2}} + \frac{\sqrt{d+2}}{\sqrt{T}} \left( \sigma^2 +\frac{ d (d+4) }{4 (d+3)^2 T}  + \frac{\E\left[  \Psi_{1:T} \left( \lambda^* \right) \right]}{2(d+2)T} \right),
        \label{eq: zo_theorem_bound}
    \end{align}
where
$\E \left[ \Psi_{1:T}(\lambda^*) \right] =
- \frac{\left( \sum_{t=1}^T \E \left[ \left\langle \mathbf{q}_t , \mathbf{d}_t^0   \right\rangle \right] \right)^2 }{\sum_{t=1}^T \E \left[ \left\| \mathbf{d}_t^0 \right\|^2 \right]}$
if \, $\sum_{t=1}^T \E\left[\left\|\mathbf{d}_t^0\right\|^2\right]>0$,
and $\E \left[ \Psi_{1:T}(\lambda^*) \right] =0$ otherwise. In particular, it holds that
$
\textstyle 
-2 (d+2) T \sigma^2
- \frac{d (d+2) (d+4)}{2 (d+3)^2} 
\leq \E[ \Psi_{1:T}(\lambda^*) ] \leq 0.
$
\end{theorem}
\end{mdframed}

One of the technical results used to prove Theorem~\ref{th: zeroth-order-scalar} is Lemma~\ref{lem:exp_q_and_diff_q_and_bar_q} in Appendix~\ref{app:C}, which relates the expected squared norm of the zeroth-order gradient estimator $\mathbf{q}_t$
to that of the corresponding first-order stochastic gradient $\mathbf{g}_t$. 
This connection allows us to express the zeroth-order convergence bound in terms of the variance $\sigma^2$ of stochastic gradients and an additional finite-difference remainder term. This analysis strategy is inspired from that adopted in the seminal work of \citet{ghadimi2013stochastic} on standard stochastic gradient-free methods without ML-labeled data; see the proof of Theorem~3.2 therein. Compared to their result, we obtain two additional terms in the upper bound in \oureqref{eq: zo_theorem_bound}. The first is the leading term in the second row in \oureqref{eq: zo_theorem_bound}, which corresponds to the regret incurred by the online FTL tuning procedure. 
% As in the first-order result, this term is $O\left(\frac{\log T}{T^{3/2}}\right)$, and therefore negligible for sufficiently large $T$. The second additional contribution is $\frac{\E\left[  \Psi_{1:T} \left( \lambda^* \right) \right]}{2(d+2)T}$. 
The regret contribution is $O\!\left((\log T)^3/T^{3/2}\right)$ and is therefore asymptotically negligible relative to the leading $T^{-1/2}$ term in the convergence bound.
As in the first-order case, its prefactor depends on the tuning
coefficients through the second moment of their running maximum.
By Cauchy--Schwarz and Assumption~\ref{A4}, the scalar FTL update satisfies
$
    |\lambda_{t+1}|^2
    \leq
    \frac{t^{-1}\sum_{s=1}^t\|\mathbf q_s\|^2}
         {t^{-1}\sum_{s=1}^t\|\mathbf d_s^0\|^2}
    \leq
    \frac{G^2\,t^{-1}\sum_{s=1}^t\|\zeta_s\|^4}
         {t^{-1}\sum_{s=1}^t\|\mathbf d_s^0\|^2}
$
whenever the denominator is positive. For independent Gaussian directions,
$t^{-1}\sum_{s=1}^t\|\zeta_s\|^4$ converges almost surely to $d(d+2)$. Thus, along almost every run in which the average squared norm of the stochastic gradient differences stays bounded away from zero, the tuning parameters remain bounded. This behavior is also supported empirically in our experiments in Appendix~\ref{app:D} by tracking the tuning parameters over the course of our zeroth-order training (see, e.g., Figures~\ref{fig:sst_zo_lam}, \ref{fig:cifar10_zo_lam}, \ref{fig:cal_zo_lam}, \ref{fig:sat_zo_lam}). The second additional term $\frac{\E\left[  \Psi_{1:T} \left( \lambda^* \right) \right]}{2(d+2)T}$ gives an explicit expression for the potential reduction in the variance contribution when ML-labeled data are used.
Under the condition that $\E\left[\|\mathbf{d}_t^0\|^2\right]$ is nonzero for at least one $t \in [T]$, the terms in the final parenthesis of \oureqref{eq: zo_theorem_bound} can be written as $ \sigma^2 +\frac{ d (d+4) }{4 (d+3)^2 T} + \frac{\E\left[  \Psi_{1:T} \left( \lambda^* \right) \right]}{2(d+2)T} =  \sigma^2 +\frac{ d (d+4) }{4 (d+3)^2 T} - \frac{1}{2(d+2)}\frac{\left( \frac{1}{T} \sum_{t=1}^T \E \left[ \left\langle \mathbf{q}_t , \mathbf{d}_t^0   \right\rangle \right] \right)^2 }{ \frac{1}{T}\sum_{t=1}^T \E \left[ \left\| \mathbf{d}_t^0 \right\|^2 \right]}$. Since the final term is nonnegative and is subtracted, this quantity is no greater than $\sigma^2+\frac{d(d+4)}{4(d+3)^2T}$, which is the corresponding contribution from stochastic-gradient variance and the finite-difference remainder in the standard zeroth-order analysis. Hence, our results suggest that a similar benefit also arises in zeroth-order optimization.

We defer the analysis of Algorithm~\ref{alg:sgfm} with vector tuning to Appendix~\ref{app:C}. 

\section{Experiments} 
In this section, we evaluate the proposed first- and zeroth-order methods through a series of experiments on image- and text-classification tasks. We present a concise summary of the experimental results here, with further details provided in  Appendix~\ref{app:D}, where we also include additional experiments with linear models on regression and multiclass classification tasks
in Appendix~\ref{app:D3} and Appendix~\ref{app:D4} respectively.

\subsection{Image Classification}

\paragraph{Dataset and setup.}
We evaluate our first-order methods on CIFAR-10 \citep{krizhevsky2009learning} by training a linear classifier on top of a frozen ImageNet-pretrained ResNet-50 \citep{he2016deep,deng2009imagenet}. Specifically, we remove the final classification layer of ResNet-50 and use its $2{,}048$-dimensional penultimate-layer representations as fixed input features for the linear classifier. We partition the original CIFAR-10 training set into $5{,}000$ examples for training the ML-labeler $f(\cdot)$, $4{,}500$ labeled examples for training the classifier, $5{,}000$ validation examples, and $35{,}500$ examples that will be served as unlabeled samples and will be labeled by the ML-labeler, with the official $10{,}000$-example test split used for evaluation. 
The ML-labeler is a separate linear classifier trained on the disjoint training set, obtaining a test accuracy of $86.47\%$. The ML-labeler is then kept fixed and used to generate the ML labels for both the labeled and unlabeled examples. 

We evaluate four training methods for $100$ epochs: standard SGD using only true-labeled examples, PP-SSL with Online Tuning \citep{shoham2025prediction}, and Algorithm~\ref{alg:sgd} with scalar and vector tuning via FTL. For brevity, we refer to the latter two methods as Scalar FTL and Vector FTL when presenting the experimental results below.
We consider several $(n,N)$ configurations, with $n$ and $N$ denoting the labeled and unlabeled batch sizes, respectively. We use a constant learning rate throughout training, which we tune independently for each method and batch-size configuration.

\paragraph{Results.}
Figure~\ref{fig:cifar10_fo_acc_main} presents test-accuracy over training epochs for three batch-size configurations. Results for other batch-size configurations 
are available in Appendix~\ref{app:D.2.1}. 
We find that across all configurations that we try, Scalar FTL provides the most consistent improvement, achieving the highest mean test accuracy in all configurations. PP-SSL also generally improves over the labeled-only SGD baseline, although by a smaller margin, compared to our method (see also Table~\ref{tab:cifar10_batch_results} in Appendix~\ref{app:D.2.1}). We observe that the behavior of Vector FTL is less consistent across batch-size configurations: it performs comparably to Scalar FTL for some choices of $(n,N)$, but exhibits a noticeable degradation for others (Table~\ref{tab:cifar10_batch_results}). The tuning trajectories help explain these differences (Figure~\ref{fig:cifar10_fo_lambda} in Appendix~\ref{app:D.2.1}). The Scalar FTL tuning parameter $\lambda$ decreases smoothly throughout training, while the PP-SSL tuning parameter rapidly approaches zero, thereby substantially reducing the contribution of the ML-labeled correction after the early stages of training.
We also find that although the mean of the Vector FTL parameters, i.e., $\mathrm{mean}(\boldsymbol{\lambda}_t)$, follows a similarly smooth decreasing trend as in scalar tuning, its coordinatewise parameters exhibit substantial dispersion, with some coordinates receiving large positive weights and others negative weights throughout training. 
Additional results for the zeroth-order methods for CIFAR-10 are provided in Appendix~\ref{app:D.2.2}.

\begin{figure}[h!]
    \centering
    \vspace{-1mm}
    \includegraphics[width=1.0\linewidth]{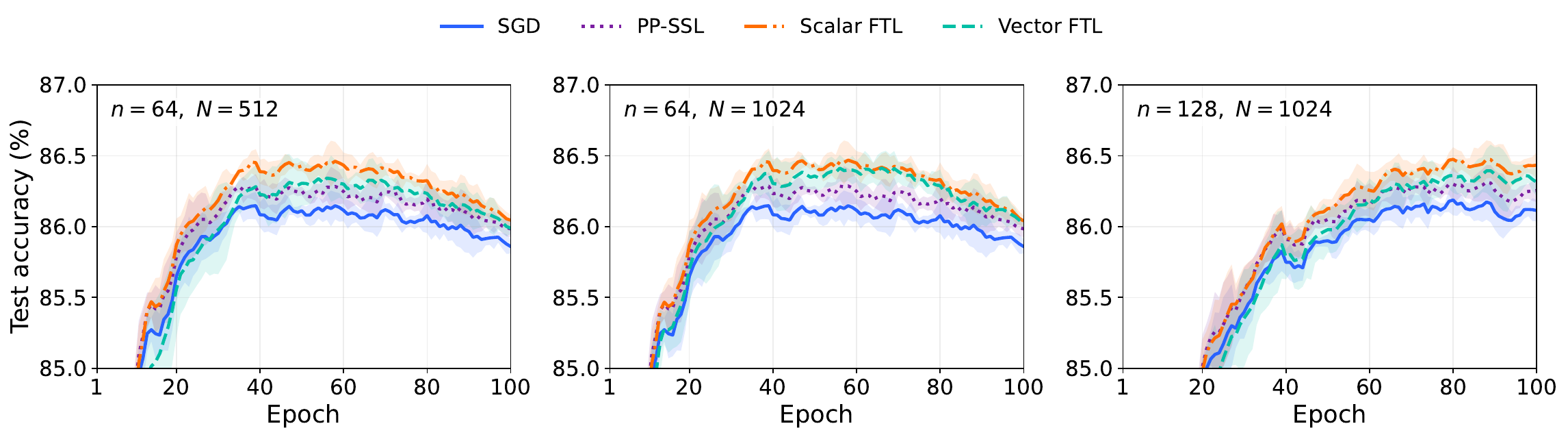}
    \vspace{-7mm}
    \caption{
    Test accuracy (\%) over training epochs for the first-order methods on CIFAR-10 with $n=64, N=512$ (left), $n=64, N=1024$ (middle), and $n=128, N=1024$ (right). Curves show the mean across five random seeds, with shaded regions indicating $\pm 1$ standard deviation. For visualization, test-accuracy curves are smoothed using a five-epoch moving average.
    }
    \label{fig:cifar10_fo_acc_main}
    \vspace{-1mm}
\end{figure}

\subsection{Text Classification}

\paragraph{Dataset and setup.}
We evaluate our zeroth-order methods on text classification using RoBERTa-large (355M parameters) \citep{liu2019robertarobustlyoptimizedbert}. Following \citet{malladi2023fine, zhang2026zeroth}, we consider few-shot fine-tuning with $K=32$ labeled examples per class on SST-2 \citep{socher2013recursive}, a binary sentiment classification task, and SST-5 \citep{socher2013recursive}, a five-class fine-grained sentiment classification task. This gives $64$ labeled training examples for SST-2 and $160$ for SST-5, with equally sized validation sets. For each task, we fine-tune a separate RoBERTa-large ML-labeler $f(\cdot)$ on $128$ labeled examples per class, obtaining test accuracies of $92.78\%$ on SST-2 and $58.30\%$ on SST-5. The ML-labeler are then kept fixed and used to generate ML-labels. The remaining training examples are treated as unlabeled data. 
We fine-tune the underlying model for $10{,}000$ optimization steps and compare three methods: standard stochastic gradient-free method without using ML-labeled data and our proposed zeroth-order methods (Algorithm~\ref{alg:sgfm}) with scalar and vector tuning, which we denote as ``ZO Scalar FTL" and ``ZO Vector FTL" respectively in the following.
We tune a constant learning rate for each method, and we use a labeled batch size of $n=32$ and an unlabeled batch size of $N=64$ for our methods.

\paragraph{Results.}
Figure~\ref{fig:sst_zo_acc_main} illustrates test-accuracy over training steps for Algorithm~\ref{alg:sgfm} with scalar and vector tuning on SST-2 and SST-5. On both SST-2 and SST-5, the proposed methods consistently improve over the ZO baseline after the early stages of training, as the tuning parameters become more stable after aggregating information over more optimization steps. In particular, we find that Algorithm~\ref{alg:sgfm} with scalar tuning performs similarly to its vector-tuning counterpart throughout most of training. The corresponding trajectories of the scalar tuning parameter and the mean of the vector tuning parameters are smooth and closely aligned. At the same time, the coordinatewise vector tuning parameters $\boldsymbol{\lambda}_t$ span a wide range, indicating that different coordinates continue to place different weights on the ML-labeled correction even when the mean of the vector tuning parameters evolves similarly to the scalar one (Figure~\ref{fig:sst_zo_lam} in Appendix~\ref{app:D.1.2}). Additional results on fine-tuning with the first-order methods are provided in Appendix~\ref{app:D.1.1}.

\begin{figure}[h]
    \centering
    \vspace{-2mm}
    \includegraphics[width=0.9\linewidth]{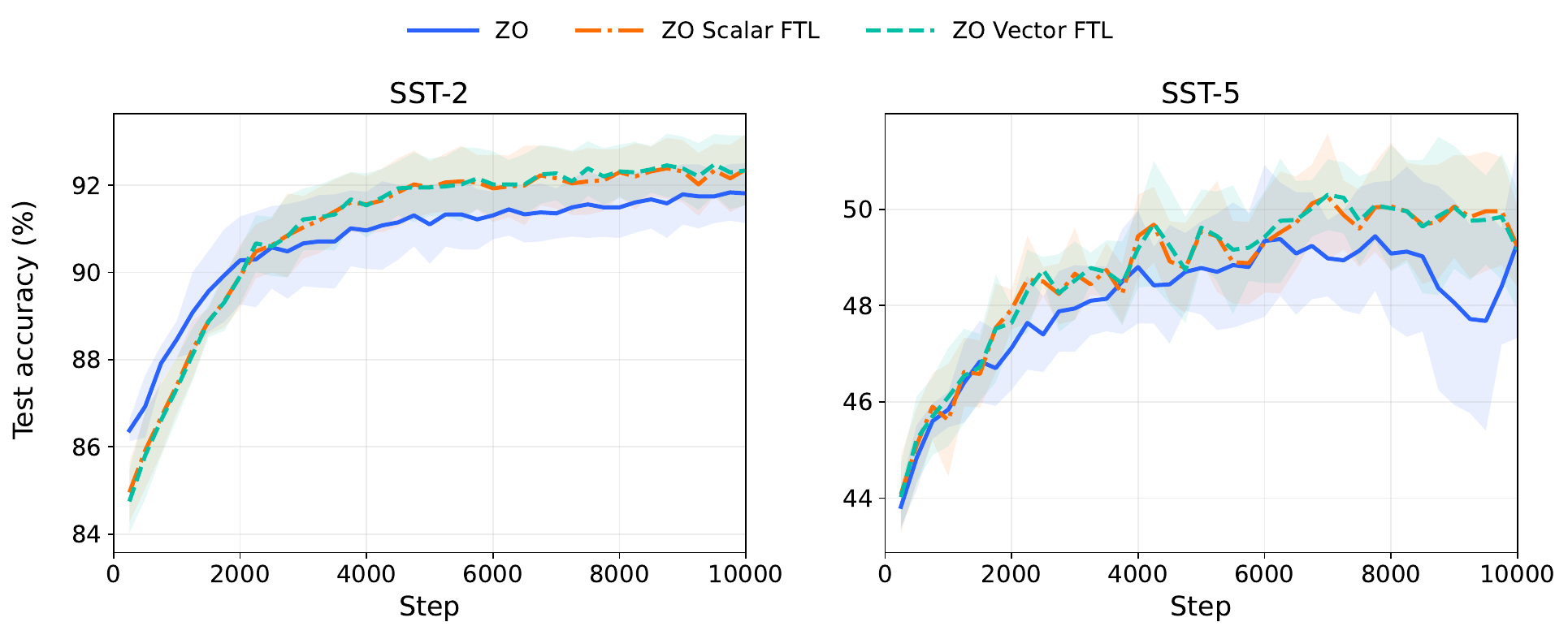}
    \vspace{-4mm}
    \caption{Test accuracy (\%) over training steps for the zeroth-order methods on SST-2 (left) and SST-5 (right). Test accuracy is evaluated every $250$ optimization steps. Solid curves denote the mean across five random seeds, and shaded regions indicate $\pm 1$ standard deviation.}
    \label{fig:sst_zo_acc_main}
\end{figure}

\section{Discussion and Future Directions}

In this work, we propose stochastic first- and zeroth-order optimization methods that use true-labeled and ML-labeled data to accelerate optimization. We consider both online scalar and vector tuning schemes that control the contribution of the ML-labeled gradient correction while preserving the convergence rate of standard stochastic optimization and enabling variance reduction.
Our results reveal a trade-off between the proposed tuning schemes. Although both may achieve similar convergence rates, vector tuning offers greater coordinatewise flexibility but exhibits greater sensitivity to batch sizes in our experiments, whereas scalar tuning tends to be more stable across configurations. A potential direction for future work is therefore to identify classes of objective functions and data conditions under which vector tuning provably outperforms scalar tuning, as well as settings in which its additional computational overhead outweighs the gains. Another related direction is to investigate intermediate levels of tuning granularity, such as layer-wise or tensor-wise tuning. Such schemes could better account for heterogeneity in gradient estimates across different blocks of model parameters while retaining some of the stability benefits of sharing a tuning parameter across multiple coordinates, which we plan to explore in future work.
Finally, our work assumes that ML labels are produced by a fixed ML model and accommodates settings in which only its labels, rather than the ML-labeler itself, are accessible. An interesting direction is to consider settings in which the ML-labeler is accessible and can also be updated on the fly.

\section*{Acknowledgements}
The authors appreciate the support from NSF CCF-2403392, as well as the Google Gemma Academic Program and Google Cloud Credits.

\bibliographystyle{plainnat}
% \bibliography{refs}

\newpage

\appendix

\section{More Related Works} \label{app:A}

\paragraph{Semi-Supervised Learning.}
Our setting is closely related to Semi-Supervised Learning (SSL), which extends classical supervised learning by leveraging abundant unlabeled data alongside a limited amount of labeled data, thus offering a practical approach to improving learning performance while reducing the reliance on costly data annotation \citep{zhu2005semisupervised, chapelle2006semisupervised, vanEngelen2020survey}. Modern SSL methods include self-training approaches \citep{xie2020self, wei2021crest, chen2022debiased}, which iteratively use model predictions on unlabeled data to construct additional supervision and retrain the model, and more broadly pseudo-labeling techniques \citep{lee2013pseudo, oh2022daso, wang2022debiased, zhang2025revisiting}, which use model-generated predictions as supervisory targets for unlabeled examples. Some popular approaches combine label guessing or pseudo-labeling with consistency regularization and data augmentation, including MixMatch \citep{berthelot2019mixmatch} and ReMixMatch \citep{berthelot2020remixmatch}. A more recent line of work on pseudo-labeling methods controls the contribution of unlabeled examples based on model confidence. In particular, FixMatch \citep{sohn2020fixmatch} uses a fixed confidence threshold, while subsequent methods such as FlexMatch \citep{zhang2021flexmatch} and FreeMatch \citep{wang2023freematch} adapt the threshold during training, and SoftMatch \citep{chen2023softmatch} replaces hard thresholding with confidence-based soft weighting. Other works employ alternative selection criteria: Dash \citep{xu2021DashSL} dynamically adjusts a loss-based selection threshold, instead of a confidence threshold, while InstanT \citep{li2023instant} introduces instance-dependent thresholds for pseudo-label selection. In contrast to the SSL approaches, our work utilizes ML-labels not as direct supervisory targets in the training objective, but as auxiliary information for constructing a stochastic gradient estimator whose contribution is tuned online to reduce variance.

\paragraph{Knowledge Distillation.}
Our work is also connected to Knowledge Distillation (KD), which focuses on  transferring knowledge from a teacher model to a student model, typically by matching the teacher’s output logits \citep{bucila2006modelcompression, ba2014dodeepnets} or softened predictive distribution \citep{ hinton2015distillingknowledgeneuralnetwork}. Unlike SSL, whose goal is to utilize unlabeled data to improve learning, KD is primarily concerned with teacher-to-student knowledge transfer and does not inherently require unlabeled data. Subsequent work has extended KD beyond output-logit matching by transferring intermediate feature representations \citep{tian2020contrastive, deng2021comprehensive, huang2023efficient}, and relational or structural information \citep{tian2020contrastive, derijk2022structural}. Other works explore alternative distillation objectives based on contrastive, correlation-based, Wasserstein, and ranking losses \citep{huang2022stronger, lv2024wasserstein, bassam2025pld}.
Beyond the design of distillation objectives and transferred representations, recent work has examined why and when distillation is effective, including optimization difficulties that can prevent students from matching the teacher’s predictive distribution \citep{stanton2021does}, bias-variance tradeoffs arising in self-distillation under label noise \citep{das2023understanding}, the transfer of teacher-specific properties to the student \citep{ojha2023what}, and the potential benefits of allowing students to deviate from their teachers \citep{nagarajan2023on}. More recently, KD has also been studied in modern generative-model settings, where distillation can induce a precision-recall tradeoff in the student's learned distribution \citep{cha2025why}. Most closely related to our work, \citet{safaryan2023knowledge} study KD through the lens of stochastic variance reduction, showing that it can partially reduce gradient noise in linear and deep linear models. Our setting differs from KD in that we do not optimize a teacher-matching objective, but rather the ML-labels are used to construct an unbiased gradient estimator for the original population objective.

\paragraph{Optimization with Surrogate Data.}
Another closely related line of research considers learning with surrogate data, where auxiliary data sources or proxy objectives are used to improve optimization or learning on a target task. \citet{Woodworth2023TwoLosses} consider optimization with an expensive target loss and a cheaper proxy loss, and develop an algorithm that exploits curvature information from the proxy loss to reduce the number of stochastic gradient evaluations of the target objective. \citet{jain2024scaling} study learning from mixtures of real and surrogate data and characterize how surrogate samples should be weighted to improve performance on the target distribution. Related work on auxiliary learning similarly investigates how additional tasks or datasets can be constructed, selected, or weighted to improve learning on a primary task
\citep{shi2020auxiliary, albalak2023improving, liu2019self}.

\newpage

\section{Proofs of the Theoretical Results in Section~\ref{sec:first_order_stochastic_optimization}} \label{app:B}

Before proving the theorems, we need a series of lemmas.

We denote $\delta_t:= \Gtvec{\bw_t} - \nabla \loss(\bw_t)$ in the following and use $\E_{t-1}[ \cdot ] := \E[ \cdot \mid \mathcal{F}_{t-1} ]$ for brevity.

\begin{mdframed}[backgroundcolor=black!10,rightline=false,leftline=false,topline=false,bottomline=false]
\begin{lemma}(Adapted from (2.9) in \citet{ghadimi2013stochastic})
Consider $\bw_{t+1} = \bw_t - \eta_t \, \Gtvec{\bw_t}$.
Under Assumption~\ref{A1} and Assumption~\ref{A2}, it holds that  
\begin{equation} \label{eq:2.9}
\sum_{t=1}^T \left( \eta_t  - \frac{\beta}{2} \eta_t^2 \right)  \| \nabla  \loss(\bw_t) \|^2 \leq \Delta_1
- \sum_{t=1}^T \left( \eta_t - \beta \eta_t^2  \right) \langle \nabla \loss(\bw_t), \delta_t \rangle 
+ \frac{\beta}{2} \sum_{t=1}^T \eta_t^2 \| \delta_t \|^2.
\end{equation}
\end{lemma}
\end{mdframed}

\begin{proof}
For reader's convenience, we replicate the analysis of \citet{ghadimi2013stochastic} below. 

By $\beta$-smoothness,
\begin{align*}
\loss(\bw_{t+1}) &\leq \loss(\bw_t) + \left\langle \nabla \loss(\bw_t), \bw_{t+1}-\bw_t \right\rangle + \frac{\beta}{2}\|\bw_{t+1}-\bw_t\|^2 \\
    &=
    \loss(\bw_t)
    - \eta_t
      \left\langle
        \nabla \loss(\bw_t),
        \nabla \loss(\bw_t)+\delta_t
      \right\rangle
    + \frac{\beta}{2}\eta_t^2
      \|\nabla \loss(\bw_t)+\delta_t\|^2 \\
    &=
    \loss(\bw_t)
    - \eta_t\|\nabla \loss(\bw_t)\|^2
    - \eta_t\langle \nabla \loss(\bw_t),\delta_t\rangle + \frac{\beta}{2}\eta_t^2
      \left(
        \|\nabla \loss(\bw_t)\|^2
        + 2\langle \nabla \loss(\bw_t),\delta_t\rangle
        + \|\delta_t\|^2
      \right) \\
    &=
    \loss(\bw_t)
    - \left(\eta_t-\frac{\beta}{2}\eta_t^2\right)
      \|\nabla \loss(\bw_t)\|^2
    - \left(\eta_t-\beta\eta_t^2\right)
      \langle \nabla \loss(\bw_t),\delta_t\rangle
    + \frac{\beta}{2}\eta_t^2\|\delta_t\|^2.
\end{align*}
Rearranging and summing over $t=1,\ldots,T$, the loss
differences telescope, which gives,
\begin{align*}
    \sum_{t=1}^T
      \left(\eta_t-\frac{\beta}{2}\eta_t^2\right)
      \|\nabla \loss(\bw_t)\|^2 & \leq
    \loss(\bw_1)-\loss(\bw_{T+1})
    - \sum_{t=1}^T
      \left(\eta_t-\beta\eta_t^2\right)
      \langle \nabla \loss(\bw_t),\delta_t\rangle
    + \frac{\beta}{2}
      \sum_{t=1}^T \eta_t^2\|\delta_t\|^2
      \\ & 
      \leq
    \Delta_1
    - \sum_{t=1}^T
      \left(\eta_t-\beta\eta_t^2\right)
      \langle \nabla \loss(\bw_t),\delta_t\rangle
    + \frac{\beta}{2}
      \sum_{t=1}^T \eta_t^2\|\delta_t\|^2.
\end{align*}
%where the last inequality follows from $\loss(\bw_{T+1})\geq \min_{\bw\in\mathbb{R}^d}\loss(\bw)$.
\end{proof}

\begin{mdframed}[backgroundcolor=black!10,rightline=false,leftline=false,topline=false,bottomline=false]
\begin{lemma} \label{lem:variance_delta}
Assume that $\boldsymbol{\lambda}_t$ is $\mathcal{F}_{t-1}$-measurable and hence $\E_{t-1}[ \Gtvec{\bw_t}  ] = \nabla \loss(\bw_t)$.
Then, $$\mathbb E_{t-1}\left[  \left \| \delta_t \right \|^2 \right] = 
\mathbb E_{t-1}\left[ \left \| \Gtvec{\bw_t}  \right \|^2 \right]  -  \| \nabla  \loss(\bw_t) \|^2.$$ 
\end{lemma}
\end{mdframed}

\begin{proof}

\begin{align}
\mathbb E_{t-1}\left[ \| \delta_t \|^2 \right] & = 
\mathbb E_{t-1}\left[\left  \| \Gtvec{\bw_t} \right \|^2  \right]
- \mathbb E_{t-1} \left[ 2 \langle \Gtvec{\bw_t} , \nabla \loss(\bw_t) \rangle  \right]
+  \mathbb E_{t-1} \left[ \| \nabla  \loss(\bw_t) \|^2  \right]
\\ & = 
\mathbb E_{t-1}\left[\left  \| \Gtvec{\bw_t} \right \|^2  \right] - \| \nabla  \loss(\bw_t) \|^2,
\end{align}
where the last equality uses that $\E_{t-1}[ \Gtvec{\bw_t}  ] = \nabla \loss(\bw_t)$.

\end{proof}

\begin{mdframed}[backgroundcolor=black!10,rightline=false,leftline=false,topline=false,bottomline=false]
\begin{lemma} \label{lem:conditional}
Under Assumption~\ref{A1} and Assumption~\ref{A2}.
Suppose that $\boldsymbol{\lambda}_t$ is $\mathcal{F}_{t-1}$-measurable and hence $\E_{t-1}[ \Gtvec{\bw_t}  ] = \nabla L(\bw_t)$.
Then,
$$
\sum_{t=1}^T \left( \eta_t  - \frac{\beta}{2} \eta_t^2 \right) \E\left[ \| \nabla  \loss(\bw_t) \|^2 \right]
\leq 
 \Delta_1
+ \frac{\beta}{2} \sum_{t=1}^T \eta_t^2  \, 
\E\left[ \mathbb E_{t-1}\left[ \left \| \Gtvec{\bw_t}  \right \|^2 \right]  -  \| \nabla  \loss(\bw_t) \|^2 \right].$$
\end{lemma}
\end{mdframed}

\begin{proof}
Taking expectation on both sides of \oureqref{eq:2.9}, noting that 
$\mathbb \E[ \langle \nabla \loss(\bw_t), \delta_t \rangle ] = 0$, we have
\begin{align}
\sum_{t=1}^T \left( \eta_t  - \frac{\beta}{2} \eta_t^2 \right) \E\left[ \| \nabla  \loss(\bw_t) \|^2 \right] & \leq \Delta_1
+ \frac{\beta}{2} \sum_{t=1}^T \eta_t^2  \,  \E\left[ \| \delta_t \|^2 \right] \notag
\\ &
=  \Delta_1
+ \frac{\beta}{2} \sum_{t=1}^T \eta_t^2  \, 
\E\left[  \E_{t-1}\left[ \left \| \delta_t  \right \|^2 \right]  \right] \notag
\\ &
=  \Delta_1
+ \frac{\beta}{2} \sum_{t=1}^T \eta_t^2  \, 
\E\left[ \mathbb E_{t-1}\left[ \left \| \Gtvec{\bw_t}  \right \|^2 \right]  -  \| \nabla  \loss(\bw_t) \|^2 \right], \label{eq:17}
\end{align}
where \oureqref{eq:17} uses Lemma~\ref{lem:variance_delta}.
\end{proof}

Now we recall the notion of vector tuning as a game introduced in Subsection~\ref{sub:our_methods},
where the sequence of loss functions are 
$\phi_t(\boldsymbol{\lambda}):=  \left \| \mathcal{G}_{t,\bw_t}^{\boldsymbol{\lambda}}  \right \|^2= \left \|  \mathbf{g}_t + \boldsymbol{\lambda} \odot \mathbf{d}_t  \right \|^2 $.
For this, we define
$$ \mathrm{Regret}_T^{\phi}(\boldsymbol{\lambda}^*):= \sum_{t=1}^T \phi_t(\boldsymbol{\lambda}_t) - \sum_{t=1}^T \phi_t(\boldsymbol{\lambda}^*),$$
where $\boldsymbol{\lambda}_t \in \reals^d$ is the action at $t$ committed by the online learner and $\boldsymbol{\lambda}^* \in \reals^d$ is a comparator. 

We also recall that the case of vector tuning actually covers the case of scalar tuning as a special case. Specifically, given a scalar tuning parameter $\lambda \in \reals$, a scalar-vector multiplication $\lambda \bv$ can be rewritten as $\lambda \bv = \lambda \mathbf{1} \odot \bv$, where $\mathbf{1} \in \reals^d$ is the all-ones vector. This allows us to invoke Lemma~\ref{lem:conditional}
and Lemma~\ref{lem:reg0} below
 by letting 
$\boldsymbol{\lambda}_t \gets \lambda_t \mathbf{1}$
and 
$\boldsymbol{\lambda}^* \gets \lambda^* \mathbf{1}$
in the lemmas when analyzing the case of scalar tuning.

\begin{mdframed}[backgroundcolor=black!10,rightline=false,leftline=false,topline=false,bottomline=false]
\begin{lemma} \label{lem:reg0}
Denote 
$\Phi_{1:T}(\boldsymbol{\lambda^*}):= \sum_{t=1}^T \left \| \g_{t,\bw_t}^{\boldsymbol{\lambda^*}} \right \|^2 -
 \sum_{t=1}^T \left \| \g_{t,\bw_t}^{0} \right \|^2$.
Set $\eta_t = \eta$, under Assumption~\ref{A1},~\ref{A2},~and~\ref{A3}, we have
\begin{equation} \label{eq:reg}
\sum_{t=1}^T \left( \eta  - \frac{\beta}{2} \eta^2 \right) \E\left[ \| \nabla  \loss(\bw_t) \|^2 \right] 
\leq
 \Delta_1
+ \frac{\beta \eta^2}{2}   \, 
\E\left[ \mathrm{Regret}_T^{\phi}(\boldsymbol{\lambda^*}) \right] + \frac{\beta \eta^2}{2}  \E[ \Phi_{1:T}(\boldsymbol{\lambda^*}) ] +
\frac{\beta \eta^2 T}{2} 
 \sigma^2.
\end{equation}
\end{lemma}
\end{mdframed}

\begin{proof}
From Lemma~\ref{lem:conditional}, 
\begin{align}
\sum_{t=1}^T \left( \eta_t  - \frac{\beta}{2} \eta_t^2 \right) & \E\left[ \| \nabla  \loss(\bw_t) \|^2 \right] 
\leq  \Delta_1
+ \frac{\beta}{2} \sum_{t=1}^T \eta_t^2  \, 
\E\left[ \mathbb E_{t-1}\left[ \left \| \mathcal{G}_{t,\bw_t}^{\boldsymbol{\lambda}_t}  \right \|^2 \right]  -  \| \nabla \loss(\bw_t) \|^2 \right]
\\& = \Delta_1
+ \frac{\beta}{2} \sum_{t=1}^T \eta_t^2  \, 
\E\left[ \mathbb E_{t-1}\left[ 
\phi_t(\boldsymbol{\lambda}_t) - \phi_t(\boldsymbol{\lambda^*})
+ \phi_t(\boldsymbol{\lambda^*}) - \phi_t(0) + \phi_t(0) \right]
-
 \| \nabla \loss(\bw_t) \|^2 \right]
\\& 
\leq \Delta_1
+ \frac{\beta}{2} \sum_{t=1}^T \eta_t^2  \, 
\E\left[ \mathbb E_{t-1}\left[ 
\phi_t(\boldsymbol{\lambda}_t) - \phi_t(\boldsymbol{\lambda^*})
+ \phi_t(\boldsymbol{\lambda^*}) - \phi_t(0)\right] \right]+ \frac{\beta}{2} \sum_{t=1}^T \eta_t^2 \sigma^2,
\end{align}
where the last inequality is due to Assumption~\ref{A3}, 
i.e., 
$
\mathbb{E} \left[ \left\| \mathbf{g}_t - \nabla \loss(\bw_t) \right\|^2 \right] \leq \sigma^2.
$

Now set $\eta_t = \eta$. We further obtain
\begin{align}
\sum_{t=1}^T \left( \eta  - \frac{\beta}{2} \eta^2 \right) \E\left[ \| \nabla \loss(\bw_t) \|^2 \right] 
\leq
 \Delta_1
+ \frac{\beta}{2}  \eta^2  \, 
\E\left[ \mathrm{Regret}_T^{\phi}(\boldsymbol{\lambda^*}) \right] + \frac{\beta}{2} \eta^2  \E[ \Phi_{1:T}(\boldsymbol{\lambda^*}) ] +
\frac{\beta T}{2} \eta^2
 \sigma^2.
\end{align}
\end{proof}

%\begin{mdframed}
As described in the main text, we propose using the Follow-the-Leader (FTL) strategy in the online tuning as a game, i.e., initialize $\lambda_1 \gets 0$, and for all $t>1$, play
\begin{align}
    \lambda_t \gets \underset{ \lambda \in \reals }{\arg\min}  \sum_{s=1}^{t-1}  \phi_s(\lambda), \text{ where } \phi_s(\lambda):= \left \|  \mathbf{g}_s + \lambda \mathbf{d}_s  \right \|^2. \label{eq: FTL}
\end{align}
%\end{mdframed}

\begin{mdframed}[backgroundcolor=black!10,rightline=false,leftline=false,topline=false,bottomline=false]
\begin{lemma}(Lemma 2.1 in \cite{ShalevShwartz2011Online}; Be-the-Leader lemma) \label{lem: shalev2025online}
Consider a learner that receives a sequence of loss functions $h_1(\cdot), h_2(\cdot), \dots, h_T(\cdot)$ in an online learning setting.
Then, Follow-the-Leader (FTL), i.e., $\theta_t \in \arg\min_{\theta \in \Theta} \sum_{s=1}^{t-1} h_s(\theta)$, satisfies
$$\mathrm{Regret}_T^{h}(\theta^*):= \sum_{t=1}^{T} h_t(\theta_t) - \sum_{t=1}^T h_t(\theta^*)
\leq \sum_{t=1}^{T} \left( h_t(\theta_t) - h_t(\theta_{t+1}) \right),$$
where $\theta^* \in \Theta$ is a comparator in the same decision space $\Theta$
as FTL.
\end{lemma}
\end{mdframed}

We observe that expanding the cumulative loss gives
\begin{align*}
    \sum_{s=1}^{t-1}  \phi_s(\lambda) 
    = \sum_{s=1}^{t-1} \left\| \mathbf{g}_s + \lambda \mathbf{d}_s  \right\|^2 
    = \sum_{s=1}^{t-1} \left\| \mathbf{g}_s \right\|^2 + 2 \lambda \sum_{s=1}^{t-1} \langle \mathbf{g}_s , \mathbf{d}_s \rangle + \lambda^2 \sum_{s=1}^{t-1} \left\| \mathbf{d}_s \right\|^2 .
\end{align*}
For $\sum_{s=1}^{t-1} \left\| \mathbf{d}_s \right\|^2>0$, the cumulative loss function is quadratic in $\lambda$. By setting $\frac{\partial}{\partial \lambda} \sum_{s=1}^{t-1}  \phi_s(\lambda) = 0$, we obtain
$    \lambda_t = - \frac{\sum_{s=1}^{t-1} \langle \mathbf{g}_s , \mathbf{d}_s \rangle}{\sum_{s=1}^{t-1} \left\| \mathbf{d}_s \right\|^2}.$
For $\sum_{s=1}^{t-1} \left\| \mathbf{d}_s \right\|^2=0$, we have $\mathbf{d}_s = 0$, for all $s<t$. Therefore, we also have $\sum_{s=1}^{t-1} \langle \mathbf{g}_s , \mathbf{d}_s \rangle=0$, and the cumulative loss function does not depend on $\lambda$. In this case, we use the convention $\lambda_t = 0$. To summarize, FTL \oureqref{eq: FTL} gives
    \begin{align*}
        \lambda_t = 
        \begin{cases}
            - \frac{\sum_{s=1}^{t-1} \left\langle \mathbf{g}_s , \mathbf{d}_s   \right\rangle}{\sum_{s=1}^{t-1}  \left\| \mathbf{d}_s \right\|^2} , & \sum_{s=1}^{t-1}  \left\| \mathbf{d}_s \right\|^2 >0, \\
            0,  & \sum_{s=1}^{t-1}  \left\| \mathbf{d}_s \right\|^2=0,
        \end{cases}
    \end{align*}
    which is \oureqref{scalar-tuning-first-order}.

\begin{mdframed}[backgroundcolor=black!10,rightline=false,leftline=false,topline=false,bottomline=false]
\begin{lemma} \label{lem: seq}
    Let $q_1 , \ldots, q_T$ be an arbitrary sequence of nonnegative numbers, and define
    \begin{align*}
        Q := \max_{1 \leq t \leq T} q_t \quad , \quad r_t := 
        \begin{cases}
            \frac{q_t^2}{\sum_{s=1}^t q_s} , &\; \sum_{s=1}^t q_s>0, \\
            0,  & \; \sum_{s=1}^t q_s=0.
        \end{cases}
    \end{align*}
    Then, for $Q > 0$,
    \begin{align*}
        \sum_{t=1}^T r_t \leq Q \left( 1 + \log T \right).
    \end{align*}
    For $Q = 0$, we have $\sum_{t=1}^T r_t = 0$.
\end{lemma}
\end{mdframed}

\begin{proof}
    If $Q = 0$, then $q_t = 0$, for all $t$, and the result is immediate. Suppose that $Q>0$. Define the function
    \begin{align*}
        \Gamma(v) := 
        \begin{cases}
            v, & 0 \leq v \leq Q ,\\
            Q \left(1 + \log \left( v/Q \right) \right), & v \geq Q .
        \end{cases}
    \end{align*}
    The function $\Gamma( \cdot )$ is concave and continuously differentiable with derivative
    \begin{align*}
        \Gamma'(v) := 
        \begin{cases}
            1, & 0 \leq v \leq Q ,\\
            Q/v, & v \geq Q .
        \end{cases}
    \end{align*}
    The derivative is nonincreasing. Therefore, for every $v \geq 0$ and every $q \in [0, Q]$ we have
    \begin{align*}
        \Gamma \left( v + q \right) - \Gamma \left( v \right) = \int_v^{v+q} \Gamma' \left( u \right) du \geq q \Gamma' \left( v + q \right).
    \end{align*}
    If $0<v+q \leq Q$, then
    \begin{align*}
        q \Gamma' \left( v+q \right) = q \geq \frac{q^2}{v+q}.
    \end{align*}
    If $v+q > Q$, then
    \begin{align*}
        q \Gamma' \left( v+q \right) = \frac{ q Q}{v + q} \geq \frac{q^2}{v+q}.
    \end{align*}
    Consequently, in both cases we have
    \begin{align*}
        \Gamma \left( v + q \right) - \Gamma \left( v \right) \geq \frac{q^2}{v+q}.
    \end{align*}
    Take $v = \sum_{s=1}^{t-1} q_s$ and $q=q_t$, so that $v+q = \sum_{s=1}^t q_s$. If $\sum_{s=1}^t q_s >0$, the above inequalities yield
    \begin{align*}
        \frac{q_t^2}{\sum_{s=1}^t q_s} \leq \Gamma \left( \sum_{s=1}^t q_s \right) - \Gamma \left( \sum_{s=1}^{t-1} q_s \right).
    \end{align*}
    If $\sum_{s=1}^t q_s = 0$, then since $\sum_{s=1}^t q_s = \sum_{s=1}^{t-1} q_s + q_t$, we also have $\sum_{s=1}^{t-1} q_s = 0$ and $q_t = 0$. Hence, by definition
    \begin{align*}
        0 = \Gamma\left(\sum_{s=1}^t q_s \right) - \Gamma \left(\sum_{s=1}^{t-1} q_s \right).
    \end{align*}
    Thus, for every $t = 1, \ldots, T$ we have
    \begin{align*}
        r_t \leq \Gamma\left(\sum_{s=1}^t q_s \right) - \Gamma \left(\sum_{s=1}^{t-1} q_s \right).
    \end{align*}
    Summing over $t$ we obtain
    \begin{align*}
        \sum_{t=1}^T r_t \leq \Gamma\left(\sum_{s=1}^T q_s \right) - \Gamma \left( 0 \right) = \Gamma\left(\sum_{s=1}^T q_s \right).
    \end{align*}
    Since $\sum_{s=1}^T q_s \geq Q$, by definition of $\Gamma(\cdot)$ we have
    \begin{align*}
        \Gamma \left( \sum_{s=1}^T q_s \right) = Q \left(1 + \log \frac{\sum_{s=1}^T q_s}{Q} \right).
    \end{align*}
    Finally, since
    \begin{align*}
        \sum_{s=1}^T q_s \leq TQ,
    \end{align*}
    we have
    \begin{align*}
        \log \frac{\sum_{s=1}^T q_s}{Q} \leq \log T.
    \end{align*}
\end{proof}

\begin{mdframed}[backgroundcolor=black!10,rightline=false,leftline=false,topline=false,bottomline=false]
\noindent
\textbf{Theorem~\ref{th: first-order-scalar}.}
%\begin{theorem} \label{th: first-order-scalar}
\textit{
Denote $C_0 := \sup_{T \geq 1} \mathbb E[\max_{1\le t\le T+1}|\lambda_t|^2]$.
Then, under Assumptions~\ref{A1}--\ref{A4}, Algorithm~\ref{alg:sgd} 
with scalar tuning and $\eta_t = \frac{1}{\beta \sqrt{T}}$ satisfies
\begin{align} 
     \frac{1}{T}\sum_{t=1}^T \E\left[ \| \nabla  \loss(\bw_t) \|^2 \right] 
        \leq
         \frac{2 \beta \Delta_1}{\sqrt{T}} 
        + \frac{8 G^2(1+2C_0)(3+\log T)}{T^{3/2}}  
        + \frac{\E[ \Phi_{1:T}(\lambda^*) ]}{T^{3/2}} +  \frac{\sigma^2}{\sqrt{T}} ,
\end{align}
where 
$    \E \left[ \Phi_{1:T}(\lambda^*) \right]
    =
    -
    \frac{\left( \sum_{t=1}^T \E \left[ \left\langle \mathbf{g}_t , \mathbf{d}_t   \right\rangle \right] \right)^2 }{\sum_{t=1}^T \E \left[ \left\| \mathbf{d}_t \right\|^2 \right]}, \mathrm{ if } \, \sum_{t=1}^T \E \left[ \left\| \mathbf{d}_t \right\|^2 \right]>0,
$ and $\E \left[ \Phi_{1:T}(\lambda^*) \right]=0$ otherwise. In particular, it holds that
$  - T \sigma^2 \leq \E \left[ \Phi_{1:T}(\lambda^*) \right] \leq 0.$
}
\end{mdframed}

\begin{proof}
    Using \oureqref{eq:reg} with $\boldsymbol{\lambda}^* \gets \lambda^* \mathbf{1}$, we have
    \begin{align} \label{eq: th1_bound1}
        \sum_{t=1}^T \left( \eta  - \frac{\beta}{2} \eta^2 \right) \E\left[ \| \nabla  \loss(\bw_t) \|^2 \right] 
        \leq
         \Delta_1
        + \frac{\beta \eta^2}{2}   \, 
        \E\left[ \mathrm{Regret}_T^{\phi}(\lambda^*) \right] + \frac{\beta \eta^2}{2}  \E[ \Phi_{1:T}(\lambda^*) ] +
        \frac{\beta \eta^2 T}{2} 
         \sigma^2,
    \end{align}
    where we will specify $\lambda^*$ in the later part of this proof.
    We first bound $\E\left[ \mathrm{Regret}_T^{\phi}(\lambda^*) \right]$. From Lemma \ref{lem: shalev2025online} we have
    \begin{align}
        \mathrm{Regret}_T^{\phi}(\lambda^*):= \sum_{t=1}^{T} \phi_t(\lambda_t) - \sum_{t=1}^T \phi_t(\lambda^*)
        \leq \sum_{t=1}^{T} \left( \phi_t(\lambda_t) - \phi_t(\lambda_{t+1}) \right). 
        \label{eq: lemma5_th1}
    \end{align}
    By FTL with scalar tuning we have
    \begin{align*}
         &\lambda_t \left( \sum_{s=1}^{t-1} \left\| \mathbf{d}_s \right\|^2 \right)
        = -\sum_{s=1}^{t-1} \langle \mathbf{g}_s, \mathbf{d}_s \rangle ,\\
        &\lambda_{t+1} \left( \sum_{s=1}^{t} \left\| \mathbf{d}_s \right\|^2 \right) = - \sum_{s=1}^{t} \langle \mathbf{g}_s, \mathbf{d}_s \rangle.
    \end{align*}  
    Subtracting the above equations and rearranging gives
    \begin{align}
        \left( \sum_{s=1}^{t} \left\| \mathbf{d}_s \right\|^2 \right) \left(\lambda_t - \lambda_{t+1} \right) = \langle \mathbf{g}_t, \mathbf{d}_t \rangle + \lambda_t \left\| \mathbf{d}_t \right\|^2.
        \label{eq:identity_lam_subtr}
    \end{align}
    Furthermore, we recall that 
    \begin{align*}
        \phi_t(\lambda):= \left \|  \mathbf{g}_t + \lambda \mathbf{d}_t  \right \|^2 = \left \|  \mathbf{g}_t \right \|^2 + 2 \lambda \langle \mathbf{g}_t, \mathbf{d}_t \rangle + \lambda^2 \left \| \mathbf{d}_t  \right \|^2.
    \end{align*}
    Hence,
    \begin{align*}
        \phi_t \left( \lambda_t \right) - \phi_t \left( \lambda_{t+1} \right) 
        &= 2 \left(\lambda_t - \lambda_{t+1} \right) \langle \mathbf{g}_t, \mathbf{d}_t \rangle  + \left( \lambda_t^2 - \lambda_{t+1}^2 \right)\left\| \mathbf{d}_t \right\|^2 \\
        &=
        2 \left(\lambda_t - \lambda_{t+1} \right) \left( \langle \mathbf{g}_t, \mathbf{d}_t \rangle + \lambda_t \left\| \mathbf{d}_t \right\|^2 \right) - \left\| \mathbf{d}_t \right\|^2  \left(\lambda_t - \lambda_{t+1} \right)^2 \\
        &\overset{\oureqref{eq:identity_lam_subtr}}{=} \left(2 \sum_{s=1}^t \left\| \mathbf{d}_s \right\|^2 - \left\| \mathbf{d}_t \right\|^2 \right) \left( \lambda_t - \lambda_{t+1} \right)^2 \\
        &\leq 2 \left(\sum_{s=1}^t \left\| \mathbf{d}_s \right\|^2 \right) \left( \lambda_t - \lambda_{t+1} \right)^2.
    \end{align*}
    Additionally, observe that for any $v\geq0, \rho>0$, we have
    \begin{align*}
         v = \frac{v^2}{\rho+v} + \frac{\rho v}{\rho + v} \leq \frac{v^2}{\rho + v} + \rho.
    \end{align*}
    Applying this inequality with $v=\sum_{s=1}^t\|\mathbf d_s\|^2$, we obtain
    \begin{align*}
        \phi_t \left( \lambda_t \right) - \phi_t \left( \lambda_{t+1} \right) 
        &\leq 2\frac{\left[\left(\sum_{s=1}^t\|\mathbf d_s\|^2\right)
        (\lambda_t-\lambda_{t+1})\right]^2}
        {\rho+\sum_{s=1}^t\|\mathbf d_s\|^2}
        +2\rho(\lambda_t-\lambda_{t+1})^2 \\
        &= 2\frac{\left(\langle\mathbf g_t,\mathbf d_t\rangle
        +\lambda_t\|\mathbf d_t\|^2\right)^2}
        {\rho+\sum_{s=1}^t\|\mathbf d_s\|^2}
        +2\rho(\lambda_t-\lambda_{t+1})^2 \\
        &\leq 4\frac{\|\mathbf g_t\|^2 \|\mathbf d_t\|^2}{\rho+\sum_{s=1}^t\|\mathbf d_s\|^2}
        +
        4\lambda_t^2\frac{\|\mathbf d_t\|^4}{\rho+\sum_{s=1}^t\|\mathbf d_s\|^2}
        +2\rho(\lambda_t-\lambda_{t+1})^2 \\
        &\leq 4G^2\frac{\|\mathbf d_t\|^2}{\rho+\sum_{s=1}^t\|\mathbf d_s\|^2}
        +
        4\lambda_t^2\frac{\|\mathbf d_t\|^4}{\rho+\sum_{s=1}^t\|\mathbf d_s\|^2}
        +2\rho(\lambda_t-\lambda_{t+1})^2,
    \end{align*}
    where the second inequality follows from $(a+b)^2 \leq 2a^2 + 2b^2 $, and the Cauchy-Schwarz inequality, and the third inequality follows from Assumption~\ref{A4}. Summing over $t=1, \ldots, T$ and using \oureqref{eq: lemma5_th1} gives
    \begin{align*}
        \mathrm{Regret}_T^\phi(\lambda^*)
        &\leq
        \sum_{t=1}^T \left( \phi_t \left( \lambda_t \right) - \phi_t \left( \lambda_{t+1} \right)  \right) \\
        &\leq 
        4G^2 \sum_{t=1}^T \frac{\|\mathbf d_t\|^2}{\rho+\sum_{s=1}^t\|\mathbf d_s\|^2} 
        +
        4 \sum_{t=1}^T \lambda_t^2\frac{\|\mathbf d_t\|^4}{\rho+\sum_{s=1}^t\|\mathbf d_s\|^2}
        + 2\rho \sum_{t=1}^T (\lambda_t-\lambda_{t+1})^2
    \end{align*}
    We will bound the three terms separately. For the first term, since $1/u$ is decreasing for $u>0$, we have
    \begin{align*}
        \sum_{t=1}^T \frac{\|\mathbf d_t\|^2}{\rho+\sum_{s=1}^t\|\mathbf d_s\|^2} 
        &\leq
        \sum_{t=1}^T \int_{\rho+\sum_{s=1}^{t-1}\|\mathbf d_s\|^2}^{\rho+\sum_{s=1}^t\|\mathbf d_s\|^2} \frac{du}{u}  \\
        &= \log \left( 1 + \frac{\sum_{t=1}^T \|\mathbf d_t\|^2}{\rho} \right) \\
        &\leq \log \left( 1 + \frac{4TG^2}{\rho} \right) ,
    \end{align*}
    where we used $\| \mathbf{d}_t \| \leq 2 G$. For the second term, we have
    \begin{align*}
        \sum_{t=1}^T
        \lambda_t^2\frac{\|\mathbf d_t\|^4}{\rho+\sum_{s=1}^t\|\mathbf d_s\|^2}
        &\leq
        \left(\max_{1\le t\le T}|\lambda_t|^2\right)
        \sum_{\substack{1\le t\le T\\\sum_{s=1}^t\|\mathbf d_s\|^2>0}}
        \frac{\|\mathbf d_t\|^4}{\sum_{s=1}^t\|\mathbf d_s\|^2}\\
        &\leq 4G^2(1+\log T)\max_{1\le t\le T}|\lambda_t|^2,
    \end{align*}
    where the last inequality uses Lemma~\ref{lem: seq} with $q_t = \|\mathbf d_t\|^2$ and \(Q=\max_{1\le t\le T} \|\mathbf d_t\|^2\le 4G^2\). For the third term, \(|\lambda_t-\lambda_{t+1}|\le 2\max_{1\le s\le T+1}|\lambda_s|\) gives
    \begin{align*}
            2\rho\sum_{t=1}^T(\lambda_t-\lambda_{t+1})^2
            \leq
            8\rho T\max_{1\le t\le T+1}|\lambda_t|^2.
    \end{align*}
    Combining these inequalities, we obtain that for every \(\rho>0\),
    \begin{align*}
        \mathrm{Regret}_T^\phi(\lambda^*)
        \le
        4G^2\log\!\left(1+\frac{4TG^2}{\rho}\right)
        +
        \left[16G^2(1+\log T)+8\rho T\right]
        \max_{1\le t\le T+1}|\lambda_t|^2.
    \end{align*}
    Taking $\rho = \frac{4G^2}{T}$
    \begin{align*}
        \mathrm{Regret}_T^\phi(\lambda^*)
        \le
        4G^2\log\!\left(1+T^2\right)
        +
        \left[16G^2(1+\log T)+32G^2\right]
        \max_{1\le t\le T+1}|\lambda_t|^2.
    \end{align*}
    Taking expectations and using $C_0 := \sup_{T \geq 1} \mathbb E[\max_{1\le t\le T+1}|\lambda_t|^2]$ we obtain
    \begin{align*}
        \mathbb E\left[\mathrm{Regret}_T^\phi(\lambda^*) \right]
        \le
        4G^2\log\!\left(1+T^2\right)
        +
        \left[16G^2(1+\log T)+32G^2\right]
        C_0.
    \end{align*}
    Using $\log(1+T^2)\le \log 2+2\log T\le 2(3+\log T)$ and rearranging gives
    \begin{align*}
        \mathbb E\left[\mathrm{Regret}_T^\phi(\lambda^*) \right]
        \le
         8G^2(1+2C_0)(3+\log T).
    \end{align*}
    Substituting the regret bound in \oureqref{eq: th1_bound1} we obtain
    \begin{align} 
        \sum_{t=1}^T \left( \eta  - \frac{\beta}{2} \eta^2 \right) \E\left[ \| \nabla  \loss(\bw_t) \|^2 \right] 
        \leq
         \Delta_1
        + 4 \beta \eta^2   
         G^2(1+2C_0)(3+\log T) + \frac{\beta \eta^2}{2} \E[ \Phi_{1:T}(\lambda^*) ] +  \frac{\beta \eta^2 T}{2}\sigma^2 . \label{eq: thm_bound}
    \end{align}
    Let $\eta =  \frac{1}{\beta \sqrt{T}}$. Then, since $\eta \leq \frac{1}{\beta}$, we have
    \begin{align*}
        \eta - \frac{\beta}{2} \eta^2 \geq \frac{\eta}{2}.
    \end{align*}
    Using the above inequality and substituting the choice of $\eta$ in \oureqref{eq: thm_bound} yields
    \begin{align*}
        \frac{1}{T}\sum_{t=1}^T \E\left[ \| \nabla  \loss(\bw_t) \|^2 \right] 
        \leq
         \frac{2 \beta \Delta_1}{\sqrt{T}} 
        + \frac{8 G^2(1+2C_0)(3+\log T)}{T^{3/2}} 
        + \frac{\E[ \Phi_{1:T}(\lambda^*) ]}{T^{3/2}} +  \frac{\sigma^2}{\sqrt{T}} ,
    \end{align*}
which is \oureqref{eq: theorem_bound}.

Next, let us analyze $\E[ \Phi_{1:T}(\lambda^*) ]$.
We observe that $\Phi_{1:T}(\lambda)$ can be expressed as
\begin{align*}
     \Phi_{1:T}(\lambda) = \sum_{t=1}^T \left\| \mathbf{g}_t + \lambda \mathbf{d}_t  \right\|^2 -  \sum_{t=1}^T \left\| \mathbf{g}_t  \right\|^2 = 2\lambda \sum_{t=1}^T \left\langle \mathbf{g}_t , \mathbf{d}_t   \right\rangle
    +
    \lambda^2
    \sum_{t=1}^T  \left\| \mathbf{d}_t \right\|^2.
\end{align*}
Therefore,
\begin{align*}
     \E \left[ \Phi_{1:T}(\lambda) \right] = 2\lambda \sum_{t=1}^T \E \left[ \left\langle \mathbf{g}_t , \mathbf{d}_t   \right\rangle \right]
    +
    \lambda^2
    \sum_{t=1}^T \E \left[ \left\| \mathbf{d}_t \right\|^2 \right].
\end{align*}
First, consider the case $\sum_{t=1}^T \E \left[ \left\| \mathbf{d}_t \right\|^2 \right] > 0$. In this case, the coefficient of $\lambda^2$ is strictly positive,
hence, $\E \left[ \Phi_{1:T}(\lambda) \right]$ is convex in $\lambda$. By setting
$\frac{\partial}{\partial \lambda}
    \E \left[ \Phi_{1:T}(\lambda) \right]
    =
    0,$
we obtain
\begin{align*}
    \lambda^*
    = - \frac{ \sum_{t=1}^T \E \left[ \left\langle \mathbf{g}_t , \mathbf{d}_t   \right\rangle \right]
    }{\sum_{t=1}^T  \E \left[ \left\| \mathbf{d}_t \right\|^2 \right]
    }.
\end{align*}
Substituting this expression into $\E \left[ \Phi_{1:T}(\lambda) \right]$ yields
\begin{align*}
    \E \left[ \Phi_{1:T}(\lambda^*) \right]
    =
    -
    \frac{\left( \sum_{t=1}^T \E \left[ \left\langle \mathbf{g}_t , \mathbf{d}_t   \right\rangle \right] \right)^2 }{\sum_{t=1}^T \E \left[ \left\| \mathbf{d}_t \right\|^2 \right]}
    \leq 0.
\end{align*}
Next, suppose that $\sum_{t=1}^T \E \left[ \left\| \mathbf{d}_t \right\|^2 \right] = 0$. Since each term in the sum is non-negative, this implies $\mathbf{d}_t = 0$ almost surely, for all $t\in [T]$. Consequently, $\left\langle \mathbf{g}_t , \mathbf{d}_t   \right\rangle = 0$ almost surely, for all $t\in [T]$, and thus
$\E \left[ \Phi_{1:T}(\lambda) \right] = 0$, for all $\lambda \in \mathbb{R}$. In this case, we use the convention $\lambda^* = 0$. Therefore, we define the fixed comparator as the minimizer of $\E \left[ \Phi_{1:T}(\lambda) \right]$, given by
\begin{align*}
    \lambda^*
    =
    \begin{cases}
    - \frac{ \sum_{t=1}^T \E \left[ \left\langle \mathbf{g}_t , \mathbf{d}_t   \right\rangle \right]
    }{\sum_{t=1}^T  \E \left[ \left\| \mathbf{d}_t \right\|^2 \right]
    }, & \sum_{t=1}^T \E \left[ \left\| \mathbf{d}_t \right\|^2 \right]>0,
    \\
    0, & \sum_{t=1}^T \E \left[ \left\| \mathbf{d}_t \right\|^2 \right] = 0.
    \end{cases}
\end{align*}
The corresponding value of the objective is
\begin{align}
    \E \left[ \Phi_{1:T}(\lambda^*) \right]
    =
    \begin{cases}
    -
    \frac{\left( \sum_{t=1}^T \E \left[ \left\langle \mathbf{g}_t , \mathbf{d}_t   \right\rangle \right] \right)^2 }{\sum_{t=1}^T \E \left[ \left\| \mathbf{d}_t \right\|^2 \right]}, & \sum_{t=1}^T \E \left[ \left\| \mathbf{d}_t \right\|^2 \right]>0,
    \\
    0, & \sum_{t=1}^T \E \left[ \left\| \mathbf{d}_t \right\|^2 \right] = 0.
    \end{cases} \label{phi:expression}
\end{align}
Hence, in either case we have $\E \left[ \Phi_{1:T}(\lambda^*) \right] \leq 0$.

For the last part, since $\E \left[ \mathbf{d}_t \big| \F_{t-1} \right] = 0$, we have 
    \begin{align*}
        \E \left[\left\langle \mathbf{g}_t , \mathbf{d}_t   \right\rangle\right] = \E \left[\left\langle \mathbf{g}_t - \nabla \loss(\bw_t), \mathbf{d}_t   \right\rangle\right].
    \end{align*}
    By Cauchy-Schwarz
    \begin{align*}
        \left(\sum_{t=1}^T \E \left[\left\langle \mathbf{g}_t , \mathbf{d}_t   \right\rangle\right] \right)^2 
        &\leq  \left( \sum_{t=1}^T \E \left[ \left\| \mathbf{g}_t - \nabla \loss(\bw_t)  \right\|^2 \right] \right) \left( \sum_{t=1}^T \E \left[ \left\| \mathbf{d}_t  \right\|^2 \right] \right) \\
        &\leq T \sigma^2 \left( \sum_{t=1}^T \E \left[ \left\| \mathbf{d}_t  \right\|^2 \right] \right).
    \end{align*}
    Therefore,
    \begin{align*}
         0 \leq \frac{\left( \sum_{t=1}^T \E \left[\left\langle \mathbf{g}_t , \mathbf{d}_t   \right\rangle\right] \right)^2}{ \sum_{t=1}^T \E \left[ \left\| \mathbf{d}_t \right\|^2 \right]} \leq T \sigma^2.
    \end{align*} 
    Equivalently, using the expression for $\E \left[ \Phi_{1:T}(\lambda^*) \right]$ gives
    \begin{align*}
        - T \sigma^2 \leq \E \left[ \Phi_{1:T}(\lambda^*) \right] \leq 0.
    \end{align*}

\end{proof}

Recall in Corollary~\ref{corollary}, we consider constant batch sizes such that $\left| B_t \right| = n$ and $\left| \tilde{B}_t \right| = N$ for all $t \in [T]$.
We defined
$\loss^f(\bw) := \E \left[ \ell \left( \bw ; \bx, f(\bx) \right) \right]$,
and we note that, conditional on $\mathcal{F}_{t-1}$,
\begin{align*}
    \E \left[\gt{f}{\bw_t} \big| \mathcal{F}_{t-1} \right] = \nabla \loss^f(\bw_t) \quad , \quad  \E \left[\gt{y}{\bw_t} \big| \mathcal{F}_{t-1} \right] = \nabla \loss(\bw_t).
\end{align*}

\begin{mdframed}[backgroundcolor=black!10,rightline=false,leftline=false,topline=false,bottomline=false]
\noindent
\textbf{Corollary~\ref{corollary}.}
\textit{
Suppose the variance of a gradient computed from a single ML-labeled sample is uniformly bounded by $\sigma_f^2$, i.e.,
$\E \left[ \left\| \nabla \ell \left( \bw ; \bx, f(\bx) \right) - \nabla \loss^f(\bw) \right\|^2 \right] \leq \sigma_f^2$, while $\sigma_y^2$ represents its counterpart using true labels.
Additionally, let $\bar{\Xi}_T:=  \frac{1}{T}\sum_{t=1}^T \E\left[ \| \nabla \loss(\bw_t) \|^2 \right]$ 
and $\Sigma_{t}^{y,f} := \operatorname{Cov} \left( \nabla \ell \left( \bw_t ; \bx, y \right), \nabla \ell \left( \bw_t ; \bx, f(\bx) \right) \big| \mathcal{F}_{t-1}\right)$ denote the conditional covariance between stochastic gradients computed using true and ML labels at step $t$.
Consider the case that $\exists t \in [T]: \E[ \| \mathbf{d}_t \|^2] > 0$.
Under the assumptions of Theorem~\ref{th: first-order-scalar},
Algorithm~\ref{alg:sgd} satisfies
\begin{align*}
     \bar{\Xi}_T   &\leq
         \frac{2 \beta \Delta_1}{\sqrt{T}} 
        + \frac{8 G^2(1+2C_0)(3+\log T)}{T^{3/2}} 
        + \frac{1}{\sqrt{T}} \frac{\sigma_y^2}{n}
        - \frac{1}{\sqrt{T}} \frac{N}{n (n+N)} \frac{ \left( \frac{1}{T}\sum_{t=1}^T \E \left[ \operatorname{tr} \left( \Sigma_{t}^{y,f}  \right) \right] \right)^2 }{\sigma_f^2}.
\end{align*}
}
\end{mdframed}

\begin{proof}
We have, for all $t \in [T]$,
$ \E \left[ \left\| \mathbf{g}_t  - \nabla \loss(\bw_t) \right\|^2 \right] \leq \frac{\sigma_y^2}{n}$, as $n$ samples with true labels are used for the mini-batch stochastic gradient $\mathbf{g}_t$.
Also, conditional on $\mathcal{F}_{t-1}$, the two
ML-label gradient estimates are independent and have
the common mean $\nabla\loss^f(\bw_t)$. Hence,
\begin{align}
\E\left[\|\mathbf{d}_t\|^2\right]
&=
\E\left[
\left\|
\tildegt{f}{\bw_t}-\gt{f}{\bw_t}
\right\|^2
\right]
\notag\\
&=
\E \left[ \left\| (\tildegt{f}{\bw_t} - \nabla \loss^f(\bw_t)  ) - (\gt{f}{\bw_t} -  \nabla \loss^f(\bw_t) ) \right\|^2 \right] 
\notag\\
&=
\E\left[
\left\|
\tildegt{f}{\bw_t}-\nabla\loss^f(\bw_t)
\right\|^2
\right]
+
\E\left[
\left\|
\gt{f}{\bw_t}-\nabla\loss^f(\bw_t)
\right\|^2
\right]
\notag\\
&\leq
\frac{\sigma_f^2}{N}
+
\frac{\sigma_f^2}{n}
=
\left(\frac{1}{n}+\frac{1}{N}\right)\sigma_f^2.
\label{eq: var_f}
\end{align}
The third equality follows because the centered cross term has conditional expectation zero. 

Furthermore, 
\begin{align}
    \E \left[ \left\langle \mathbf{g}_t , \mathbf{d}_t \right\rangle\right]
    =
    \E \left[ \left\langle \gt{y}{\bw_t}  , \tildegt{f}{\bw_t} - \gt{f}{\bw_t} \right\rangle\right]
    =
    \E \left[ \left\langle \gt{y}{\bw_t}  , \tildegt{f}{\bw_t} \right\rangle\right] - \E \left[ \left\langle \gt{y}{\bw_t}  , \gt{f}{\bw_t} \right\rangle\right].
    \label{eq: inner_prod_decomp}
\end{align}
Since $\tildegt{f}{\bw_t}$ is computed from the independent unlabeled batch, conditional on $\mathcal{F}_{t-1}$ it is independent of the labeled estimator $\gt{y}{\bw_t}$. Hence, for the first term in \oureqref{eq: inner_prod_decomp} we have
\begin{align}
    \E \left[ \left\langle \gt{y}{\bw_t}  , \tildegt{f}{\bw_t} \right\rangle\right] = \E \left[ \left\langle \nabla \loss(\bw_t) , \nabla \loss^f(\bw_t) \right\rangle\right].
    \label{eq: first_term}
\end{align}
For the second term in \oureqref{eq: inner_prod_decomp}, $\gt{y}{\bw_t}$ and $\gt{f}{\bw_t}$ are not independent, because they are computed using the same labeled batch. We can express the second term as follows
\begin{align}
    \E \left[ \left\langle \gt{y}{\bw_t}  , \gt{f}{\bw_t} \right\rangle\right] = \E \left[ \left\langle \nabla \loss(\bw_t) , \nabla \loss^f(\bw_t) \right\rangle\right] + \E \left[ \left\langle \gt{y}{\bw_t} - \nabla \loss(\bw_t) , \gt{f}{\bw_t} - \nabla \loss^f(\bw_t) \right\rangle\right].
    \label{eq: second_term}
\end{align}
Substituting \oureqref{eq: first_term} and \oureqref{eq: second_term} into \oureqref{eq: inner_prod_decomp}, we get
\begin{align*}
    \E \left[ \left\langle \mathbf{g}_t , \mathbf{d}_t \right\rangle\right] &= - \E \left[ \left\langle \gt{y}{\bw_t} - \nabla \loss(\bw_t) , \gt{f}{\bw_t} - \nabla \loss^f(\bw_t) \right\rangle\right]  \\
    &= - \E \left[ \left\langle \frac{1}{ n } \sum_{(\bx, y) \in B_t} \nabla \ell \left( \bw_t ; \bx, y \right) - \nabla \loss(\bw_t) , \frac{1}{ n } \sum_{(\bx, y) \in B_t} \nabla \ell \left( \bw_t ; \bx, f(\bx) \right) - \nabla \loss^f(\bw_t) \right\rangle\right]  \\
    &= - \frac{1}{n} \E \left[ \left\langle \nabla \ell \left( \bw_t ; \bx, y \right) - \nabla \loss(\bw_t) , \nabla \ell \left( \bw_t ; \bx, f(\bx) \right) - \nabla \loss^f(\bw_t) \right\rangle\right] \\
    &= - \frac{1}{n} \E \left[ \E \left[ \left\langle \nabla \ell \left( \bw_t ; \bx, y \right) - \nabla \loss(\bw_t) , \nabla \ell \left( \bw_t ; \bx, f(\bx) \right) - \nabla \loss^f(\bw_t) \right\rangle \big| \mathcal{F}_{t-1}\right] \right] \\
    &= - \frac{1}{n} \E \left[ \operatorname{tr} \left( \operatorname{Cov} \left( \nabla \ell \left( \bw_t ; \bx, y \right) , \nabla \ell \left( \bw_t ; \bx, f(\bx) \right) \big| \mathcal{F}_{t-1}\right) \right) \right].
\end{align*}
Define
$    \Sigma_{t}^{y,f} : = \operatorname{Cov} \left( \nabla \ell \left( \bw_t ; \bx, y \right) , \nabla \ell \left( \bw_t ; \bx, f(\bx) \right) \big| \mathcal{F}_{t-1}\right).$
Then, we can write
\begin{align}
    \E \left[ \left\langle \mathbf{g}_t , \mathbf{d}_t \right\rangle\right]
    =
    - \frac{1}{n} \E \left[ \operatorname{tr} \left( \Sigma_{t}^{y,f}  \right) \right].
    \label{eq: inner_prod}
\end{align}
For $\sum_{t=1}^T \E \left[ \left\| \mathbf{d}_t \right\|^2 \right] > 0$,
substituting \oureqref{eq: var_f}, \oureqref{eq: inner_prod} into the expression for $\E \left[ \Phi_{1:T}(\lambda^*) \right]$ \oureqref{phi:expression}, we obtain
\begin{align*}
    \E \left[ \Phi_{1:T}(\lambda^*) \right] \leq - \frac{1}{T} \frac{N}{n(n+N)} \frac{\left(\sum_{t=1}^T \E \left[ \operatorname{tr} \left( \Sigma_{t}^{y,f}  \right) \right] \right)^2}{\sigma_{f}^2}.
\end{align*}
Therefore, using the above bound and letting $\sigma^2 = \frac{\sigma_y^2}{n}$ in \oureqref{eq: thm_bound}, we obtain
\begin{align*}
    \bar{\Xi}_T
        &\leq
         \frac{2 \beta \Delta_1}{\sqrt{T}} 
        + \frac{8 G^2(1+2C_0)(3+\log T)}{T^{3/2}} 
        + \frac{1}{\sqrt{T}} \frac{\sigma_y^2}{n}
        - \frac{1}{\sqrt{T}} \frac{N}{n (n+N)} \frac{ \left(  \frac{1}{T} \sum_{t=1}^T \E \left[ \operatorname{tr} \left( \Sigma_{t}^{y,f}  \right) \right] \right)^2 }{\sigma_f^2}.
\end{align*}
\end{proof}

Having finished the analysis of the scalar tuning case, we turn to vector tuning below.
%\begin{mdframed}
In vector tuning, we also adopt FTL, i.e., initialize $\boldsymbol{\lambda}_1 \gets \boldsymbol{0}$ and for all $t>1$, play
\begin{align}
    \boldsymbol{\lambda}_t \gets \underset{ \boldsymbol{\lambda} \in \reals^d }{\arg\min}  \sum_{s=1}^{t-1}  \phi_s(\boldsymbol{\lambda}), \text{ where } \phi_s(\boldsymbol{\lambda}):= \left\| \mathbf{g}_s + \boldsymbol{\lambda} \odot \mathbf{d}_s  \right\|^2. \label{eq: coord_FTL}
\end{align}
%\end{mdframed}
Because the loss function is separable, we have for all $i \in [d]$:
\begin{align*}
    \boldsymbol{\lambda}_t [i] \gets \underset{ \lambda \in \reals }{\arg\min} \sum_{s=1}^{t-1} \phi_s^{(i)} \left( \lambda \right) , \; \text{where } \phi_t^{(i)} \left( \lambda \right) := \left( \mathbf{g}_t[i] + \lambda \mathbf{d}_t[i] \right)^2.
\end{align*}
Expanding the cumulative loss function gives
\begin{align*}
    \sum_{s=1}^{t-1} \phi_t^{(i)} \left( \lambda \right) 
    &= \sum_{s=1}^{t-1} \left( \mathbf{g}_s[i] + \lambda \mathbf{d}_s[i] \right)^2 \\
    &= \sum_{s=1}^{t-1} \mathbf{g}_s[i]^2 + 2 \lambda \sum_{s=1}^{t-1}\mathbf{g}_s[i] \mathbf{d}_s[i] + \lambda^2 \sum_{s=1}^{t-1} \mathbf{d}_s[i]^2
\end{align*}
For $\sum_{s=1}^{t-1} \mathbf{d}_s[i]^2 > 0$, the objective is quadratic in $\lambda$. By setting $\frac{\partial}{\partial \lambda} \sum_{s=1}^{t-1} \phi_t^{(i)} \left( \lambda \right) = 0$, we obtain
$    \boldsymbol{\lambda}_t [i] = - \frac{\sum_{s=1}^{t-1}\mathbf{g}_s[i] \mathbf{d}_s[i]}{\sum_{s=1}^{t-1} \mathbf{d}_s[i]^2}$.
Similarly to the case of scalar tuning, for $\sum_{s=1}^{t-1} \mathbf{d}_s[i]^2 = 0$ the cumulative loss function is independent in $\lambda$, and thus in that case we use $\boldsymbol{\lambda}_t [i] = 0$. Therefore, for all $i \in [d]$ we have
\begin{align*}
    \boldsymbol{\lambda}_t [i] =
    \begin{cases}
        - \frac{\sum_{s=1}^{t-1}\mathbf{g}_s[i] \mathbf{d}_s[i]}{\sum_{s=1}^{t-1} \mathbf{d}_s[i]^2} , & \sum_{s=1}^{t-1} \mathbf{d}_s[i]^2 > 0, \\
        0, &\sum_{s=1}^{t-1} \mathbf{d}_s[i]^2 = 0, 
    \end{cases}
\end{align*}
which is \oureqref{vector-tuning-first-order}.
In the following, we denote
\begin{align}
    \mathrm{Regret}_T^{\phi}(\boldsymbol{\lambda}^*)
    &:= \sum_{t=1}^{T} \phi_t(\boldsymbol{\lambda}_t) - \sum_{t=1}^T \phi_t(\boldsymbol{\lambda}^*). \label{eq: def_regret_lam_vec}
\end{align}
Additionally, define
$\Phi_{1:T} \left( \boldsymbol{\lambda} \right) 
    :=  \sum_{t=1}^T \left \| \g_{t,\bw_t}^{\boldsymbol{\lambda}} \right \|^2 -
    \sum_{t=1}^T \left \| \g_{t,\bw_t}^{\boldsymbol{0}} \right \|^2.$
We are now ready to prove Theorem~\ref{th: first-order-vector}.

\begin{mdframed}[backgroundcolor=black!10,rightline=false,leftline=false,topline=false,bottomline=false]
\begin{theorem} \label{th: first-order-vector}
Denote $\widetilde C_0:= \sup_{T\ge 1}\mathbb E\!\left[
\max_{1\le t\le T+1}\|\boldsymbol{\lambda}_t\|_\infty^2
\right]$.
Then, under Assumptions~\ref{A1}-\ref{A4} Algorithm~\ref{alg:sgd} 
with vector tuning and $\eta_t = \frac{1}{\beta \sqrt{T}}$ satisfies
\begin{align*} 
    \frac{1}{T}\sum_{t=1}^T \E\left[ \| \nabla  \loss(\bw_t) \|^2 \right] 
        \leq
         \frac{2 \beta \Delta_1}{\sqrt{T}} 
        + \frac{8dG^2(1+2\widetilde C_0)(3+\log T)}{T^{3/2}} 
        + \frac{\E[ \Phi_{1:T}(\boldsymbol{\lambda}^*) ]}{T^{3/2}} +  \frac{\sigma^2}{\sqrt{T}} ,
\end{align*}
where 
$    \E \left[ \Phi_{1:T}(\boldsymbol{\lambda}^*) \right]
    =
    - \sum_{ i \in \mathfrak{I} } \frac{\left( \sum_{t=1}^{T} \E \left[ \mathbf{g}_t[i] \mathbf{d}_t[i] \right] \right)^2}{\sum_{t=1}^{T} \E \left[ \mathbf{d}_t[i]^2 \right] } \leq 0,
$ when the set $\mathfrak{I}:=\{ i \in [d]: \sum_{t=1}^{T} \E \left[ \mathbf{d}_t[i]^2 \right] > 0 \}$ is nonempty; otherwise, $\E \left[ \Phi_{1:T}(\boldsymbol{\lambda}^*) \right]=0$. In particular, it holds that
$  - T \sigma^2 \leq \E \left[ \Phi_{1:T}(\boldsymbol{\lambda}^*) \right] \leq 0.$
\end{theorem}
\end{mdframed}

\begin{proof}

From \oureqref{eq:reg}, we have
    \begin{align} \label{eq: th_bound2}
        \sum_{t=1}^T \left( \eta  - \frac{\beta}{2} \eta^2 \right) \E\left[ \| \nabla \loss(\bw_t) \|^2 \right] 
        \leq
         \Delta_1
        + \frac{\beta \eta^2}{2}   \, 
        \E\left[ \mathrm{Regret}_T^{\phi}(\boldsymbol{\lambda}^*) \right] + \frac{\beta \eta^2}{2}  \E[ \Phi_{1:T}(\boldsymbol{\lambda}^*) ] +
        \frac{\beta \eta^2 T}{2} 
         \sigma^2,
    \end{align}
    where we will specify $\boldsymbol{\lambda}^*$ in the later part of the proof. 
    We first bound $\E\left[ \mathrm{Regret}_T^{\phi}(\boldsymbol{\lambda}^*) \right]$. From the definition of regret \oureqref{eq: def_regret_lam_vec} we have
    \begin{align}
        \mathrm{Regret}_T^{\phi}(\boldsymbol{\lambda}^*)
        &:= \sum_{t=1}^{T} \phi_t(\boldsymbol{\lambda}_t) - \sum_{t=1}^T \phi_t(\boldsymbol{\lambda}^*) \notag \\
        &= \sum_{t=1}^{T} \sum_{i=1}^{d} \phi_t^{(i)}(\boldsymbol{\lambda}_t[i]) - \sum_{t=1}^{T} \sum_{i=1}^{d} \phi_t^{(i)}(\boldsymbol{\lambda}^*[i]) \notag \\
        &= \sum_{i=1}^{d} \underbrace{ \left( \sum_{t=1}^{T} \phi_t^{(i)}(\boldsymbol{\lambda}_t[i])  - \sum_{t=1}^T \phi_t^{(i)}(\boldsymbol{\lambda}^*[i]) \right)}_{:= \mathrm{Regret}_T^{\phi^{(i)}}(\boldsymbol{\lambda}^*)} . \label{eq: overall_reg_lam_vec1}
    \end{align}
    Using Lemma \ref{lem: shalev2025online} we have
    \begin{align}
        \mathrm{Regret}_T^{\phi^{(i)}}(\boldsymbol{\lambda}^*) := \sum_{t=1}^{T} \phi_t^{(i)}(\boldsymbol{\lambda}_t[i])  - \sum_{t=1}^T \phi_t^{(i)}(\boldsymbol{\lambda}^*[i]) \leq 
        \sum_{t=1}^{T} \left( \phi_t^{(i)}(\boldsymbol{\lambda}_t[i]) - \phi_t^{(i)}(\boldsymbol{\lambda}_{t+1}[i]) \right), \label{eq: reg_i_lam_vec}
    \end{align}
    for all $i \in [d]$. Combining \oureqref{eq: overall_reg_lam_vec1} and \oureqref{eq: reg_i_lam_vec} we obtain
    \begin{align}
        \mathrm{Regret}^{\phi}_T(\boldsymbol{\lambda}^*) \leq \sum_{i=1}^{d} \sum_{t=1}^{T} \left( \phi_t^{(i)}(\boldsymbol{\lambda}_t[i]) - \phi_t^{(i)}(\boldsymbol{\lambda}_{t+1}[i]) \right). \label{eq: overall_reg_lam_vec_bound}
    \end{align}
    Fix some $i \in [d]$. By FTL with vector tuning we have
    \begin{align*}
        \boldsymbol{\lambda}_t[i]\sum_{s=1}^{t-1}\mathbf d_s[i]^2
        &=-\sum_{s=1}^{t-1}\mathbf g_s[i]\mathbf d_s[i],\\
        \boldsymbol{\lambda}_{t+1}[i]\sum_{s=1}^{t}\mathbf d_s[i]^2
        &=-\sum_{s=1}^{t}\mathbf g_s[i]\mathbf d_s[i].
    \end{align*}
    Subtracting the second identity from the first and adding \(\boldsymbol{\lambda}_t[i]\mathbf d_t[i]^2\) to both sides gives
    \begin{align}
        \left(\sum_{s=1}^{t}\mathbf d_s[i]^2\right)
        \bigl(\boldsymbol{\lambda}_t[i]-\boldsymbol{\lambda}_{t+1}[i]\bigr)
        =
        \mathbf g_t[i]\mathbf d_t[i]
        +\boldsymbol{\lambda}_t[i]\mathbf d_t[i]^2.
        \label{eq: vector_lambda_diff_th_3}
    \end{align}
    This identity also holds when \(\sum_{s=1}^t\mathbf d_s[i]^2=0\), since both sides then vanish. We recall that 
    \begin{align*}
        \phi_t^{(i)}(\lambda)=(\mathbf g_t[i]+\lambda\mathbf d_t[i])^2.
    \end{align*}
    Therefore,
    \begin{align*}
    &\phi_t^{(i)}(\boldsymbol{\lambda}_t[i])
    -\phi_t^{(i)}(\boldsymbol{\lambda}_{t+1}[i])\\
    &=2\bigl(\boldsymbol{\lambda}_t[i]-\boldsymbol{\lambda}_{t+1}[i]\bigr)
           \mathbf g_t[i]\mathbf d_t[i]
      +\bigl(\boldsymbol{\lambda}_t[i]^2-\boldsymbol{\lambda}_{t+1}[i]^2\bigr)
           \mathbf d_t[i]^2\\
    &=2\bigl(\boldsymbol{\lambda}_t[i]-\boldsymbol{\lambda}_{t+1}[i]\bigr)
           \bigl(\mathbf g_t[i]\mathbf d_t[i]
                 +\boldsymbol{\lambda}_t[i]\mathbf d_t[i]^2\bigr)
      -\mathbf d_t[i]^2
           \bigl(\boldsymbol{\lambda}_t[i]-\boldsymbol{\lambda}_{t+1}[i]\bigr)^2\\
    &\overset{\oureqref{eq: vector_lambda_diff_th_3}}{=}
      \left(2\sum_{s=1}^{t}\mathbf d_s[i]^2-\mathbf d_t[i]^2\right)
           \bigl(\boldsymbol{\lambda}_t[i]-\boldsymbol{\lambda}_{t+1}[i]\bigr)^2\\
    &\le
      2\left(\sum_{s=1}^{t}\mathbf d_s[i]^2\right)
           \bigl(\boldsymbol{\lambda}_t[i]-\boldsymbol{\lambda}_{t+1}[i]\bigr)^2.
    \end{align*}
    For any \(v\ge 0\) and \(\rho>0\), we have
    \begin{align*}
        v=\frac{v^2}{\rho+v}+\frac{\rho v}{\rho+v}
    \le \frac{v^2}{\rho+v}+\rho.
    \end{align*}
    Applying this inequality with \(v=\sum_{s=1}^{t}\mathbf d_s[i]^2\), we obtain
    \begin{align*}
        &\phi_t^{(i)}(\boldsymbol{\lambda}_t[i])
        -\phi_t^{(i)}(\boldsymbol{\lambda}_{t+1}[i])\\
        &\le
        2\frac{
        \left[
          \left(\sum_{s=1}^{t}\mathbf d_s[i]^2\right)
          \bigl(\boldsymbol{\lambda}_t[i]-\boldsymbol{\lambda}_{t+1}[i]\bigr)
        \right]^2}
        {\rho+\sum_{s=1}^{t}\mathbf d_s[i]^2}
        +2\rho\bigl(\boldsymbol{\lambda}_t[i]-\boldsymbol{\lambda}_{t+1}[i]\bigr)^2\\
        &\overset{\oureqref{eq: vector_lambda_diff_th_3}}{=}
        2\frac{\bigl(\mathbf g_t[i]\mathbf d_t[i]
               +\boldsymbol{\lambda}_t[i]\mathbf d_t[i]^2\bigr)^2}
               {\rho+\sum_{s=1}^{t}\mathbf d_s[i]^2}
        +2\rho\bigl(\boldsymbol{\lambda}_t[i]-\boldsymbol{\lambda}_{t+1}[i]\bigr)^2\\
        &\le
        4G^2\frac{\mathbf d_t[i]^2}{\rho+\sum_{s=1}^{t}\mathbf d_s[i]^2}
        +4\boldsymbol{\lambda}_t[i]^2
           \frac{\mathbf d_t[i]^4}{\rho+\sum_{s=1}^{t}\mathbf d_s[i]^2}
        +2\rho\bigl(\boldsymbol{\lambda}_t[i]-\boldsymbol{\lambda}_{t+1}[i]\bigr)^2,
    \end{align*}
    where the last inequality follows from \((a+b)^2\le 2a^2+2b^2\) and \(|\mathbf g_t[i]|\le\|\mathbf g_t\|\le G\) by Assumption~\ref{A4}. Summing over \(t=1,\ldots,T\) and using the coordinatewise regret inequality \oureqref{eq: reg_i_lam_vec}, we obtain
    \begin{align*}
        \mathrm{Regret}_T^{\phi^{(i)}}(\boldsymbol{\lambda}^*)
        \leq
        4G^2\sum_{t=1}^{T}
         \frac{\mathbf d_t[i]^2}{\rho+\sum_{s=1}^{t}\mathbf d_s[i]^2}
        +4\sum_{t=1}^{T}\boldsymbol{\lambda}_t[i]^2
         \frac{\mathbf d_t[i]^4}{\rho+\sum_{s=1}^{t}\mathbf d_s[i]^2}
        +2\rho\sum_{t=1}^{T}
         \bigl(\boldsymbol{\lambda}_t[i]-\boldsymbol{\lambda}_{t+1}[i]\bigr)^2.
    \end{align*}
    We will bound the three terms separately. For the first term, since $1/u$ is decreasing for \(u>0\), we have
    \begin{align*}
        \sum_{t=1}^{T}
        \frac{\mathbf d_t[i]^2}{\rho+\sum_{s=1}^{t}\mathbf d_s[i]^2}
        &\le
        \sum_{t=1}^{T}
        \int_{\rho+\sum_{s=1}^{t-1}\mathbf d_s[i]^2}
             ^{\rho+\sum_{s=1}^{t}\mathbf d_s[i]^2}
        \frac{du}{u}\\
        &=\log\!\left(1+\frac{\sum_{t=1}^{T}\mathbf d_t[i]^2}{\rho}\right)\\
        &\le \log\!\left(1+\frac{4TG^2}{\rho}\right),
    \end{align*}
    where the last inequality uses \(|\mathbf d_t[i]|\le\|\mathbf d_t\|\le 2G\). For the second term, we have
    \begin{align*}
    \sum_{t=1}^{T}\boldsymbol{\lambda}_t[i]^2
        \frac{\mathbf d_t[i]^4}{\rho+\sum_{s=1}^{t}\mathbf d_s[i]^2}
        &\le
        \max_{1\le t\le T}|\boldsymbol{\lambda}_t[i]|^2
        \sum_{\substack{1\le t\le T\\ \sum_{s=1}^{t}\mathbf d_s[i]^2>0}}
        \frac{\mathbf d_t[i]^4}{\sum_{s=1}^{t}\mathbf d_s[i]^2}\\
        &\le
        \max_{1\le t\le T}|\boldsymbol{\lambda}_t[i]|^2
        \left(\max_{1\le t\le T}\mathbf d_t[i]^2\right)(1+\log T)\\
        &\le
        4G^2(1+\log T)\max_{1\le t\le T}|\boldsymbol{\lambda}_t[i]|^2.
    \end{align*}
    The second inequality follows from Lemma~\ref{lem: seq} with \(q_t=\mathbf d_t[i]^2\) and \(Q=\max_{1\le t\le T}\mathbf d_t[i]^2\le 4G^2\), where the bound uses \(|\mathbf d_t[i]|\le\|\mathbf d_t\|\le 2G\). For the third term, the
    triangle inequality gives
    \[
    |\boldsymbol{\lambda}_t[i]-\boldsymbol{\lambda}_{t+1}[i]|
    \le 2\max_{1\le s\le T+1}|\boldsymbol{\lambda}_s[i]|.
    \]
    Consequently,
    \[
    2\rho\sum_{t=1}^{T}
    \bigl(\boldsymbol{\lambda}_t[i]-\boldsymbol{\lambda}_{t+1}[i]\bigr)^2
    \le
    8\rho T\max_{1\le t\le T+1}|\boldsymbol{\lambda}_t[i]|^2.
    \]
    Substituting these three bounds into the regret bound yields
    \begin{align*}
        \mathrm{Regret}_T^{\phi^{(i)}}(\boldsymbol{\lambda}^*)
        \leq
        4G^2\log\!\left(1+\frac{4TG^2}{\rho}\right)
        +\bigl[16G^2(1+\log T)+8\rho T\bigr]
        \max_{1\le t\le T+1}|\boldsymbol{\lambda}_t[i]|^2.
    \end{align*}
    Taking \(\rho=4G^2/T\) for \(G>0\) gives
    \begin{align*}
        \mathrm{Regret}_T^{\phi^{(i)}}(\boldsymbol{\lambda}^*)
        \leq
        4G^2\log(1+T^2)
        +
        16G^2(3+\log T)
        \max_{1\le t\le T+1}|\boldsymbol{\lambda}_t[i]|^2.
    \end{align*}
    Finally, summing over the \(d\) coordinates and taking expectations, we obtain
    \[
    \begin{aligned}
    \mathbb E\!\left[\mathrm{Regret}_T^\phi(\boldsymbol{\lambda}^*)\right]
    &=\sum_{i=1}^{d}
      \mathbb E\!\left[\mathrm{Regret}_T^{\phi^{(i)}}(\boldsymbol{\lambda}^*)\right]\\
    &\le
    4dG^2\log(1+T^2)
    +
    16G^2(3+\log T)
    \sum_{i=1}^{d}
    \mathbb E\!\left[\max_{1\le t\le T+1}|\boldsymbol{\lambda}_t[i]|^2\right]\\
    &\le
    4dG^2\log(1+T^2)
    +
    16dG^2(3+\log T)\widetilde C_0,
    \end{aligned}
    \]
    where the last inequality uses the definition of \(\widetilde C_0\). Since \(\log(1+T^2)\le 2(3+\log T)\) for \(T\ge 1\), it follows that
    \[
    \mathbb E\!\left[\mathrm{Regret}_T^\phi(\boldsymbol{\lambda}^*)\right]
    \le
    8dG^2(1+2\widetilde C_0)(3+\log T).
    \]
    Substituting the regret bound in \oureqref{eq: th_bound2} we obtain
    \begin{align} \label{eq: thm_bound2}
        \sum_{t=1}^T \left( \eta  - \frac{\beta}{2} \eta^2 \right) \E\left[ \| \nabla  L(\bw_t) \|^2 \right] 
        \leq
         \Delta_1
        + 4 \beta \eta^2 dG^2(1+2\widetilde C_0)(3+\log T)+ \frac{\beta \eta^2}{2}  \E[ \Phi_{1:T}(\boldsymbol{\lambda}^*) ] +
        \frac{\beta \eta^2 T}{2} 
         \sigma^2.
    \end{align}
    Let $\eta = \frac{1}{\beta \sqrt{T}}$. Then, since $\eta \leq \frac{1}{\beta}$, we have
    \begin{align*}
        \eta - \frac{\beta}{2} \eta^2 \geq \frac{\eta}{2}.
    \end{align*}
    Using the above inequality and substituting the choice of $\eta$ in \oureqref{eq: thm_bound2} yields
    \begin{align*}
        \frac{1}{T}\sum_{t=1}^T \E\left[ \| \nabla \loss(\bw_t) \|^2 \right] 
        \leq
         \frac{2 \beta \Delta_1}{\sqrt{T}} 
        + \frac{8dG^2(1+2\widetilde C_0)(3+\log T)}{T^{3/2}} 
        + \frac{\E[ \Phi_{1:T}(\boldsymbol{\lambda}^*) ]}{T^{3/2}} +  \frac{\sigma^2}{\sqrt{T}}.
    \end{align*}
We continue by analyzing $\E[ \Phi_{1:T}(\boldsymbol{\lambda}^*) ]$.
We first observe that $\Phi_{1:T} \left( \boldsymbol{\lambda} \right)$ can be expressed as
\begin{align*}
    \Phi_{1:T} \left( \boldsymbol{\lambda} \right) 
    &=  \sum_{t=1}^T \left\| \mathbf{g}_t + \boldsymbol{\lambda} \odot \mathbf{d}_t  \right\|^2 -  \sum_{t=1}^T \left\| \mathbf{g}_t  \right\|^2 \\
    &= 2 \sum_{t=1}^T \langle \mathbf{g}_t , \boldsymbol{\lambda} \odot \mathbf{d}_t \rangle + \sum_{t=1}^T \left\| \boldsymbol{\lambda} \odot \mathbf{d}_t  \right\|^2 \\
    &= 2 \sum_{t=1}^T \sum_{i=1}^d \boldsymbol{\lambda}[i] \mathbf{g}_t [i] \mathbf{d}_t [i] + \sum_{t=1}^T \sum_{i=1}^d \boldsymbol{\lambda}[i]^2 \mathbf{d}_t [i]^2 \\
    &= \sum_{i=1}^d \left( 2 \boldsymbol{\lambda}[i] \sum_{t=1}^T \mathbf{g}_t [i] \mathbf{d}_t [i] +  \boldsymbol{\lambda}[i]^2 \sum_{t=1}^T \mathbf{d}_t [i]^2\right) .
\end{align*}
Therefore, 
\begin{align}
    \E \left[ \Phi_{1:T} \left( \boldsymbol{\lambda} \right) \right] 
    = \sum_{i=1}^d \left( 2 \boldsymbol{\lambda}[i] \sum_{t=1}^T \E \left[ \mathbf{g}_t [i] \mathbf{d}_t [i] \right] +  \boldsymbol{\lambda}[i]^2 \sum_{t=1}^T \E \left[ \mathbf{d}_t [i]^2 \right]\right) . \label{expression:phi-1:T}
\end{align}
Observe that $\E \left[ \Phi_{1:T} \left( \boldsymbol{\lambda} \right) \right]$ is separable. We now define the comparator $\boldsymbol{\lambda}^*$ as the minimizer of $\E \left[ \Phi_{1:T} \left( \boldsymbol{\lambda} \right) \right]$ given by
\begin{align*}
    \boldsymbol{\lambda}^*[i] = 
    \begin{cases}
        - \frac{\sum_{t=1}^{T} \E \left[ \mathbf{g}_t[i] \mathbf{d}_t[i] \right]}{\sum_{t=1}^{T} \E \left[ \mathbf{d}_t[i]^2 \right] } , & \sum_{t=1}^{T} \E \left[ \mathbf{d}_t[i]^2 \right] > 0, \\
        0, &\sum_{t=1}^{T} \E \left[  \mathbf{d}_t[i]^2 \right] = 0,
    \end{cases}
\end{align*}
for all $i \in [d]$. Substituting $\boldsymbol{\lambda} \gets \boldsymbol{\lambda}^*$ into the expression for $\E \left[ \Phi_{1:T} \left( \boldsymbol{\lambda} \right) \right]$ \oureqref{expression:phi-1:T}, we obtain
\begin{align*}
    \E \left[ \Phi_{1:T} \left( \boldsymbol{\lambda}^* \right) \right] =
    - \sum_{\substack{i \in [d] \\ \sum_{t=1}^{T} \E \left[ \mathbf{d}_t[i]^2 \right] > 0}} \frac{\left( \sum_{t=1}^{T} \E \left[ \mathbf{g}_t[i] \mathbf{d}_t[i] \right] \right)^2}{\sum_{t=1}^{T} \E \left[ \mathbf{d}_t[i]^2 \right] } \leq 0.
\end{align*}

For the last part, define the centered stochastic gradient
$\boldsymbol{\xi}_t:=\mathbf{g}_t-\nabla\loss(\bw_t).$
Then, Cauchy--Schwarz gives
\begin{align}
\left( \sum_{t=1}^T \E\left[\boldsymbol{\xi}_t[i]\mathbf{d}_t[i]\right] \right)^2
\leq \left( \sum_{t=1}^T \E\left[\boldsymbol{\xi}_t[i]^2\right] \right)
\left( \sum_{t=1}^T \E\left[\mathbf{d}_t[i]^2\right] \right). \label{e41}
\end{align}
Also, we observe that
\begin{align*}
\E\left[ \nabla\loss(\bw_t)[i]\,\mathbf{d}_t[i] \right]
&= \E\left[ \nabla\loss(\bw_t)[i]\, \E\left[\mathbf{d}_t[i]\mid\mathcal{F}_{t-1}\right] \right] =0,
\end{align*}
which implies that
\begin{equation}
\sum_{t=1}^T
\E\left[\boldsymbol{\xi}_t[i]\mathbf{d}_t[i]\right]
=
\sum_{t=1}^T
\E\left[\mathbf{g}_t[i]\mathbf{d}_t[i]\right]. \label{e42}
\end{equation}
Combining \oureqref{e41} and \oureqref{e42}, we obtain
\begin{equation*}
\left( \sum_{t=1}^T \E\left[\mathbf{g}_t[i]\mathbf{d}_t[i]\right] \right)^2
\leq \left( \sum_{t=1}^T \E\left[\boldsymbol{\xi}_t[i]^2\right] \right)
\left( \sum_{t=1}^T \E\left[\mathbf{d}_t[i]^2\right] \right). 
\end{equation*}
Hence, for every $i\in\mathfrak{I}$,
\[
\frac{\left( \sum_{t=1}^T \E\left[\mathbf{g}_t[i]\mathbf{d}_t[i]\right] \right)^2}{ \sum_{t=1}^T \E\left[\mathbf{d}_t[i]^2\right] }
\leq
\sum_{t=1}^T \E\left[\boldsymbol{\xi}_t[i]^2\right].
\]
Summing over the active coordinates, we obtain
\begin{align*}
0
\leq
%\underset{=- \E \left[ \Phi_{1:T}(\boldsymbol{\lambda}^*) \right]}{
\underbrace{
\sum_{i\in\mathfrak{I}}
\frac{\left( \sum_{t=1}^T \E\left[\mathbf{g}_t[i]\mathbf{d}_t[i]\right] \right)^2}{ \sum_{t=1}^T \E\left[\mathbf{d}_t[i]^2\right] }
}_{= - \E \left[ \Phi_{1:T}(\boldsymbol{\lambda}^*) \right]}
%}
\leq
\sum_{i\in\mathfrak{I}}
\sum_{t=1}^T
\E\left[\boldsymbol{\xi}_t[i]^2\right]
\leq
\sum_{t=1}^T
\E\left[\|\boldsymbol{\xi}_t\|^2\right]
\leq T\sigma^2.
\end{align*}
This shows that 
$  - T \sigma^2 \leq \E \left[ \Phi_{1:T}(\boldsymbol{\lambda}^*) \right] \leq 0.$

\end{proof}

\clearpage

\section{Proofs of the Theoretical Results in Section~\ref{sec:zeroth_order_stochastic_optimization}} \label{app:C}

We will need a series of lemmas first. In the following, we denote $$\delta_t:= \mathcal{Q}_{t,\bw_t}^{\boldsymbol{\lambda}_t} - \nabla \loss_{\mu}(\bw_t)$$ and define
%\begin{equation*}
$\loss_{\mu}(\bw) := \mathbb{E}_{\zeta} \left[ \loss( \bw + \mu \zeta ) \right].$
%\end{equation*}
Additionally, we denote
\begin{align*}
    \mathbf{q}_t := \qt{y}{\bw_t}, 
    \quad  \quad 
    \mathbf{d}^0_t := \tildeqt{f}{\bw_t} - \qt{f}{\bw_t}.
\end{align*}
We note that for the scalar case, one can let $\boldsymbol{\lambda}_t \gets \lambda_t \odot \mathbf{1}$, where $\mathbf{1}$ is the all-one-vector; consequently one can write
$\mathbf{q}_t + \boldsymbol{\lambda}_t \odot \mathbf{d}_t^0 = \mathbf{q}_t  + \lambda_t \mathbf{d}_t^0$. Hence, the following lemmas also apply to the case of scalar tuning. 

\begin{mdframed}[backgroundcolor=black!10,rightline=false,leftline=false,topline=false,bottomline=false]
\begin{lemma}(Adapted from the proof of Theorem 3.2 (a) in \cite{ghadimi2013stochastic}) \label{lem:zeroth}
Consider $\bw_{t+1} = \bw_t - \eta \mathcal{Q}_{t,\bw_t}^{\boldsymbol{\lambda}_t}.$
Under Assumption~\ref{A1} and Assumption~\ref{A2},
it holds that
$$ \eta \sum_{t=1}^T  \left \| \nabla  \loss(\bw_{t}) \right \|^2
\leq
2 \Delta_1 + 2\mu^2 \beta d 
+  \frac{\mu^2}{2} \beta^2 (d+3)^3 \eta T
 - 2 \eta \sum_{t=1}^T  \langle \nabla \loss_{\mu}(\bw_t),  \delta_t \rangle +  \beta \eta^2  \sum_{t=1}^T \left \| \mathcal{Q}_{t,\bw_t}^{\boldsymbol{\lambda}_t}\right \|^2.
$$
\end{lemma}
\end{mdframed}

\begin{proof}
Denote $\Delta^{\mu}_1:= \loss_{\mu}(\bw_1) - \inf_{\bw} \loss_{\mu}(\bw)$. 
We have for every $t \in [T]$,
\begin{align}
\loss_{\mu}(\bw_{t+1}) & \leq \loss_{\mu}(\bw_{t})
- \eta \langle \nabla \loss_{\mu}(\bw_t), \mathcal{Q}_{t,\bw_t}^{\boldsymbol{\lambda}_t} \rangle
+ \frac{\beta}{2} \eta^2 \left \| \mathcal{Q}_{t,\bw_t}^{\boldsymbol{\lambda}_t} \right \|^2 
\\ &
=  \loss_{\mu}(\bw_{t})
- \eta \left \|  \nabla \loss_{\mu}(\bw_t) \right \|^2
- \eta \langle \nabla \loss_{\mu}(\bw_t),  \delta_t \rangle
+ \frac{\beta}{2} \eta^2 \left \| \mathcal{Q}_{t,\bw_t}^{\boldsymbol{\lambda}_t} \right \|^2. 
\end{align}
A telescoping sum of the inequalities for all $t \in [T]$ above leads to
\begin{align}
\sum_{t=1}^T \eta \left \| \nabla  \loss_{\mu}(\bw_{t}) \right \|^2
& \leq \Delta^{\mu}_1 - \sum_{t=1}^T \eta \langle \nabla \loss_{\mu}(\bw_t),  \delta_t \rangle + \frac{\beta}{2} \sum_{t=1}^T \eta^2  \left \| \mathcal{Q}_{t,\bw_t}^{\boldsymbol{\lambda}_t}  \right \|^2 \notag
\\ &
\leq \Delta_1 + \mu^2 \beta d  - \sum_{t=1}^T \eta \langle \nabla \loss_{\mu}(\bw_t),  \delta_t \rangle + \frac{\beta}{2} \sum_{t=1}^T \eta^2  \left \| \mathcal{Q}_{t,\bw_t}^{\boldsymbol{\lambda}_t} \right \|^2, \label{j1}
\end{align}
where the last inequality follows from the known result in the literature that connects the initial gap of the smoothed objective and the initial gap of the original objective function, i.e., 
$\Delta^{\mu}_1 \leq \Delta_1 + \mu^2 \beta d $, see e.g., (3.11) in \citet{ghadimi2013stochastic}.

Furthermore, (3.9) in Theorem 3.1 of \citet{ghadimi2013stochastic} provides the following lower bound relating the gradient norm of the smoothed objective to that of the original objective:
\begin{align}
\frac{1}{2} \left( \| \nabla \loss(\bw_t) \|^2 - \frac{\mu^2}{2} \beta^2 (d+3)^3 \right)
\leq  \| \nabla \loss_\mu(\bw_t) \|^2. \label{j2}
\end{align}
Combining \oureqref{j1} and \oureqref{j2}, we obtain
\begin{align}
\eta \sum_{t=1}^T  \left \| \nabla  \loss(\bw_{t}) \right \|^2
\leq
2 \Delta_1 + 2\mu^2 \beta d 
+  \frac{\mu^2}{2} \beta^2 (d+3)^3 \eta T
 - 2 \eta \sum_{t=1}^T  \langle \nabla \loss_{\mu}(\bw_t),  \delta_t \rangle +  \beta \eta^2  \sum_{t=1}^T \left \| \mathcal{Q}_{t,\bw_t}^{\boldsymbol{\lambda}_t}  \right \|^2.
\end{align}

\end{proof}

\begin{mdframed}[backgroundcolor=black!10,rightline=false,leftline=false,topline=false,bottomline=false]
\begin{lemma} \label{lem:unbiasedness}
Assume that $\boldsymbol{\lambda}_t$ and $\bw$ are $\mathcal{F}_{t-1}$-measurable.
Assume also that 
the labeled observations in $B_t$ and the unlabeled covariates in
$\widetilde B_t$ have the same covariate marginal, and that $\zeta_t \sim \mathcal{N}(0,I_d)$ is independent of $\mathcal{H}_t:=\sigma(\mathcal{F}_{t-1},B_t,\widetilde B_t)$.
Then, $\E_{t-1}\left[ \mathcal{Q}_{t,\bw}^{\boldsymbol{\lambda}_t} \right] = \nabla \loss_{\mu}(\bw)$.
\end{lemma}
\end{mdframed}

\begin{proof}
We have
\begin{align}
\mathbb{E}_{t-1}\left[\qt{y}{\bw}\right]
 &=
\mathbb{E}_{\zeta_t\sim \mathcal{N}(0,I_d)}
\left[
\mathbb{E}_{(\bx,y)}
\left[
\frac{
\ell(\bw+\mu\zeta_t;\bx,y)-\ell(\bw;\bx,y)
}{\mu} \zeta_t
\middle|  \zeta_t 
\right]
\right]  \notag
\\
&=
\mathbb{E}_{\zeta_t\sim \mathcal{N}(0,I_d)}
\left[
\frac{
\loss(\bw+\mu\zeta_t)-\loss(\bw)
}{\mu} \zeta_t
\right]  \notag
\\
&=
\nabla\loss_\mu(\bw),\label{gaussian-smooth}
\end{align}
where \oureqref{gaussian-smooth} follows from the standard identity for the Gaussian zeroth-order gradient estimator; see, e.g., (21) in \citet{nesterov2017random}.

At the same time, we also have
\begin{align}
\mathbb{E}_{t-1}\left[\qt{f}{\bw}\right]
 &=
\mathbb{E}_{\zeta_t\sim \mathcal{N}(0,I_d)}
\left[
\mathbb{E}_{\bx}
\left[
\frac{
\ell(\bw+\mu\zeta_t;\bx,f(\bx))-\ell(\bw;\bx,f(\bx))
}{\mu} \zeta_t
\middle|  \zeta_t 
\right]
\right]  \notag
\\
&=
\mathbb{E}_{\zeta_t\sim \mathcal{N}(0,I_d)}
\left[
\frac{
\loss^f(\bw+\mu\zeta_t)-\loss^f(\bw)
}{\mu} \zeta_t
\right]  \notag
\\
&=
\nabla\loss^f_\mu(\bw), \notag
\end{align}
where we denote $\loss^f(\bw) := \E_{\bx}[ \ell(\bw;\bx, f(\bx)) ]$ and 
$\loss^f_{\mu}(\bw)$ is the objective function after smoothing.
Furthermore, under the assumption that the labeled observations in $B_t$ and the unlabeled covariates in $\widetilde B_t$ have the same covariate marginal,
we also have
\begin{align}
\mathbb{E}_{t-1}\left[\tildeqt{f}{\bw}\right]
 &=
\mathbb{E}_{\zeta_t\sim \mathcal{N}(0,I_d)}
\left[
\mathbb{E}_{\tilde{\bx}}
\left[
\frac{
\ell(\bw+\mu\zeta_t;\tilde{\bx},f(\tilde{\bx}))-\ell(\bw;\tilde{\bx},f(\tilde{\bx}))
}{\mu} \zeta_t
\middle|  \zeta_t 
\right]
\right]  \notag
\\
&=
\mathbb{E}_{\zeta_t\sim \mathcal{N}(0,I_d)}
\left[
\frac{
\loss^f(\bw+\mu\zeta_t)-\loss^f(\bw)
}{\mu} \zeta_t
\right]  \notag
\\
&=
\nabla\loss^f_\mu(\bw). \notag
\end{align}

Combining all the above, we obtain
\begin{equation*}
\mathbb E_{t-1}\left[\mathcal{Q}_{t,\bw}^{\boldsymbol{\lambda}_t}\right]
=\nabla\loss_\mu(\bw) + 
\boldsymbol{\lambda}_t \odot \left( \nabla\loss^f_\mu(\bw) -  \nabla\loss^f_\mu(\bw) \right)
= \nabla\loss_\mu(\bw).
\end{equation*}

\end{proof}

\begin{mdframed}[backgroundcolor=black!10,rightline=false,leftline=false,topline=false,bottomline=false]
\begin{lemma} \label{lem:conditional2}
Follow the assumptions and notations in Lemma~\ref{lem:zeroth} and Lemma~\ref{lem:unbiasedness}, it holds that
\begin{align} \label{eq: zo-scalar-grad-bound}
    \eta \sum_{t=1}^T \E\left[ \left \| \nabla  \loss(\bw_{t}) \right \|^2 \right]
    \leq
    2 \Delta_1 + 2\mu^2 \beta d 
    +  \frac{\mu^2}{2} \beta^2 (d+3)^3 \eta T
    +  \beta \eta^2  \sum_{t=1}^T
     \E\left[ \left \| \mathcal{Q}_{t,\bw_t}^{\boldsymbol{\lambda}_t}  \right \|^2 \right].
\end{align}
\end{lemma}
\end{mdframed}

\begin{proof}
Taking expectation on both sides of the inequality in Lemma~\ref{lem:zeroth} and noting that, by Lemma~\ref{lem:unbiasedness}, $\mathbb \E[ \langle \nabla \loss_{\mu}(\bw_t), \delta_t \rangle ] = 0$,
we obtain
\begin{align*}
 \eta \sum_{t=1}^T \E\left[ \left \| \nabla  \loss(\bw_{t}) \right \|^2 \right]
\leq
2 \Delta_1 + 2\mu^2 \beta d 
+  \frac{\mu^2}{2} \beta^2 (d+3)^3 \eta T
+  \beta \eta^2  \sum_{t=1}^T
 \E\left[ \left \| \mathcal{Q}_{t,\bw_t}^{\boldsymbol{\lambda}_t}   \right \|^2 \right].
\end{align*}

\end{proof}

\begin{mdframed}[backgroundcolor=black!10,rightline=false,leftline=false,topline=false,bottomline=false]
\begin{lemma} \label{lem:exp_q_and_diff_q_and_bar_q}
    Let $\zeta_t\sim N(0,I_d)$ be a common Gaussian perturbation used in every function-value difference appearing in $\q_{t,\bw_t}^{\boldsymbol{\lambda}}$ at iteration $t$, and assume that $\zeta_t$ is independent of $\mathcal{H}_t := \sigma(\mathcal{F}_{t-1},B_t,\widetilde B_t)$. Define $\overline{\mathbf{q}}_t := 
    \E \left[ \mathbf{q}_t \big| \F_{t-1}, \zeta_t \right]$. Then, under Assumptions~\ref{A2} and \ref{A3}, for every $t \in [T]$,
    \begin{align*}
        \E\left[\|\mathbf{q}_t\|^2\right]
        \leq
        2(d+2)\E\left[\|\mathbf{g}_t\|^2\right]
        +
        \frac{\beta^2\mu^2}{2}d(d+2)(d+4),
    \end{align*}
    and
    \begin{align*}
        \E \left[ \left\| \mathbf{q}_t - \overline{\mathbf{q}}_t \right\|^2 \right]
        \leq 
        2 (d+2) \sigma^2
        + \frac{\beta^2 \mu^2}{2} d (d+2) (d+4).
    \end{align*}
\end{lemma}
\end{mdframed}

\begin{proof}
    Define the empirical loss function
    \begin{align*}
    \widehat L_{t,y}(\bw) &:= \frac{1}{|B_t|} \sum_{(\bx,y)\in B_t} \ell(\bw;\bx,y),
    \end{align*}
    whose gradient satisfies
    \begin{align*}
        \nabla \widehat L_{t,y}(\bw) = \frac{1}{|B_t|} \sum_{(\bx,y)\in B_t} \nabla \ell(\bw;\bx,y) = \gt{y}{\bw}.
    \end{align*}
    Then, by definition of $\mathbf{q}_t$, 
    \begin{align*}
        \mathbf{q}_t
        = \qt{y}{\bw_t} 
        = \frac{1}{ |B_t| } \sum_{(\bx, y) \in B_t}
    \frac{\ell\left(\bw_t + \mu \zeta_t; \bx,y \right) - \ell(\bw_t; \bx, y )  }{ \mu } \zeta_t
        = \frac{ \widehat L_{t,y}(\bw_t+\mu\zeta_t)
        -
        \widehat L_{t,y}(\bw_t)}{\mu} \zeta_t.
    \end{align*}
    By the integral form of Taylor's theorem (e.g., Theorem 2.1 in \citet{wright_recht_2022}), we have
    \begin{align*}
    \widehat L_{t,y}(\bw_t+\mu\zeta_t)
    -
    \widehat L_{t,y}(\bw_t)
    &=
    \mu \left\langle \nabla \widehat L_{t,y}(\bw_t),\zeta_t \right\rangle + 
    \underbrace{
    \int_0^1
    \left\langle
        \nabla \widehat L_{t,y}(\bw_t+s\mu\zeta_t)
        -
        \nabla \widehat L_{t,y}(\bw_t),
        \mu\zeta_t
    \right\rangle \,ds}_{:=R_{t,y}} \\
    &=
    \mu \left\langle  \gt{y}{\bw_t} ,\zeta_t \right\rangle + 
    R_{t,y} \\
    &= 
    \mu \left\langle  \mathbf{g}_t ,\zeta_t \right\rangle + 
    R_{t,y} .
    \end{align*}
    Therefore,
    \begin{align}
        \mathbf{q}_t = \left\langle \mathbf{g}_t,\zeta_t \right\rangle \zeta_t + \frac{R_{t,y}}{\mu} \zeta_t.
        \label{eq:q_expression}
    \end{align}
    Using the inequality $\| \mathbf{a}+ \mathbf{b}\|^2 \leq 2 \| \mathbf{a}\|^2 + 2 \| \mathbf{b}\|^2 $, for any two vectors $\mathbf{a}$ and $\mathbf{b}$, we obtain
    \begin{align*}
        \left\| \mathbf{q}_t \right\|^2 \leq 2 \left\langle \mathbf{g}_t,\zeta_t \right\rangle^2 \left\| \zeta_t \right\|^2 + \frac{2 R_{t,y}^2}{\mu^2}  \left\| \zeta_t \right\|^2 .
    \end{align*}
    Taking expectation,
    \begin{align}
        \E \left[ \left\| \mathbf{q}_t \right\|^2 \right] \leq 2 \E \left[ \left\langle \mathbf{g}_t,\zeta_t \right\rangle^2 \left\| \zeta_t \right\|^2 \right] + \frac{2}{\mu^2} \E \left[ R_{t,y}^2 \left\| \zeta_t \right\|^2 \right].
        \label{eq:lem_intermediate_exp_q}
    \end{align}
    We now bound the two terms separately. Let $\mathcal H_t := \sigma\!\left(\mathcal F_{t-1},B_t,\widetilde B_t\right).$ Conditional on $\mathcal H_t$, the stochastic gradient $\mathbf{g}_t$ is fixed, while the search direction remains independent and satisfies $\zeta_t\mid\mathcal H_t\sim\mathcal N(0,I_d)$. Let $\delta_{ij}$ denote the Kronecker delta, i.e., 
    \begin{align*}
        \delta_{ij} = 
        \begin{cases}
            1, & i=j, \\
            0, & i \neq j.
        \end{cases}
    \end{align*}
    Then, for standard Gaussian variables $\zeta \sim\mathcal N(0,I_d)$,  Isserlis's theorem (e.g., Theorem 1.28 in \cite{janson1997gaussian}) gives
    \begin{align}
        \E \Big[ \zeta[i] \zeta[j] \zeta[k] \zeta[l] \Big] 
        =
        \delta_{ij} \delta_{kl} + \delta_{ik} \delta_{jl} + \delta_{il} \delta_{jk}.
        \label{eq:isserliss_th}
    \end{align}
    Expanding the first term in \oureqref{eq:lem_intermediate_exp_q} and using the above identity for $\zeta_t\mid\mathcal H_t\sim\mathcal N(0,I_d)$, yields
    \begin{align*}
        \E \left[ \left\langle  \mathbf{g}_t ,\zeta_t \right\rangle^2 \left\| \zeta_t \right\|^2  \big| \mathcal H_t  \right] 
        &= \sum_{i,j,k=1}^d \mathbf{g}_t [i] \mathbf{g}_t [j]  \E \Big[ \zeta_t[i] \zeta_t[j] \zeta_t[k]^2 \big| \mathcal H_t \Big] \\
        &= \sum_{i,j,k=1}^d \mathbf{g}_t [i] \mathbf{g}_t [j] \left( \delta_{ij} + 2 \delta_{ik} \delta_{jk} \right) \\
        &= d \sum_{i=1}^d \mathbf{g}_t [i]^2 + 2 \sum_{i=1}^d \mathbf{g}_t [i]^2  \\
        &= (d+2) \left\| \mathbf{g}_t \right\|^2.
    \end{align*}
    Therefore, the first term in \oureqref{eq:lem_intermediate_exp_q} satisfies
    \begin{align}
        2 \E \left[ \left\langle \mathbf{g}_t ,\zeta_t \right\rangle^2 \left\| \zeta_t \right\|^2 \right] 
        &= 2 (d+2) \E \left[ \left\| \mathbf{g}_t  \right\|^2 \right] 
        \label{eq:first_term1}
    \end{align}
    We next bound the second term in \oureqref{eq:lem_intermediate_exp_q}. We observe that the remainder satisfies 
    \begin{align*}
        \left| R_{t,y} \right| 
        &= 
        \left| \int_0^1
        \left\langle
            \nabla \widehat L_{t,y}(\bw_t+s\mu\zeta_t)
            -
            \nabla \widehat L_{t,y}(\bw_t),
            \mu\zeta_t
        \right\rangle \,ds 
        \right| \notag \\
        &\leq  
        \int_0^1 \left|
        \left\langle
            \nabla \widehat L_{t,y}(\bw_t+s\mu\zeta_t)
            -
            \nabla \widehat L_{t,y}(\bw_t),
            \mu\zeta_t
        \right\rangle \right| \,ds \notag \\
        &\leq 
        \int_0^1 \left\| \nabla \widehat L_{t,y}(\bw_t+s\mu\zeta_t)
            -
            \nabla \widehat L_{t,y}(\bw_t) \right\| \mu \left\| \zeta_t \right\| ds \notag \\
        &\leq 
        \int_0^1 \beta s\mu\|\zeta_t\| \mu\|\zeta_t\| \,ds \notag \\
        &=
        \frac{\beta \mu^2}{2}
        \|\zeta_t\|^2,
    \end{align*}
    where the first inequality follows from the triangle inequality for integrals, the second from the Cauchy-Schwarz inequality, and the third holds since the function $\widehat L_{t,y}(\bw)$ has $\beta$-Lipschitz continuous gradients by Assumption~\ref{A2}. Squaring the above and multiplying both sides by $\left\| \zeta_t \right\|^2$ yields
    \begin{align*}
         R_{t,y}^2 \left\| \zeta_t \right\|^2 \leq \frac{\beta^2 \mu^4}{4}
        \|\zeta_t\|^6.
    \end{align*}
    Consequently, 
    \begin{align}
        \frac{2}{\mu^2} \E \left[ R_{t,y}^2 \left\| \zeta_t \right\|^2 \right] 
        &\leq 
        \frac{2}{\mu^2} \frac{\beta^2 \mu^4}{4} \E \left[ \|\zeta_t\|^6 \right] \notag \\
        &= \frac{\beta^2 \mu^2}{2} \E \left[ \|\zeta_t\|^6 \right] \notag \\
        &= \frac{\beta^2 \mu^2}{2} d (d+2) (d+4),
        \label{eq:first_term2}
    \end{align}
    where the last equality is obtained because $\|\zeta_t\|^2$ conditionally follows a chi-squared distribution with $d$ degrees of freedom, i.e., $\| \zeta_t\|^2 \mid\mathcal H_t \sim \chi_d^2$. Substituting \oureqref{eq:first_term1} and \oureqref{eq:first_term2} into \oureqref{eq:lem_intermediate_exp_q} gives
    \begin{align*}
        \E\left[\|\mathbf{q}_t\|^2\right]
        \leq
        2(d+2)\E\left[\|\mathbf{g}_t\|^2\right]
        +
        \frac{\beta^2\mu^2}{2}d(d+2)(d+4),
    \end{align*}
    which proves the first result. For the second result, we first observe that since $\E \left[ \mathbf{g}_t \big| \F_{t-1} \right] = \nabla \loss (\bw_t)$ we have
    \begin{align}
        \overline{\mathbf{q}}_t 
        &=  \E \left[ \mathbf{q}_t \big| \F_{t-1}, \zeta_t \right] \notag \\
        &= \E \left[\left\langle \mathbf{g}_t,\zeta_t \right\rangle \zeta_t + \frac{R_{t,y}}{\mu} \zeta_t \Big| \F_{t-1}, \zeta_t \right] \notag \\
        &=  \left\langle \nabla \loss (\bw_t) ,\zeta_t \right\rangle \zeta_t + \frac{\E \left[ R_{t,y} \big| \F_{t-1}, \zeta_t \right] }{\mu} \zeta_t.
        \label{eq:q_bar_expression}
    \end{align}
    Subtracting \oureqref{eq:q_bar_expression} from \oureqref{eq:q_expression} yields
    \begin{align*}
        \mathbf{q}_t - \overline{\mathbf{q}}_t = \left\langle  \mathbf{g}_t - \nabla\loss (\bw_t) ,\zeta_t \right\rangle \zeta_t + \frac{R_{t,y} - \E \left[ R_{t,y} \big| \F_{t-1}, \zeta_t \right] }{\mu} \zeta_t.
    \end{align*}
    Hence,
    \begin{align}
        \E \left[ \left\| \mathbf{q}_t - \overline{\mathbf{q}}_t \right\|^2 \right]
        \leq 2 \E \left[ \left\langle \mathbf{g}_t - \nabla \loss (\bw_t) ,\zeta_t \right\rangle^2 \left\| \zeta_t \right\|^2 \right] + \frac{2}{\mu^2} \E \left[ \left( R_{t,y} - \E \left[ R_{t,y} \big| \F_{t-1}, \zeta_t \right] \right)^2 \left\| \zeta_t \right\|^2 \right],
        \label{eq:exp_diff}
    \end{align}
    where the above inequality uses that $\| \mathbf{a}+ \mathbf{b}\|^2 \leq 2 \| \mathbf{a}\|^2 + 2 \| \mathbf{b}\|^2 $, for any two vectors $\mathbf{a}$ and $\mathbf{b}$.
    We now bound the two terms in \oureqref{eq:exp_diff} separately. 
    Let $\mathbf{v}_t :=  \mathbf{g}_t - \nabla \loss (\bw_t)$.
    Expanding the first term in \oureqref{eq:exp_diff} and using the identity \oureqref{eq:isserliss_th} for $\zeta_t\mid\mathcal H_t\sim\mathcal N(0,I_d)$, yields
    \begin{align*}
        \E \left[ \left\langle  \mathbf{v}_t ,\zeta_t \right\rangle^2 \left\| \zeta_t \right\|^2  \big| \mathcal H_t  \right] 
        &= \sum_{i,j,k=1}^d \mathbf{v}_t [i] \mathbf{v}_t [j]  \E \Big[ \zeta_t[i] \zeta_t[j] \zeta_t[k]^2 \big| \mathcal H_t \Big] \\
        &= \sum_{i,j,k=1}^d \mathbf{v}_t [i] \mathbf{v}_t [j] \left( \delta_{ij} + 2 \delta_{ik} \delta_{jk} \right) \\
        &= d \sum_{i=1}^d \mathbf{v}_t [i]^2 + 2 \sum_{i=1}^d \mathbf{v}_t [i]^2  \\
        &= (d+2) \left\| \mathbf{v}_t \right\|^2.
    \end{align*}
    Therefore, the first term in \oureqref{eq:exp_diff} satisfies
    \begin{align}
        2 \E \left[ \left\langle \mathbf{g}_t - \nabla \loss (\bw_t) ,\zeta_t \right\rangle^2 \left\| \zeta_t \right\|^2 \right] 
        &= 2 (d+2) \E \left[ \left\| \mathbf{g}_t - \nabla \loss (\bw_t) \right\|^2 \right] \notag \\
        &\leq 2 (d+2) \sigma^2,
        \label{eq:term1}
    \end{align}
    where the inequality follows from Assumption~\ref{A3}.
    We next bound the second term in \oureqref{eq:exp_diff}. We observe that
    \begin{align*}
        \E \left[ \left( R_{t,y} - \E \left[ R_{t,y} \big| \F_{t-1}, \zeta_t \right] \right)^2 \big| \F_{t-1}, \zeta_t  \right]
        &= \text{Var} \left( R_{t,y} \big| \F_{t-1}, \zeta_t \right) \\
        &\leq  \E \left[ R_{t,y}^2 \big| \F_{t-1}, \zeta_t  \right]
    \end{align*}
    Hence, 
    \begin{align*}
        \E \left[ \left( R_{t,y} - \E \left[ R_{t,y} \big| \F_{t-1}, \zeta_t \right] \right)^2 \left\| \zeta_t \right\|^2 \right]
        &\leq  \E \left[ R_{t,y}^2 \left\| \zeta_t \right\|^2  \right].
    \end{align*}
    Multiplying the above with $\frac{2}{\mu^2}$ and using \oureqref{eq:first_term2} we further obtain
    \begin{align}
        \frac{2}{\mu^2} \E \left[ \left( R_{t,y} - \E \left[ R_{t,y} \big| \F_{t-1}, \zeta_t \right] \right)^2 \left\| \zeta_t \right\|^2 \right]
        &\leq \frac{2}{\mu^2} \E \left[ R_{t,y}^2 \left\| \zeta_t \right\|^2  \right] \notag \\
        &\leq \frac{\beta^2 \mu^2}{2} d (d+2) (d+4).
         \label{eq:term2}
    \end{align}
     Substituting \oureqref{eq:term1} and \oureqref{eq:term2} into \oureqref{eq:exp_diff} yields
    \begin{align*}
        \E \left[ \left\| \mathbf{q}_t - \overline{\mathbf{q}}_t \right\|^2 \right]
        \leq 
        2 (d+2) \sigma^2
        + \frac{\beta^2 \mu^2}{2} d (d+2) (d+4),
    \end{align*}
    which proves the second result.
    \end{proof}

Similar to our results in stochastic first-order optimization, we consider updating $\boldsymbol{\lambda}_t$ as a game, where in each round $t$, the online learner
\begin{itemize}
\item Plays $\boldsymbol{\lambda}_t \in \reals^d$.
\item After that, receives the loss function $\psi_t(\boldsymbol{\lambda}):=  \left \| \Q{\bw_t}  \right \|^2= \left \|  \mathbf{q}_t +\boldsymbol{\lambda} \odot \mathbf{d}^0_t  \right \|^2 $.
\item Incurs loss $\psi_t(\boldsymbol{\lambda}_t)$.
\end{itemize}
Therefore, below, we denote:
$$ \mathrm{Regret}^{\psi}_T(\boldsymbol{\lambda}^*):= \sum_{t=1}^T \psi_t(\boldsymbol{\lambda}_t) - \sum_{t=1}^T \psi_t(\boldsymbol{\lambda}^*),$$
where $\boldsymbol{\lambda}^* \in \reals^d$ is a comparator.
We reiterate that the case of vector tuning covers the case of scalar tuning as a special case. Specifically, given a scalar tuning parameter $\lambda \in \reals$, a scalar-vector multiplication $\lambda \bv$ can be rewritten as $\lambda \bv = \lambda \mathbf{1} \odot \bv$, where $\mathbf{1} \in \reals^d$ is the all-ones vector.

Similar to the first-order case, the FTL strategy under scalar tuning solves 
the following:
$$   \lambda_t \gets \underset{ \lambda \in \reals }{\arg\min}  \sum_{s=1}^{t-1}  \psi_s(\lambda), \text{ where } \psi_s(\lambda):= \left \|  \mathbf{q}_s + \lambda \mathbf{d}^0_s  \right \|^2,$$
 which gives
\begin{align*}
    \lambda_{t} = 
    \begin{cases}
        - \frac{\sum_{s=1}^{t-1} \left\langle \mathbf{q}_s , \mathbf{d}_s^0   \right\rangle}{\sum_{s=1}^{t-1}  \left\| \mathbf{d}_s^0 \right\|^2} , & \sum_{s=1}^{t-1}  \left\| \mathbf{d}_s^0 \right\|^2 >0, \\
        0,  & \sum_{s=1}^{t-1}  \left\| \mathbf{d}_s^0 \right\|^2=0.
    \end{cases}
\end{align*}
Additionally, define
$\Psi_{1:T} \left( \lambda \right) 
    :=  \sum_{t=1}^T \left \| \q_{t,\bw_t}^{\lambda} \right \|^2 -
    \sum_{t=1}^T \left \| \q_{t,\bw_t}^{0} \right \|^2.$
We are now ready to prove Theorem~\ref{th: zeroth-order-scalar}.

\begin{mdframed}[backgroundcolor=black!10,rightline=false,leftline=false,topline=false,bottomline=false]
\noindent
\textbf{Theorem~\ref{th: zeroth-order-scalar}.}
\textit{
Denote $D_0:=
\sup_{T\ge1}\mathbb E\!\left[\max_{1\le t\le T+1}|\lambda_t|^2\right]$,
and $\bar{\Xi}_T:=\frac{1}{T} \sum_{t=1}^T \E\left[ \left \| \nabla  \loss(\bw_{t}) \right \|^2 \right] $. 
Then, under Assumptions~\ref{A1}-\ref{A4}, 
Algorithm~\ref{alg:sgfm} with scalar tuning, $\eta = \min \left\{ \frac{1}{4 \beta (d+2)}, \frac{1}{4 \beta \sqrt{(d+2) T}}   \right\}$, and $\mu=\frac{1}{\beta (d+3) \sqrt{T}}$ satisfies
    \begin{align*}
\textstyle  \bar{\Xi}_T
        &\textstyle  \leq 
        16 \beta \Delta_1 \left( \frac{d+2}{T} + \frac{\sqrt{d+2}}{\sqrt{T}} \right) + \frac{16 d}{(d+3)^2} \left( \frac{d+2}{T^{2}} + \frac{\sqrt{d+2}}{T^{3/2}} \right) + \frac{d+3}{T} \\
        &\textstyle \quad + \frac{16G^2(1+2D_0)\bigl(d+4\log(eT)\bigr)^2(4+\log T)}
        {\sqrt{d+2}\,T^{3/2}} + \frac{\sqrt{d+2}}{\sqrt{T}} \left( \sigma^2 +\frac{ d (d+4) }{4 (d+3)^2 T}  + \frac{\E\left[  \Psi_{1:T} \left( \lambda^* \right) \right]}{2(d+2)T} \right),
    \end{align*}
where
$\E \left[ \Psi_{1:T}(\lambda^*) \right] =
- \frac{\left( \sum_{t=1}^T \E \left[ \left\langle \mathbf{q}_t , \mathbf{d}_t^0   \right\rangle \right] \right)^2 }{\sum_{t=1}^T \E \left[ \left\| \mathbf{d}_t^0 \right\|^2 \right]}$
if \, $\sum_{t=1}^T \E\left[\left\|\mathbf{d}_t^0\right\|^2\right]>0$,
and $\E \left[ \Psi_{1:T}(\lambda^*) \right] =0$ otherwise. In particular, it holds that
\begin{align*}
\textstyle 
-2 (d+2) T \sigma^2
- \frac{d (d+2) (d+4)}{2 (d+3)^2} 
\leq \E[ \Psi_{1:T}(\lambda^*) ] \leq 0.
\end{align*}
}
\end{mdframed}

\begin{proof}
    From Lemma \ref{lem:conditional2} we have
    \begin{align}
        \eta \sum_{t=1}^T \E\left[ \left \| \nabla  \loss(\bw_{t}) \right \|^2 \right]
        \leq
        2 \Delta_1 + 2\mu^2 \beta d 
        +  \frac{\mu^2}{2} \beta^2 (d+3)^3 \eta T
        +  \beta \eta^2  \sum_{t=1}^T
         \E\left[ \left \| \Qt{\bw_t}  \right \|^2 \right]. \label{eq: bound_zo_scalar}
    \end{align}
    It remains to bound $\sum_{t=1}^T \E\left[ \left \| \Qt{\bw_t}  \right \|^2 \right]$. We observe that for any comparator $\lambda^* \in \reals$ we have
    \begin{align}
        &\E\left[
        \sum_{t=1}^T  \left \| \Qt{\bw_t}  \right \|^2 \right] \notag \\
        & =
        \E
        \left[ 
        \sum_{t=1}^T \left \| \Qt{\bw_t}  \right \|^2 
        -
        \sum_{t=1}^T \left \| \q_{t,\bw_t}^{\lambda^*}  \right \|^2 
        \right]
        +  \E\left[ \sum_{t=1}^T \left \| \q_{t,\bw_t}^{\lambda^*} \right \|^2 \right]
        \notag \\ 
        &= 
        \E\left[ \mathrm{Regret}^{\psi}_T(\lambda^*) \right]
        + \sum_{t=1}^T \E\left[ \left \| \q_{t,\bw_t}^{\lambda^*} \right \|^2 \right] \notag \\
        &= 
        \E\left[ \mathrm{Regret}^{\psi}_T(\lambda^*) \right]
        + \sum_{t=1}^T \E\left[ \left \|  \mathbf{q}_t + \lambda^* \mathbf{d}^0_t  \right \|^2 \right] \notag \\
        &= 
        \E\left[ \mathrm{Regret}^{\psi}_T(\lambda^*) \right]
        + \sum_{t=1}^T \E\left[ \left \|  \mathbf{q}_t \right \|^2 \right]
        +  \E\left[ \underbrace{ \sum_{t=1}^T \left \|  \mathbf{q}_t + \lambda^* \mathbf{d}^0_t  \right \|^2 -  \sum_{t=1}^T  \left \|  \mathbf{q}_t \right \|^2 }_{:= \Psi_{1:T} \left( \lambda^* \right) }  \right]
        \label{eq:zo_scalar_1}
    \end{align}
    Furthermore, by Lemma~\ref{lem:exp_q_and_diff_q_and_bar_q} we have
    \begin{align}
        \sum_{t=1}^T \mathbb E\left[  \left\| \mathbf{q_t} \right\|^2 \right] \leq 2 (d+2) \sum_{t=1}^T \E \left[ \left\| \mathbf{g_t} \right\|^2 \right] + 
        \frac{\beta^2 \mu^2}{2} d \left(d+2\right) \left(d+4\right) T.
        \label{eq:zo_scalar_2}
    \end{align}
    Combining \oureqref{eq:zo_scalar_1} with \oureqref{eq:zo_scalar_2}, we obtain
    \begin{align}
        \E\left[\sum_{t=1}^T  \left \| \Qt{\bw_t}  \right \|^2 \right] 
        &\leq 
        \E\left[ \mathrm{Regret}^{\psi}_T(\lambda^*) \right]
        + 2(d+2) \sum_{t=1}^T \mathbb \E \left[ \left\| \mathbf{g_t} \right\|^2 \right] \notag \\
        &\quad + \frac{\mu^2\beta^2}{2} d(d+2)(d+4) T + \E \left[ \Psi_{1:T} \left( \lambda^* \right) \right] \notag \\
        &= \E\left[ \mathrm{Regret}^{\psi}_T(\lambda^*) \right]
        + 2(d+2) \sum_{t=1}^T \mathbb E\left[ \left\| \nabla \mathcal{L} \left( \bw_t \right) \right\|^2 \right] \notag \\
        &\quad + 2(d+2) \sum_{t=1}^T \mathbb E\left[ \left\| \mathbf{g}_t - \nabla \mathcal{L} \left( \bw_t \right) \right\|^2 \right] 
        \notag \\
        & \quad + \frac{\mu^2\beta^2}{2} d(d+2)(d+4) T + \E \left[ \Psi_{1:T} \left( \lambda^* \right) \right] \notag \\
        &\leq \E\left[ \mathrm{Regret}^{\psi}_T(\lambda^*) \right]
        + 2(d+2) \sum_{t=1}^T \mathbb E\left[ \left\| \nabla \mathcal{L} \left( \bw_t \right) \right\|^2 \right] + 2(d+2) \sigma^2 T 
        \notag \\
        & \quad + \frac{\mu^2\beta^2}{2} d(d+2)(d+4) T + \E \left[ \Psi_{1:T} \left( \lambda^* \right) \right], \label{eq: zero-order-grad-bound}
    \end{align}
    where the last inequality follows from Assumption~\ref{A3}. We next bound $\E\left[ \mathrm{Regret}_T^{\psi}\left(\lambda^* \right) \right]$.
    From Lemma \ref{lem: shalev2025online} we have
    \begin{align}
        \mathrm{Regret}_T^{\psi}\left(\lambda^* \right):= \sum_{t=1}^{T} \psi_t(\lambda_{t}) - \sum_{t=1}^T \psi_t(\lambda^*)
        \leq \sum_{t=1}^{T} \left( \psi_t(\lambda_{t}) - \psi_t(\lambda_{t+1}) \right).
        \label{eq: lemma5_th2}
    \end{align}
    By FTL with scalar tuning we have
    \begin{align*}
         &\lambda_t \left( \sum_{s=1}^{t-1} \left\| \mathbf{d}_s^0 \right\|^2 \right)
        = -\sum_{s=1}^{t-1} \langle \mathbf{q}_s, \mathbf{d}_s^0 \rangle ,\\
        &\lambda_{t+1} \left( \sum_{s=1}^{t} \left\| \mathbf{d}_s^0 \right\|^2 \right) = - \sum_{s=1}^{t} \langle \mathbf{q}_s, \mathbf{d}_s^0 \rangle.
    \end{align*}  
    Subtracting the above equations and rearranging gives
    \begin{align*}
        \left(\sum_{s=1}^{t}\|\mathbf d_s^0\|^2\right)
        (\lambda_t-\lambda_{t+1})
        =
        \langle\mathbf q_t,\mathbf d_t^0\rangle
        +\lambda_t\|\mathbf d_t^0\|^2.
    \end{align*}
    Furthermore, we recall that 
    \begin{align*}
        \psi_t(\lambda):= \left \|  \mathbf{q}_t + \lambda \mathbf{d}_t^0  \right \|^2 = \left \|  \mathbf{q}_t \right \|^2 + 2 \lambda \langle \mathbf{q}_t, \mathbf{d}_t^0 \rangle + \lambda^2 \left \| \mathbf{d}_t^0  \right \|^2.
    \end{align*}
    Hence, proceeding as in the proof of Theorem \ref{th: first-order-scalar}, we obtain
    \begin{align*}
        \psi_t(\lambda_t)-\psi_t(\lambda_{t+1})
        =
        \left(2\sum_{s=1}^t\|\mathbf d_s^0\|^2-\|\mathbf d_t^0\|^2\right)
        (\lambda_t-\lambda_{t+1})^2.
    \end{align*}
    Applying \(v\le v^2/(\rho+v)+\rho\) with \(v=\sum_{s=1}^t\|\mathbf d_s^0\|^2\) yields, for every \(\rho>0\),
    \begin{align*}
    \psi_t(\lambda_t)-\psi_t(\lambda_{t+1})
    &\le
    2\frac{\bigl(\langle\mathbf q_t,\mathbf d_t^0\rangle
                 +\lambda_t\|\mathbf d_t^0\|^2\bigr)^2}
            {\rho+\sum_{s=1}^t\|\mathbf d_s^0\|^2}
    +2\rho(\lambda_t-\lambda_{t+1})^2\\
    &\le
    4\|\mathbf q_t\|^2
    \frac{\|\mathbf d_t^0\|^2}{\rho+\sum_{s=1}^t\|\mathbf d_s^0\|^2}
    +
    4\lambda_t^2
    \frac{\|\mathbf d_t^0\|^4}{\rho+\sum_{s=1}^t\|\mathbf d_s^0\|^2}
    +
    2\rho(\lambda_t-\lambda_{t+1})^2
    \end{align*}
    where the second inequality follows from $(a+b)^2 \leq 2a^2 + 2b^2 $, and the Cauchy-Schwarz inequality.
    Summing over $t=1, \ldots, T$ and using \oureqref{eq: lemma5_th2} gives
    \begin{align}
        \mathrm{Regret}_T^\psi(\lambda^*)
        &\leq
        \sum_{t=1}^T \left( \psi_t \left( \lambda_t \right) - \psi_t \left( \lambda_{t+1} \right)  \right) \notag \\
        &\leq 
        4 \sum_{t=1}^T \frac{\|\mathbf q_t\|^2 \|\mathbf d_t^0\|^2}{\rho+\sum_{s=1}^t\|\mathbf d_s^0\|^2} 
        +
        4 \sum_{t=1}^T \lambda_t^2\frac{\|\mathbf d_t^0\|^4}{\rho+\sum_{s=1}^t\|\mathbf d_s^0\|^2}
        + 2\rho \sum_{t=1}^T (\lambda_t-\lambda_{t+1})^2,
        \label{eq:reg_bound_th3_zo}
    \end{align}
    We will bound the three terms above separately. First, we observe that by Assumption~\ref{A4}, each sample loss is \(G\)-Lipschitz in \(w\), so for \(\mu>0\) we have
    \[
    \left|
    \frac{\ell(\mathbf w_t+\mu\zeta_t;x,y)-\ell(\mathbf w_t;x,y)}{\mu}
    \right|
    \le
    \frac{G\|(\mathbf w_t+\mu\zeta_t)-\mathbf w_t\|}{\mu}
    =G\|\zeta_t\|.
    \]
    Since the same \(\zeta_t\) is used for every sample in \(B_t\), substituting the above into the definition of the zeroth-order minibatch gradient
    \[
    \begin{aligned}
    \|\mathbf q_t\|
    &=
    \left|
    \frac{1}{|B_t|}\sum_{(x,y)\in B_t}
    \frac{\ell(\mathbf w_t+\mu\zeta_t;x,y)-\ell(\mathbf w_t;x,y)}{\mu}
    \right|\|\zeta_t\|\\
    &\le
    \frac{1}{|B_t|}\sum_{(x,y)\in B_t}
    \left|
    \frac{\ell(\mathbf w_t+\mu\zeta_t;x,y)-\ell(\mathbf w_t;x,y)}{\mu}
    \right|\|\zeta_t\|\\
    &\le G\|\zeta_t\|^2.
    \end{aligned}
    \]
    Since Algorithm~\ref{alg:sgfm} uses the same direction \(\zeta_t\) for the two ML-labeled minibatch gradients whose difference defines \(\mathbf d_t^0\), we also obtain
    \begin{align*}
        \|\mathbf d_t^0\|\le 2G\|\zeta_t\|^2.
    \end{align*}
    Moreover, since \(\|\zeta_t\|^2\sim\chi_d^2\), its moment generating function implies \(\mathbb E[\exp(\|\zeta_t\|^2/4)]=2^{d/2}\). Hence, by a union bound and Markov’s inequality, for every \(u>0\),
    \[
    \mathbb P\!\left(\max_{1\le t\le T}\|\mathbf d_t^0\|>2Gu\right)
    \le \sum_{t=1}^T\mathbb P\!\left(\|\zeta_t\|^2>u\right)
    \le T\,2^{d/2}e^{-u/4}.
    \]
    Taking \(u=2d\log 2+4\log(T/\delta)\), we obtain, with probability at least \(1-\delta\),
    \[
    \max_{1\le t\le T}\|\mathbf d_t^0\|
    \le 2G\bigl(2d\log 2+4\log(T/\delta)\bigr),
    \qquad \delta\in(0,1).
    \]
    Next, define
    \begin{align*}
        b_T:=2d+8\log(eT),\qquad  I_T:=\{t\in[T]:\|\zeta_t\|^2\le b_T\}.
    \end{align*}
    By the Gaussian tail bound above, this choice ensures that, for every \(t\in[T]\),
    \[
    \mathbb P(t\notin I_T)
    \le 2^{d/2}e^{-b_T/4}\le T^{-2}.
    \]
    First, consider rounds $I_T$. On each round \(t\in I_T\), we have \(\|\mathbf q_t\|\le G b_T\) and \(\|\mathbf d_t^0\|\le 2G b_T\). Therefore, 
    \[
    \begin{aligned}
    \sum_{t\in I_T}
    \frac{\|\mathbf d_t^0\|^2}{\rho+\sum_{s=1}^t\|\mathbf d_s^0\|^2}
    &\le
    \sum_{t\in I_T}
    \int_{\rho+\sum_{s<t,\,s\in I_T}\|\mathbf d_s^0\|^2}
         ^{\rho+\sum_{s\le t,\,s\in I_T}\|\mathbf d_s^0\|^2}
    \frac{du}{u}\\
    &=
    \log\!\left(1+\frac{\sum_{t\in I_T}\|\mathbf d_t^0\|^2}{\rho}\right)
    \le
    \log\!\left(1+\frac{T 4G^2 b_T^2}{\rho}\right),
    \end{aligned}
    \]
    where we have used that \(1/u\) is decreasing for $u>0$. For the second term, using Lemma~\ref{lem: seq} with $q_t = \|\mathbf d_t^0\|^2\mathbf 1_{\{t\in I_T\}}$ and $Q = \max_{1 \leq t \leq T} q_t \leq  4G^2 b_T^2$ gives
    \[
    \begin{aligned}
    \sum_{t\in I_T}
    \frac{\|\mathbf d_t^0\|^4}{\rho+\sum_{s=1}^t\|\mathbf d_s^0\|^2}
    &\le
    \sum_{\substack{t\in I_T\\\|\mathbf d_t^0\|>0}}
    \frac{\|\mathbf d_t^0\|^4}
         {\sum_{s\le t,\,s\in I_T}\|\mathbf d_s^0\|^2}\\
    &\le 4G^2 b_T^2(1+\log T).
    \end{aligned}
    \]
    For the third term, we have
    \[
    2\rho\sum_{t\in I_T}(\lambda_t-\lambda_{t+1})^2
    \le 8\rho T\max_{1\le t\le T+1}|\lambda_t|^2.
    \]
    Combining these inequalities in \oureqref{eq:reg_bound_th3_zo}, taking \(\rho=4G^2 b_T^2/T\), and then taking expectations gives
    \begin{align}
    \mathbb E\!\left[
    \sum_{t\in I_T}\bigl(\psi_t(\lambda_t)-\psi_t(\lambda_{t+1})\bigr)
    \right]
    \le
    4G^2 b_T^2\log(1+T^2)+16G^2 b_T^2(3+\log T)D_0,
    \label{eq: sum_I_T}
    \end{align}
    where we have used $D_0:=\sup_{T\ge1}\mathbb E\!\left[\max_{1\le t\le T+1}|\lambda_t|^2\right]$. \\
    \noindent Next, consider the remaining rounds. For each $t \notin I_T$, since \(\psi_t(\lambda_{t+1})\ge0\), we have
    \begin{align*}
    \psi_t(\lambda_t)-\psi_t(\lambda_{t+1})
    &\le \|\mathbf q_t+\lambda_t\mathbf d_t^0\|^2\\
    &\le 2\|\mathbf q_t\|^2+2\lambda_t^2\|\mathbf d_t^0\|^2\\
    &\le 2G^2(1+4\lambda_t^2)\|\zeta_t\|^4.
    \end{align*}
    Additionally, by the choice of \(b_T\),
    \[
    \mathbb P(\|\zeta_t\|^2>b_T)
    \le 2^{d/2}e^{-b_T/4}\le T^{-2}.
    \]
    By the Cauchy–Schwarz inequality and the Gaussian eighth moment we have
    \[
    \begin{aligned}
    \mathbb E\!\left[\|\zeta_t\|^4\mathbf 1_{\{\|\zeta_t\|^2>b_T\}}\right]
    &\le
    \bigl(\mathbb E\|\zeta_t\|^8\bigr)^{1/2}
    \mathbb P(\|\zeta_t\|^2>b_T)^{1/2}\\
    &\le \frac{\sqrt{d(d+2)(d+4)(d+6)}}{T}\\
    &\le \frac{(d+3)^2}{T}.
    \end{aligned}
    \]
    Since \(\lambda_t\) depends only on the first \(t-1\) rounds and \(\zeta_t\) is sampled independently of that history, \(\lambda_t\) is independent of \(\|\zeta_t\|^4\mathbf1_{\{\|\zeta_t\|^2>b_T\}}\). Using independence and that \(\mathbb E[\lambda_t^2]\le D_0\), we obtain
    \begin{align}
    \mathbb E\!\left[
    \sum_{t\notin I_T}\bigl(\psi_t(\lambda_t)-\psi_t(\lambda_{t+1})\bigr)
    \right]
    &\le
    2G^2\sum_{t=1}^T
    (1+4\mathbb E[\lambda_t^2])
    \mathbb E\!\left[\|\zeta_t\|^4\mathbf 1_{\{\|\zeta_t\|^2>b_T\}}\right] \notag \\
    &\le 2G^2(1+4D_0)(d+3)^2.
    \label{eq: sum_not_I_T}
    \end{align}
    Adding \oureqref{eq: sum_I_T} and \oureqref{eq: sum_not_I_T}, we further obtain
    \begin{align*}
        \mathbb E\left[\mathrm{Regret}_T^\psi(\lambda^*) \right]
        &\leq \mathbb E\left[\sum_{t=1}^T \bigl(\psi_t(\lambda_t)-\psi_t(\lambda_{t+1})\bigr) \right] \\
        &\leq 4G^2 b_T^2 \log(1+T^2) +16 G^2 b_T^2(3+\log T)D_0 +2G^2(1+4D_0)(d+3)^2.
    \end{align*}
    Substituting \(b_T=2d+8\log(eT)\) and rearranging gives
    \begin{align}
    \mathbb E\!\left[\mathrm{Regret}_T^\psi(\lambda^*)\right]
    &\leq
    4G^2\bigl(2d+8\log(eT)\bigr)^2
    \Bigl[\log(1+T^2)+4D_0(3+\log T)\Bigr] +2G^2(1+4D_0)(d+3)^2 \notag \\
    &\leq 16G^2\bigl(d+4\log(eT)\bigr)^2
    \left[\log(1+T^2)+4D_0(3+\log T)+\frac{1+4D_0}{8}\right] \notag \\
    &\leq 32G^2(1+2D_0)\bigl(d+4\log(eT)\bigr)^2(4+\log T),
    \label{eq: zo-scalar-reg-bound}
    \end{align}
    where in the second inequality we have used that for \(T\ge1\), we have \(d+3\le d+4\log(eT)\), and the third inequality follows from \(\log(1+T^2)\le \log 2+2\log T\), \(\log 2+\tfrac18\le 8\), and \(3+\tfrac18\le 4\).
    Combining \oureqref{eq: zo-scalar-reg-bound}, \oureqref{eq: zero-order-grad-bound}, and \oureqref{eq: bound_zo_scalar} yields
    \begin{align}
         \left(\eta - 2 \beta \left(d+2\right) \eta^2 \right)  \sum_{t=1}^T \E\left[ \left \| \nabla  \loss(\bw_{t}) \right \|^2 \right]
        &\leq
        2 \Delta_1 + 2\mu^2 \beta d 
        +  \frac{\mu^2}{2} \beta^2 (d+3)^3 \eta T \notag \\
        & \quad + 32\beta\eta^2G^2(1+2D_0) \bigl(d+4\log(eT)\bigr)^2(4+\log T) \notag \\
        &\quad+ 2 \beta \eta^2 (d+2) \sigma^2 T 
        + \beta \eta^2 \E \left[ \Psi_{1:T} \left( \lambda^* \right) \right] \notag \\
        &\quad+  \frac{\mu^2\beta^3}{2}\eta^2 d(d+2)(d+4) T. \label{eq: zo-scalar-intermediate-bound}
    \end{align}
    We continue by analyzing $\E[ \Psi_{1:T}(\lambda^*) ]$. We observe that $\Psi_{1:T}(\lambda)$ can be expressed as
    \begin{align*}
         \Psi_{1:T}(\lambda) = \sum_{t=1}^T \left\| \mathbf{q}_t + \lambda \mathbf{d}_t^0  \right\|^2 -  \sum_{t=1}^T \left\| \mathbf{q}_t  \right\|^2 = 2\lambda \sum_{t=1}^T \left\langle \mathbf{q}_t , \mathbf{d}_t^0   \right\rangle
        +
        \lambda^2
        \sum_{t=1}^T  \left\| \mathbf{d}_t^0 \right\|^2.
    \end{align*}
    Therefore,
    \begin{align*}
         \E \left[ \Psi_{1:T}(\lambda) \right] = 2\lambda \sum_{t=1}^T \E \left[ \left\langle \mathbf{q}_t , \mathbf{d}_t^0   \right\rangle \right]
        +
        \lambda^2
        \sum_{t=1}^T \E \left[ \left\| \mathbf{d}_t^0 \right\|^2 \right].
    \end{align*}
    First, consider the case $\sum_{t=1}^T \E \left[ \left\| \mathbf{d}_t^0 \right\|^2 \right] > 0$. In this case, the coefficient of $\lambda^2$ is strictly positive, hence, $\E \left[ \Psi_{1:T}(\lambda) \right]$ is convex in $\lambda$. By setting
    $\frac{\partial}{\partial \lambda}
        \E \left[ \Psi_{1:T}(\lambda) \right]
        =
        0,$
    we obtain
    \begin{align*}
        \lambda^*
        = - \frac{ \sum_{t=1}^T \E \left[ \left\langle \mathbf{q}_t , \mathbf{d}_t^0   \right\rangle \right]
        }{\sum_{t=1}^T  \E \left[ \left\| \mathbf{d}_t^0 \right\|^2 \right]
        }.
    \end{align*}
    Substituting this expression into $\E \left[ \Psi_{1:T}(\lambda) \right]$ yields
    \begin{align*}
        \E \left[ \Psi_{1:T}(\lambda^*) \right]
        =
        -
        \frac{\left( \sum_{t=1}^T \E \left[ \left\langle \mathbf{q}_t , \mathbf{d}_t^0   \right\rangle \right] \right)^2 }{\sum_{t=1}^T \E \left[ \left\| \mathbf{d}_t^0 \right\|^2 \right]}
        \leq 0.
    \end{align*}
    Next, suppose that $\sum_{t=1}^T \E \left[ \left\| \mathbf{d}_t^0 \right\|^2 \right] = 0$. Since each term in the sum is non-negative, this implies $\mathbf{d}_t^0 = 0$ almost surely, for all $t\in [T]$. Consequently, $\left\langle \mathbf{q}_t , \mathbf{d}_t^0   \right\rangle = 0$ almost surely, for all $t\in [T]$, and thus
    $\E \left[ \Psi_{1:T}(\lambda) \right] = 0$, for all $\lambda \in \mathbb{R}$. In this case, we use the convention $\lambda^* = 0$. Therefore, we define the fixed comparator as the minimizer of $\E \left[ \Psi_{1:T}(\lambda) \right]$, given by
    \begin{align*}
        \lambda^*
        =
        \begin{cases}
        - \frac{ \sum_{t=1}^T \E \left[ \left\langle \mathbf{q}_t , \mathbf{d}_t^0   \right\rangle \right]
        }{\sum_{t=1}^T  \E \left[ \left\| \mathbf{d}_t^0 \right\|^2 \right]
        }, & \sum_{t=1}^T \E \left[ \left\| \mathbf{d}_t^0 \right\|^2 \right]>0,
        \\
        0, & \sum_{t=1}^T \E \left[ \left\| \mathbf{d}_t^0 \right\|^2 \right] = 0.
        \end{cases}
    \end{align*}
    The corresponding value of $\E \left[ \Psi_{1:T}(\lambda^*) \right]$ is
    \begin{align*}
        \E \left[ \Psi_{1:T}(\lambda^*) \right]
        =
        \begin{cases}
        -
        \frac{\left( \sum_{t=1}^T \E \left[ \left\langle \mathbf{q}_t , \mathbf{d}_t^0   \right\rangle \right] \right)^2 }{\sum_{t=1}^T \E \left[ \left\| \mathbf{d}_t^0 \right\|^2 \right]}, & \sum_{t=1}^T \E \left[ \left\| \mathbf{d}_t^0 \right\|^2 \right]>0,
        \\
        0, & \sum_{t=1}^T \E \left[ \left\| \mathbf{d}_t^0 \right\|^2 \right] = 0.
        \end{cases}
    \end{align*}
    Hence, we have $\E \left[ \Psi_{1:T}(\lambda^*) \right] \leq 0$. Next, we define
    \begin{align*}
        \overline{\mathbf{q}}_t := 
        \E \left[ \mathbf{q}_t \big| \F_{t-1}, \zeta_t \right] = \frac{\mathcal{L}\left( \mathbf{w}_t + \mu \zeta_t \right) - \mathcal{L}\left( \mathbf{w}_t\right)}{\mu} \zeta_t.
    \end{align*}
    Since $\E \left[ \mathbf{d}_t^0 \big| \F_{t-1}, \zeta_t  \right] = 0$, we have
    \begin{align*}
        \E \left[ \langle  \mathbf{q}_t , \mathbf{d}_t^0 \rangle \right] = \E \left[ \langle  \mathbf{q}_t - \overline{\mathbf{q}}_t  , \mathbf{d}_t^0 \rangle  \right] .
    \end{align*}
    By Cauchy-Schwarz
    \begin{align*}
        \left(\sum_{t=1}^T \E \left[\left\langle \mathbf{q}_t , \mathbf{d}_t^0   \right\rangle\right] \right)^2 
        = \left(\sum_{t=1}^T \E \left[\left\langle \mathbf{q}_t  - \overline{\mathbf{q}}_t  , \mathbf{d}_t^0   \right\rangle\right] \right)^2 
        \leq  \left( \sum_{t=1}^T \E \left[ \left\| \mathbf{q}_t - \overline{\mathbf{q}}_t  \right\|^2 \right] \right) \left( \sum_{t=1}^T \E \left[ \left\| \mathbf{d}_t^0  \right\|^2 \right] \right).
    \end{align*}
    Therefore, 
    \begin{align*}
        \E[ \Psi_{1:T}(\lambda^*) ] \geq - \sum_{t=1}^T \E \left[ \left\| \mathbf{q}_t - \overline{\mathbf{q}}_t  \right\|^2 \right].
    \end{align*}
    Additionally, by Lemma~\ref{lem:exp_q_and_diff_q_and_bar_q}, we have
    \begin{align*}
        \sum_{t=1}^T \E \left[ \left\| \mathbf{q}_t - \overline{\mathbf{q}}_t \right\|^2 \right]
        \leq 
        2 (d+2) T \sigma^2
        + \frac{\beta^2 \mu^2}{2} d (d+2) (d+4) T.
    \end{align*}
    Consequently, we obtain the following bounds for $\E[ \Psi_{1:T}(\lambda^*) ]$,
    \begin{align}
        -2 (d+2) T \sigma^2
        - \frac{\beta^2 \mu^2}{2} d (d+2) (d+4) T 
        \leq \E[ \Psi_{1:T}(\lambda^*) ] \leq 0.
        \label{eq:bounds_exp_psi}
    \end{align}
    Next, set $\eta = \min \Bigg\{ \frac{1}{4 \beta (d+2)}, \frac{1}{4 \beta \sqrt{(d+2) T}}   \Bigg\}$. Then, 
    \begin{align*}
        2 \beta \left(d+2\right) \eta \leq \frac{1}{2},
    \end{align*}
    and hence
    \begin{align*}
        \left(\eta - 2 \beta \left(d+2\right) \eta^2 \right)  \geq \frac{\eta}{2}.
    \end{align*}
    Let $M := \max \Big\{ d+2,  \sqrt{(d+2) T}  \Big\}$. Using the choice of $\eta$ in \oureqref{eq: zo-scalar-intermediate-bound} and rearranging, we obtain
    \begin{align}
        \frac{1}{T} \sum_{t=1}^T \E\left[ \left \| \nabla  \loss(\bw_{t}) \right \|^2 \right] 
        &\leq \left( 16 \beta \Delta_1
        + 16 \mu^2 \beta^2 d \right) \frac{M}{T}
        + \mu^2 \beta^2 (d+3)^3 \\
        &\quad+ \frac{16G^2(1+2D_0)\bigl(d+4\log(eT)\bigr)^2(4+\log T)}{MT}\\
        &\quad + \frac{d+2}{M} \left( \sigma^2 + \frac{1}{2(d+2)T} \E\left[  \Psi_{1:T} \left( \lambda^* \right) \right] +  \frac{\mu^2 \beta^2 d (d+4) }{4} \right).
        \label{eq: zo-scalar-intermediate-bound2}
    \end{align}
    Using \oureqref{eq:bounds_exp_psi} we have
    \begin{align*}
         0 \leq \sigma^2 + \frac{1}{2(d+2)T} \E\left[  \Psi_{1:T} \left( \lambda^* \right) \right] +  \frac{\mu^2 \beta^2 d (d+4) }{4} 
         \leq \sigma^2 +  \frac{\mu^2 \beta^2 d (d+4) }{4},
    \end{align*}
    and thus the last term in \oureqref{eq: zo-scalar-intermediate-bound2} is nonnegative. Using the fact that
    \begin{align*}
        \sqrt{(d+2) T} \leq M \leq d+2 + \sqrt{(d+2) T},
    \end{align*}
    the bound \oureqref{eq: zo-scalar-intermediate-bound2} further implies
    \begin{align*}
        \frac{1}{T} \sum_{t=1}^T \E\left[ \left \| \nabla  \loss(\bw_{t}) \right \|^2 \right] 
        &\leq \left( 16 \beta \Delta_1 + 16 \mu^2 \beta^2 d \right) \left( \frac{d+2}{T} + \frac{\sqrt{d+2}}{\sqrt{T}} \right) \\
        & \quad + \mu^2 \beta^2 (d+3)^3 
        + \frac{16G^2(1+2D_0)\bigl(d+4\log(eT)\bigr)^2(4+\log T)}
{\sqrt{d+2}\,T^{3/2}} \\
        & \quad + \frac{\sqrt{d+2}}{\sqrt{T}} \left(
       \sigma^2
        + \frac{\E[\Psi_{1:T}(\lambda^*)]}{2(d+2)T}
        + \frac{\mu^2\beta^2d(d+4)}{4} \right).
    \end{align*}
    Furthermore, we can choose $\mu = \frac{1}{\beta (d+3) \sqrt{T}}$, for which we obtain
    \begin{align*}
        \frac{1}{T} \sum_{t=1}^T \E\left[ \left \| \nabla  \loss(\bw_{t}) \right \|^2 \right] 
        &\leq 16 \beta \Delta_1 \left( \frac{d+2}{T} + \frac{\sqrt{d+2}}{\sqrt{T}} \right)  \\
        &\quad + \frac{16 d}{(d+3)^2} \left( \frac{d+2}{T^{2}} + \frac{\sqrt{d+2}}{T^{3/2}} \right) + \frac{d+3}{T} \\
        & \quad + \frac{16G^2(1+2D_0)\bigl(d+4\log(eT)\bigr)^2(4+\log T)}
        {\sqrt{d+2}\,T^{3/2}} \\
        &\quad + \frac{\sqrt{d+2}}{\sqrt{T}} \left( \sigma^2 + \frac{\E\left[  \Psi_{1:T} \left( \lambda^* \right) \right]}{2(d+2)T} +\frac{ d (d+4) }{4 (d+3)^2 T}  \right).
    \end{align*}
    Moreover, substituting the choice of $\mu$ in \oureqref{eq:bounds_exp_psi}, we obtain
    \begin{align*}
        -2 (d+2) T \sigma^2
        - \frac{d (d+2) (d+4)}{2 (d+3)^2} 
        \leq \E[ \Psi_{1:T}(\lambda^*) ] \leq 0.
    \end{align*}
\end{proof}

We now turn to vector tuning below. In vector tuning, we also adopt FTL, i.e., initialize $\boldsymbol{\lambda}_1 \gets \boldsymbol{0}$ and for all $t>1$, play
\begin{align}
    \boldsymbol{\lambda}_t \gets \underset{ \boldsymbol{\lambda} \in \reals^d }{\arg\min}  \sum_{s=1}^{t-1}  \psi_s(\boldsymbol{\lambda}), \text{ where } \psi_s(\boldsymbol{\lambda}):= \left\| \mathbf{q}_s + \boldsymbol{\lambda} \odot \mathbf{d}_s^0  \right\|^2. 
\end{align}
Because the loss function is separable, we have for all $i \in [d]$:
\begin{align*}
    \boldsymbol{\lambda}_t [i] \gets \underset{ \lambda \in \reals }{\arg\min} \sum_{s=1}^{t-1} \psi_s^{(i)} \left( \lambda \right) , \; \text{where } \psi_t^{(i)} \left( \lambda \right) := \left( \mathbf{q}_t[i] + \lambda \mathbf{d}_t^0[i] \right)^2.
\end{align*}
Similar to the first-order case, the FTL strategy under vector tuning gives
for all $i \in [d]$
\begin{align*}
    \boldsymbol{\lambda}_t [i] =
    \begin{cases}
        - \frac{\sum_{s=1}^{t-1}\mathbf{q}_s[i] \mathbf{d}_s^0[i]}{\sum_{s=1}^{t-1} \mathbf{d}_s^0[i]^2} , & \sum_{s=1}^{t-1} \mathbf{d}_s^0[i]^2 > 0, \\
        0, &\sum_{s=1}^{t-1} \mathbf{d}_s^0[i]^2 = 0,
    \end{cases}
\end{align*}
which is \oureqref{vector-tuning-first-order-zoo}. In the following, we denote
\begin{align}
    \mathrm{Regret}_T^{\psi}(\boldsymbol{\lambda}^*)
    &:= \sum_{t=1}^{T} \psi_t(\boldsymbol{\lambda}_t) - \sum_{t=1}^T \psi_t(\boldsymbol{\lambda}^*). \label{eq: def_regret_lam_vec_zoo}
\end{align}
Additionally, define
$\Psi_{1:T} \left( \boldsymbol{\lambda} \right) 
    :=  \sum_{t=1}^T \left \| \q_{t,\bw_t}^{\boldsymbol{\lambda}} \right \|^2 -
    \sum_{t=1}^T \left \| \q_{t,\bw_t}^{\boldsymbol{0}} \right \|^2.$
We are now ready to prove Theorem \ref{th: zeroth-order-vector}.

\begin{mdframed}[backgroundcolor=black!10,rightline=false,leftline=false,topline=false,bottomline=false]
\begin{theorem} \label{th: zeroth-order-vector}
Denote $\widetilde D_0:= \sup_{T\ge1} \mathbb E\!\left[
\max_{1\le t\le T+1}\|\boldsymbol{\lambda}_t\|_\infty^2
\right]$,
and $\bar{\Xi}_T:=\frac{1}{T} \sum_{t=1}^T \E\left[ \left \| \nabla  \loss(\bw_{t}) \right \|^2 \right] $. 
Then, under Assumptions~\ref{A1}-\ref{A4}, 
Algorithm~\ref{alg:sgfm} with vector tuning, $\eta = \min \left\{ \frac{1}{4 \beta (d+2)}, \frac{1}{4 \beta \sqrt{(d+2) T}}   \right\}$, and $\mu=\frac{1}{\beta (d+3) \sqrt{T}}$ satisfies
    \begin{align*}
\textstyle        \bar{\Xi}_T
        &\textstyle  \leq 
        16 \beta \Delta_1 \left( \frac{d+2}{T} + \frac{\sqrt{d+2}}{\sqrt{T}} \right) + \frac{16 d}{(d+3)^2} \left( \frac{d+2}{T^{2}} + \frac{\sqrt{d+2}}{T^{3/2}} \right) + \frac{d+3}{T} \\
        &\textstyle \quad + \frac{16dG^2(1+2\widetilde D_0)
        \bigl(d+4\log(eT)\bigr)^2(4+\log T)}
        {\sqrt{d+2}\,T^{3/2}} + \frac{\sqrt{d+2}}{\sqrt{T}} \left( \sigma^2 +\frac{ d (d+4) }{4 (d+3)^2 T}  + \frac{\E\left[  \Psi_{1:T} \left( \boldsymbol{\lambda}^* \right) \right]}{2(d+2)T}  \right),
    \end{align*}
where
$ \E \left[ \Psi_{1:T}(\boldsymbol{\lambda}^*) \right]
    =
    - \sum_{ i \in \mathfrak{J} } \frac{\left( \sum_{t=1}^{T} \E \left[ \mathbf{q}_t[i] \mathbf{d}_t^0[i] \right] \right)^2}{\sum_{t=1}^{T} \E \left[ \mathbf{d}_t^0[i]^2 \right] } \leq 0,
$ when the set $\mathfrak{J}:=\{ i \in [d]: \sum_{t=1}^{T} \E \left[ \mathbf{d}_t^0[i]^2 \right] > 0 \}$ is nonempty; otherwise, $\E \left[ \Psi_{1:T}(\boldsymbol{\lambda}^*) \right]=0$. In particular, it holds that
\begin{align*}
\textstyle 
-2 (d+2) T \sigma^2
- \frac{d (d+2) (d+4)}{2 (d+3)^2} 
\leq \E[ \Psi_{1:T}(\boldsymbol{\lambda}^*) ] \leq 0.
\end{align*}
\end{theorem}
\end{mdframed}

\begin{proof}
    From Lemma \ref{lem:conditional2} we have
    \begin{align}
         \eta \sum_{t=1}^T \E\left[ \left \| \nabla  \loss(\bw_{t}) \right \|^2 \right]
        \leq
        2 \Delta_1 + 2\mu^2 \beta d 
        +  \frac{\mu^2}{2} \beta^2 (d+3)^3 \eta T
        +  \beta \eta^2  \sum_{t=1}^T
         \E\left[ \left \| \mathcal{Q}_{t,\bw_t}^{\boldsymbol{\lambda}_t}   \right \|^2 \right].
         \label{eq: bound_zo_vector}
    \end{align}
    It remains to bound $\sum_{t=1}^T \E\left[ \left \| \mathcal{Q}_{t,\bw_t}^{\boldsymbol{\lambda}_t}  \right \|^2 \right]$. We observe that for any comparator $\boldsymbol{\lambda}^* \in \reals^d$ we have
    \begin{align}
        &\E\left[
        \sum_{t=1}^T  \left \| \q_{t,\bw_t}^{\boldsymbol{\lambda}_t}  \right \|^2 \right] \notag \\
        & =
        \E
        \left[ 
        \sum_{t=1}^T \left \| \q_{t,\bw_t}^{\boldsymbol{\lambda}_t}  \right \|^2 
        -
        \sum_{t=1}^T \left \| \q_{t,\bw_t}^{\boldsymbol{\lambda}^*}  \right \|^2 
        \right]
        +  \E\left[ \sum_{t=1}^T \left \| \q_{t,\bw_t}^{\boldsymbol{\lambda}^*} \right \|^2 \right]
        \notag \\ 
        &= 
        \E\left[ \mathrm{Regret}^{\psi}_T(\boldsymbol{\lambda}^*) \right]
        + \sum_{t=1}^T \E\left[ \left \| \q_{t,\bw_t}^{\boldsymbol{\lambda}^*} \right \|^2 \right] \notag \\
        &= 
        \E\left[ \mathrm{Regret}^{\psi}_T(\boldsymbol{\lambda}^*) \right]
        + \sum_{t=1}^T \E\left[ \left \|  \mathbf{q}_t + \boldsymbol{\lambda}^* \odot \mathbf{d}^0_t  \right \|^2 \right] \notag \\
        &= 
        \E\left[ \mathrm{Regret}^{\psi}_T(\boldsymbol{\lambda}^*) \right]
        + \sum_{t=1}^T \E\left[ \left \|  \mathbf{q}_t \right \|^2 \right]
        +  \E\left[ \underbrace{ \sum_{t=1}^T \left \|  \mathbf{q}_t + \boldsymbol{\lambda}^* \odot \mathbf{d}^0_t  \right \|^2 -  \sum_{t=1}^T  \left \|  \mathbf{q}_t \right \|^2 }_{:= \Psi_{1:T} \left( \boldsymbol{\lambda}^* \right) }  \right]
        \label{eq:zo_vector_1}
    \end{align}
    Furthermore, by Lemma~\ref{lem:exp_q_and_diff_q_and_bar_q} we have
    \begin{align}
        \sum_{t=1}^T \mathbb E\left[  \left\| \mathbf{q_t} \right\|^2 \right] \leq 2 (d+2) \sum_{t=1}^T \E \left[ \left\| \mathbf{g_t} \right\|^2 \right] + 
        \frac{\beta^2 \mu^2}{2} d \left(d+2\right) \left(d+4\right) T.
        \label{eq:zo_vector_2}
    \end{align}
    Combining \oureqref{eq:zo_vector_1} with \oureqref{eq:zo_vector_2}, we obtain
    \begin{align}
        \E\left[\sum_{t=1}^T  \left \| \q_{t,\bw_t}^{\boldsymbol{\lambda}_t}  \right \|^2 \right] 
        &\leq 
        \E\left[ \mathrm{Regret}^{\psi}_T(\boldsymbol{\lambda}^*) \right]
        + 2(d+2) \sum_{t=1}^T \mathbb \E \left[ \left\| \mathbf{g_t} \right\|^2 \right] \notag \\
        &\quad + \frac{\mu^2\beta^2}{2} d(d+2)(d+4) T + \E \left[ \Psi_{1:T} \left( \boldsymbol{\lambda}^* \right) \right] \notag \\
        &= \E\left[ \mathrm{Regret}^{\psi}_T(\boldsymbol{\lambda}^*) \right]
        + 2(d+2) \sum_{t=1}^T \mathbb E\left[ \left\| \nabla \mathcal{L} \left( \bw_t \right) \right\|^2 \right] \notag \\
        &\quad + 2(d+2) \sum_{t=1}^T \mathbb E\left[ \left\| \mathbf{g}_t - \nabla \mathcal{L} \left( \bw_t \right) \right\|^2 \right] 
        \notag \\
        & \quad + \frac{\mu^2\beta^2}{2} d(d+2)(d+4) T + \E \left[ \Psi_{1:T} \left( \boldsymbol{\lambda}^* \right) \right] \notag \\
        &\leq \E\left[ \mathrm{Regret}^{\psi}_T(\boldsymbol{\lambda}^*) \right]
        + 2(d+2) \sum_{t=1}^T \mathbb E\left[ \left\| \nabla \mathcal{L} \left( \bw_t \right) \right\|^2 \right] + 2(d+2) \sigma^2 T 
        \notag \\
        & \quad + \frac{\mu^2\beta^2}{2} d(d+2)(d+4) T + \E \left[ \Psi_{1:T} \left( \boldsymbol{\lambda}^* \right) \right], \label{eq: zero-order-grad-bound-2}
    \end{align}
    We next bound $\E\left[ \mathrm{Regret}_T^{\psi}\left(\boldsymbol{\lambda}^* \right) \right]$.
    From the definition of regret in \oureqref{eq: def_regret_lam_vec_zoo} we have
    \begin{align}
        \mathrm{Regret}_T^{\psi}(\boldsymbol{\lambda}^*)
        &:= \sum_{t=1}^{T} \psi_t(\boldsymbol{\lambda}_t) - \sum_{t=1}^T \psi_t(\boldsymbol{\lambda}^*) \notag \\
        &= \sum_{t=1}^{T} \sum_{i=1}^{d} \psi_t^{(i)}(\boldsymbol{\lambda}_t[i]) - \sum_{t=1}^{T} \sum_{i=1}^{d} \psi_t^{(i)}(\boldsymbol{\lambda}^*[i]) \notag \\
        &= \sum_{i=1}^{d} \underbrace{ \left( \sum_{t=1}^{T} \psi_t^{(i)}(\boldsymbol{\lambda}_t[i])  - \sum_{t=1}^T \psi_t^{(i)}(\boldsymbol{\lambda}^*[i]) \right)}_{:= \mathrm{Regret}_T^{\psi^{(i)}}(\boldsymbol{\lambda}^*)} . \label{eq: overall_reg_lam_vec}
    \end{align}
    Using Lemma \ref{lem: shalev2025online} we have
    \begin{align}
        \mathrm{Regret}_T^{\psi^{(i)}}(\boldsymbol{\lambda}^*) := \sum_{t=1}^{T} \psi_t^{(i)}(\boldsymbol{\lambda}_t[i])  - \sum_{t=1}^T \psi_t^{(i)}(\boldsymbol{\lambda}^*[i]) \leq 
        \sum_{t=1}^{T} \left( \psi_t^{(i)}(\boldsymbol{\lambda}_t[i]) - \psi_t^{(i)}(\boldsymbol{\lambda}_{t+1}[i]) \right), \label{eq: reg_zo_i_lam_vec}
    \end{align}
    for all $i \in [d]$. Combining the above gives
    \begin{align*}
        \mathrm{Regret}^{\psi}_T(\boldsymbol{\lambda}^*) \leq \sum_{i=1}^{d} \sum_{t=1}^{T} \left( \psi_t^{(i)}(\boldsymbol{\lambda}_t[i]) - \psi_t^{(i)}(\boldsymbol{\lambda}_{t+1}[i]) \right).
    \end{align*}
    Fix some $i \in [d]$. By FTL with vector tuning we have
    \begin{align*}
        \boldsymbol{\lambda}_t[i]\sum_{s=1}^{t-1}\mathbf d_s^0[i]^2
        &=-\sum_{s=1}^{t-1}\mathbf q_s[i]\mathbf d_s^0[i],\\
        \boldsymbol{\lambda}_{t+1}[i]\sum_{s=1}^{t}\mathbf d_s^0[i]^2
        &=-\sum_{s=1}^{t}\mathbf q_s[i]\mathbf d_s^0[i].
    \end{align*}
    Subtracting the second identity from the first and adding \(\boldsymbol{\lambda}_t[i]\mathbf d_t^0[i]^2\) to both sides gives
    \begin{align}
        \left(\sum_{s=1}^{t}\mathbf d_s^0[i]^2\right)
        \bigl(\boldsymbol{\lambda}_t[i]-\boldsymbol{\lambda}_{t+1}[i]\bigr)
        =
        \mathbf q_t[i]\mathbf d_t^0[i]
        +\boldsymbol{\lambda}_t[i]\mathbf d_t^0[i]^2.
        \label{eq: vector_lambda_diff_th_4}
    \end{align}
    This identity also holds when \(\sum_{s=1}^t\mathbf d_s^0[i]^2=0\), in which case both sides vanish. We recall that 
    \begin{align*}
        \psi_t^{(i)}(\lambda)=(\mathbf q_t[i]+\lambda\mathbf d_t^0[i])^2.
    \end{align*}
    Proceeding as in the proof of Theorem~\ref{th: first-order-vector}, we obtain
    \begin{align*}
    &\psi_t^{(i)}(\boldsymbol{\lambda}_t[i])
    -\psi_t^{(i)}(\boldsymbol{\lambda}_{t+1}[i])
    =
    \left(2\sum_{s=1}^t\mathbf d_s^0[i]^2-\mathbf d_t^0[i]^2\right)
    \bigl(\boldsymbol{\lambda}_t[i]-\boldsymbol{\lambda}_{t+1}[i]\bigr)^2.
    \end{align*}
    Applying \(v\le v^2/(\rho+v)+\rho\) with \(v=\sum_{s=1}^t\mathbf d_s^0[i]^2\) yields, for every \(\rho>0\),
    \begin{align}
    &\psi_t^{(i)}(\boldsymbol{\lambda}_t[i])
    -\psi_t^{(i)}(\boldsymbol{\lambda}_{t+1}[i]) \notag \\
    &\le
    4\mathbf q_t[i]^2
    \frac{\mathbf d_t^0[i]^2}{\rho+\sum_{s=1}^t\mathbf d_s^0[i]^2}
    +
    4\boldsymbol{\lambda}_t[i]^2
    \frac{\mathbf d_t^0[i]^4}{\rho+\sum_{s=1}^t\mathbf d_s^0[i]^2}
    +
    2\rho\bigl(\boldsymbol{\lambda}_t[i]-\boldsymbol{\lambda}_{t+1}[i]\bigr)^2.
    \label{eq: psi_diff_zo_vector}
    \end{align}
    As in the proof of Theorem~\ref{th: zeroth-order-scalar} for scalar tuning, let
    \[
    b_T:=2d+8\log(eT),
    \qquad
    I_T:=\{t\in[T]:\|\zeta_t\|^2\le b_T\}.
    \]
    First, consider rounds \(t\in I_T\). The bounds
    \(\|\mathbf q_t\|\le G\|\zeta_t\|^2\) and
    \(\|\mathbf d_t^0\|\le 2G\|\zeta_t\|^2\) give for every $i\in[d]$
    \[
    |\mathbf q_t[i]|\le Gb_T,
    \qquad
    |\mathbf d_t^0[i]|\le 2Gb_T.
    \]
    For the first term in \oureqref{eq: psi_diff_zo_vector}, the same integral argument as in the scalar proof gives
    \[
    \sum_{t\in I_T}
    \frac{\mathbf d_t^0[i]^2}{\rho+\sum_{s=1}^t\mathbf d_s^0[i]^2}
    \le
    \log\!\left(1+\frac{\sum_{t\in I_T}\mathbf d_t^0[i]^2}{\rho}\right)
    \le
    \log\!\left(1+\frac{4TG^2b_T^2}{\rho}\right).
    \]
    For the second term, applying Lemma~\ref{lem: seq} with $q_t = \mathbf d_t^0[i]^2\mathbf 1_{\{t\in I_T\}}$ and $Q = \max_{1 \leq t \leq T} q_t \leq  4G^2 b_T^2$, gives
    \[
    \sum_{t\in I_T}
    \frac{\mathbf d_t^0[i]^4}{\rho+\sum_{s=1}^t\mathbf d_s^0[i]^2}
    \le
    \sum_{\substack{t\in I_T\\\mathbf d_t^0[i]\ne0}}
    \frac{\mathbf d_t^0[i]^4}{\sum_{s\le t,\,s\in I_T}\mathbf d_s^0[i]^2}
    \le 4G^2b_T^2(1+\log T).
    \]
    For the third term, we have
    \[
    2\rho\sum_{t\in I_T}
    \bigl(\boldsymbol{\lambda}_t[i]-\boldsymbol{\lambda}_{t+1}[i]\bigr)^2
    \le
    8\rho T\max_{1\le s\le T+1}|\boldsymbol{\lambda}_s[i]|^2.
    \]
    Summing \oureqref{eq: psi_diff_zo_vector} over \(t\in I_T\) and using these three bounds therefore yields
    \begin{align*}
    &\sum_{t\in I_T}
    \left(\psi_t^{(i)}(\boldsymbol{\lambda}_t[i])
    -\psi_t^{(i)}(\boldsymbol{\lambda}_{t+1}[i])\right)\\
    &\le
    4G^2b_T^2\log\!\left(1+\frac{4TG^2b_T^2}{\rho}\right)
    +
    \left[16G^2b_T^2(1+\log T)+8\rho T\right]
    \max_{1\le s\le T+1}|\boldsymbol{\lambda}_s[i]|^2.
    \end{align*}
    Taking \(\rho=4G^2b_T^2/T\), summing over \(i\in[d]\), and using
    \[
    \sum_{i=1}^d\mathbb E\!\left[\max_{1\le s\le T+1}
    |\boldsymbol{\lambda}_s[i]|^2\right]\le d\widetilde D_0,
    \]
    we obtain
    \begin{align}
    \mathbb E\!\left[
    \sum_{t\in I_T}
    \bigl(\psi_t(\boldsymbol{\lambda}_t)-\psi_t(\boldsymbol{\lambda}_{t+1})\bigr)
    \right]
    \le
    4dG^2b_T^2\log(1+T^2)
    +16dG^2b_T^2(3+\log T)\widetilde D_0.
     \label{eq: sum_I_T_zo_vector}
    \end{align}
    Next, consider the remaining rounds. Since \(\psi_t(\boldsymbol{\lambda}_{t+1})\ge0\) and
    \(\|\boldsymbol{\lambda}_t\odot\mathbf d_t^0\|
    \le\|\boldsymbol{\lambda}_t\|_\infty\|\mathbf d_t^0\|\), we have
    \begin{align*}
    \psi_t(\boldsymbol{\lambda}_t)-\psi_t(\boldsymbol{\lambda}_{t+1})
    &\le \|\mathbf q_t+\boldsymbol{\lambda}_t\odot\mathbf d_t^0\|^2\\
    &\le 2\|\mathbf q_t\|^2
    +2\|\boldsymbol{\lambda}_t\|_\infty^2\|\mathbf d_t^0\|^2\\
    &\le 2G^2\bigl(1+4\|\boldsymbol{\lambda}_t\|_\infty^2\bigr)\|\zeta_t\|^4.
    \end{align*}
    In the proof of Theorem~\ref{th: zeroth-order-scalar}, we showed that
    \[
    \mathbb E\!\left[\|\zeta_t\|^4
    \mathbf1_{\{\|\zeta_t\|^2>b_T\}}\right]\le\frac{(d+3)^2}{T}.
    \]
    Since \(\boldsymbol{\lambda}_t\) depends only on previous rounds and \(\zeta_t\) is sampled independently of that history, it follows that
    \begin{align}
    \mathbb E\!\left[
    \sum_{t\notin I_T}
    \bigl(\psi_t(\boldsymbol{\lambda}_t)-\psi_t(\boldsymbol{\lambda}_{t+1})\bigr)
    \right]
    &\le
    2G^2\sum_{t=1}^T
    \bigl(1+4\mathbb E\|\boldsymbol{\lambda}_t\|_\infty^2\bigr)
    \mathbb E\!\left[\|\zeta_t\|^4\mathbf1_{\{\|\zeta_t\|^2>b_T\}}\right] \notag\\
    &\le 2G^2(1+4\widetilde D_0)(d+3)^2.
    \label{eq: sum_not_I_T_zo_vector}
    \end{align}
    Adding \oureqref{eq: sum_I_T_zo_vector} and \oureqref{eq: sum_not_I_T_zo_vector}, we further obtain
    \begin{align*}
        \mathbb E\left[\mathrm{Regret}_T^\psi(\boldsymbol \lambda^*) \right]
        &\leq \mathbb E\left[\sum_{t=1}^T \bigl(\psi_t(\boldsymbol{\lambda}_t)-\psi_t(\boldsymbol{\lambda}_{t+1})\bigr) \right] \notag \\
        &\leq4dG^2b_T^2\log(1+T^2) +16dG^2b_T^2(3+\log T)\widetilde D_0 +2G^2(1+4\widetilde D_0)(d+3)^2.
    \end{align*}
    Substituting \(b_T=2d+8\log(eT)\) and applying the same simplification as in the scalar case gives
    \begin{align}
    \mathbb E[\mathrm{Regret}_T^\psi(\boldsymbol{\lambda}^*)]
    \le
    32dG^2(1+2\widetilde D_0)
    \bigl(d+4\log(eT)\bigr)^2(4+\log T).
    \label{eq: zo-vector-regret-bound}
    \end{align}
    Combining \oureqref{eq: bound_zo_vector}, \oureqref{eq: zero-order-grad-bound-2}, and \oureqref{eq: zo-vector-regret-bound} yields
    \begin{align}
      \left(\eta - 2 \beta \left(d+2\right) \eta^2 \right)  \sum_{t=1}^T \E\left[ \left \| \nabla  \loss(\bw_{t}) \right \|^2 \right]
        &\leq
        2 \Delta_1 + 2\mu^2 \beta d 
        +  \frac{\mu^2}{2} \beta^2 (d+3)^3 \eta T \notag \\
        &\quad + 
        32d\,\beta\eta^2G^2(1+2\widetilde D_0)
        \bigl(d+4\log(eT)\bigr)^2(4+\log T) \notag \\
        &\quad 
        + 2 \beta \eta^2 (d+2) T \sigma^2 
        + \beta \eta^2 \E\left[  \Psi_{1:T} \left( \boldsymbol{\lambda}^* \right) \right] \notag \\
        &\quad  + \frac{\mu^2\beta^3}{2} \eta^2 d(d+2)(d+4) T
        \label{eq: zo-vector-intermediate-bound}
    \end{align}
    We continue by analyzing $\E[ \Psi_{1:T}(\boldsymbol{\lambda}^*) ]$. We observe that $\Psi_{1:T}(\boldsymbol{\lambda})$ can be expressed as
    \begin{align*}
        \Psi_{1:T} \left( \boldsymbol{\lambda} \right) 
        &=  \sum_{t=1}^T \left\| \mathbf{q}_t + \boldsymbol{\lambda} \odot \mathbf{d}_t^0  \right\|^2 -  \sum_{t=1}^T \left\| \mathbf{q}_t  \right\|^2 \\
        &= 2 \sum_{t=1}^T \langle \mathbf{q}_t , \boldsymbol{\lambda} \odot \mathbf{d}_t^0 \rangle + \sum_{t=1}^T \left\| \boldsymbol{\lambda} \odot \mathbf{d}_t^0  \right\|^2 \\
        &= 2 \sum_{t=1}^T \sum_{i=1}^d \boldsymbol{\lambda}[i] \mathbf{q}_t [i] \mathbf{d}_t^0 [i] + \sum_{t=1}^T \sum_{i=1}^d \boldsymbol{\lambda}[i]^2 \mathbf{d}_t^0 [i]^2 \\
        &= \sum_{i=1}^d \left( 2 \boldsymbol{\lambda}[i] \sum_{t=1}^T \mathbf{q}_t [i] \mathbf{d}_t^0 [i] +  \boldsymbol{\lambda}[i]^2 \sum_{t=1}^T \mathbf{d}_t^0 [i]^2\right) .
    \end{align*}
    Therefore, 
    \begin{align}
        \E \left[ \Psi_{1:T} \left( \boldsymbol{\lambda} \right) \right] 
        = \sum_{i=1}^d \left( 2 \boldsymbol{\lambda}[i] \sum_{t=1}^T \E \left[ \mathbf{q}_t [i] \mathbf{d}_t^0 [i] \right] +  \boldsymbol{\lambda}[i]^2 \sum_{t=1}^T \E \left[ \mathbf{d}_t^0 [i]^2 \right]\right) .
        \label{expression:psi-1:T}
    \end{align}
    Hence, $\E \left[ \Psi_{1:T} \left( \boldsymbol{\lambda} \right) \right]$ is separable. Similar to the first-order case, we define the comparator $\boldsymbol{\lambda}^*$ as the minimizer of $\E \left[ \Psi_{1:T} \left( \boldsymbol{\lambda} \right) \right]$ given by
    \begin{align*}
        \boldsymbol{\lambda}^*[i] = 
        \begin{cases}
            - \frac{\sum_{t=1}^{T} \E \left[ \mathbf{q}_t[i] \mathbf{d}_t^0[i] \right]}{\sum_{t=1}^{T} \E \left[ \mathbf{d}_t^0[i]^2 \right] } , & \sum_{t=1}^{T} \E \left[ \mathbf{d}_t^0[i]^2 \right] > 0, \\
            0, &\sum_{t=1}^{T} \E \left[  \mathbf{d}_t^0[i]^2 \right] = 0,
        \end{cases}
    \end{align*}
    for all $i \in [d]$. Substituting $\boldsymbol{\lambda}^*$ into the expression for $\E \left[ \Psi_{1:T} \left( \boldsymbol{\lambda} \right) \right]$ \oureqref{expression:psi-1:T}, we obtain
    \begin{align*}
        \E \left[ \Psi_{1:T} \left( \boldsymbol{\lambda}^* \right) \right] =
        - \sum_{\substack{i \in [d] \\ \sum_{t=1}^{T} \E \left[ \mathbf{d}_t^0[i]^2 \right] > 0}} \frac{\left( \sum_{t=1}^{T} \E \left[ \mathbf{q}_t[i] \mathbf{d}_t^0[i] \right] \right)^2}{\sum_{t=1}^{T} \E \left[ \mathbf{d}_t^0[i]^2 \right] } \leq 0.
    \end{align*}
    Next, we define
    \begin{align*}
        \overline{\mathbf{q}}_t := 
        \E \left[ \mathbf{q}_t \big| \F_{t-1}, \zeta_t \right] = \frac{\mathcal{L}\left( \mathbf{w}_t + \mu \zeta_t \right) - \mathcal{L}\left( \mathbf{w}_t\right)}{\mu} \zeta_t.
    \end{align*}
    Since $\E \left[ \mathbf{d}_t^0 \big| \F_{t-1}, \zeta_t  \right] = 0$, we have
    \begin{align*}
        \E \left[ \mathbf{q}_t[i] \mathbf{d}_t^0[i]  \right] = \E \left[  \left( \mathbf{q}_t[i] - \overline{\mathbf{q}}_t[i] \right)  \mathbf{d}_t^0[i]   \right] .
    \end{align*}
    By Cauchy-Schwarz
    \begin{align*}
        \left(\sum_{t=1}^T \E \left[ \mathbf{q}_t[i]  \mathbf{d}_t^0[i]   \right] \right)^2 
        &= \left(\sum_{t=1}^T \E \left[ \left(\mathbf{q}_t[i] -\overline{\mathbf{q}}_t[i] \right) \mathbf{d}_t^0[i]   \right] \right)^2 \\
        &\leq  \left( \sum_{t=1}^T \E \left[ \left( \mathbf{q}_t[i] - \overline{\mathbf{q}}_t[i]  \right)^2 \right] \right) \left( \sum_{t=1}^T \E \left[ \left( \mathbf{d}_t^0[i] \right)^2 \right] \right).
    \end{align*}
    Therefore, 
    \begin{align*}
        \E \left[ \Psi_{1:T} \left( \boldsymbol{\lambda}^* \right) \right] &\geq - \sum_{i=1}^d \sum_{t=1}^T \E \left[ \left( \mathbf{q}_t[i] - \overline{\mathbf{q}}_t[i]  \right)^2 \right]
        = - \sum_{t=1}^T \E \left[ \left\| \mathbf{q}_t - \overline{\mathbf{q}}_t  \right\|^2 \right].
    \end{align*}
    Additionally, by Lemma~\ref{lem:exp_q_and_diff_q_and_bar_q}, we have
    \begin{align*}
        \sum_{t=1}^T \E \left[ \left\| \mathbf{q}_t - \overline{\mathbf{q}}_t \right\|^2 \right]
        \leq 
        2 (d+2) T \sigma^2
        + \frac{\beta^2 \mu^2}{2} d (d+2) (d+4) T.
    \end{align*}
    Consequently, we obtain the following bounds for $\E[ \Psi_{1:T}(\boldsymbol{\lambda}^*) ]$,
    \begin{align}
        -2 (d+2) T \sigma^2
        - \frac{\beta^2 \mu^2}{2} d (d+2) (d+4) T 
        \leq \E[ \Psi_{1:T}(\boldsymbol{\lambda}^*) ] \leq 0.
        \label{eq:bounds_exp_psi_vector}
    \end{align}
    Next, set $\eta = \min \Bigg\{ \frac{1}{4 \beta (d+2)}, \frac{1}{4 \beta \sqrt{(d+2) T}}   \Bigg\}$. Then, 
    \begin{align*}
        2 \beta \left(d+2\right) \eta \leq \frac{1}{2},
    \end{align*}
    and hence
    \begin{align*}
        \left(\eta - 2 \beta \left(d+2\right) \eta^2 \right)  \geq \frac{\eta}{2}.
    \end{align*}
    Let $M := \max \Big\{ d+2,  \sqrt{(d+2) T}  \Big\}$. Using the choice of $\eta$ in \oureqref{eq: zo-vector-intermediate-bound} and rearranging, we obtain
    \begin{align}
        \frac{1}{T} \sum_{t=1}^T \E\left[ \left \| \nabla  \loss(\bw_{t}) \right \|^2 \right] 
        &\leq \left( 16 \beta \Delta_1
        + 16 \mu^2 \beta^2 d \right) \frac{M}{T}
        + \mu^2 \beta^2 (d+3)^3  \notag \\
        &\quad +
        \frac{16dG^2(1+2\widetilde D_0)
        \bigl(d+4\log(eT)\bigr)^2(4+\log T)}{MT} \notag \\
        &\quad + \frac{d+2}{M} \left( \sigma^2 + \frac{1}{2(d+2)T} \E\left[  \Psi_{1:T} \left( \boldsymbol{\lambda}^* \right) \right] +  \frac{\mu^2 \beta^2 d (d+4) }{4} \right).
        \label{eq: zo-vector-intermediate-bound2}
    \end{align}
    Using \oureqref{eq:bounds_exp_psi_vector} we have
    \begin{align*}
         0 \leq \sigma^2 + \frac{1}{2(d+2)T} \E\left[  \Psi_{1:T} \left( \boldsymbol{\lambda}^* \right) \right] +  \frac{\mu^2 \beta^2 d (d+4) }{4} 
         \leq \sigma^2 +  \frac{\mu^2 \beta^2 d (d+4) }{4},
    \end{align*}
    and thus the last term in \oureqref{eq: zo-vector-intermediate-bound2} is nonnegative. Using the fact that
    \begin{align*}
        \sqrt{(d+2) T} \leq M \leq d+2 + \sqrt{(d+2) T},
    \end{align*}
    the bound \oureqref{eq: zo-vector-intermediate-bound2} further implies
    \begin{align*}
        \frac{1}{T} \sum_{t=1}^T \E\left[ \left \| \nabla  \loss(\bw_{t}) \right \|^2 \right] 
        &\leq \left( 16 \beta \Delta_1 + 16 \mu^2 \beta^2 d \right) \left( \frac{d+2}{T} + \frac{\sqrt{d+2}}{\sqrt{T}} \right) \\
        & \quad + \mu^2 \beta^2 (d+3)^3 
        + \frac{16dG^2(1+2\widetilde D_0)
        \bigl(d+4\log(eT)\bigr)^2(4+\log T)}
        {\sqrt{d+2}\,T^{3/2}} \\
        & \quad + \frac{\sqrt{d+2}}{\sqrt{T}} \left(
       \sigma^2
        + \frac{\E[\Psi_{1:T}(\boldsymbol{\lambda}^*)]}{2(d+2)T}
        + \frac{\mu^2\beta^2d(d+4)}{4} \right).
    \end{align*}
    Furthermore, we can choose $\mu = \frac{1}{\beta (d+3) \sqrt{T}}$, for which we obtain
    \begin{align*}
        \frac{1}{T} \sum_{t=1}^T \E\left[ \left \| \nabla  \loss(\bw_{t}) \right \|^2 \right] 
        &\leq 16 \beta \Delta_1 \left( \frac{d+2}{T} + \frac{\sqrt{d+2}}{\sqrt{T}} \right) \\
        &\quad + \frac{16 d}{(d+3)^2} \left( \frac{d+2}{T^{2}} + \frac{\sqrt{d+2}}{T^{3/2}} \right) \\
        & \quad + \frac{d+3}{T} +
        \frac{16dG^2(1+2\widetilde D_0)
        \bigl(d+4\log(eT)\bigr)^2(4+\log T)}
        {\sqrt{d+2}\,T^{3/2}} \\
        &\quad + \frac{\sqrt{d+2}}{\sqrt{T}} \left( \sigma^2 + \frac{\E\left[  \Psi_{1:T} \left( \boldsymbol{\lambda}^* \right) \right]}{2(d+2)T} +\frac{ d (d+4) }{4 (d+3)^2 T}  \right).
    \end{align*}
    Moreover, substituting the choice of $\mu$ in \oureqref{eq:bounds_exp_psi_vector}, we obtain
    \begin{align*}
        -2 (d+2) T \sigma^2
        - \frac{d (d+2) (d+4)}{2 (d+3)^2} 
        \leq \E[ \Psi_{1:T}(\boldsymbol{\lambda}^*) ] \leq 0.
    \end{align*}

\end{proof}

\newpage

\section{Extensions to Follow-the-Regularized-Leader} \label{app:FTRL}

\subsection{First-order Optimization}

We can also consider using the Follow-the-Regularized-Leader (FTRL) strategy in the online tuning as a game, i.e., for $t=1,\ldots,T+1$, play
\begin{align} 
 \lambda_t
 \gets \underset{ \lambda \in \reals }{\arg\min}
 \left\{\sum_{s=1}^{t-1}\phi_s(\lambda)+\epsilon\lambda^2\right\}, \text{ where } \phi_s(\lambda):= \left \|  \mathbf{g}_s + \lambda \mathbf{d}_s  \right \|^2 \text{ and } \epsilon>0.
 \label{eq:ftrl}
\end{align}
The hyperparameter $\epsilon>0$ is fixed throughout a run and will be specified in the theorem. The empty-sum convention gives $\lambda_1=0$.
By setting $\frac{\partial 
\sum_{s=1}^{t-1}\phi_s(\lambda)+\epsilon\lambda^2
}{\partial \lambda}
 = 0$, we obtain
$$    \lambda_t = - \frac{\sum_{s=1}^{t-1} \langle \mathbf{g}_s , \mathbf{d}_s \rangle}{\epsilon + \sum_{s=1}^{t-1} \left\| \mathbf{d}_s \right\|^2}.$$

We use the notation of the scalar tuning game:
$\mathbf g_t=\gt{y}{\bw_t}$,
$\mathbf d_t=\tildegt{f}{\bw_t}-\gt{f}{\bw_t}$, and
$\g_{t,\bw_t}^{\lambda_t}=\mathbf g_t+\lambda_t\mathbf d_t$.
The stochastic gradient update is
\begin{equation}
    \bw_{t+1}=\bw_t-\eta_t\g_{t,\bw_t}^{\lambda_t}.
    \label{eq:sgd-update}
\end{equation}
For the vector notation in Lemma~\ref{lem:reg0}, write
$\g_{t,\bw_t}^{\boldsymbol\lambda}
=\mathbf g_t+\boldsymbol\lambda\odot\mathbf d_t$.
Let $\F_{t-1}$ denote the information available before round $t$.
As in the prediction-powered setting, we assume
\[
\E[\mathbf g_t\mid\F_{t-1}]=\nabla\loss(\bw_t),
\qquad
\E[\mathbf d_t\mid\F_{t-1}]=0.
\]
The initial iterate $\bw_1$ is deterministic, and
$\Delta_1:=\loss(\bw_1)-\inf_{\bw\in\reals^d}\loss(\bw)$.
The FTRL coefficient $\lambda_t$ depends only on observations before
round $t$, so the corrected gradient is conditionally unbiased.
All expectations below are taken along the FTRL-generated trajectory.

\begin{mdframed}[backgroundcolor=black!10,rightline=false,leftline=false,topline=false,bottomline=false]
\begin{theorem}[Scalar tuning with FTRL]
Suppose Assumptions~\ref{A1}--\ref{A4} hold. Fix $T\geq2$
and choose any $\epsilon\geq 4G^2/(\sqrt T-1)$. Run
\oureqref{eq:sgd-update} with scalar FTRL tuning \oureqref{eq:ftrl}
and step size $\eta_t=\frac{1}{\beta\sqrt T}$. Then
\begin{align}
    \frac{1}{T}\sum_{t=1}^T
        \E\!\left[\|\nabla\loss(\bw_t)\|^2\right]
    \leq{}&
    \frac{2\beta\Delta_1}{\sqrt{T}}
    + \frac{1+4G^2/\epsilon}{\sqrt{T}}
    \left[
        \sigma^2
        -
        \frac{
            \left(\sum_{t=1}^T
                \E[\langle\mathbf g_t,\mathbf d_t\rangle]\right)^2
        }{
            T\left(\epsilon+
                \sum_{t=1}^T\E[\|\mathbf d_t\|^2]\right)
        }
    \right],
\end{align}
where
\begin{align*}
    0
    \leq
    \frac{
        \left(\sum_{t=1}^T
            \E[\langle\mathbf g_t,\mathbf d_t\rangle]\right)^2
    }{
        T\left(\epsilon+\sum_{t=1}^T\E[\|\mathbf d_t\|^2]\right)
    }
    \leq
    \sigma^2\frac{4TG^2}{\epsilon+4TG^2}.
\end{align*}
\end{theorem}
\end{mdframed}

\begin{proof}

Using \oureqref{eq:reg} with $\boldsymbol{\lambda}^* \gets \lambda^* \mathbf{1}$, we have
    \begin{align} \label{eq: th_bound1}
        \sum_{t=1}^T \left( \eta  - \frac{\beta}{2} \eta^2 \right) \E\left[ \| \nabla  \loss(\bw_t) \|^2 \right] 
        \leq
         \Delta_1
        + \frac{\beta \eta^2}{2}   \, 
        \E\left[ \mathrm{Regret}_T^{\phi}(\lambda^*) \right] + \frac{\beta \eta^2}{2}  \E[ \Phi_{1:T}(\lambda^*) ] +
        \frac{\beta \eta^2 T}{2} 
         \sigma^2,
    \end{align}
    where we will specify $\lambda^*$ in the later part of this proof.
    We first bound $\E\left[ \mathrm{Regret}_T^{\phi}(\lambda^*) \right]$. We recall that by definition of regret
    \begin{align}
        \mathrm{Regret}_T^{\phi}(\lambda^*)
        &:= \sum_{t=1}^{T} \phi_t(\lambda_t) - \sum_{t=1}^T \phi_t(\lambda^*) \notag \\
        &=\sum_{t=1}^T
          \left[\phi_t(\lambda_t)-\phi_t(\lambda_{t+1})\right]
          +\sum_{t=1}^T\phi_t(\lambda_{t+1})
          -\sum_{t=1}^T\phi_t(\lambda^*). \label{rdecomposition}
    \end{align}
    Additionally, we note that
    \begin{align*}
        \phi_t(\lambda)& =  \left \|  \mathbf{g}_t + \lambda \mathbf{d}_t  \right \|^2 \notag \\
        &=\|\mathbf g_t\|^2
          +2\lambda\langle\mathbf g_t,\mathbf d_t\rangle
          +\lambda^2\|\mathbf d_t\|^2.
    \end{align*}
    Therefore,
    \begin{align}
        \phi_t(\lambda_t)-\phi_t(\lambda_{t+1})
        &=2(\lambda_t-\lambda_{t+1}) \langle\mathbf g_t,\mathbf d_t\rangle
          +(\lambda_t^2-\lambda_{t+1}^2) \|\mathbf d_t\|^2
          \notag\\
        &=2(\lambda_t-\lambda_{t+1}) \left(\langle\mathbf g_t,\mathbf d_t\rangle
                      +\lambda_t\|\mathbf d_t\|^2\right)
             - (\lambda_t-\lambda_{t+1})^2 \|\mathbf d_t\|^2 \notag \\
        &=2 (\lambda_t-\lambda_{t+1}) \langle\mathbf g_t+\lambda_t\mathbf d_t,\mathbf d_t\rangle - (\lambda_t-\lambda_{t+1})^2\|\mathbf d_t\|^2
        \label{eq: phi_diff_new}
    \end{align}
    Furthermore, the FTRL update gives
    \begin{align*}
        \lambda_t
        &=-\frac{\sum_{s=1}^{t-1}
                     \langle\mathbf g_s,\mathbf d_s\rangle}
                   {\epsilon+\sum_{s=1}^{t-1}\|\mathbf d_s\|^2},\\
        \lambda_{t+1}
        &=-\frac{\sum_{s=1}^{t}
                     \langle\mathbf g_s,\mathbf d_s\rangle}
                   {\epsilon+\sum_{s=1}^{t}\|\mathbf d_s\|^2}
        =\frac{\left(\epsilon+\sum_{s=1}^{t-1}\|\mathbf d_s\|^2\right)
                      \lambda_t-\langle\mathbf g_t,\mathbf d_t\rangle}
                   {\epsilon+\sum_{s=1}^{t}\|\mathbf d_s\|^2}.
    \end{align*}
    Subtracting the last expression from $\lambda_t$, we obtain
    \begin{align}
        \lambda_t-\lambda_{t+1}
        &=\lambda_t-
          \frac{\left(\epsilon+\sum_{s=1}^{t-1}\|\mathbf d_s\|^2\right)
                      \lambda_t-\langle\mathbf g_t,\mathbf d_t\rangle}
                   {\epsilon+\sum_{s=1}^{t}\|\mathbf d_s\|^2}
         =\frac{\langle\mathbf g_t,\mathbf d_t\rangle
                      +\lambda_t\|\mathbf d_t\|^2}
                   {\epsilon+\sum_{s=1}^{t}\|\mathbf d_s\|^2}
        =\frac{\langle\mathbf g_t+\lambda_t\mathbf d_t,\mathbf d_t\rangle}
         {\epsilon+\sum_{s=1}^{t}\|\mathbf d_s\|^2}.
        \label{eq: lambda_diff_new}
    \end{align}
    Here every denominator is at least $\epsilon>0$, so these
    identities remain valid even when all past $\mathbf d_s$ are zero. Substituting \oureqref{eq: lambda_diff_new} into \oureqref{eq: phi_diff_new} we obtain
        \begin{align}
        \phi_t(\lambda_t)-\phi_t(\lambda_{t+1})
        &= 2\left(\frac{\langle\mathbf g_t+\lambda_t\mathbf d_t,\mathbf d_t\rangle}
         {\epsilon+\sum_{s=1}^{t}\|\mathbf d_s\|^2}\right)
         \langle\mathbf g_t+\lambda_t\mathbf d_t,\mathbf d_t\rangle
         -\left(\frac{\langle\mathbf g_t+\lambda_t\mathbf d_t,\mathbf d_t\rangle}
         {\epsilon+\sum_{s=1}^{t}\|\mathbf d_s\|^2}\right)^2 \|\mathbf d_t\|^2 \notag \\
        &=\left(
             \frac{2}{\epsilon+\sum_{s=1}^{t}\|\mathbf d_s\|^2}
             -\frac{\|\mathbf d_t\|^2}
                    {\left(\epsilon+\sum_{s=1}^{t}\|\mathbf d_s\|^2\right)^2}
           \right) 
          \left(\langle\mathbf g_t+\lambda_t\mathbf d_t,\mathbf d_t\rangle\right)^2
          \notag\\
        &=\frac{2\epsilon+2\sum_{s=1}^{t-1}\|\mathbf d_s\|^2
                         +\|\mathbf d_t\|^2}
                   {\left(\epsilon+\sum_{s=1}^{t}\|\mathbf d_s\|^2\right)^2}
          \left(\langle\mathbf g_t+\lambda_t\mathbf d_t,\mathbf d_t\rangle\right)^2 \notag \\
        &=\frac{ \left( \epsilon+ \sum_{s=1}^{t-1}\|\mathbf d_s\|^2 \right)
                + \left( \epsilon + \sum_{s=1}^{t}\|\mathbf d_s\|^2 \right)}
                   {\left(\epsilon+\sum_{s=1}^{t}\|\mathbf d_s\|^2\right)^2}
          \left(\langle\mathbf g_t+\lambda_t\mathbf d_t,\mathbf d_t\rangle\right)^2  \notag \\
          &= \left(\epsilon+\sum_{s=1}^{t-1}\|\mathbf d_s\|^2\right)
        (\lambda_t-\lambda_{t+1})^2
        +
        \frac{\langle\mathbf g_t+\lambda_t\mathbf d_t,\mathbf d_t\rangle^2}
             {\epsilon+\sum_{s=1}^{t}\|\mathbf d_s\|^2}.
        \label{eq: diff_phi}
    \end{align}
We next evaluate the term $\sum_{t=1}^T\phi_t(\lambda_{t+1})$
in the regret decomposition. Recall that $\lambda_t$ minimizes
$\sum_{s=1}^{t-1}\phi_s(\lambda)+ \epsilon\lambda^2$.
This is a quadratic in $\lambda$ with quadratic coefficient
$\sum_{s=1}^{t-1}\|\mathbf d_s\|^2 + \epsilon$. Completing the square at its minimizer therefore gives, for every $\lambda\in\mathbb R$,
\begin{align}
    \epsilon\lambda^2 + \sum_{s=1}^{t-1}\phi_s(\lambda) 
    ={}&
    \epsilon\lambda_t^2 + \sum_{s=1}^{t-1}\phi_s(\lambda_t) 
    +\left(\epsilon + \sum_{s=1}^{t-1}\|\mathbf d_s\|^2 \right)
      (\lambda-\lambda_t)^2.
    \label{eq:ftrl_quadratic_identity}
\end{align}
Setting $\lambda=\lambda_{t+1}$ and rearranging yields
\begin{align*}
    \phi_t(\lambda_{t+1})
    ={}&
    \left[
    \epsilon\lambda_{t+1}^2 + \sum_{s=1}^{t}\phi_s(\lambda_{t+1})
    \right]
    -
    \left[
    \epsilon\lambda_t^2 + \sum_{s=1}^{t-1}\phi_s(\lambda_t) 
    \right]\\
    &-
    \left( \epsilon + \sum_{s=1}^{t-1}\|\mathbf d_s\|^2 \right)
    (\lambda_t-\lambda_{t+1})^2.
\end{align*}
Summing over $t=1, \ldots, T$ and using $\lambda_1=0$, we obtain
\begin{align}
    \sum_{t=1}^T\phi_t(\lambda_{t+1})
    ={}&
    \epsilon\lambda_{T+1}^2 + \sum_{t=1}^T\phi_t(\lambda_{T+1})
    -
    \sum_{t=1}^T
    \left(\sum_{s=1}^{t-1}\|\mathbf d_s\|^2 + \epsilon \right)
    (\lambda_t-\lambda_{t+1})^2.
    \label{eq:ftrl_leader_sum}
\end{align}
Substituting \oureqref{eq: diff_phi} and
\oureqref{eq:ftrl_leader_sum} into the regret decomposition \oureqref{rdecomposition}, the
terms involving $(\lambda_t-\lambda_{t+1})^2$ cancel. Thus
\begin{align*}
    \mathrm{Regret}_T^\phi(\lambda^*)
    ={}&
    \epsilon\lambda_{T+1}^2
    +\sum_{t=1}^T\phi_t(\lambda_{T+1})
    -\sum_{t=1}^T\phi_t(\lambda^*)
    +
    \sum_{t=1}^T
    \frac{\langle\mathbf g_t+\lambda_t\mathbf d_t,\mathbf d_t\rangle^2}
         {\epsilon+\sum_{s=1}^{t}\|\mathbf d_s\|^2}.
\end{align*}
Since $\lambda_{T+1}$ minimizes the regularized cumulative loss
through round $T$, we have
\[
    \epsilon\lambda_{T+1}^2+\sum_{t=1}^T\phi_t(\lambda_{T+1})
    \leq
    \epsilon(\lambda^*)^2+\sum_{t=1}^T\phi_t(\lambda^*).
\]
Consequently,
\begin{align}
    \mathrm{Regret}_T^\phi(\lambda^*)
    \leq
    \epsilon(\lambda^*)^2+
    \sum_{t=1}^T
    \frac{\langle\mathbf g_t+\lambda_t\mathbf d_t,\mathbf d_t\rangle^2}
         {\epsilon+\sum_{s=1}^{t}\|\mathbf d_s\|^2}.
    \label{eq:ftrl_regret_residual}
\end{align}
By Cauchy--Schwarz,
\[
    \langle\mathbf g_t+\lambda_t\mathbf d_t,\mathbf d_t\rangle^2
    \leq \underbrace{\|\mathbf g_t+\lambda_t\mathbf d_t\|^2}_{:= \phi_t(\lambda_t)} \|\mathbf d_t\|^2.
\]
Moreover, since $\|\mathbf d_t\|\leq D$ and
$x/(\epsilon+x)$ is increasing for $x\geq0$,
\[
    \frac{\|\mathbf d_t\|^2}
         {\epsilon+\sum_{s=1}^{t}\|\mathbf d_s\|^2}
    \leq
    \frac{\|\mathbf d_t\|^2}{\epsilon+\|\mathbf d_t\|^2}
    \leq \frac{4G^2}{\epsilon+4G^2}.
\]
Applying these bounds to \oureqref{eq:ftrl_regret_residual} and using
the definition of regret gives
\begin{align*}
    \mathrm{Regret}_T^\phi(\lambda^*)
    &\leq
    \epsilon(\lambda^*)^2+
    \frac{4G^2}{\epsilon+4G^2}\sum_{t=1}^T\phi_t(\lambda_t)\\
    &=
    \epsilon(\lambda^*)^2+
    \frac{4G^2}{\epsilon+4G^2}
    \left[
        \mathrm{Regret}_T^\phi(\lambda^*)
        +\sum_{t=1}^T\phi_t(\lambda^*)
    \right].
\end{align*}
Moving the regret term to the left and dividing by $\epsilon/(\epsilon+4G^2)$ yields
\begin{align}
    \mathrm{Regret}_T^\phi(\lambda^*)
    \leq
    (\epsilon+4G^2)(\lambda^*)^2
    +\frac{4G^2}{\epsilon}\sum_{t=1}^T\phi_t(\lambda^*).
    \label{eq:ftrl_pathwise}
\end{align}
For a deterministic comparator $\lambda^*$, the definition of
$\Phi_{1:T}$ gives
\[
    \E\!\left[\sum_{t=1}^T\phi_t(\lambda^*)\right]
    =
    \E[\Phi_{1:T}(\lambda^*)]
    +\sum_{t=1}^T\E[\|\mathbf g_t\|^2].
\]
Using conditional unbiasedness
$\E[\mathbf g_t-\nabla\loss(\bw_t)\mid\F_{t-1}]=0$ and Assumption~\ref{A3}, we obtain
\[
    \E[\|\mathbf g_t\|^2]
    =
    \E[\|\nabla\loss(\bw_t)\|^2]
    +\E[\|\mathbf g_t-\nabla\loss(\bw_t)\|^2]
    \leq \E[\|\nabla\loss(\bw_t)\|^2]+\sigma^2.
\]
Taking expectations in \oureqref{eq:ftrl_pathwise} therefore yields
\begin{align}
    \E[\mathrm{Regret}_T^\phi(\lambda^*)]
    \leq{}&
    (\epsilon+4G^2)(\lambda^*)^2
    +\frac{4G^2}{\epsilon}
    \left[
        \E[\Phi_{1:T}(\lambda^*)]
        +\sum_{t=1}^T\E[\|\nabla\loss(\bw_t)\|^2]
        +T\sigma^2
    \right].
    \label{eq:ftrl_expected}
\end{align}
Substituting \oureqref{eq:ftrl_expected} into \oureqref{eq: th_bound1}
and rearranging gives
\begin{align}
    \eta\left[
    1-\frac{\beta\eta}{2}\left(1+\frac{4G^2}{\epsilon}\right)
    \right]
    \sum_{t=1}^T\E[\|\nabla\loss(\bw_t)\|^2]
    \leq
    \Delta_1+
    \frac{\beta\eta^2}{2}\left(1+\frac{4G^2}{\epsilon}\right)
    \left[
        T\sigma^2+\epsilon(\lambda^*)^2+\E[\Phi_{1:T}(\lambda^*)]
    \right],
    \label{eq:ftrl_before_minimization}
\end{align}
where we used
$\epsilon+4G^2=\epsilon(1+4G^2/\epsilon)$. It remains to choose $\lambda^*$. Expanding the comparator-dependent
terms above gives
\begin{align*}
    \epsilon\lambda^2+\E[\Phi_{1:T}(\lambda)]
    =
    2\lambda\sum_{t=1}^T\E[\langle\mathbf g_t,\mathbf d_t\rangle]
    +\lambda^2
    \left(\epsilon+\sum_{t=1}^T\E[\|\mathbf d_t\|^2]\right),
\end{align*}
where we recall that
    $\Phi_{1:T}(\lambda)
        :=\sum_{t=1}^T\|\mathbf g_t+\lambda\mathbf d_t\|^2
          -\sum_{t=1}^T\|\mathbf g_t\|^2 =2\lambda\sum_{t=1}^T
              \langle\mathbf g_t,\mathbf d_t\rangle
          +\lambda^2\sum_{t=1}^T\|\mathbf d_t\|^2.$
The above quadratic has a positive quadratic coefficient. Setting its derivative equal to zero gives
\begin{align*}
    \lambda^*
    =
    -\frac{\sum_{t=1}^T\E[\langle\mathbf g_t,\mathbf d_t\rangle]}
          {\epsilon+\sum_{t=1}^T\E[\|\mathbf d_t\|^2]}.
\end{align*}
Substituting this choice into the quadratic gives
\begin{align*}
        \epsilon(\lambda^*)^2+\E[\Phi_{1:T}(\lambda^*)]
    =
    -\frac{
        \left(\sum_{t=1}^T\E[\langle\mathbf g_t,\mathbf d_t\rangle]\right)^2
    }{
        \epsilon+\sum_{t=1}^T\E[\|\mathbf d_t\|^2]
    }.
\end{align*}
Finally, take $\eta=\frac{1}{\beta\sqrt T}$. Since $T\geq2$ and
$\epsilon\geq 4G^2/(\sqrt T-1)$, we have
$4G^2/\epsilon\leq\sqrt T-1$. Hence
\[
    \eta\left[
        1-\frac{\beta\eta}{2}\left(1+\frac{4G^2}{\epsilon}\right)
    \right]
    =\eta\left[
        1-\frac{1+4G^2/\epsilon}{2\sqrt T}
    \right]
    \geq\frac{\eta}{2}.
\]
Applying this lower bound to \oureqref{eq:ftrl_before_minimization},
substituting the comparator value, and dividing by $\eta T/2$ yields
\begin{align*}
    \frac{1}{T}\sum_{t=1}^T\E[\|\nabla\loss(\bw_t)\|^2]
    \leq{}&
    \frac{2\beta\Delta_1}{\sqrt T}
    +\frac{1+4G^2/\epsilon}{\sqrt T}
    \left[
        \sigma^2
        -
        \frac{
            \left(\sum_{t=1}^T
                \E[\langle\mathbf g_t,\mathbf d_t\rangle]\right)^2
        }{
            T\left(\epsilon+\sum_{t=1}^T\E[\|\mathbf d_t\|^2]\right)
        }
    \right].
\end{align*}
It remains to bound the fraction in square brackets.
Since $\bw_t$ is $\F_{t-1}$-measurable and
$\E[\mathbf d_t\mid\F_{t-1}]=0$, the tower property gives
\[
    \E[\langle\nabla\loss(\bw_t),\mathbf d_t\rangle]
    =
    \E\!\left[
        \left\langle\nabla\loss(\bw_t),
        \E[\mathbf d_t\mid\F_{t-1}]\right\rangle
    \right]
    =0.
\]
Using the above property and applying Cauchy--Schwarz first in
expectation and then across rounds, we obtain
\begin{align*}
    \left(
        \sum_{t=1}^T \E[\langle\mathbf g_t,\mathbf d_t\rangle]
    \right)^2
    &=
    \left(
        \sum_{t=1}^T
        \E[\langle\mathbf g_t-\nabla\loss(\bw_t),\mathbf d_t\rangle]
    \right)^2\\
    &\leq
    \left(
        \sum_{t=1}^T
        \sqrt{\E[\|\mathbf g_t-\nabla\loss(\bw_t)\|^2]}
        \sqrt{\E[\|\mathbf d_t\|^2]}
    \right)^2\\
    &\leq
    \left(\sum_{t=1}^T \E[\|\mathbf g_t-\nabla\loss(\bw_t)\|^2]\right)
    \left(\sum_{t=1}^T \E[\|\mathbf d_t\|^2]\right)\\
    &\leq
    T\sigma^2\sum_{t=1}^T \E[\|\mathbf d_t\|^2],
\end{align*}
where the last inequality follows from Assumption~\ref{A3}. Dividing by the positive denominator yields
\begin{align*}
    0
    \leq
    \frac{
        \left(\sum_{t=1}^T
        \E[\langle\mathbf g_t,\mathbf d_t\rangle]\right)^2
    }{
        T\left(\epsilon+\sum_{t=1}^T\E[\|\mathbf d_t\|^2]\right)
    }
    &\leq
    \sigma^2
    \frac{\sum_{t=1}^T\E[\|\mathbf d_t\|^2]}
         {\epsilon+\sum_{t=1}^T\E[\|\mathbf d_t\|^2]}\\
    &\leq
    \sigma^2\frac{4TG^2}{\epsilon+4TG^2},
\end{align*}
where the final inequality uses
$\sum_{t=1}^T\E[\|\mathbf d_t\|^2]\leq 4TG^2$
from Assumption~\ref{A4} and the fact that
$x/(\epsilon+x)$ is increasing for $x\geq0$.
Consequently, the entire bracket satisfies
\[
    \frac{\epsilon\sigma^2}{\epsilon+4TG^2}
    \leq
    \sigma^2-
    \frac{
        \left(\sum_{t=1}^T
        \E[\langle\mathbf g_t,\mathbf d_t\rangle]\right)^2
    }{
        T\left(\epsilon+\sum_{t=1}^T\E[\|\mathbf d_t\|^2]\right)
    }
    \leq \sigma^2.
\]

\end{proof}

For the vector tuning game, Follow-the-Regularized-Leader (FTRL) selects a coefficient for each coordinate of the correction $\mathbf d_t$. Fix $\epsilon>0$ throughout the run.
For $t=1,\ldots,T+1$, play
\begin{align}
    \boldsymbol{\lambda}_t
    \gets
    \underset{\boldsymbol{\lambda}\in\reals^d}{\arg\min}
    \left\{
        \sum_{s=1}^{t-1}\phi_s(\boldsymbol{\lambda})
        +\epsilon\|\boldsymbol{\lambda}\|^2
    \right\},
    \qquad
    \phi_s(\boldsymbol{\lambda})
    :=
    \|\mathbf g_s+\boldsymbol{\lambda}\odot\mathbf d_s\|^2.
    \label{eq:vector_ftrl_update}
\end{align}
The empty-sum convention gives $\boldsymbol{\lambda}_1=\mathbf 0$. Both the loss and the regularizer are separable across coordinates since
\begin{align*}
    \sum_{s=1}^{t-1}\phi_s(\boldsymbol{\lambda})
    +\epsilon\|\boldsymbol{\lambda}\|^2
    =
    \sum_{i=1}^d
    \left[
        \sum_{s=1}^{t-1}
        \bigl(\mathbf g_s[i]+\boldsymbol{\lambda}[i]\mathbf d_s[i]\bigr)^2
        +\epsilon\boldsymbol{\lambda}[i]^2
    \right].
\end{align*}
Minimizing each scalar quadratic separately gives
\begin{align}
    \boldsymbol{\lambda}_t[i]
    =
    -\frac{\sum_{s=1}^{t-1}\mathbf g_s[i]\mathbf d_s[i]}
          {\epsilon+\sum_{s=1}^{t-1}\mathbf d_s[i]^2},
    \qquad i=1,\ldots,d.
    \label{eq:vector_ftrl_closed_form}
\end{align}
Since the regret in the vector-tuning scheme is the sum of the coordinatewise regrets, the corresponding guarantee follows by applying the scalar FTRL analysis to each coordinate and summing the resulting bounds. We therefore omit the formal statement and proof for vector tuning.

\subsection{Zeroth-order Optimization}

For the scalar tuning game in zeroth-order optimization, we use FTRL to choose the
coefficient of the correction $\mathbf d_t^0$. Specifically, fix
$\epsilon>0$ throughout the run and, for $t=1,\ldots,T+1$, play
\begin{align}
    \lambda_t
    \gets
    \underset{\lambda\in\reals}{\arg\min}
    \left\{
        \sum_{s=1}^{t-1}\psi_s(\lambda)+\epsilon\lambda^2
    \right\},
    \qquad
    \psi_s(\lambda)
    :=
    \|\mathbf q_s+\lambda\mathbf d_s^0\|^2.
    \label{eq:zo_scalar_ftrl}
\end{align}
The empty-sum convention gives $\lambda_1=0$. Differentiating the
objective with respect to $\lambda$ and setting the derivative equal
to zero gives
\[
    \lambda_t
    =
    -\frac{\sum_{s=1}^{t-1}\langle\mathbf q_s,\mathbf d_s^0\rangle}
          {\epsilon+\sum_{s=1}^{t-1}\|\mathbf d_s^0\|^2}.
\]
We use the notation of the scalar tuning game in the zeroth-order case:
$\mathbf q_t=\qt{y}{\bw_t}$,
$\mathbf d_t^0=\tildeqt{f}{\bw_t}-\qt{f}{\bw_t}$, and
$\q_{t,\bw_t}^{\lambda_t}=\mathbf q_t+\lambda_t\mathbf d_t^0$.
The stochastic gradient update is
\begin{equation}
    \bw_{t+1}=\bw_t-\eta_t\q_{t,\bw_t}^{\lambda_t},
    \label{eq:sgfm-update}
\end{equation}
where 
\begin{align*}
    \E \left[ \mathcal{Q}_{t,\bw}^{\boldsymbol{\lambda}_t} \mid\F_{t-1} \right] = \nabla \loss_{\mu}(\bw)
    \qquad
    \loss_{\mu}(\bw) := \mathbb{E}_{\zeta} \left[ \loss( \bw + \mu \zeta ) \right].
\end{align*}

\begin{mdframed}[backgroundcolor=black!10,rightline=false,leftline=false,
topline=false,bottomline=false]
\begin{theorem}[Zeroth-order scalar tuning with FTRL]
Suppose Assumptions~\ref{A1}--\ref{A4} hold, and choose any $\epsilon \geq \frac{128G^2(d+6\log(eT))^2}{3\max\{1,\sqrt{T/(d+2)}\}-2}$. Run
\oureqref{eq:sgfm-update} with scalar FTRL tuning \oureqref{eq:zo_scalar_ftrl}
with step size $ \eta= \min\left\{\frac{1}{4\beta(d+2)},\frac{1}{4\beta\sqrt{(d+2)T}}\right\}$ and $\mu = \frac{1}{\beta(d+3)\sqrt T}$.
Then
\begin{align*}
    \textstyle &\frac1T\sum_{t=1}^T\E[\|\nabla\loss(\bw_t)\|^2]
    \leq{}
    32\left(\beta\Delta_1+\frac{d}{(d+3)^2T}\right)
    \left(\frac{d+2}{T}+\sqrt{\frac{d+2}{T}}\right)
    +\frac{2(d+3)}{T}\\
    \textstyle  &+
    2\sqrt{\frac{d+2}{T}}
    \left(1+\frac{64G^2(d+6\log(eT))^2}{\epsilon}\right)
    \left[
        \sigma^2+\frac{d(d+4)}{4(d+3)^2T}
        -
        \frac{
            \left(\sum_{t=1}^T
            \E[\langle\mathbf q_t,\mathbf d_t^0\rangle]\right)^2
        }{
            2(d+2)T
            \left(\epsilon+\sum_{t=1}^T\E[\|\mathbf d_t^0\|^2]\right)
        }
    \right]\\
    \textstyle  &+
    \frac{G^2}{T^2\max\{d+2,\sqrt{(d+2)T}\}}
    \left[(d+3)^2+\frac{4G^2}{\epsilon}(d+7)^4\right].
\end{align*}
where
\begin{align*}
    \textstyle  0
    \leq
    \frac{
        \left(\sum_{t=1}^T
        \E[\langle\mathbf q_t,\mathbf d_t^0\rangle]\right)^2
    }{
        2(d+2)T
        \left(\epsilon+\sum_{t=1}^T\E[\|\mathbf d_t^0\|^2]\right)
    }
    \leq
    \frac{4G^2T\,d(d+2)}{\epsilon+4G^2T\,d(d+2)}
    \left(\sigma^2+\frac{d(d+4)}{4(d+3)^2T}\right).
\end{align*}
\end{theorem}
\end{mdframed}

\begin{proof}
From \oureqref{eq: zo-scalar-grad-bound}, we have
\begin{align*}
    \eta\sum_{t=1}^T\E[\|\nabla\loss(\bw_t)\|^2]
    \leq{}&
    2\Delta_1+2\mu^2\beta d
    +\frac{\mu^2\beta^2(d+3)^3}{2}\eta T
    +\beta\eta^2\sum_{t=1}^T
    \E[\|\mathcal Q_{t,\bw_t}^{\lambda_t}\|^2],
\end{align*}
where
\[
    \|\mathcal Q_{t,\bw_t}^{\lambda_t}\|^2
    =\|\mathbf q_t+\lambda_t\mathbf d_t^0\|^2
    =\psi_t(\lambda_t).
\]
We first bound the sum of these squared norms. The losses
$\psi_t(\lambda)=\|\mathbf q_t+\lambda\mathbf d_t^0\|^2$ have the same quadratic form as the losses in the scalar FTRL bound in first-order optimization
\oureqref{eq:ftrl_pathwise}. We may therefore apply that bound with
$(\mathbf g_t,\mathbf d_t)$ replaced by $(\mathbf q_t,\mathbf d_t^0)$. Thus, for every fixed $\lambda^*\in\reals$,
\begin{align*}
    \sum_{t=1}^T \psi_t(\lambda_t)-\sum_{t=1}^T \psi_t(\lambda^*)
    &\leq
    \epsilon(\lambda^*)^2
    +\frac{\max_{1\le s\le T}\|\mathbf d_s^0\|^2}{\epsilon}
    \left(
        \epsilon(\lambda^*)^2+\sum_{t=1}^T\psi_t(\lambda^*)
    \right).
\end{align*}
Adding $\sum_{t=1}^T\psi_t(\lambda^*)$ to both sides and factoring out
$\epsilon(\lambda^*)^2+\sum_{t=1}^T\psi_t(\lambda^*)$ gives
\begin{align}
    \sum_{t=1}^T\psi_t(\lambda_t)
    &\leq
    \left(
        1+\frac{\max_{1\le s\le T}\|\mathbf d_s^0\|^2}{\epsilon}
    \right)
    \left(
        \epsilon(\lambda^*)^2+\sum_{t=1}^T\psi_t(\lambda^*)
    \right).
    \label{eq:zo_ftrl_gaussian_pathwise}
\end{align}
As in the proof for scalar FTL, Assumption~\ref{A4} gives
\[
    \left|
    \frac{\ell(\bw_t+\mu\zeta_t;\bx,z)-\ell(\bw_t;\bx,z)}{\mu}
    \right|
    \leq G\|\zeta_t\|.
\]
Averaging over the corresponding minibatches and using the common
direction $\zeta_t$, we obtain
\begin{align}
    \|\mathbf q_t\|\leq G\|\zeta_t\|^2,
    \qquad
    \|\mathbf d_t^0\|\leq 2G\|\zeta_t\|^2.
    \label{eq:zo_ftrl_gaussian_direction_bound}
\end{align}
Define
\[
    b_T:=4\bigl(d+6\log(eT)\bigr),
    \qquad
    I_T:=\bigl\{t\in[T]:\|\zeta_t\|^2\leq b_T\bigr\}.
\]
Since $\zeta_t\sim\mathcal N(\mathbf 0,I_d)$, its squared norm
$\|\zeta_t\|^2$ follows a $\chi_d^2$ distribution. Evaluating its
moment-generating function at $1/4$ gives
$\E[\exp(\|\zeta_t\|^2/4)]=2^{d/2}$. Thus, by Markov's inequality,
\[
    \mathbb P(\|\zeta_t\|^2>b_T)
    \leq e^{-b_T/4}\E[\exp(\|\zeta_t\|^2/4)]
    = 2^{d/2}e^{-b_T/4} \leq T^{-6},
\]
where the last inequality follows by substituting $b_T$. 
Since the multiplier in \oureqref{eq:zo_ftrl_gaussian_pathwise} involves the maximum over all $T$ rounds, we split the expectation according to whether $I_T=[T]$ or $I_T\neq[T]$. Since $I_T\neq[T]$ precisely when at least one direction exceeds the
threshold, a union bound gives
\begin{align}
    \mathbb P(I_T\neq[T])\leq T^{-5}
    \label{eq: event_prob}
\end{align}
First, consider the event $\{I_T=[T]\}$. Then,
\eqref{eq:zo_ftrl_gaussian_direction_bound} gives
\[
    \max_{1\leq t\leq T}\|\mathbf d_t^0\|
    \leq 2G\max_{1\leq t\leq T}\|\zeta_t\|^2
    \leq 2G b_T.
\]
Thus, for every fixed comparator $\lambda^*\in\reals$,
\oureqref{eq:zo_ftrl_gaussian_pathwise} implies
\begin{align*}
    \E\!\left[
        \mathbf 1_{\{I_T=[T]\}}\sum_{t=1}^T\psi_t(\lambda_t)
    \right]
    &\leq
    \left(1+\frac{4G^2 b_T^2}{\epsilon}\right)
    \E\!\left[
        \mathbf 1_{\{I_T=[T]\}}
        \left(
            \epsilon(\lambda^*)^2+\sum_{t=1}^T\psi_t(\lambda^*)
        \right)
    \right]\\
    &\leq
    \left(1+\frac{4G^2 b_T^2}{\epsilon}\right)
    \left(
        \epsilon(\lambda^*)^2+\E[\Psi_{1:T}(\lambda^*)]
        +\sum_{t=1}^T\E[\|\mathbf q_t\|^2]
    \right),
\end{align*}
where in the second inequality, we drop the indicator since the regularized comparator loss is nonnegative, and then use the definition of $\Psi_{1:T}(\lambda^*)$.\\
\noindent
Next, consider the event $\{I_T\neq[T]\}$. Taking $\lambda^*=0$ in
\oureqref{eq:zo_ftrl_gaussian_pathwise} and using
$\psi_t(0)=\|\mathbf q_t\|^2$, we obtain
\begin{align}
    \E\!\left[
        \mathbf 1_{\{I_T\neq[T]\}}\sum_{t=1}^T\psi_t(\lambda_t)
    \right]
    &\leq
    \E\!\left[
        \mathbf 1_{\{I_T\neq[T]\}}
        \left(
            1+\frac{\max_{1\leq s\leq T}\|\mathbf d_s^0\|^2}{\epsilon}
        \right)
        \sum_{t=1}^T\|\mathbf q_t\|^2
    \right] \notag \\
    &\leq
    G^2\E\!\left[
        \mathbf 1_{\{I_T\neq[T]\}}
        \left(
            1+\frac{4G^2}{\epsilon}\max_{1\leq s\leq T}\|\zeta_s\|^4
        \right)
        \sum_{t=1}^T\|\zeta_t\|^4
    \right],
    \label{eq: sum_psi_event}
\end{align}
where the last inequality follows from
\oureqref{eq:zo_ftrl_gaussian_direction_bound}.
Furthermore, for $k\in\{4,8\}$, the Cauchy-Schwarz inequality and \oureqref{eq: event_prob} gives
\begin{align}
    \E\!\left[
        \mathbf 1_{\{I_T\neq[T]\}}
        \max_{1\leq t\leq T}\|\zeta_t\|^k
    \right]
    &\leq
    \mathbb P(I_T\neq[T])^{1/2}
    \left(
        \E\!\left[\max_{1\leq t\leq T}\|\zeta_t\|^{2k}\right]
    \right)^{1/2} \notag \\
    &\leq
    T^{-5/2}
    \left(\sum_{t=1}^T\E[\|\zeta_t\|^{2k}]\right)^{1/2}
    \notag \\
    &=
    \frac{\bigl(\E[\|\zeta_1\|^{2k}]\bigr)^{1/2}}{T^2}.
    \notag
\end{align}
Using the Gaussian moment formula and the arithmetic-geometric mean inequality, we further have
\[
    \bigl(\E[\|\zeta_1\|^{2k}]\bigr)^{1/2}
    =
    \left(\prod_{j=0}^{k-1}(d+2j)\right)^{1/2}
    \leq (d+k-1)^{k/2}.
\]
Therefore, 
\begin{align*}
    \E\!\left[
        \mathbf 1_{\{I_T\neq[T]\}}
        \max_{1\leq t\leq T}\|\zeta_t\|^k
    \right]
    &\leq
    \frac{(d+k-1)^{k/2}}{T^2}.
\end{align*}
Since $\sum_{t=1}^T\|\zeta_t\|^4
\leq T\max_{1\leq t\leq T}\|\zeta_t\|^4$, applying the above inequality with $k=4$ and $k=8$ in \oureqref{eq: sum_psi_event} yields
\begin{align*}
    \E\!\left[
        \mathbf 1_{\{I_T\neq[T]\}}\sum_{t=1}^T\psi_t(\lambda_t)
    \right]
    &\leq TG^2\E\!\left[
        \mathbf 1_{\{I_T\neq[T]\}}
        \left(
            \max_{1\leq t\leq T}\|\zeta_t\|^4
            +\frac{4G^2}{\epsilon}
             \max_{1\leq t\leq T}\|\zeta_t\|^8
        \right)
    \right]\\
    &\leq \frac{G^2}{T}
    \left[(d+3)^2+\frac{4G^2}{\epsilon}(d+7)^4\right].
\end{align*}

Adding the contributions from $\{I_T=[T]\}$ and $\{I_T\neq[T]\}$ gives
\begin{align}
    &\sum_{t=1}^T\E[\|\mathcal Q_{t,\bw_t}^{\lambda_t}\|^2]
    = \E\left[ \sum_{t=1}^T \psi_t(\lambda_t) \right] \\
    &\leq
    \left(1+\frac{4G^2b_T^2}{\epsilon}\right)
    \left[
        \epsilon(\lambda^*)^2+\E[\Psi_{1:T}(\lambda^*)]
        +\sum_{t=1}^T\E[\|\mathbf q_t\|^2]
    \right] +\frac{G^2}{T}
    \left[(d+3)^2+\frac{4G^2}{\epsilon}(d+7)^4\right].
    \label{eq:zo_ftrl_gaussian_second_moment}
\end{align}
Moreover, by Lemma~\ref{lem:exp_q_and_diff_q_and_bar_q}
and Assumption~\ref{A3} we have
\begin{align*}
    \E[\|\mathbf q_t\|^2]
    &\leq 2(d+2)\E[\|\mathbf g_t\|^2]
    +\frac{\mu^2\beta^2}{2}d(d+2)(d+4)\\
    &\leq 2(d+2)\E[\|\nabla\loss(\bw_t)\|^2]
    +2(d+2)\left(\sigma^2+\frac{\mu^2\beta^2d(d+4)}4\right),
\end{align*}
where the second inequality uses
$\E[\mathbf g_t\mid\F_{t-1}]=\nabla\loss(\bw_t)$ and Assumption~\ref{A3}. Substituting the above into
\oureqref{eq:zo_ftrl_gaussian_second_moment} and then into
\oureqref{eq: zo-scalar-grad-bound}, and collecting the gradient terms
on the left, gives
\begin{align}
    &\eta\left[
        1-2\beta\eta(d+2)
        \left(1+\frac{4G^2b_T^2}{\epsilon}\right)
    \right]
    \sum_{t=1}^T\E[\|\nabla\loss(\bw_t)\|^2]\notag\\
    &\quad\leq
    2\Delta_1+2\mu^2\beta d
    +\frac{\mu^2\beta^2(d+3)^3}{2}\eta T\notag\\
    &\qquad+
    \beta\eta^2\left(1+\frac{4G^2b_T^2}{\epsilon}\right)
    \left[
        \epsilon(\lambda^*)^2+\E[\Psi_{1:T}(\lambda^*)]
        +2(d+2)T\left(\sigma^2+\frac{\mu^2\beta^2d(d+4)}4\right)
    \right]\notag\\
    &\qquad+
    \frac{\beta\eta^2G^2}{T}
    \left[(d+3)^2+\frac{4G^2}{\epsilon}(d+7)^4\right].
    \label{eq:zo_ftrl_gaussian_before_minimization}
\end{align}
Next, we choose the comparator $\lambda^*$ as the minimizer of the right-hand side. By the definition of $\Psi_{1:T}$,
\[
    \epsilon\lambda^2+\E[\Psi_{1:T}(\lambda)]
    =
    2\lambda\sum_{t=1}^T\E[\langle\mathbf q_t,\mathbf d_t^0\rangle]
    +\lambda^2
    \left(\epsilon+\sum_{t=1}^T\E[\|\mathbf d_t^0\|^2]\right).
\]
Since $\epsilon>0$, this quadratic has the unique minimizer
\[
    \lambda^*
    =
    -\frac{\sum_{t=1}^T\E[\langle\mathbf q_t,\mathbf d_t^0\rangle]}
          {\epsilon+\sum_{t=1}^T\E[\|\mathbf d_t^0\|^2]}.
\]
For this choice,
\begin{align}
    \epsilon(\lambda^*)^2+\E[\Psi_{1:T}(\lambda^*)]
    =
    -\frac{
        \left(\sum_{t=1}^T\E[\langle\mathbf q_t,\mathbf d_t^0\rangle]\right)^2
    }{
        \epsilon+\sum_{t=1}^T\E[\|\mathbf d_t^0\|^2]
    }.
    \label{eq:zo_ftrl_gaussian_comparator}
\end{align}

With $\eta=1/(4\beta\max\{d+2,\sqrt{(d+2)T}\})$ and
\[
    \epsilon\geq
    \frac{4G^2b_T^2}
         {\tfrac32\max\{1,\sqrt{T/(d+2)}\}-1},
\]
we have
\[
    2\beta\eta(d+2)
    \left(1+\frac{4G^2b_T^2}{\epsilon}\right)
    =
    \frac{1+4G^2b_T^2/\epsilon}
         {2\max\{1,\sqrt{T/(d+2)}\}}
    \leq \frac34.
\]
Consequently, the coefficient of the gradient sum in
\oureqref{eq:zo_ftrl_gaussian_before_minimization} is at least
$\eta/4$. Substituting \oureqref{eq:zo_ftrl_gaussian_comparator}
and dividing by $\eta T/4$ therefore gives
\begin{align}
    &\frac1T\sum_{t=1}^T\E[\|\nabla\loss(\bw_t)\|^2] \notag \\
    &\leq{}
    \frac{8(\Delta_1+\mu^2\beta d)}{\eta T}
    +2\mu^2\beta^2(d+3)^3 \notag \\
    &\quad+
    8\beta\eta(d+2)
    \left(1+\frac{4G^2b_T^2}{\epsilon}\right)
    \left[
        \sigma^2+\frac{\mu^2\beta^2d(d+4)}4
        -
        \frac{
            \left(\sum_{t=1}^T
            \E[\langle\mathbf q_t,\mathbf d_t^0\rangle]\right)^2
        }{
            2(d+2)T
            \left(\epsilon+\sum_{t=1}^T\E[\|\mathbf d_t^0\|^2]\right)
        }
    \right] \notag \\
    &\quad+
    \frac{4\beta\eta G^2}{T^2}
    \left[(d+3)^2+\frac{4G^2}{\epsilon}(d+7)^4\right].
    \label{eq: conv_bound_FTRL_zo}
\end{align}

To bound the term in square brackets, we recall from the proof of
Lemma~\ref{lem:unbiasedness} that
$\E[\mathbf d_t^0\mid\F_{t-1},\zeta_t]=0$. Hence,
\[
    \E[\langle\mathbf q_t,\mathbf d_t^0\rangle]
    =
    \E\!\left[
        \left\langle
            \mathbf q_t-\E[\mathbf q_t\mid\F_{t-1},\zeta_t],
            \mathbf d_t^0
        \right\rangle
    \right].
\]
By Cauchy--Schwarz and the centered second-moment bound in
Lemma~\ref{lem:exp_q_and_diff_q_and_bar_q},
\begin{align*}
    \left(\sum_{t=1}^T
        \E[\langle\mathbf q_t,\mathbf d_t^0\rangle]\right)^2
    &\leq
    \left(\sum_{t=1}^T
        \E\!\left[
            \|\mathbf q_t-\E[\mathbf q_t\mid\F_{t-1},\zeta_t]\|^2
        \right]\right)
    \left(\sum_{t=1}^T\E[\|\mathbf d_t^0\|^2]\right)\\
    &\leq
    2(d+2)T\left(\sigma^2+\frac{\mu^2\beta^2d(d+4)}4\right)
    \sum_{t=1}^T\E[\|\mathbf d_t^0\|^2].
\end{align*}
Moreover, \oureqref{eq:zo_ftrl_gaussian_direction_bound} and the Gaussian
fourth moment give
\[
    \sum_{t=1}^T\E[\|\mathbf d_t^0\|^2]
    \leq 4G^2\sum_{t=1}^T\E[\|\zeta_t\|^4]
    =4G^2T\,d(d+2).
\]
Since $x/(\epsilon+x)$ is increasing for $x\geq0$, it follows that
\[
    0\leq
    \frac{
        \left(\sum_{t=1}^T\E[\langle\mathbf q_t,\mathbf d_t^0\rangle]\right)^2
    }{
        2(d+2)T
        \left(\epsilon+\sum_{t=1}^T\E[\|\mathbf d_t^0\|^2]\right)
    }
    \leq
    \frac{4G^2T\,d(d+2)}{\epsilon+4G^2T\,d(d+2)}
    \left(\sigma^2+\frac{\mu^2\beta^2d(d+4)}4\right).
\]
Taking $\mu=1/(\beta(d+3)\sqrt T)$ gives $\mu^2\beta^2d(d+4)/4=d(d+4)/(4(d+3)^2T)$. Hence,
\begin{align*}
    0
    &\leq
    \frac{
        \left(\sum_{t=1}^T
        \E[\langle\mathbf q_t,\mathbf d_t^0\rangle]\right)^2
    }{
        2(d+2)T
        \left(\epsilon+\sum_{t=1}^T\E[\|\mathbf d_t^0\|^2]\right)
    }\\
    &\leq
    \frac{\sum_{t=1}^T\E[\|\mathbf d_t^0\|^2]}
         {\epsilon+\sum_{t=1}^T\E[\|\mathbf d_t^0\|^2]}
    \left(\sigma^2+\frac{d(d+4)}{4(d+3)^2T}\right)\\
    &\leq
    \frac{4G^2T\,d(d+2)}{\epsilon+4G^2T\,d(d+2)}
    \left(\sigma^2+\frac{d(d+4)}{4(d+3)^2T}\right),
\end{align*}
where the last inequality uses
$\sum_{t=1}^T\E[\|\mathbf d_t^0\|^2]\leq 4G^2T\,d(d+2)$
and the monotonicity of $x/(\epsilon+x)$ for $x\geq0$. Thus, the quantity in square brackets in \oureqref{eq: conv_bound_FTRL_zo} is nonnegative. Furthermore, using
\begin{align*}
        &\frac{2(d+2)}{\max\{d+2,\sqrt{(d+2)T}\}}
        =2\min\left\{1,\sqrt{\frac{d+2}{T}}\right\}
        \leq 2\sqrt{\frac{d+2}{T}}, \\
        &4\beta\max\left\{\frac{d+2}{T},\sqrt{\frac{d+2}{T}}\right\}
      \leq 4\beta\left(\frac{d+2}{T}+\sqrt{\frac{d+2}{T}}\right)
\end{align*}
and substituting $b_T$ in \oureqref{eq: conv_bound_FTRL_zo} yields
\begin{align*}
    &\frac1T\sum_{t=1}^T\E[\|\nabla\loss(\bw_t)\|^2] \\
    &\leq{}
    32\left(\beta\Delta_1+\frac{d}{(d+3)^2T}\right)
    \left(\frac{d+2}{T}+\sqrt{\frac{d+2}{T}}\right)
    +\frac{2(d+3)}{T}\\
    &\quad +
    2\sqrt{\frac{d+2}{T}}
    \left(1+\frac{64G^2[d+6\log(eT)]^2}{\epsilon}\right)
    \left[
        \sigma^2+\frac{d(d+4)}{4(d+3)^2T}
        -
        \frac{
            \left(\sum_{t=1}^T
            \E[\langle\mathbf q_t,\mathbf d_t^0\rangle]\right)^2
        }{
            2(d+2)T
            \left(\epsilon+\sum_{t=1}^T\E[\|\mathbf d_t^0\|^2]\right)
        }
    \right]\\
    &\quad +
    \frac{G^2}{T^2\max\{d+2,\sqrt{(d+2)T}\}}
    \left[(d+3)^2+\frac{4G^2}{\epsilon}(d+7)^4\right].
\end{align*}

\end{proof}

For the zeroth-order vector tuning game, we can use FTRL to choose one
coefficient for each coordinate of the correction $\mathbf d_t^0$. Fix $\epsilon>0$
throughout the run and, for $t=1,\ldots,T+1$ play
\begin{align}
    \boldsymbol{\lambda}_t
    \gets
    \underset{\boldsymbol{\lambda}\in\reals^d}{\arg\min}
    \left\{
        \sum_{s=1}^{t-1}\psi_s(\boldsymbol{\lambda})
        +\epsilon\|\boldsymbol{\lambda}\|^2
    \right\},
    \qquad
    \psi_s(\boldsymbol{\lambda})
    :=
    \|\mathbf q_s+\boldsymbol{\lambda}\odot\mathbf d_s^0\|^2.
    \label{eq:zo_vector_ftrl}
\end{align}
The empty-sum convention gives $\boldsymbol{\lambda}_1=\mathbf 0$.
Since both the loss and the regularizer are sums over coordinates,
the minimization separates across coordinates. For each
$i=1,\ldots,d$, differentiating with respect to
$\boldsymbol{\lambda}[i]$ and setting the derivative equal to zero gives
\[
    \sum_{s=1}^{t-1}
    \bigl(\mathbf q_s[i]+\boldsymbol{\lambda}_t[i]\mathbf d_s^0[i]\bigr)
    \mathbf d_s^0[i]
    +\epsilon\boldsymbol{\lambda}_t[i]=0.
\]
Solving for $\boldsymbol{\lambda}_t[i]$ yields
\begin{align}
    \boldsymbol{\lambda}_t[i]
    =
    -\frac{\sum_{s=1}^{t-1}\mathbf q_s[i]\mathbf d_s^0[i]}
          {\epsilon+\sum_{s=1}^{t-1}\mathbf d_s^0[i]^2}.
    \label{eq:zo_vector_ftrl_closed_form}
\end{align}
Similar to the first-order analysis, the regret in the vector tuning scheme is the sum of the coordinatewise regrets, and thus the corresponding guarantee follows by applying the scalar FTRL analysis to each coordinate and summing the resulting bounds. We therefore omit the formal statement and proof for vector tuning.

\newpage

\section{Experimental Setup} \label{app:D}

\subsection{Text Classification on SST-2 and SST-5}

\paragraph{Tasks and model.}
We conduct experiments with RoBERTa-large (355M parameters) \citep{liu2019robertarobustlyoptimizedbert} on text classification. Following the experimental setups of \citet{malladi2023fine, zhang2026zeroth}, we consider few-shot fine-tuning with $K=32$ labeled examples per class on two text classification benchmarks: SST-2 \citep{socher2013recursive, wang2018glue}, a binary sentiment classification task, and SST-5 \citep{socher2013recursive}, a five-class fine-grained sentiment classification task. We perform full-parameter fine-tuning, updating all model parameters, including the embedding layers, transformer blocks, and masked-language-model head. For a detailed description of the downstream tasks and prompts we refer the reader to \citet{malladi2023fine}.

\paragraph{Dataset.}
For each task, we construct the labeled training and validation sets by independently sampling $K=32$ examples per class from the original training split. Thus, SST-2 contains a total of $64$ examples in each of the training and validation sets, whereas SST-5 contains a total of $160$ examples in each set. For the test set, we use the official SST-2 validation split of 872 examples, since the labels for the GLUE test split are not publicly available. For SST-5, we use a class-balanced subset of $1{,}000$ examples drawn from the official labeled test split.

\paragraph{ML-labeler.}
To construct the ML-labeler $f(\cdot)$, we fine-tune a separate RoBERTa-large model for each task using a disjoint subset of the original training data, with 128 labeled examples per class. The labeler is trained for 500 optimization steps using Adam with a batch size of $8$, a learning rate of $1e-5$, and a linear learning-rate decay schedule. The resulting labelers achieve test accuracies of $92.78\%$ on SST-2, and $58.30\%$ on SST-5. All remaining examples from the original SST-2 and SST-5 training sets, excluding those used for the training, validation, test, and labeler-training sets, are treated as unlabeled data.  

\paragraph{Loss functions.}
For each labeled and unlabeled input, we pass the example through the corresponding fixed ML-labeler model and extract its logits. We apply a softmax to the ML-labeler logits to obtain a probability distribution over the task labels, which we cache and use as soft targets throughout fine-tuning. On labeled examples, the supervised loss is computed using the standard cross-entropy loss between the model predictions and the ground-truth class labels. The labeler-based loss terms, computed on both labeled and unlabeled examples, use the KL divergence between the ML-labeler distribution and the predictive distribution of the model being fine-tuned. Since the ML-labeler remains fixed throughout fine-tuning, its predictions can be precomputed and cached, avoiding additional forward passes through the ML-labeler during optimization.

\paragraph{Randomness.}
We report results across five random seeds.
% , $\{13, 21, 42, 87, 100\}$. 
Randomness arises from the sampling of the training and validation sets, the sampling of labeled and unlabeled mini-batches during fine-tuning, the initialization of task-specific parameters, and, for zeroth-order methods, the Gaussian search directions $\zeta_t \sim N(0,I_d)$.

\subsubsection{First-order methods} \label{app:D.1.1}

For the first-order experiments, we fine-tune the model for $1{,}000$ steps and compare four methods: standard SGD using only labeled data, PP-SSL with Online Tuning \citep{shoham2025prediction}, and the proposed
Algorithm~\ref{alg:sgd} with the proposed scalar tuning and vector tuning schemes. For brevity, we refer to the latter two methods as Scalar FTL and Vector FTL when presenting the experimental results below.
We use a labeled batch size of $n=16$ for all methods. For PP-SSL and our methods, we additionally use an unlabeled batch size of $N=32$. Following \citet{malladi2023fine}, we use a linear learning-rate decay schedule for all methods. For each method, we independently tune the initial learning rate over $\{1e-5, 5e-5, 1e-4, 5e-4\}$. Learning rates are selected based on the mean validation accuracy across seeds and are reported in \tabref{tab:sst_fo_selected_lr}.

\paragraph{Results.}
Figures~\ref{fig:sst_fo_acc} and \ref{fig:sst_fo_lam} report the test-accuracy results and corresponding power-tuning parameter trajectories for the first-order methods on SST-2 and SST-5.

On SST-2, PP-SSL, Scalar FTL, and Vector FTL all improve over the SGD baseline, with Scalar FTL achieving slightly higher mean test accuracy over most of the later stages of training. 
%However, the differences among the three ML-labeled methods remain modest relative to the across-seed variability. 
We find that the values of the tuning parameters of our methods generally decrease smoothly over training. In contrast, the PP-SSL tuning parameter exhibits noticeably larger fluctuations throughout training, despite achieving comparable test accuracy.

On SST-5, the differences between the methods are slightly more pronounced. Vector FTL achieves the strongest performance overall, whereas Scalar FTL remains comparable to SGD without a noticeable improvement. This behavior is consistent with the tuning trajectories: the Scalar FTL tuning parameter quickly decays toward zero, effectively reducing the contribution of the ML-labeled correction, while the Vector FTL tuning parameter retains substantial coordinatewise variation for much longer, allowing it to continue exploiting informative directions.

\begin{figure}[h!]
    \centering
    \includegraphics[width=0.98\linewidth]{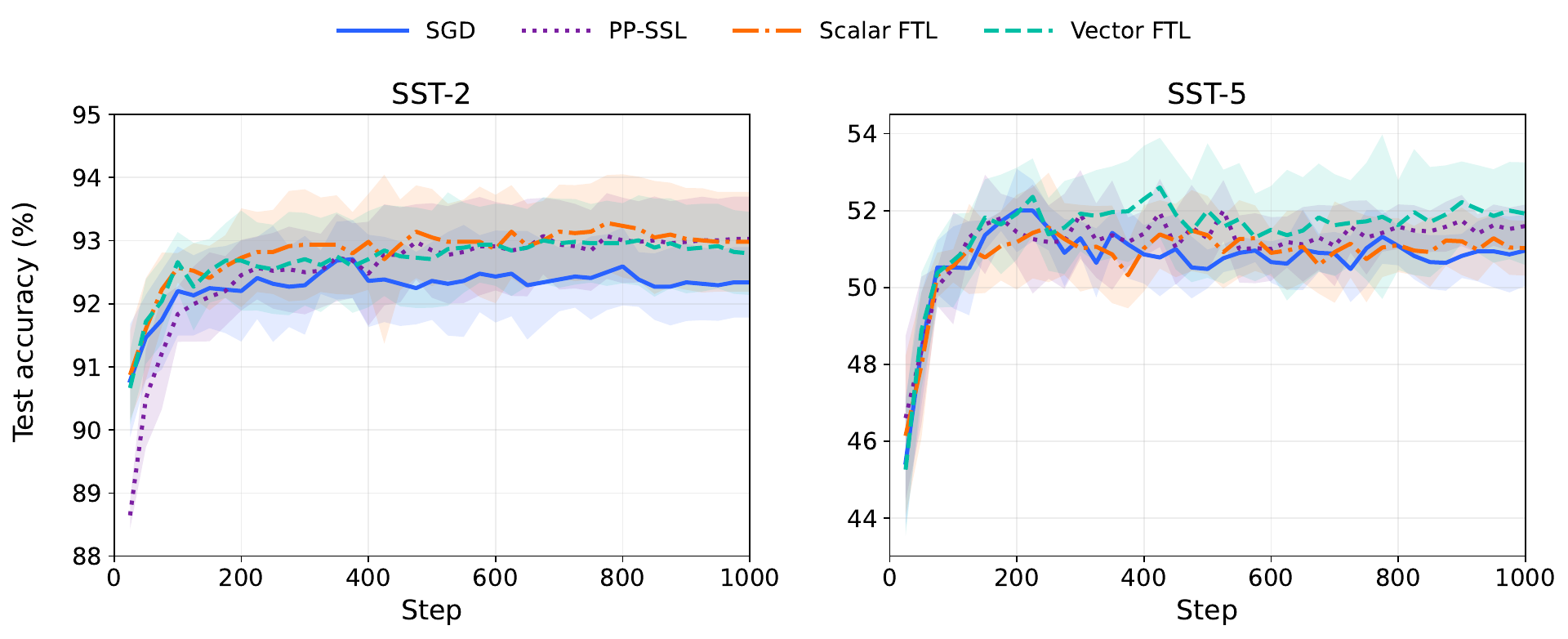}
    \vspace{-3mm}
    \caption{Test accuracy (\%) over training steps for the first-order methods on SST-2 (left) and SST-5 (right). Test accuracy is evaluated every $25$ optimization steps. Solid curves denote the mean across five random seeds, and shaded regions indicate $\pm 1$ standard deviation.}
    \label{fig:sst_fo_acc}
\end{figure}

\begin{figure}[h!]
    \centering
    \includegraphics[width=0.98\linewidth]{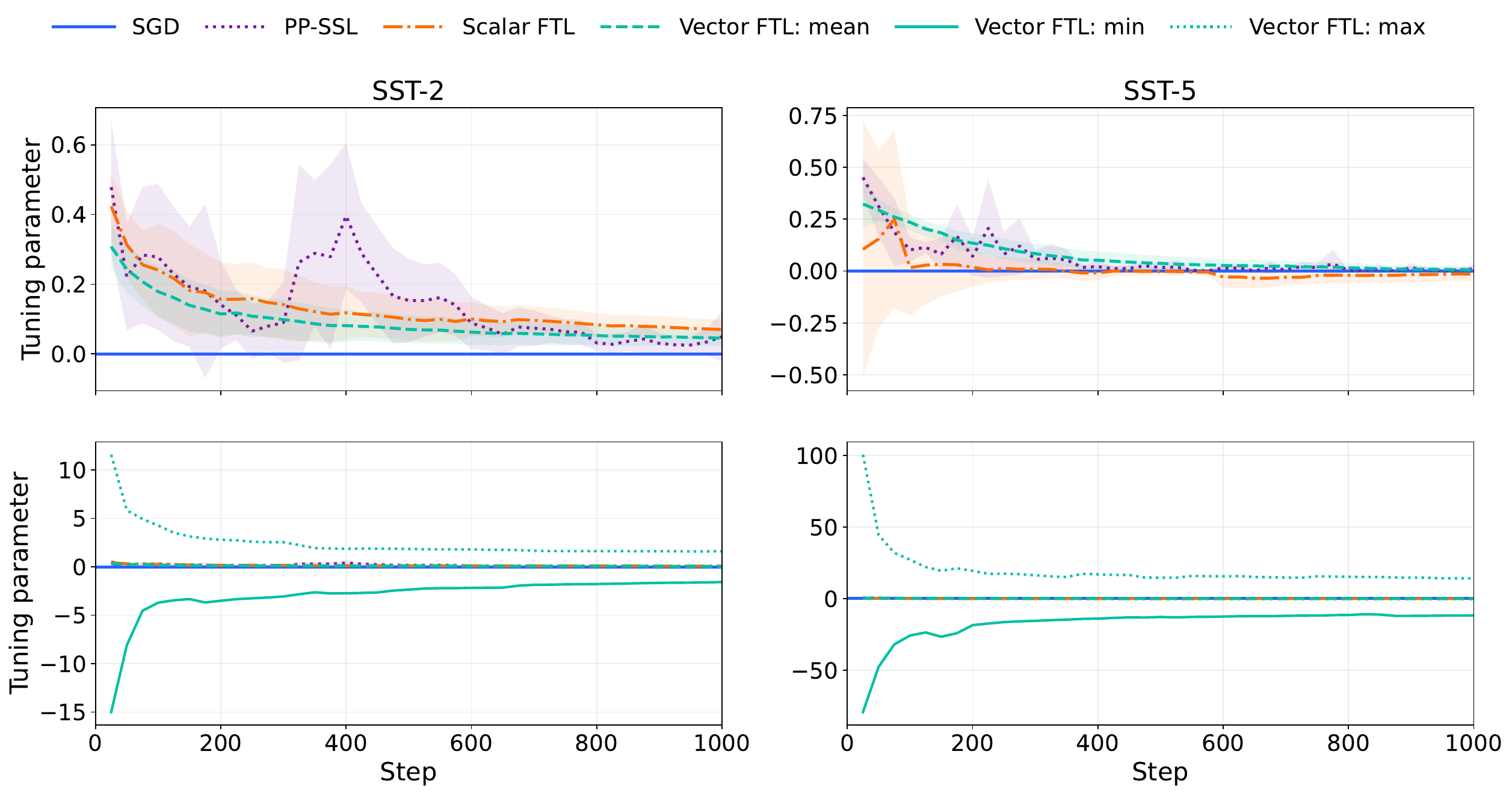}
    \vspace{-3mm}
    \caption{
    Tuning parameter trajectories for the first-order methods on SST-2 (left) and SST-5 (right). Tuning parameters are recorded every $25$ optimization steps. \textbf{Top row}: Power-tuning parameter trajectories over optimization steps. \textbf{Bottom row}: Power-tuning parameter trajectories together with the across-seed means of the minimum and maximum coordinatewise Vector FTL parameters. Curves show the mean across five random seeds, with shaded regions indicating $\pm 1$ standard deviation.
    }
    \label{fig:sst_fo_lam}
\end{figure}

\begin{table}[h!]
\centering
\label{tab:sst_fo_selected_lr_original}
\simpletablelabel{tab:sst_fo_selected_lr}
\vspace{2mm}
\begin{tabular}{ccc}
\toprule
\textbf{Method} & \textbf{SST-2} & \textbf{SST-5} \\
\midrule
SGD        & $1e-4$ & $1e-4$ \\
PP-SSL     & $5e-5$ & $1e-4$ \\
Scalar FTL & $1e-4$ & $1e-4$ \\
Vector FTL & $1e-4$ & $1e-4$ \\
\bottomrule
\end{tabular}
\caption{Learning rates selected using mean validation accuracy for each first-order method on SST-2 and SST-5.}
\end{table}

% \begin{table}[h!]
% \centering
% \label{tab:sst_fo_selected_lr}
% \vspace{2mm}
% \begin{tabular}{cccc}
% \toprule
% \multicolumn{4}{c}{Learning Rate} \\
% \cmidrule{1-4}
% SGD & PP-SSL & Scalar FTL & Vector FTL \\
% \midrule
% $5e-2$ & $5e-2$ & $5e-2$ & $5e-2$ \\
% \bottomrule
% \end{tabular}
% \caption{Learning rates selected using mean validation accuracy for each method on SST.}
% \end{table}

\subsubsection{Zeroth-order methods} \label{app:D.1.2}

For the zeroth-order experiments, we fine-tune the model for $10{,}000$ steps to account for the higher variance of the zeroth-order gradient estimates. Despite the larger number of optimization steps, we note that zeroth-order updates can be more computationally efficient per step, as they require only forward passes and avoid backpropagation. We compare three methods: standard 
stochastic gradient-free method without using ML-labeled data as well as Algorithm~\ref{alg:sgfm} with scalar tuning and vector tuning, which we denote as ``ZO Scalar FTL" and ``ZO Vector FTL" respectively in the following.
We use a labeled batch size of $n=32$ for all methods. For Algorithm~\ref{alg:sgfm}, we additionally use an unlabeled batch size of $N=64$. We set the perturbation parameter to $\mu = 1e-3$. For each method, we independently tune the learning rate over $\{5e-7, 1e-6, 5e-6\}$ and use a constant learning rate throughout fine-tuning. Learning rates are selected based on the mean validation accuracy across seeds and are reported in \tabref{tab:sst_zo_selected_lr}.

\paragraph{Results.}
Figures~\ref{fig:sst_zo_acc} and \ref{fig:sst_zo_lam} report the test-accuracy results and corresponding tuning parameter trajectories for the zeroth-order methods on SST-2 and SST-5. On both SST-2 and SST-5, our methods consistently improve over the ZO baseline after the early stages of training, as the tuning parameters become more stable after aggregating information over more optimization steps. ZO Scalar FTL and ZO Vector FTL perform similarly throughout most of training. The corresponding tuning trajectories are smooth and closely aligned for the scalar and the mean of vector parameters. At the same time, each entry of $\boldsymbol{\lambda}_t$ in vector tuning retains a wider range, indicating that different coordinates continue to place different weights on the ML-labeled correction even when the average tuning parameter evolves similarly to the scalar one.

\begin{figure}[h]
    \centering
    \includegraphics[width=0.98\linewidth]{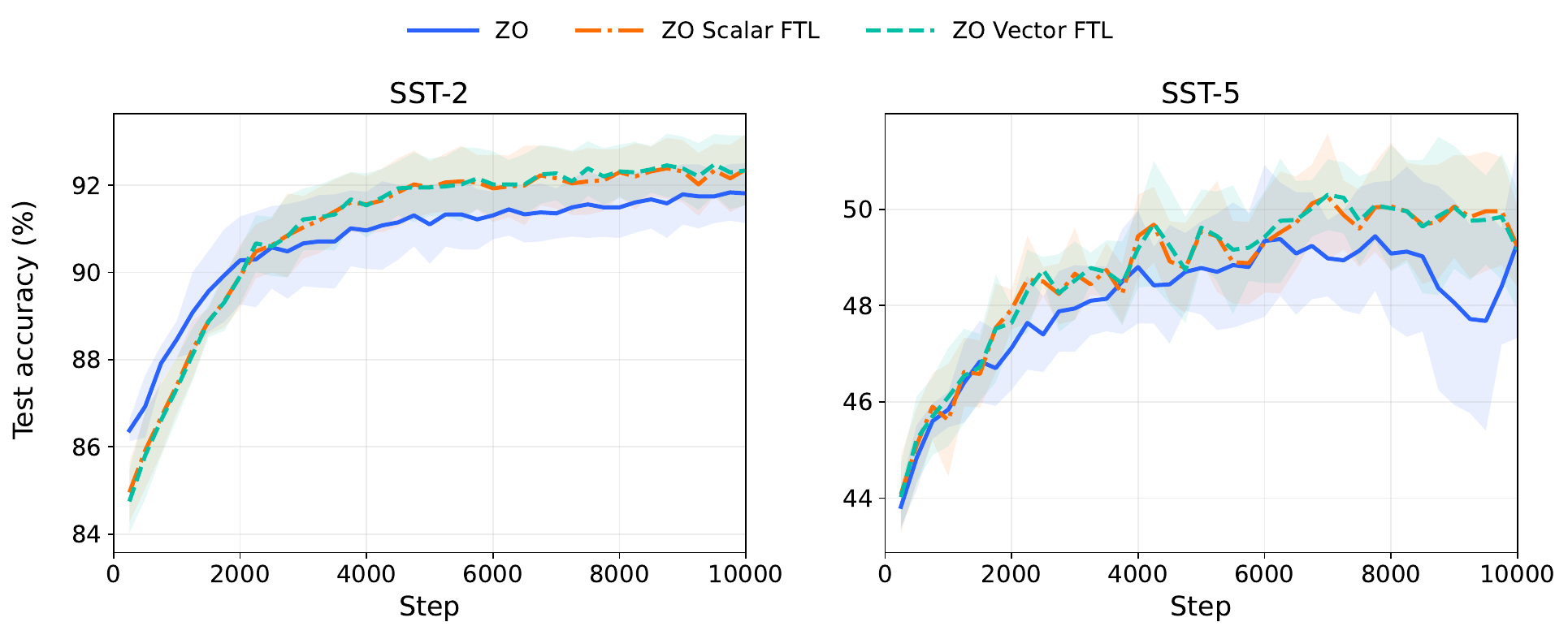}
    \vspace{-2mm}
    \caption{Test accuracy (\%) over training steps for the zeroth-order methods on SST-2 (left) and SST-5 (right). Test accuracy is evaluated every $250$ optimization steps. Solid curves denote the mean across five random seeds, and shaded regions indicate $\pm 1$ standard deviation.}
    \label{fig:sst_zo_acc}
\end{figure}

\begin{figure}[h]
    \centering
    \includegraphics[width=0.98\linewidth]{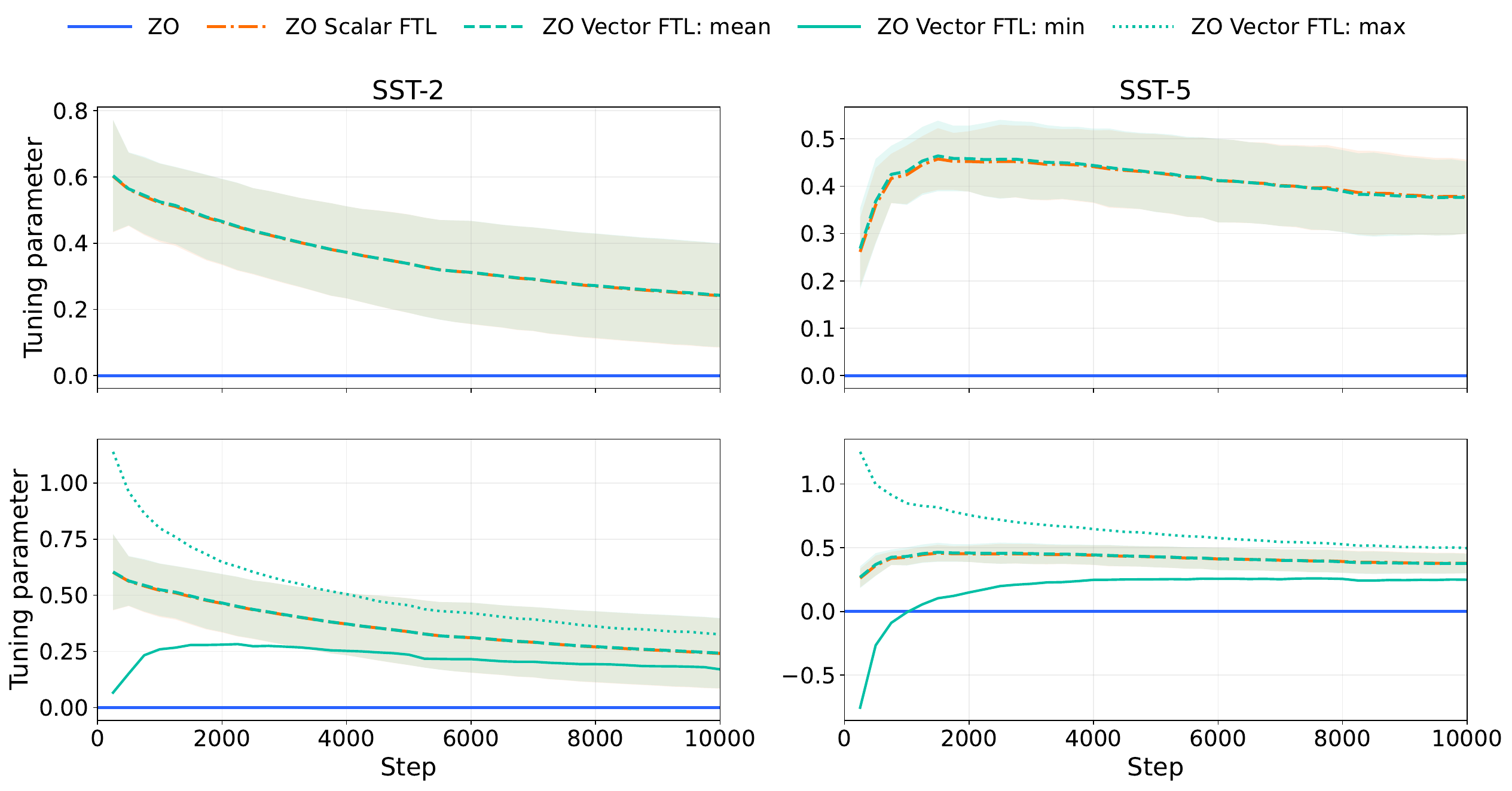}
    \vspace{-2mm}
    \caption{
    Tuning parameter trajectories for the zeroth-order methods on SST-2 (left) and SST-5 (right). Tuning parameters are recorded every $250$ optimization steps. \textbf{Top row}: Power-tuning parameter trajectories over optimization steps. \textbf{Bottom row}: Power-tuning parameter trajectories together with the across-seed means of the minimum and maximum coordinatewise ZO Vector FTL parameters. Curves show the mean across five random seeds, with shaded regions indicating $\pm 1$ standard deviation.
    }
    \label{fig:sst_zo_lam}
\end{figure}

\begin{table}[h]
\centering
\label{tab:sst_zo_selected_lr_original}
\simpletablelabel{tab:sst_zo_selected_lr}
\vspace{2mm}
\begin{tabular}{ccc}
\toprule
\textbf{Method} & \textbf{SST-2} & \textbf{SST-5} \\
\midrule
ZO            & $1e-6$ & $1e-6$ \\
ZO Scalar FTL & $1e-6$ & $1e-6$ \\
ZO Vector FTL & $1e-6$ & $1e-6$ \\
\bottomrule
\end{tabular}
\caption{Learning rates selected using mean validation accuracy for each zeroth-order method on SST-2 and SST-5.}
\end{table}

\subsection{Image Classification on CIFAR-10}

\paragraph{Task and model.} 
We conduct experiments on image classification using CIFAR-10 \citep{krizhevsky2009learning}. We train a linear classification layer on top of a fixed ResNet-50 \citep{he2016deep} pretrained on ImageNet \citep{deng2009imagenet}. CIFAR-10 contains $50{,}000$ training images and $10{,}000$ test images distributed uniformly across 10 classes. Each image is resized and normalized using the preprocessing associated with the ImageNet-pretrained ResNet-50. We remove the final classification layer of the network and use the resulting $2{,}048$-dimensional penultimate-layer representation as input to a linear classifier with \(10\) outputs. The ResNet-50 feature extractor remains fixed throughout all labeler- and model-training experiments.

\paragraph{Dataset.}
We partition the original CIFAR-10 training set into four mutually disjoint, class-balanced subsets. We use $5{,}000$ examples to train the labeler model, $4{,}500$ examples as the true-labeled training set, $5{,}000$ examples as the validation set, and the remaining $35{,}500$ examples as unlabeled data. The official CIFAR-10 test split of $10{,}000$ images is used for final evaluation. The same data partition is used for all methods and random seeds. For the labeled, unlabeled, and labeler-training subsets, we apply a random resized crop to $224\times224$ pixels and a random horizontal flip before extracting the ResNet-50 features. Validation and test images are resized to $256\times256$ pixels and center-cropped to $224\times224$ pixels. Because the extracted representations are cached, the training augmentations are sampled once during feature extraction.

\paragraph{ML-labeler.}
To construct the ML-labeler model $f(\cdot)$, we train a linear classifier on top of the fixed $2{,}048$-dimensional features extracted by the ImageNet-pretrained ResNet-50. The classifier maps the feature representation to the $10$ CIFAR-10 classes and is trained exclusively on the disjoint labeler-training set. We train the labeler for 50 epochs using AdamW with a learning rate of 1e-3, weight decay $0.01$, and batch size 256. The resulting labeler achieves a test accuracy of $86.47\%$ on CIFAR-10 and is subsequently kept fixed and shared across all experimental configurations.

\paragraph{Loss functions.}
For every labeled and unlabeled feature vector, we pass the representation through the fixed labeler to obtain its logits. A softmax converts these logits into a probability distribution over the 10 CIFAR-10 classes. On labeled examples, the supervised loss is the standard cross-entropy loss between the underlying model's prediction and the ground-truth label. The labeler-based losses on the labeled and unlabeled examples use the KL divergence between the labeler distribution and the predictive distribution of the model being trained.

\paragraph{Randomness.}
We report results across five random seeds.
% , $\{13,21,42,87,100\}$. 
Randomness across seeds arises from the initialization of the classification layer and the ordering of the labeled and unlabeled mini-batches. For zeroth-order methods, it additionally arises from the Gaussian search directions $\zeta_t\sim\mathcal{N}(0,I_d)$, with a new direction sampled at every optimization step.

\subsubsection{First-order methods} \label{app:D.2.1}

For the first-order experiments, we train each model for $100$ epochs and compare four methods: standard SGD using only labeled examples, PP-SSL with Online Tuning \citep{shoham2025prediction}, and our method with scalar and vector-valued tuning parameters (Algorithm~\ref{alg:sgd}). We refer to ours as Scalar FTL and Vector FTL, respectively, when presenting the experimental results below for brevity.
We tune the learning rate independently for each method and batch-size configuration over $\{5e-3,\;1e-2,\;5e-2,\;1e-1\}$ and use a constant learning rate throughout training. For each method and batch-size configuration, the learning rate is selected based on the mean validation accuracy across five random seeds. The selected learning rates are reported in Table~\ref{tab:cifar10_fo_selected_lr}. To study the effect of labeled and unlabeled batch sizes, we consider several $(n,N)$ configurations, with $n$ and $N$ denoting the labeled and unlabeled batch sizes, respectively. The full set of configurations is shown in Table~\ref{tab:cifar10_batch_results}.

\begin{table}[h]
    \centering
    \label{tab:cifar10_batch_results_original}
    \simpletablelabel{tab:cifar10_batch_results}
    \vspace{2mm}
    \begin{tabular}{cccccc}
        \toprule
        \textbf{n} & \textbf{N} & \textbf{SGD} & \textbf{PP-SSL} & \textbf{Scalar FTL} & \textbf{Vector FTL} \\
        \midrule
        64  & 128  & $86.29 \pm 0.08$ & $86.38 \pm 0.10$
            & $\mathbf{86.59 \pm 0.05}$ & $85.78 \pm 0.47$ \\
        64  & 256  & $86.29 \pm 0.08$ & $86.44 \pm 0.05$
            & $\mathbf{86.63 \pm 0.08}$ & $86.57 \pm 0.04$ \\
        64  & 512  & $86.29 \pm 0.08$ & $86.43 \pm 0.05$
            & $\mathbf{86.63 \pm 0.06}$ & $86.48 \pm 0.17$ \\
        64  & 1024 & $86.29 \pm 0.08$ & $86.43 \pm 0.05$
            & $\mathbf{86.63 \pm 0.06}$ & $\mathbf{86.63 \pm 0.12}$ \\
        128 & 256  & $86.26 \pm 0.07$ & $86.39 \pm 0.05$
            & $\mathbf{86.47 \pm 0.06}$ & $85.97 \pm 0.20$ \\
        128 & 512  & $86.26 \pm 0.07$ & $86.40 \pm 0.04$
            & $\mathbf{86.53 \pm 0.06}$ & $86.32 \pm 0.09$ \\
        128 & 1024 & $86.26 \pm 0.07$ & $86.43 \pm 0.05$
            & $\mathbf{86.65 \pm 0.05}$ & $86.40 \pm 0.09$ \\
        256 & 512  & $86.43 \pm 0.07$ & $\mathbf{86.52 \pm 0.05}$
            & $\mathbf{86.52 \pm 0.08}$ & $86.24 \pm 0.06$ \\
        256 & 1024 & $86.43 \pm 0.07$ & $86.55 \pm 0.05$
            & $\mathbf{86.64 \pm 0.05}$ & $86.25 \pm 0.13$ \\
        512 & 1024 & $86.08 \pm 0.07$ & $86.12 \pm 0.08$
            & $\mathbf{86.38 \pm 0.03}$ & $85.84 \pm 0.09$ \\
        \bottomrule
    \end{tabular}
    \caption{Test accuracy (\%) on CIFAR-10. For each method and batch
    configuration, the learning rate and training checkpoint are selected based on mean validation accuracy, and test accuracy is reported at the selected checkpoint. Results report the mean $\pm$ standard error over five random seeds. For each batch configuration, the highest test accuracy is shown in bold.}
\end{table}

% \begin{table}[h]
%     \centering
%     \caption{Learning rates selected using mean validation accuracy for each
%     method and batch-size configuration on CIFAR-10.}
%     \label{tab:cifar10_selected_lr}
%     \vspace{2mm}
%     \begin{tabular}{cccccc}
%         \toprule
%         n & N
%         & SGD & PP-SSL & Scalar FTL & Vector FTL \\
%         \midrule
%         $64$  & $128$  & $5e-2$ & $5e-2$ & $5e-2$ & $5e-2$ \\
%         $64$  & $256$  & $5e-2$ & $5e-2$ & $5e-2$ & $5e-2$ \\
%         $64$  & $512$  & $5e-2$ & $5e-2$ & $5e-2$ & $5e-2$ \\
%         $64$  & $1024$ & $5e-2$ & $5e-2$ & $5e-2$ & $5e-2$ \\
%         $128$ & $256$  & $5e-2$ & $5e-2$ & $5e-2$ & $5e-2$ \\
%         $128$ & $512$  & $5e-2$ & $5e-2$ & $5e-2$ & $5e-2$ \\
%         $128$ & $1024$ & $5e-2$ & $5e-2$ & $5e-2$ & $5e-2$ \\
%         $256$ & $512$  & $1e-1$ & $1e-1$ & $1e-1$ & $1e-1$ \\
%         $256$ & $1024$ & $1e-1$ & $1e-1$ & $1e-1$ & $1e-1$ \\
%         $512$ & $1024$ & $1e-1$ & $1e-1$ & $1e-1$ & $1e-1$ \\
%         \bottomrule
%     \end{tabular}
% \end{table}

\paragraph{Results.}
Table~\ref{tab:cifar10_batch_results} reports the test accuracy obtained after selecting the learning rate and training checkpoint based on mean validation accuracy, while Figures~\ref{fig:cifar10_fo_acc} and \ref{fig:cifar10_fo_lambda} illustrate the corresponding test-accuracy and tuning-parameter trajectories for chosen batch-size configurations. Across the considered configurations, Scalar FTL provides the most consistent improvement, achieving the highest mean test accuracy in all configurations. PP-SSL also generally improves over the labeled-only SGD baseline, although by a smaller margin. In contrast, Vector FTL is more sensitive to the batch-size configuration: it performs comparably to Scalar FTL for some choices of $(n,N)$, but exhibits a noticeable degradation for others.

The tuning trajectories help explain these differences. The Scalar FTL tuning parameter decreases smoothly throughout training, while the PP-SSL tuning parameter rapidly approaches zero, thereby substantially reducing the contribution of the ML-labeled correction after the early stages of training. Although the mean of the Vector FTL parameters follows a similarly smooth decreasing trend, its coordinatewise parameters exhibit substantial dispersion, with some coordinates receiving large positive weights and others negative weights throughout training. This additional flexibility can be beneficial, but the pronounced coordinatewise variation is consistent with the greater variability of Vector FTL across batch-size configurations. Overall, these results suggest that the scalar parameter provides a favorable tradeoff between exploiting the ML-labeled data and maintaining stability across different labeled and unlabeled batch sizes.

\begin{figure}[h!]
    \centering
    \includegraphics[width=0.98\linewidth]{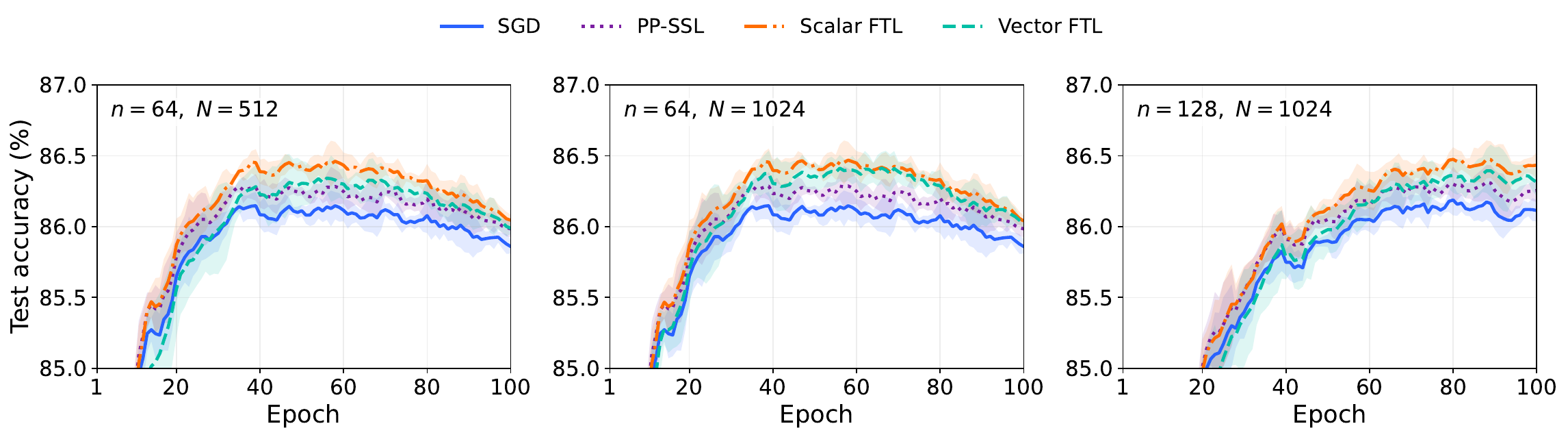}
    \vspace{-2mm}
    \caption{
    Test accuracy (\%) over training epochs for the first-order methods on CIFAR-10 with $n=64, N=512$ (left), $n=64, N=1024$ (middle), and $n=128, N=1024$ (right). Curves show the mean across five random seeds, with shaded regions indicating $\pm 1$ standard deviation. For visualization, test-accuracy curves are smoothed using a five-epoch moving average.
    }
    \label{fig:cifar10_fo_acc}
\end{figure}

\begin{figure}[h!]
    \centering
    \includegraphics[width=0.98\linewidth]{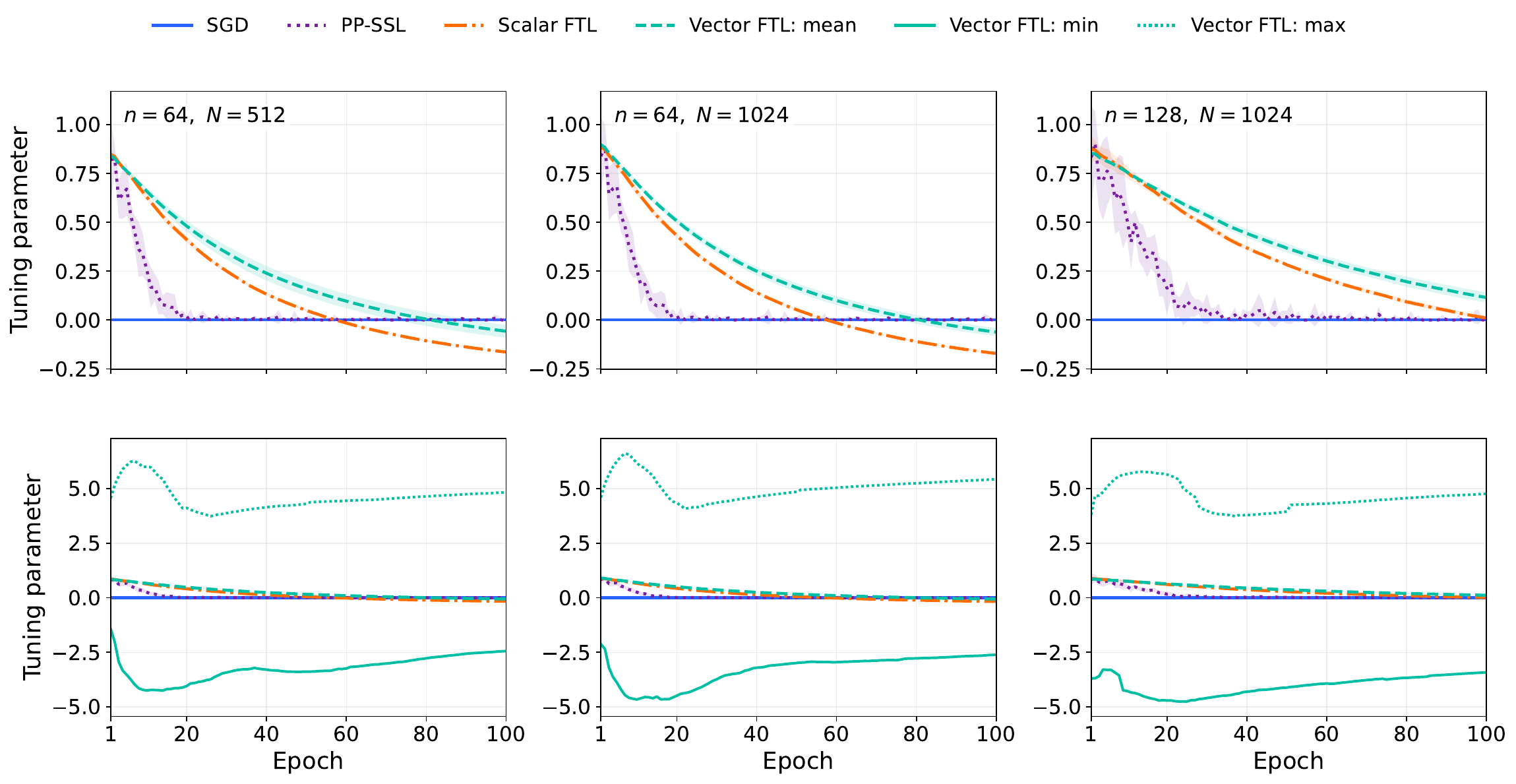}
    \vspace{-2mm}
    \caption{
    Tuning parameter trajectories for the first-order methods on CIFAR-10 with $n=64, N=512$ (left), $n=64, N=1024$ (middle), and $n=128, N=1024$ (right). 
    \textbf{Top row}: Tuning parameter trajectories over training epochs. 
    \textbf{Bottom row}: Tuning parameter trajectories together with the across-seed mean of the minimum and maximum coordinatewise Vector FTL parameters at each epoch. 
    Curves show the mean across five random seeds, with shaded regions indicating $\pm 1$ standard deviation.
    }
    \label{fig:cifar10_fo_lambda}
\end{figure}

\begin{table}[h]
\centering
\label{tab:cifar10_fo_selected_lr_original}
\simpletablelabel{tab:cifar10_fo_selected_lr}
\vspace{2mm}
\begin{tabular}{cccccc}
\toprule
\textbf{n} & \textbf{N} & \textbf{SGD} & \textbf{PP-SSL} & \textbf{Scalar FTL} & \textbf{Vector FTL} \\
\midrule
$64$  & $128$  & $5e-2$ & $5e-2$ & $5e-2$ & $5e-2$ \\
$64$  & $256$  & $5e-2$ & $5e-2$ & $5e-2$ & $5e-2$ \\
$64$  & $512$  & $5e-2$ & $5e-2$ & $5e-2$ & $5e-2$ \\
$64$  & $1024$ & $5e-2$ & $5e-2$ & $5e-2$ & $5e-2$ \\
$128$ & $256$  & $5e-2$ & $5e-2$ & $5e-2$ & $5e-2$ \\
$128$ & $512$  & $5e-2$ & $5e-2$ & $5e-2$ & $5e-2$ \\
$128$ & $1024$ & $5e-2$ & $5e-2$ & $5e-2$ & $5e-2$ \\
$256$ & $512$  & $1e-1$ & $1e-1$ & $1e-1$ & $1e-1$ \\
$256$ & $1024$ & $1e-1$ & $1e-1$ & $1e-1$ & $1e-1$ \\
$512$ & $1024$ & $1e-1$ & $1e-1$ & $1e-1$ & $1e-1$ \\
\bottomrule
\end{tabular}
\caption{Learning rates selected using mean validation accuracy for each
first-order method and batch-size configuration on CIFAR-10.}
\end{table}

\subsubsection{Zeroth-order methods} \label{app:D.2.2}

For the zeroth-order experiments, we train each model for $2{,}000$ epochs and compare three methods: standard stochastic gradient-free method without using ML-labeled data and the proposed Algorithm~\ref{alg:sgfm} with scalar tuning and vector tuning, which we denote as ``ZO Scalar FTL" and ``ZO Vector FTL" respectively in the following.
We tune the learning rate independently for each method and batch-size configuration over $\{5e-4,\;1e-3,\;5e-3\}$ and use a constant learning rate throughout training. The selected learning rates are reported in Table~\ref{tab:cifar10_zo_selected_lr}. We set the perturbation parameter to $\mu = 1e-3$. To study the effect of labeled and unlabeled batch sizes, we consider $(n,N)\in\{(256,512),(256,1024),(512,1024)\}$, where $n$ and $N$ denote the labeled and unlabeled batch sizes, respectively.

\paragraph{Results.}
Figures~\ref{fig:cifar10_zo_acc} and \ref{fig:cifar10_zo_lam} show the test-accuracy results and corresponding tuning parameter trajectories for the zeroth-order methods on CIFAR-10. 
Similar to previous subsection, Algorithm~\ref{alg:sgfm} consistently improve over the ZO baseline across all three batch-size configurations, with ZO Scalar FTL and ZO Vector FTL achieving similar test accuracy throughout training. The improvement is sustained throughout the training, and the relative behavior is stable across the different choices of $(n,N)$. The corresponding tuning trajectories are also notably smooth: the ZO Scalar FTL parameter and the mean of the tuning parameters of ZO Vector FTL remain closely aligned and gradually decrease over training. For a fixed labeled batch size $n=256$, we observe both reduced variability in the tuning-parameter trajectories and larger tuning-parameter values as the unlabeled batch size $N$ increases from $512$ to $1024$. Although the coordinatewise ZO Vector FTL parameters initially exhibit a wider range, this dispersion is reduced substantially as the training progresses. Overall, in the zeroth-order setting, both methods appear to exploit the ML-labeled correction in a stable manner across batch-size configurations, with little difference in predictive performance between the scalar and vector schemes.

\begin{figure}[h!]
    \centering
    \includegraphics[width=0.98\linewidth]{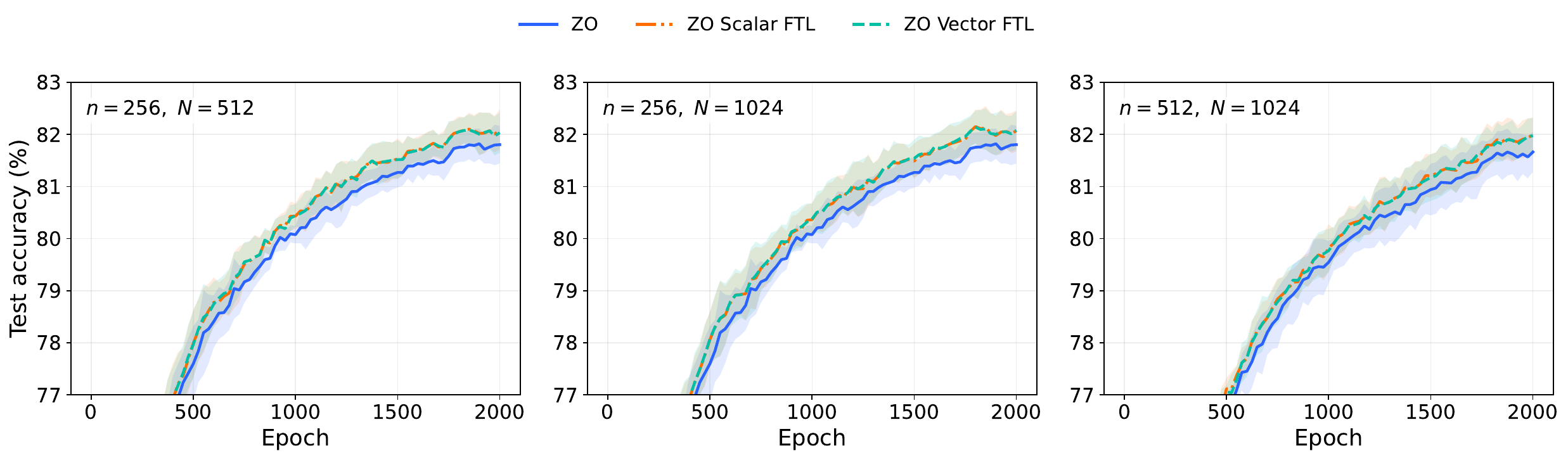}
    \vspace{-2mm}
    \caption{
    Test accuracy (\%) over training epochs for the zeroth-order methods on CIFAR-10 with $n=256, N=512$ (left), $n=256, N=1024$ (middle), and $n=512, N=1024$ (right). Values are reported every 25 epochs. Curves show the mean across five random seeds, with shaded regions indicating $\pm 1$ standard deviation.
    }
    \label{fig:cifar10_zo_acc}
\end{figure}

\begin{figure}[h!]
    \centering
    \includegraphics[width=0.98\linewidth]{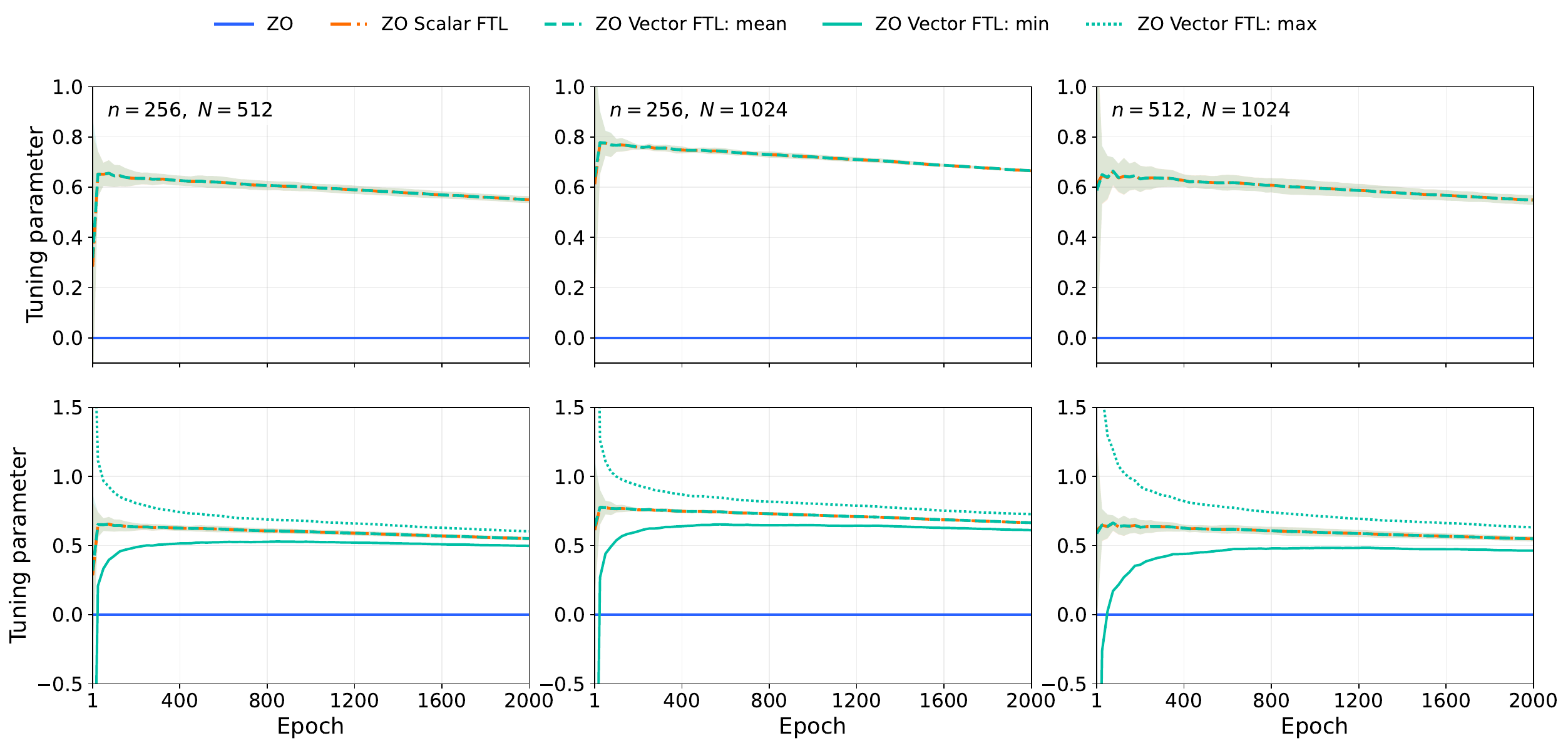}
    \vspace{-2mm}
    \caption{
    Tuning parameter trajectories for the zeroth-order methods on CIFAR-10 with $n=256, N=512$ (left), $n=256, N=1024$ (middle), and $n=512, N=1024$ (right).
    \textbf{Top row}: Tuning parameter trajectories over training epochs. 
    \textbf{Bottom row}: Tuning parameter trajectories together with the across-seed mean of the minimum and maximum coordinatewise ZO Vector FTL parameters at each epoch.
    Values are reported every 25 epochs.
    Curves show the mean across five random seeds, with shaded regions indicating $\pm 1$ standard deviation.
    }
    \label{fig:cifar10_zo_lam}
\end{figure}

\begin{table}[h]
\centering
\label{tab:cifar10_zo_selected_lr_original}
\simpletablelabel{tab:cifar10_zo_selected_lr}
\vspace{2mm}
\begin{tabular}{cccccc}
\toprule
\textbf{n} & \textbf{N} & \textbf{ZO} & \textbf{ZO Scalar FTL} & \textbf{ZO Vector FTL} \\
\midrule
$256$  & $512$  & $5e-4$ & $5e-4$ & $5e-4$  \\
$256$  & $1024$  & $5e-4$ & $5e-4$ & $5e-4$  \\
$512$  & $1024$  & $1e-3$ & $1e-3$ & $1e-3$  \\
\bottomrule
\end{tabular}
\caption{Learning rates selected using mean validation accuracy for each
zeroth-order method and batch-size configuration on CIFAR-10.}
\end{table}

\subsection{Linear Regression on California Housing} \label{app:D3}

\paragraph{Task and model.}
We conduct experiments on linear regression using the California Housing dataset from the LIBSVM repository \citep{chang2011libsvm}. The task is to predict the median house value of a California census block from eight numerical features. We use a linear regression model that maps the standardized eight-dimensional input vector to a single standardized target value. The same model architecture is used for the labeler and all models.

\paragraph{Dataset.}
The complete dataset contains $20{,}640$ examples. Because the dataset does not provide an official test split, we construct a fixed random split. We reserve $4{,}128$ examples, corresponding to $20\%$ of the dataset, for testing. From the remaining examples, we use $500$ examples to train the labeler, $80$ examples as the true-labeled training set, and $80$ examples as the validation set. The remaining $15{,}852$ examples are treated as unlabeled data. These subsets are mutually disjoint, and the same partition is used for all methods and random seeds.

\paragraph{ML-labeler.}
To construct the labeler model $f(\cdot)$, we train a linear regressor on the disjoint labeler-training set of $500$ examples. The labeler is trained for $10$ epochs using AdamW with a learning rate of $1e-2$, weight decay $1e-4$, and batch size $64$. The resulting labeler achieves a standardized test MSE of $0.4191$. After training, the labeler remains fixed and is shared across all experimental configurations.

\paragraph{Loss functions.}
For each labeled and unlabeled input, we pass the standardized feature vector through the fixed labeler and cache its scalar prediction. On labeled examples, the supervised loss is the mean squared error between the model's prediction and the ground-truth standardized target. The labeler-based losses on both labeled and unlabeled examples are mean squared errors between the model's prediction and the corresponding cached labeler prediction. Thus, the labeler-based terms use continuous soft regression targets. Because the labeler remains fixed, its predictions are computed once and reused throughout optimization.

\paragraph{Randomness.}
We report results across ten random seeds. The dataset partition and labeler model are held fixed across seeds. Randomness arises from the initialization of the linear model and the ordering and sampling of the labeled and unlabeled mini-batches. For zeroth-order methods, additional randomness arises from the Gaussian search directions $\zeta_t\sim\mathcal{N}(0,I_d)$.

\subsubsection{First-order methods}

For the first-order experiments, we train the linear regression model for $100$ optimization steps and compare four methods: standard SGD using only labeled data, PP-SSL with Online Tuning \citep{shoham2025prediction}, and SGD with scalar tuning and vector tuning (i.e., Algorithm~\ref{alg:sgd}). We refer to the latter two as Scalar FTL and Vector FTL respectively below.
We use a labeled batch size of $n=16$ for all methods and consider unlabeled batch sizes $N\in\{64,128,256\}$ for PP-SSL and our methods. Since SGD does not use unlabeled examples, its trajectory is shared across the three batch-size configurations. We use a constant learning rate throughout training and consider the grid $\{5e-3,1e-2,5e-2,1e-1\}$. Based on mean validation performance across seeds, we use a learning rate of $5e-2$ for all methods in the final comparison. The reported curves show the mean across ten random seeds, with shaded regions indicating $\pm1$ standard deviation.

\paragraph{Results.}
Figures~\ref{fig:cal_fo_acc} and \ref{fig:cal_fo_lam} show the test MSE and corresponding tuning-parameter trajectories for the first-order methods on California Housing. Across all three unlabeled batch sizes, the methods that incorporate ML-labeled data improve upon the labeled-only SGD baseline after the initial optimization steps. PP-SSL and Scalar FTL perform comparably across configurations and generally achieve the lowest test MSE, while Vector FTL approaches similar performance later in training but exhibits greater variability early in optimization. The relative behavior of the methods remains broadly consistent as the unlabeled batch size increases from $N=64$ to $N=256$. The tuning-parameter trajectories stabilize quickly after the first few steps. Compared to our methods, the PP-SSL tuning parameter exhibits more pronounced fluctuations, particularly for smaller $N$. The PP-SSL, Scalar FTL, and the mean of Vector FTL parameters remain in a similar positive range throughout training, although their relative magnitudes vary across configurations. We also observe a slightly reduced across-seed variability in the tuning parameters as $N$ increases, particularly after the first few oprimization steps. The coordinatewise Vector FTL parameters exhibit substantial dispersion, with some coordinates receiving negative weights and others considerably larger positive weights. This coordinatewise spread decreases over training but remains visible at the end of optimization, even as the mean of the Vector FTL tuning parameters stays close to the scalar tuning parameters.

\begin{figure}[h]
    \centering
    \includegraphics[width=0.98\linewidth]{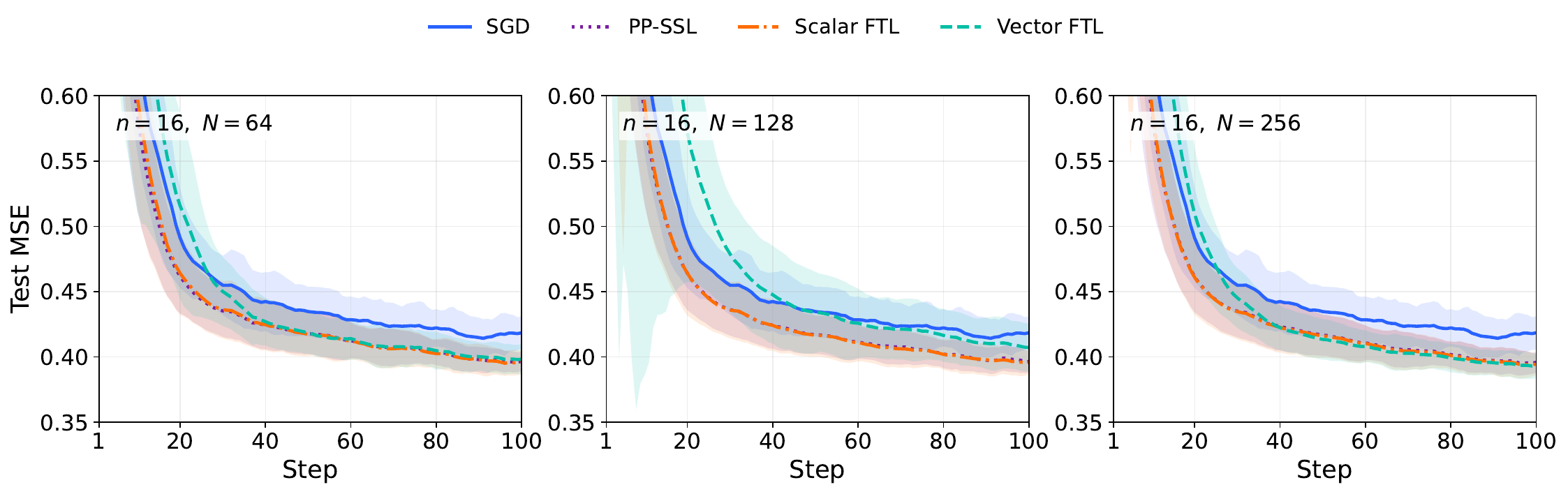}
    \vspace{-2mm}
    \caption{
    Test MSE over training steps for the first-order methods on California Housing with $n=16, N=64$ (left), $n=16, N=128$ (middle), and $n=16, N=256$ (right). Values are reported every 25 steps. Curves show the mean across ten random seeds, with shaded regions indicating $\pm 1$ standard deviation. For visualization, test-MSE curves are smoothed using a five-step moving average.
    }
    \label{fig:cal_fo_acc}
\end{figure}

\begin{figure}[h!]
    \centering
    \includegraphics[width=0.98\linewidth]{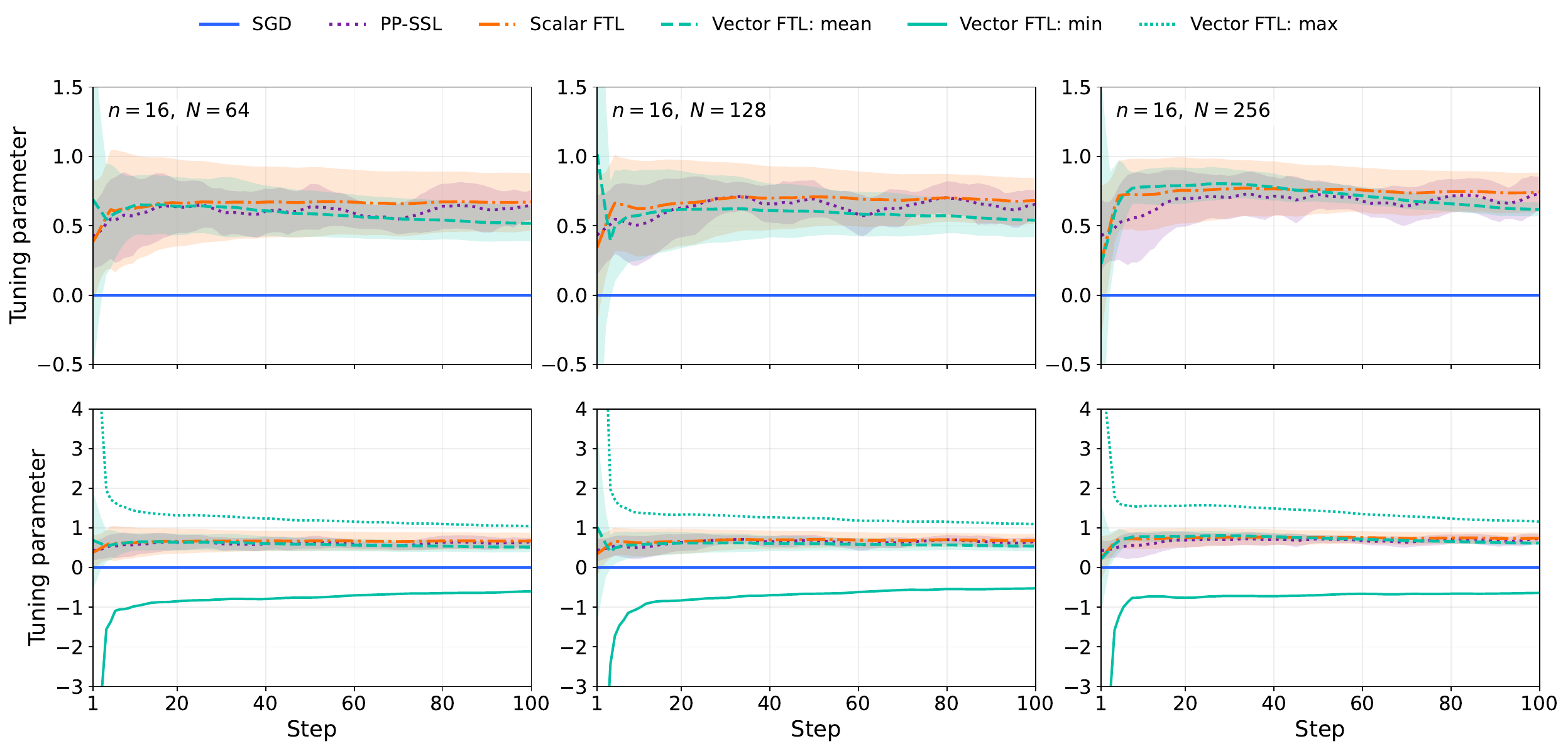}
    \vspace{-2mm}
    \caption{
    Tuning parameter trajectories for the first-order methods on California Housing with $n=16, N=64$ (left), $n=16, N=128$ (middle), and $n=16, N=256$ (right).
    \textbf{Top row}: Tuning parameter trajectories over training steps. 
    \textbf{Bottom row}: Tuning parameter trajectories together with the across-seed mean of the minimum and maximum coordinatewise Vector FTL parameters at each step.
    Curves show the mean across ten random seeds, with shaded regions indicating $\pm 1$ standard deviation. For visualization, the tuning parameter curves are smoothed using a five-step moving average.
    }
    \label{fig:cal_fo_lam}
\end{figure}

\subsubsection{Zeroth-order methods}
For the zeroth-order experiments, we train the linear regression model for $1{,}000$ optimization steps and compare three methods: Standard stochastic gradient-free method without using ML-labeled data and Algorithm~\ref{alg:sgfm} with scalar tuning and vector tuning, denoted as ``ZO Scalar FTL" and ``ZO Vector FTL" respectively. We tune the learning rate independently for each method and batch-size configuration over $\{1e-3, 5e-3,\;1e-2\}$ and use a constant learning rate throughout training. Learning rates are selected based on the mean validation MSE across ten random seeds and are reported in Table~\ref{tab:cal_zo_selected_lr}. We set the perturbation parameter to $\mu=1e\text{-}3$ and sample a new Gaussian search direction at every optimization step. To study the effect of the unlabeled batch size, we fix the labeled batch size to $n=32$ and consider $(n,N)\in\{(32,128),(32,256),(32,512)\}$, where $N$ denotes the unlabeled batch size.

\paragraph{Results.}
Figures~\ref{fig:cal_zo_acc} and \ref{fig:cal_zo_lam} present the test MSE and corresponding tuning-parameter trajectories for the zeroth-order methods on California Housing. Across all three unlabeled batch sizes, the proposed Algorithm~\ref{alg:sgfm} improve over the true-labeled-only ZO baseline after the initial optimization period, with ZO Scalar FTL and ZO Vector FTL achieving similar test MSE throughout most of training. The improvement becomes particularly clear later in optimization and remains consistent as the unlabeled batch size increases from $N=128$ to $N=512$. The test-MSE trajectories also become somewhat less variable for larger $N$ for the proposed methods. The tuning parameters stabilize rapidly after the initial optimization steps. The ZO Scalar FTL parameter and the mean of the ZO Vector FTL parameters remain closely aligned throughout training. For fixed $n$, the across-seed variability of these tuning parameters also tends to decrease as $N$ increases. The coordinatewise ZO Vector FTL parameters initially exhibit a substantially wider range, but this dispersion is reduced quickly and continues to decrease throughout training. Despite this coordinatewise flexibility, the close agreement between the mean vector and scalar tuning parameters is consistent with the similar predictive performance.

\begin{figure}[h!]
    \centering
    \includegraphics[width=0.98\linewidth]{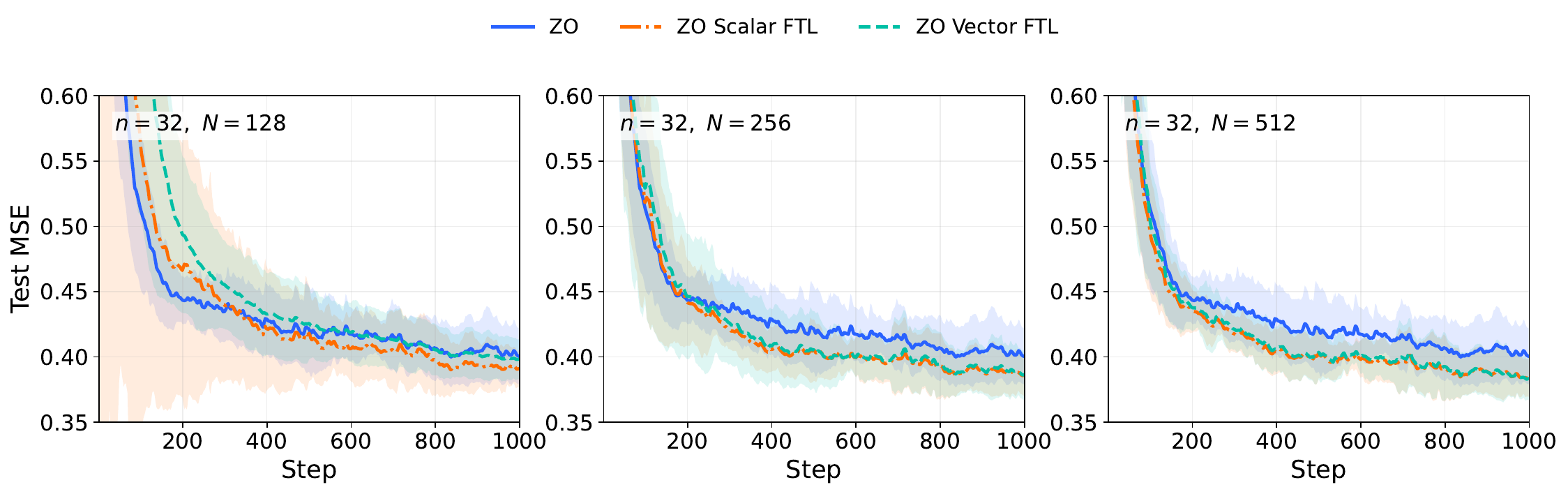}
    \vspace{-2mm}
    \caption{
    Test MSE over training steps for the zeroth-order methods on California Housing with $n=32, N=128$ (left), $n=32, N=256$ (middle), and $n=32, N=512$ (right). Values are reported every 25 steps. Curves show the mean across ten random seeds, with shaded regions indicating $\pm 1$ standard deviation. For visualization, test-MSE curves are smoothed using a five-step moving average.
    }
    \label{fig:cal_zo_acc}
\end{figure}

\begin{figure}[h!]
    \centering
    \includegraphics[width=0.98\linewidth]{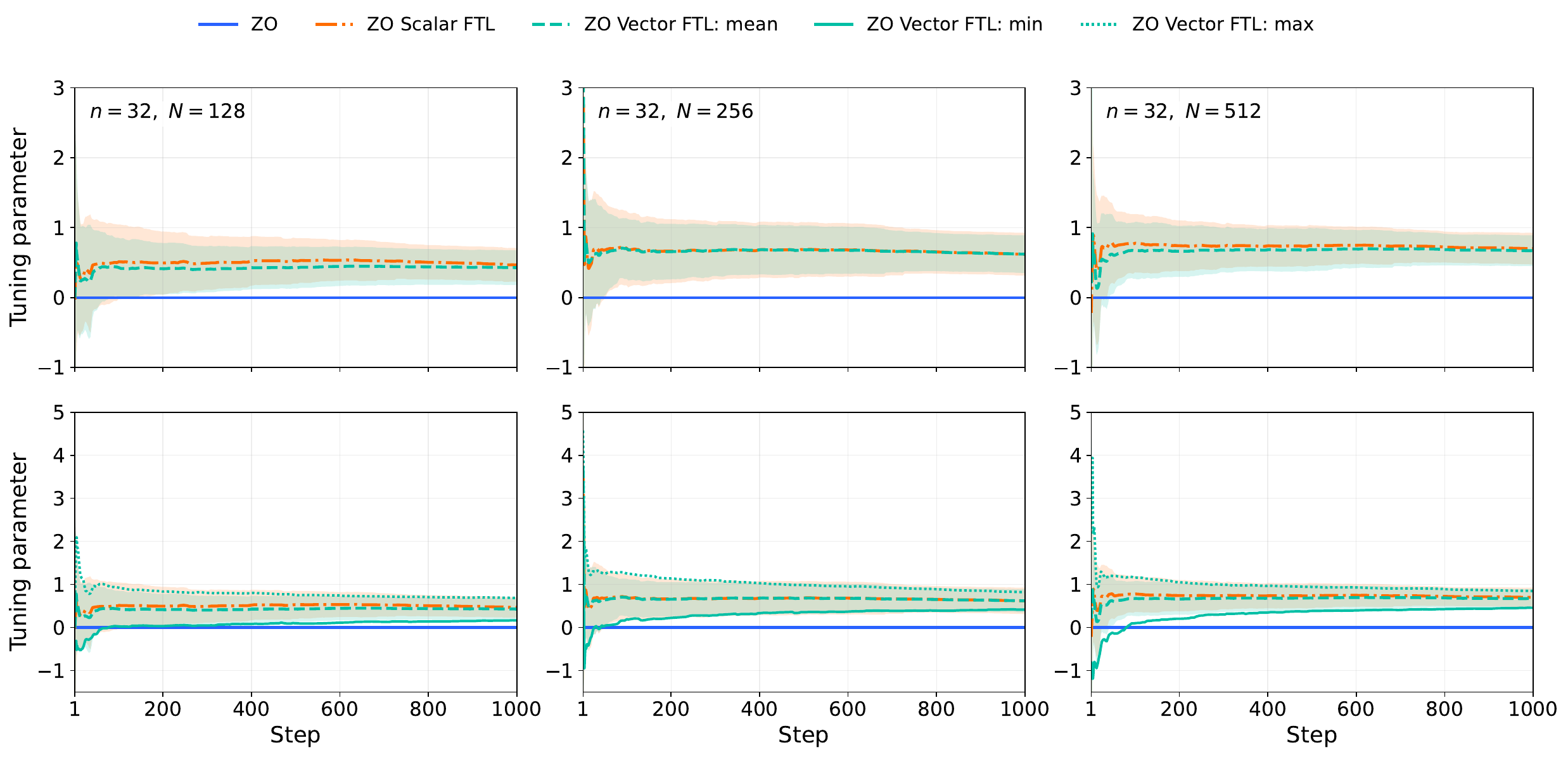}
    \vspace{-2mm}
    \caption{
    Tuning parameter trajectories for the zeroth-order methods on California Housing with $n=32, N=128$ (left), $n=32, N=256$ (middle), and $n=32, N=512$ (right).
    \textbf{Top row}: Tuning parameter trajectories over training steps. 
    \textbf{Bottom row}: Tuning parameter trajectories together with the across-seed mean of the minimum and maximum coordinatewise Vector FTL parameters at each step.
    Curves show the mean across ten random seeds, with shaded regions indicating $\pm 1$ standard deviation. For visualization, the tuning parameter curves are smoothed using a five-step moving average.
    }
    \label{fig:cal_zo_lam}
\end{figure}

\begin{table}[h!]
\centering
\label{tab:cal_zo_selected_lr_original}
\simpletablelabel{tab:cal_zo_selected_lr}
\vspace{2mm}
\begin{tabular}{cccccc}
\toprule
\textbf{n} & \textbf{N} & \textbf{ZO} & \textbf{ZO Scalar FTL} & \textbf{ZO Vector FTL} \\
\midrule
$32$  & $128$  & $1e-2$ & $1e-2$ & $5e-3$  \\
$32$  & $256$  & $1e-2$ & $1e-2$ & $1e-2$  \\
$32$  & $512$  & $1e-2$ & $1e-2$ & $1e-2$  \\
\bottomrule
\end{tabular}
\caption{Learning rates selected using mean validation MSE for each zeroth-order method and batch-size configuration on California Housing.}
\end{table}

\newpage

\subsection{Classification on Satimage} \label{app:D4}

\paragraph{Task and model.}
We conduct experiments on multiclass classification using the Statlog (Landsat Satellite) dataset, distributed as \texttt{satimage.scale} in the LIBSVM repository \citep{chang2011libsvm}. The task is to classify the soil type associated with the center pixel of a satellite-image neighborhood. Each example contains $36$ numerical features, corresponding to four spectral measurements for each pixel in a $3\times3$ neighborhood, and belongs to one of six classes. We train a linear classifier that maps the $36$-dimensional input vector to six output logits. The dataset features are taken from the scaled version provided by LIBSVM.

\paragraph{Dataset.}
The dataset provides an official training split containing $4{,}435$ examples and an official test split containing $2{,}000$ examples. We partition the official training split into four mutually disjoint subsets. We use $1{,}200$ examples to train the labeler, $100$ examples as the true-labeled training set, and $100$ examples as the validation set. The remaining $3{,}035$ examples are treated as unlabeled data. The split approximately preserves the original class proportions in each subset. The official test split is used for final evaluation. The partition is generated once and is shared across all methods and training seeds.

\paragraph{ML-labeler.}
To construct the labeler model $f(\cdot)$, we train a two-layer multilayer perceptron on the disjoint labeler-training set of $1{,}200$ examples. The labeler consists of a linear layer mapping the $36$ input features to $128$ hidden units, followed by a ReLU activation and a linear output layer with six logits. We train the labeler for $2{,}000$ optimization steps using AdamW with a learning rate of $1e\text{-}3$, weight decay $1e\text{-}4$, and batch size $128$. The labeler is trained using the standard cross-entropy loss and achieves a test accuracy of $87.55\%$. After training, the labeler remains fixed and is shared across all experimental configurations.

\paragraph{Loss functions.}
For every labeled and unlabeled input, we pass the feature vector through the fixed labeler and apply a softmax to its logits to obtain a probability distribution over the six classes. These labeler distributions are cached before training. On labeled examples, the supervised loss is the standard cross-entropy loss between the prediction of the model being trained and the ground-truth class label. The labeler-based losses on both labeled and unlabeled examples use the KL divergence between the cached labeler distribution and the predictive distribution from the model being trained. Because the labeler remains fixed, its predictions are computed once and reused throughout optimization.

\paragraph{Randomness.}
We report results across ten random seeds. Randomness arises from the initialization of the linear classifier and the sampling of labeled and unlabeled mini-batches. For the zeroth-order methods, additional randomness arises from the Gaussian search directions $\zeta_t\sim\mathcal{N}(0,I_d)$.

\subsubsection{First-order methods}
For the first-order experiments, we train the linear classifier for $500$ optimization steps and compare four methods: standard SGD using only labeled data, PP-SSL with Online Tuning \citep{shoham2025prediction}, and our proposed methods with scalar tuning and vector tuning (i.e., Algorithm~\ref{alg:sgd}). We use a labeled batch size of $n=16$ for all methods and consider unlabeled batch sizes $N\in\{64,128,256\}$ for PP-SSL and our methods. We use a constant learning rate throughout training and independently tune it for each method and batch-size configuration over $\{5e-3,1e-2,5e-2,1e-1\}$. Based on the mean validation accuracy across the initial seeds, a learning rate of $1e-1$ is selected for all methods and configurations in the final comparison. 

\paragraph{Results.}
Figures~\ref{fig:sat_fo_acc} and \ref{fig:sat_fo_lam} present the test-accuracy results and corresponding tuning-parameter trajectories for the first-order methods on Satimage. Across all three unlabeled batch sizes, PP-SSL, Scalar FTL, and Vector FTL consistently outperform the labeled-only SGD baseline after the initial optimization period. The three methods that leverage ML-labeled data achieve similar test accuracy overall. The performance of Vector FTL improves as the unlabeled batch size increases, with larger $N$ yielding higher test accuracy in the later stages of training. The tuning-parameter trajectories stabilize rapidly after the first few optimization steps. The Scalar FTL and mean of Vector FTL parameters remain close to one another throughout training and are relatively smooth, whereas the PP-SSL tuning parameter exhibits somewhat larger fluctuations. The coordinatewise Vector FTL parameters show a great dispersion, with some coordinates initially receiving negative weights and others much larger positive weights. This dispersion reduces sharply early in training and then remains relatively stable, while the mean of the Vector FTL parameters stays close to the scalar tuning parameter.

\begin{figure}[h!]
    \centering
    \includegraphics[width=0.98\linewidth]{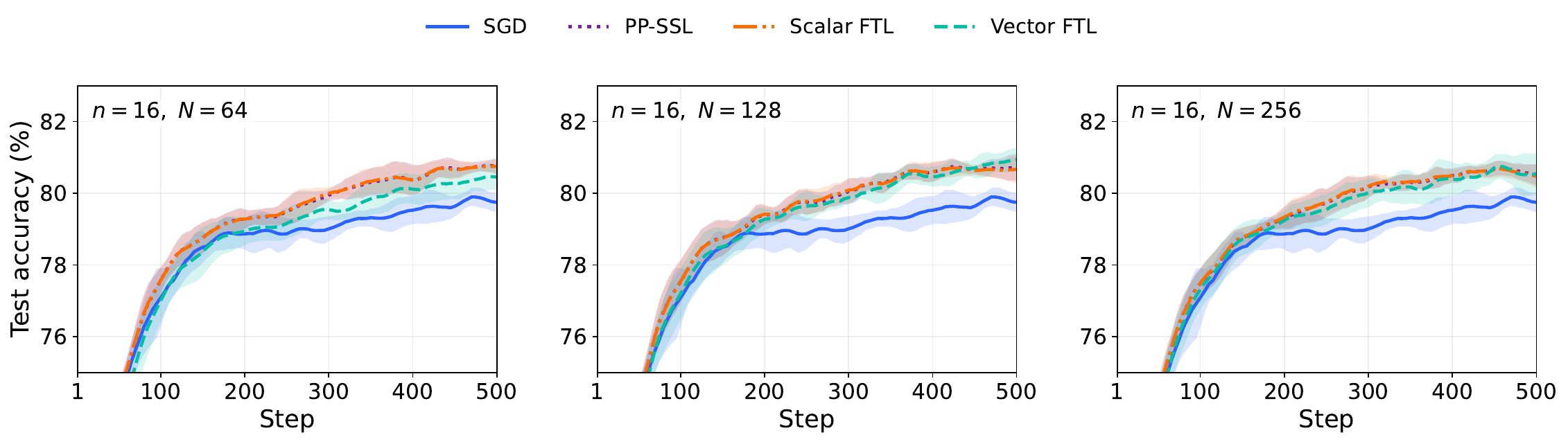}
    \vspace{-2mm}
    \caption{
    Test Accuracy (\%) over training steps for the first-order methods on Satimage with $n=16, N=64$ (left), $n=16, N=128$ (middle), and $n=16, N=256$ (right). Values are reported every 25 steps. Curves show the mean across ten random seeds, with shaded regions indicating $\pm 1$ standard deviation. For visualization, test-accuracy curves are smoothed using a five-step moving average.
    }
    \label{fig:sat_fo_acc}
\end{figure}

\begin{figure}[h!]
    \centering
    \includegraphics[width=0.98\linewidth]{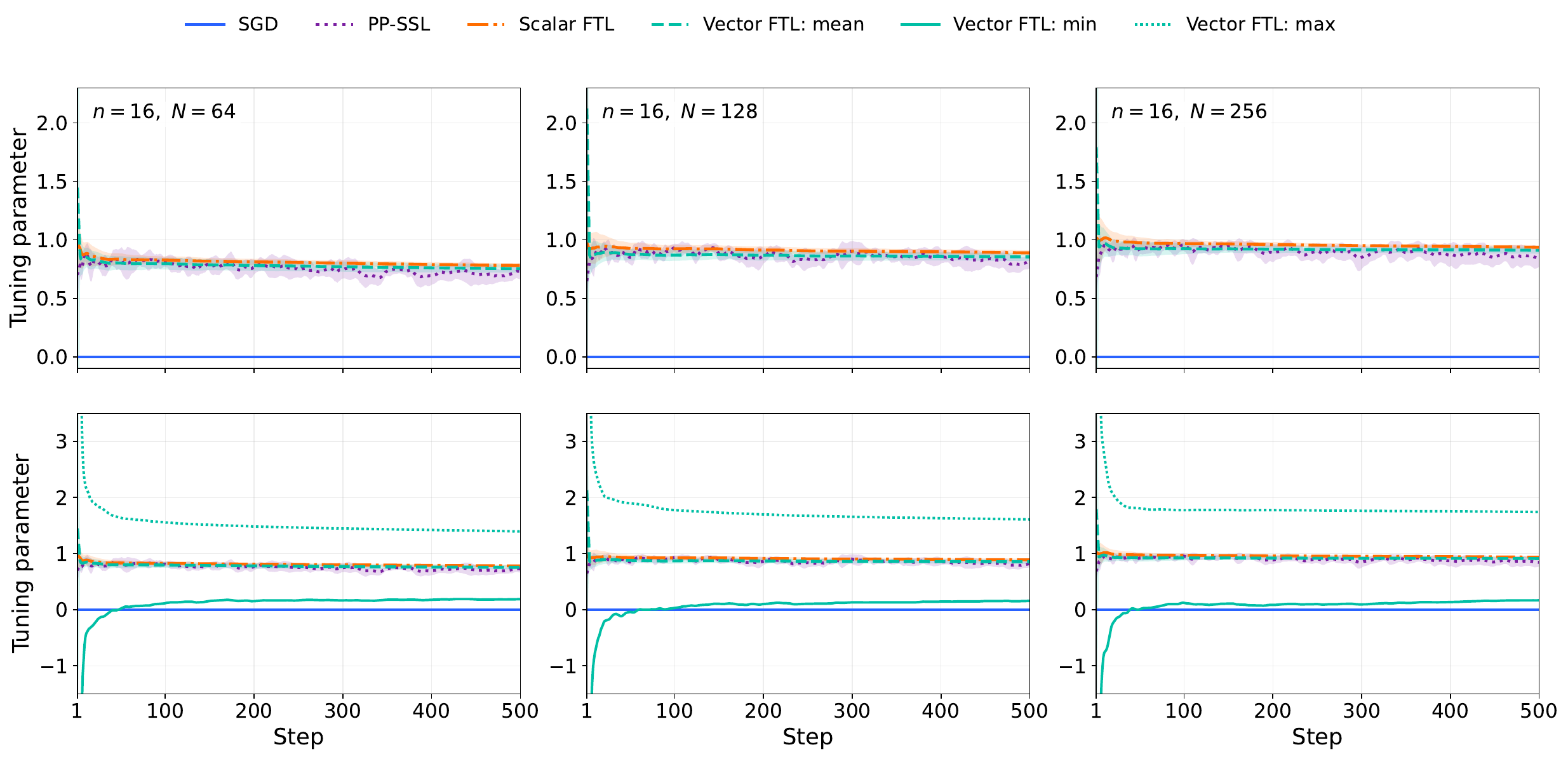}
    \vspace{-2mm}
    \caption{
    Tuning parameter trajectories for the first-order methods on Satimage with $n=16, N=64$ (left), $n=16, N=128$ (middle), and $n=16, N=256$ (right).
    \textbf{Top row}: Tuning parameter trajectories over training steps. 
    \textbf{Bottom row}: Tuning parameter trajectories together with the across-seed mean of the minimum and maximum coordinatewise Vector FTL parameters at each step.
    Curves show the mean across ten random seeds, with shaded regions indicating $\pm 1$ standard deviation. For visualization, the tuning parameter curves are smoothed using a five-step moving average.
    }
    \label{fig:sat_fo_lam}
\end{figure}

\subsubsection{Zeroth-order methods}

For the zeroth-order experiments, we train the linear classifier for $5{,}000$ optimization steps and compare three methods: standard stochastic zeroth-order method without using ML-labeled data and our proposed Algorithm~\ref{alg:sgfm} with scalar tuning and vector tuning. We use a labeled batch size of $n=32$ for all methods and consider unlabeled batch sizes $N\in\{128,256,512\}$ for our methods that use ML-labeled data. We set the perturbation parameter to $\mu=1e\text{-}3$. We use a constant learning rate throughout training and independently tune it for each method and batch-size configuration over $\{1e-3,5e-3,1e-2\}$. Based on mean validation accuracy across seeds, a learning rate of $1e\text{-}2$ is selected for all methods and configurations in the final comparison.

\begin{figure}[h!]
    \centering
    \includegraphics[width=0.98\linewidth]{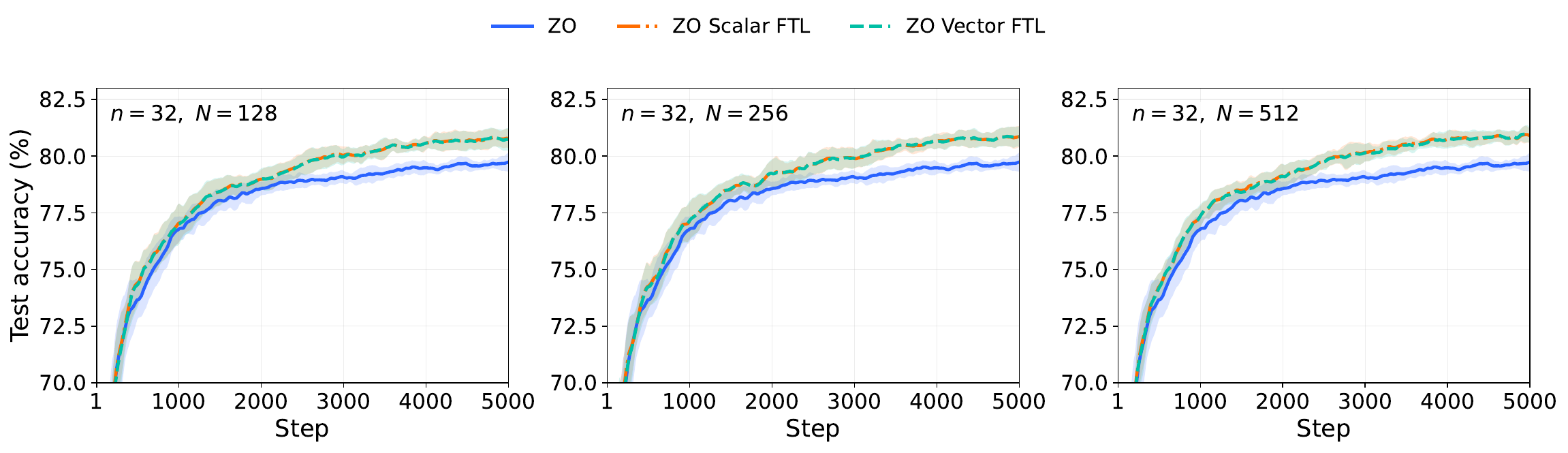}
    \vspace{-2mm}
    \caption{
    Test Accuracy (\%) over training steps for the zeroth-order methods on Satimage with $n=32, N=128$ (left), $n=32, N=256$ (middle), and $n=32, N=512$ (right). Values are reported every 25 steps. Curves show the mean across ten random seeds, with shaded regions indicating $\pm 1$ standard deviation. For visualization, test-accuracy curves are smoothed using a ten-step moving average.
    }
    \label{fig:sat_zo_acc}
\end{figure}

\begin{figure}[h!]
    \centering
    \includegraphics[width=0.98\linewidth]{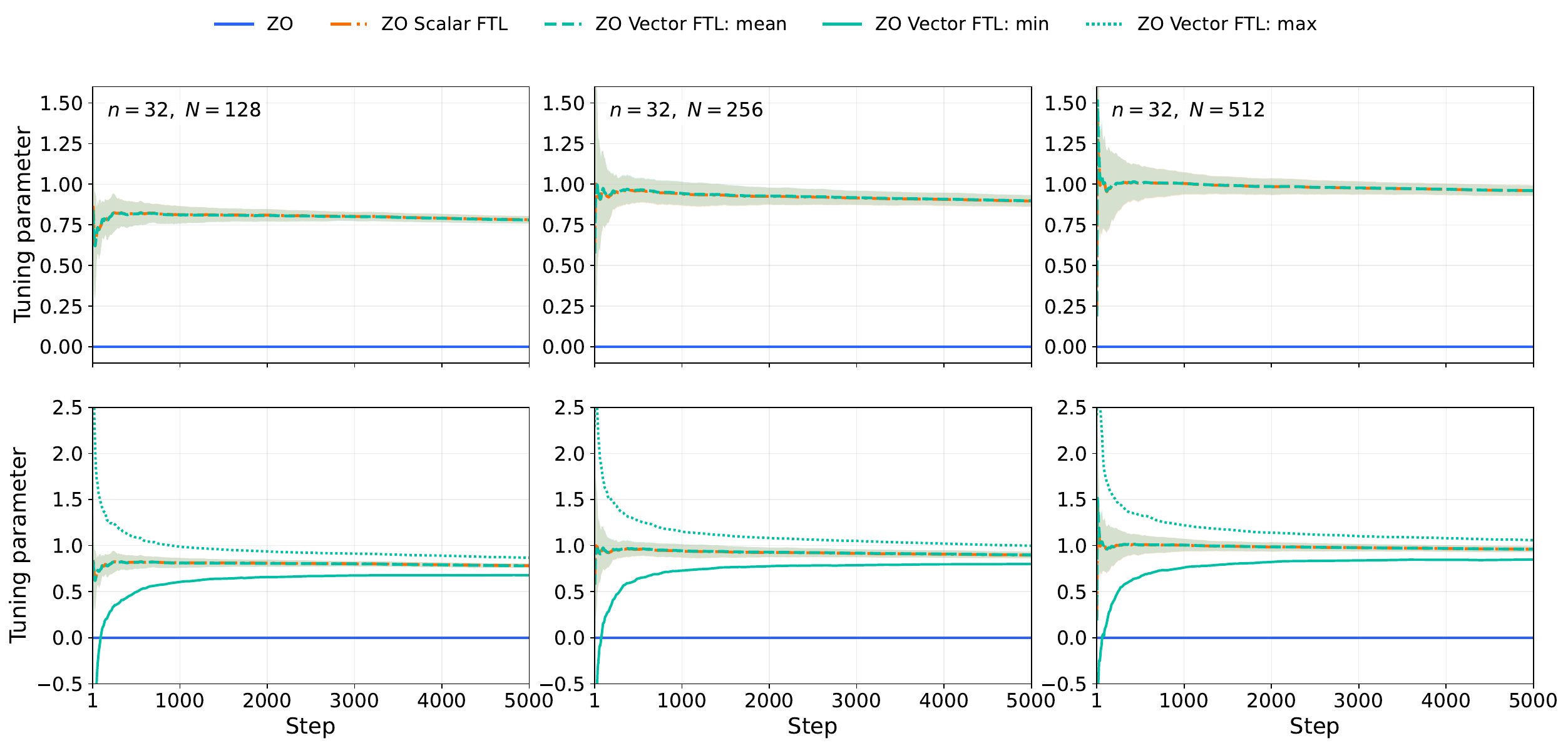}
    \vspace{-2mm}
    \caption{
    Tuning parameter trajectories for the zeroth-order methods on Satimage with $n=32, N=128$ (left), $n=32, N=256$ (middle), and $n=32, N=512$ (right).
    \textbf{Top row}: Tuning parameter trajectories over training steps. 
    \textbf{Bottom row}: Tuning parameter trajectories together with the across-seed mean of the minimum and maximum coordinatewise ZO Vector FTL parameters at each step.
    Curves show the mean across ten random seeds, with shaded regions indicating $\pm 1$ standard deviation. For visualization, the tuning parameter curves are smoothed using a ten-step moving average.
    }
    \label{fig:sat_zo_lam}
\end{figure}

\paragraph{Results.}
Figures~\ref{fig:sat_zo_acc} and \ref{fig:sat_zo_lam} present the test-accuracy results and corresponding tuning-parameter trajectories for the zeroth-order methods on Satimage. Across all three unlabeled batch sizes, Algorithm~\ref{alg:sgfm} with either of the tuning schemes consistently outperform the ZO baseline without using ML-labeled data after the initial stages of training. ZO Scalar FTL and ZO Vector FTL achieve nearly identical test accuracy throughout training, and their improvement over the baseline is sustained through the later optimization steps. This behavior remains consistent as the unlabeled batch size increases from $N=128$ to $N=512$. The corresponding tuning trajectories stabilize rapidly after the initial optimization steps, with the ZO Scalar FTL parameter and the mean of ZO Vector FTL parameter remaining closely aligned throughout training. As $N$ increases for fixed $n$, the tuning parameters tend to stabilize at larger values, indicating that the methods place greater weight on the ML-labeled correction when it is estimated from a larger unlabeled batch. The coordinatewise ZO Vector FTL parameters initially span a substantially wider range, but this dispersion reduces quickly and continues to decrease over training. Despite this additional coordinatewise flexibility, the close agreement between the scalar and the mean of vector tuning parameters is consistent with the nearly identical predictive performance of the two methods.

\subsection{Synthetic linear regression with controlled ML-labeler}
\label{app:synthetic_regression}

\paragraph{Task and model.}
To empirically quantify the variance-reduction benefits characterized in Corollary \ref{corollary}, we construct a synthetic linear regression experiment in which both the accuracy and the error structure of the ML-labeler are explicitly controlled. We sample each input independently as
\[
    \bx \sim \mathcal{N}(0,I_d),
\]
where the dimension is set to \(d=50\). We also sample a ground-truth parameter \(\bw^\star\in\mathbb{R}^d\) and normalize it so that \(\|\bw^\star\|_2=1\). 
To construct the labels, we take the regression function $y^\star(\bx) = \bx^\top\bw^\star$ and add Gaussian noise, i.e., 
\[
    y = \bx^\top\bw^\star+\sigma_y\epsilon,
    \qquad
    \epsilon\sim\mathcal{N}(0,1),
\]
where we set \(\sigma_y=0.5\). We train a linear regression model
\[
    h(\bx)=\bx^\top\bw+b.
\]

\paragraph{Dataset.}
We generate four mutually disjoint synthetic subsets using a fixed dataset seed. The labeled training set contains \(256\) examples, the validation set contains \(2{,}048\) examples, the test set contains \(4{,}096\) examples, and the unlabeled set contains \(8{,}192\) examples. The validation and test sets contain noisy observed labels generated using the same noise level \(\sigma_y=0.5\).

\paragraph{ML-labeler.}
Instead of training a separate labeler, we construct labeles whose error relative to the regression function is known analytically. Specifically, let \(\bu\in\mathbb{R}^d\) be a fixed unit vector orthogonal to \(\bw^\star\), so that
\[
    \|\bu\|_2=1,
    \qquad
    \bu^\top\bw^\star=0.
\]
For each example \(\bx\), we also sample an independent noise variable
\(\xi_{\bx}\sim\mathcal{N}(0,1)\). This variable is sampled once and cached. For parameters \(\tau\geq 0\) and \(\alpha\in[0,1]\), the ML-generated label is
\[
    f_{\alpha,\tau}(\bx)
    =
    \bx^\top
    \left(
        \bw^\star+\tau\sqrt{\alpha}\,\bu
    \right)
    +
    \tau\sqrt{1-\alpha}\,\xi_{\bx}.
\]
Consequently, the ML-labeler error relative to the target $y^\star(\bx) = \bx^\top\bw^\star$ is
\[
    f_{\alpha,\tau}(\bx)-y^\star(\bx)
    =
    \tau\sqrt{\alpha}\,\bx^\top\bu
    +
    \tau\sqrt{1-\alpha}\,\xi_{\bx}.
\]
Since \(\bx^\top\bu\sim\mathcal{N}(0,1)\), \(\xi_{\bx}\sim\mathcal{N}(0,1)\), and these two terms are independent, the expected labeler MSE is
\[
\begin{aligned}
    \mathbb{E}\!\left[
        \left(
            f_{\alpha,\tau}(\bx)-y^\star(\bx)
        \right)^2
    \right]
    &=
    \tau^2\alpha\,
    \mathbb{E}\!\left[(\bx^\top\bu)^2\right]
    +
    \tau^2(1-\alpha)\,
    \mathbb{E}\!\left[\xi_{\bx}^2\right] \\
    &=\tau^2.
\end{aligned}
\]
Thus, \(\tau\) is the population RMSE of the ML-labeler relative to the regression function, while \(\alpha\) controls the structure of the error. We consider
\[
    \alpha\in\{0,0.5,1\},
    \qquad
    \tau\in\{0.1,0.25,0.5,1\}.
\]
When \(\alpha=1\), all labeler error is systematic and corresponds to a fixed misspecification of the labeler parameter in the direction \(\bu\). When \(\alpha=0\), the error consists entirely of Gaussian noise. The choice \(\alpha=0.5\) assigns equal error variance to systematic misspecification and example-specific noise. Across different values of \(\alpha\) and \(\tau\), we reuse the same features, direction \(\bu\), and base noise variables \(\xi_{\bx}\), thereby isolating the effect of the error magnitude and error structure.

\paragraph{Loss function.}
On labeled examples, the supervised loss is the mean squared error between the model's prediction and the ground-truth target $y^*$. The labeler-based losses on both labeled and unlabeled examples are mean squared errors between the model's prediction and the corresponding cached labeler prediction.

\paragraph{Randomness.}
We report results across ten random seeds. The synthetic dataset, ground-truth parameter, systematic error direction, cached ML-labeler noise, and data partition are held fixed across seeds. Randomness across training seeds arises from the initialization of the student model and the sampling of labeled and unlabeled mini-batches.

\vspace{4mm}

We train the linear regressor for \(500\) optimization steps and compare four methods: standard SGD using only the observed labeled examples, PP-SSL with Online Tuning \citep{shoham2025prediction}, and our FTL methods with scalar and vector-valued tuning parameters. We use a labeled batch size of \(n=16\) for all methods and an unlabeled batch size of \(N=64\) for PP-SSL and the FTL methods. Mini-batches are sampled with replacement.

We use a constant learning rate throughout training and tune it independently for every method and \((\alpha,\tau)\) configuration over $\{5e-3,\;1e-2,\;5e-2,\;1e-1\}$. For each method and ML-labeler configuration, the learning rate and training checkpoint are selected by minimizing the mean validation MSE across the ten seeds. We report test MSE relative to the regression function $y^\star(\bx) = \bx^\top\bw^\star$, which isolates estimation error from the irreducible label noise. 

To empirically investigate the variance-reduction mechanism characterized in Corollary~1, we evaluate the population covariance between the single-example supervised and ML-labeled gradients along each realized optimization trajectory. Let
\[
    \bg_t^y=\nabla\ell(\bw_t;\bx,y),
    \qquad
    \bg_t^f=\nabla\ell(\bw_t;\bx,f_{\alpha,\tau}(\bx)).
\]
Since the Gaussian data-generating model allows us to compute the corresponding covariance and variance explicitly, we use their exact values along the optimization trajectory rather than the uniform variance upper bound \(\sigma_f^2\) appearing in Corollary~1. We define the trajectory-averaged quantities
\[
    \overline{\mathcal C}_T
    =
    \frac{1}{T}\sum_{t=1}^T
    \operatorname{Cov}(\bg_t^y,\bg_t^f),
    \qquad
    \overline{\mathcal V}_T
    =
    \frac{1}{T}\sum_{t=1}^T
    \operatorname{Var}(\bg_t^f),
\]
where the covariance and variance are evaluated under the population distribution at the current iterate \(\bw_t\).

For Scalar FTL, we report
\[
    R_T^{\mathrm{scalar}}
    =
    \frac{\operatorname{tr}(\overline{\mathcal C}_T)^2}
         {\operatorname{tr}(\overline{\mathcal V}_T)},
\]
whereas for Vector FTL, we consider the corresponding coordinatewise variance-reduction quantity
\[
    R_T^{\mathrm{vector}}
    =
    \sum_{j=1}^{d+1}
    \frac{\overline{\mathcal C}_{T,jj}^{\,2}}
         {\overline{\mathcal V}_{T,jj}},
\]
including the intercept coordinate. We compute these quantities over \(T=500\) optimization steps and average the resulting values across ten random seeds to obtain the heatmaps.

These quantities characterize the potential reduction in gradient variance obtained by optimally combining the true-labeled and ML-labeled gradients. Specifically, the corresponding variance reduction is given by
\[
    \frac{N}{n(n+N)}R_T,
\]
where \(R_T\) denotes the scalar or vector quantity, respectively. Since \(n=16\) and \(N=64\) are fixed throughout the experiment, this batch-size factor is identical across all \((\alpha,\tau)\) configurations and is omitted from the heatmaps. Thus, larger heatmap values indicate a greater potential variance reduction relative to using only the true-labeled gradient.

\paragraph{Results.} 
Figure \ref{fig:synthetic_mse} presents the test MSE over 500 optimization steps for SGD, PP-SSL, Scalar FTL, and Vector FTL under different combinations of the labeler error magnitude $\tau$ and error structure parameter $\alpha$. Figure \ref{fig:synthetic_heatmap} shows heatmaps of the empirical variance-reduction quantities for Scalar FTL and Vector FTL, computed from the gradient covariance and variance along the optimization trajectories.
As the labeler error $\tau$ increases, the empirical variance-reduction quantities decrease, and the optimization performance of the FTL methods deteriorates, reducing their advantage over SGD. In contrast, varying the error structure through $\alpha$ has a comparatively small and non-monotone effect, with no clear relationship between the type of labeler error and the resulting variance reduction or optimization performance. Overall, these results corroborate the theoretical findings of Corollary~\ref{corollary} by demonstrating that the potential variance reduction predicted by the gradient covariance and variance is reflected in the observed optimization performance. 
% They further suggest that, in our experimental setting, the magnitude of the labeler error plays a more significant role than its specific structure.

\begin{figure}[h]
    \centering
    \includegraphics[width=0.98\linewidth]{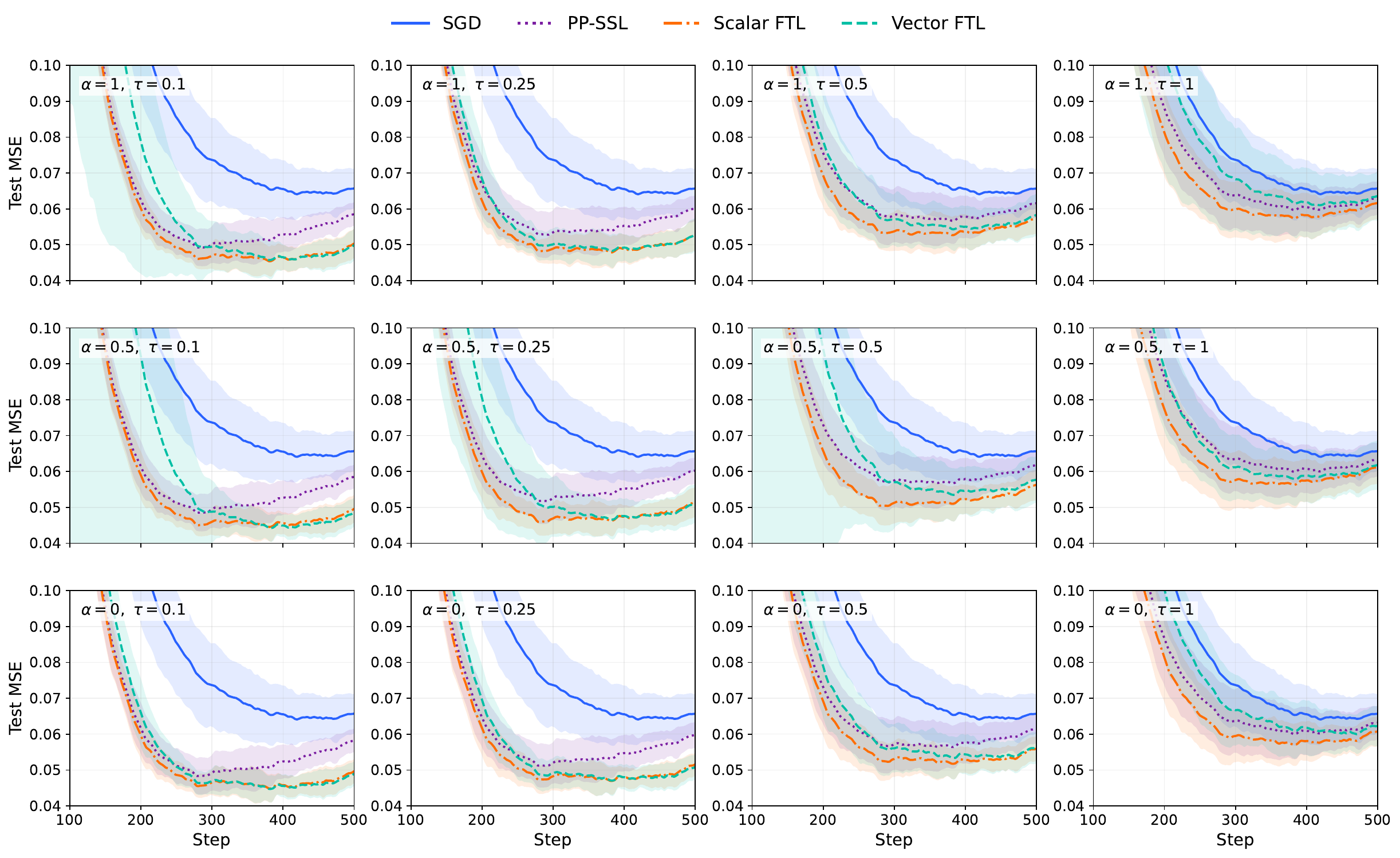}
    \vspace{-2mm}
    \caption{Test MSE over training steps for the first-order methods in the synthetic linear-regression experiment. Rows correspond to $\alpha \in\{1,0.5,0\}$, while columns correspond to $\tau \in\{0.1,0.25,0.5,1\}$. Curves show the mean across ten random seeds, with shaded regions indicating $\pm1$ standard deviation. For visualization, the test-MSE curves and their standard-deviation bounds are smoothed using a five-step moving average.}
    \label{fig:synthetic_mse}
\end{figure}

\begin{figure}[h]
    \centering
    \includegraphics[width=1.0\linewidth]{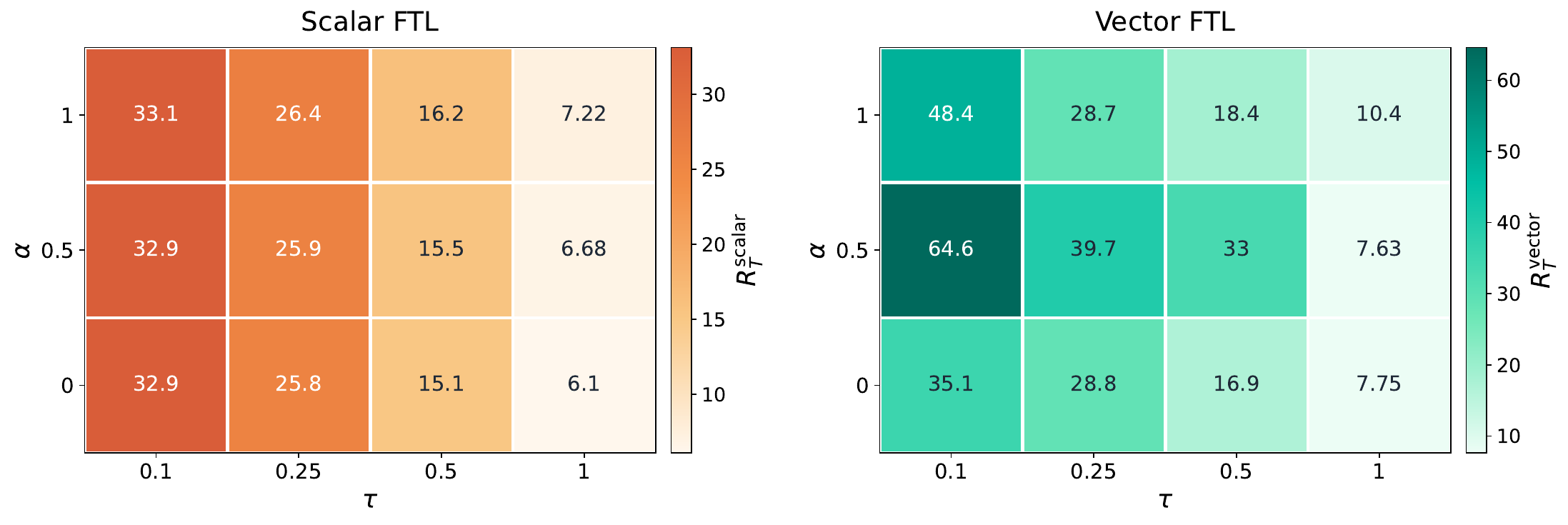}
    \vspace{-5mm}
    \caption{Empirical variance-reduction quantities for Scalar FTL (left) and Vector FTL (right) in the synthetic linear-regression experiment. Rows correspond to $\alpha\in\{1,0.5,0\}$, and columns correspond to $\tau\in\{0.1,0.25,0.5,1\}$. Each cell reports the terminal quantity $R_T^{\mathrm{scalar}}$ or $R_T^{\mathrm{vector}}$ at $T=500$, averaged across ten random seeds.}
    \label{fig:synthetic_heatmap}
\end{figure}

\end{document}